\documentclass{article} 
\usepackage[final]{colm2026_conference}

\usepackage{microtype}
\usepackage{hyperref}
\usepackage{url}
\usepackage{booktabs}

\usepackage{amsmath}
\usepackage{amssymb}
\usepackage{mathtools}
\usepackage{amsthm}

\usepackage{lineno}

\usepackage{amsfonts}
\usepackage{bbm}
\usepackage{enumerate}
\usepackage{enumitem}
\usepackage{array}
\usepackage{listings}
\usepackage{dsfont}
\usepackage{placeins}
\usepackage[ruled,vlined,linesnumbered]{algorithm2e}
\usepackage{caption} 
\usepackage{subcaption}

\usepackage{adjustbox} 
\usepackage{wrapfig}
\usepackage{titletoc}
\usepackage{makecell}

\usepackage[capitalize,noabbrev]{cleveref}
\usepackage[table,xcdraw, dvipsnames]{xcolor}
\definecolor{gtrow}{RGB}{255, 255, 200}   
\definecolor{seabornbluemid}{HTML}{08519C}
\definecolor{seabornorangemid}{HTML}{F16913}
\definecolor{seaborngreymid}{HTML}{737373}
\definecolor{seabornredmid}{HTML}{EF3B2C}

\makeatletter
\def\And{\end{tabular}\hspace{3em}%
        \begin{tabular}[t]{c}\bf\rule{\z@}{24pt}\ignorespaces}%
\def\AND{\end{tabular}\hfil\linebreak[4]\hfil
        \begin{tabular}[t]{c}\bf\rule{\z@}{24pt}\ignorespaces}%
\def\@maketitle{\vbox{\hsize\textwidth
{\Large\bf \@title\par}
\ifcolmfinal
    \lhead{Published as a conference paper at COLM 2026}
    \centering\begin{tabular}[t]{c}\bf\rule{\z@}{24pt}\@author\end{tabular}%
\else\ifcolmpreprint
    \lhead{Preprint.}
    \centering\begin{tabular}[t]{c}\bf\rule{\z@}{24pt}\@author\end{tabular}%
\else
    \lhead{Under review as a conference paper at COLM 2026}
    \centering\begin{tabular}[t]{c}\bf\rule{\z@}{24pt}Anonymous authors\\Paper under double-blind review\end{tabular}%
\fi\fi
\vskip 0.3in minus 0.1in}}
\makeatother

\newif\ifusemath
\usemathtrue

\newcommand{\todo}[1]{\textcolor{red}{\textbf{TODO:} #1}}
\newcommand{\custompar}[1]{\vspace{.1cm} \noindent{\bf #1.}\:}

\SetKwFunction{TopK}{TopK}
\SetKwFunction{LogSoftmax}{LogSoftmax}
\SetKwFunction{LogSumExp}{LogSumExp}

\newlist{steplist}{enumerate}{1}
\setlist[steplist]{label=(\Alph*), ref=(\Alph*)}

\usepackage{amsmath,amsfonts,bm}

\newcommand{\newterm}[1]{{\bf #1}}

\def\eqref#1{equation~\ref{#1}}
\def\Eqref#1{Equation~\ref{#1}}

\def\ceil#1{\lceil #1 \rceil}

\def\1{\bm{1}}

\def\vb{{\bm{b}}}
\def\vc{{\bm{c}}}

\def\vr{{\bm{r}}}

\def\vu{{\bm{u}}}
\def\vv{{\bm{v}}}
\def\vw{{\bm{w}}}
\def\vx{{\bm{x}}}
\def\vy{{\bm{y}}}
\def\vz{{\bm{z}}}

\DeclareMathAlphabet{\mathsfit}{\encodingdefault}{\sfdefault}{m}{sl}
\SetMathAlphabet{\mathsfit}{bold}{\encodingdefault}{\sfdefault}{bx}{n}

\def\sA{{\mathbb{A}}}
\def\sB{{\mathbb{B}}}
\def\sC{{\mathbb{C}}}
\def\sD{{\mathbb{D}}}
\def\sF{{\mathbb{F}}}

\def\sL{{\mathbb{L}}}

\def\sN{{\mathbb{N}}}

\def\sR{{\mathbb{R}}}
\def\sS{{\mathbb{S}}}

\def\sU{{\mathbb{U}}}
\def\sV{{\mathbb{V}}}

\def\sX{{\mathbb{X}}}
\def\sY{{\mathbb{Y}}}
\def\sZ{{\mathbb{Z}}}

\newcommand{\E}{\mathbb{E}}

\newcommand{\R}{\mathbb{R}}

\newcommand{\Var}{\mathrm{Var}}
\newcommand{\standarderror}{\mathrm{SE}}

\DeclareMathOperator*{\argmax}{arg\,max}

\usepackage{amsthm}

\ifusemath
\theoremstyle{plain}
\newtheorem{theorem}{Theorem}
\newtheorem{proposition}[theorem]{Proposition}
\newtheorem{lemma}[theorem]{Lemma}
\newtheorem{corollary}[theorem]{Corollary}
\theoremstyle{definition}

\theoremstyle{definition}

\theoremstyle{definition}

\theoremstyle{remark}
\newtheorem{remark}[theorem]{Remark}
\theoremstyle{plain}

\fi

\newcommand{\expect}[2]{\mathds{E}_{{#1}} \left[ {#2} \right]}

\newcommand{\model}{\theta}
\newcommand{\dec}{\phi}
\newcommand{\logit}{\vy}
\newcommand{\dist}{\mathsf{dist}}
\newcommand{\ham}{\mathsf{Hamming}}
\newcommand{\hamshort}{\mathsf{Ham}}
\newcommand{\lev}{\mathsf{Levenshtein}}
\newcommand{\levshort}{\mathsf{Lev}}
\newcommand{\normlev}{\mathsf{Norm\_Levenshtein}}
\newcommand{\softmaxsf}{\mathsf{softmax}}
\newcommand{\gen}{\mathsf{generate}}
\newcommand{\vocab}{\sV}
\newcommand{\tok}{z}
\newcommand{\seq}{\vz}

\newcommand{\genseq}{\hat{\seq}} 
\newcommand{\gentok}{\hat{\tok}} 

\newcommand{\dataset}{\sD}
\newcommand{\seqset}{\sZ}

\newcommand{\preend}{a}
\newcommand{\prelen}{\preend}
\newcommand{\preinds}{1:\prelen}
\newcommand{\prefix}{\seq_{\preinds}}

\newcommand{\sufstart}{\preend+1}
\newcommand{\suflen}{T}
\newcommand{\sufend}{\preend+\suflen}
\newcommand{\sufinds}{\sufstart:\sufend}
\newcommand{\suffix}{\seq_{\sufinds}}
\newcommand{\gensuffix}{\genseq_{\sufinds}}

\newcommand{\suc}{s}
\newcommand{\greedy}{\mathsf{greedy}}
\newcommand{\pseq}{p_\seq}
\newcommand{\pseqe}{p_{\seq,\varepsilon}^\dist}
\newcommand{\pseqeham}{p_{\seq,\varepsilon}^\hamshort}
\newcommand{\pseqelev}{p_{\seq,\varepsilon}^\levshort}
\newcommand{\hatpseqe}{\hat{p}_{\seq,\varepsilon}^\dist}
\newcommand{\taumin}{\tau_{\min}}

\newcommand{\seqb}{\vb}
\newcommand{\tokb}{b}
\newcommand{\seqc}{\vc}
\newcommand{\tokc}{c}

\newcommand{\bw}{B}
\newcommand{\ballset}{\sB}
\newcommand{\ball}{\ballset^{\dist}_{\varepsilon}}
\newcommand{\hamball}{\ballset^{\hamshort}_{\varepsilon}}
\newcommand{\levball}{\ballset^{\levshort}_{\varepsilon}}

\newcommand{\pre}{\seq_{\textnormal{(pre)}}}
\newcommand{\suf}{\seq_{\textnormal{(suf)}}}
\newcommand{\gensuf}{\genseq_{\textnormal{(cont)}}}
\newcommand{\cont}{\seq_{\textnormal{(cont)}}}
\newcommand{\mismatch}{n}
\newcommand{\suftok}[1]{\tok^{\textnormal{(suf)}}_{#1}}         
\newcommand{\lb}{\mathrm{LB}_{\varepsilon,\dist}}
\newcommand{\ub}{\mathrm{UB}_{\varepsilon,\dist}}

\newcommand{\cover}{\mathrm{covered\_mass}}
\newcommand{\temp}{\beta}
\newcommand{\dptab}{D}
\newcommand{\subsuf}[1]{\seq^{\textnormal{(suf)}}_{#1}}
\newcommand{\subgensuf}[1]{\genseq^{\textnormal{(cont)}}_{#1}}
\newcommand{\subgensuftok}[1]{\gentok^{\textnormal{(cont)}}_{#1}}

\newcommand{\customtarget}[2]{\hypertarget{#1}{\textbf{#2}}}
\newcommand{\customlink}[2]{\hyperlink{#1}{#2}}

\definecolor{darkblue}{rgb}{0, 0, 0.5}
\hypersetup{colorlinks=true, citecolor=darkblue, linkcolor=darkblue, urlcolor=darkblue}

\title{Estimating near-verbatim extraction risk in language models with decoding-constrained beam search}

\author{%
  A. Feder Cooper\thanks{Corresponding author: \texttt{a.feder.cooper@yale.edu}} \\
  Stanford \& Yale\\
  \And 
  Mark A. Lemley\\
  Stanford\\
  \And
  Christopher De Sa\\
  Cornell\\
  \AND
  Lea Duesterwald\\
  Cornell\\
  \And 
  Allison Casasola\\
  Stanford\\
  \And 
  Jamie Hayes\\
  Google DeepMind
  \AND 
  Katherine Lee\\
  \And
  Daniel E. Ho\\
  Stanford\\
  \And
  Percy Liang\\
  Stanford
}

\begin{document}

\ifcolmsubmission
\linenumbers
\fi

\maketitle

\vspace{-.3cm}
\begin{abstract}
\vspace{-.1cm}
Recent work shows that standard greedy-decoding extraction methods for quantifying memorization in LLMs miss how extraction risk varies across sequences.
Probabilistic extraction---computing the probability of generating a target 
suffix given a prefix under a decoding scheme---addresses this, but is tractable only for verbatim memorization, missing near-verbatim instances that pose similar privacy and copyright risks.
Quantifying near-verbatim extraction risk is expensive:
the set of near-verbatim suffixes is combinatorially large, and reliable Monte Carlo (MC) estimation can require ${\approx}\,100{,}000$ samples per sequence.
To mitigate this cost, we introduce decoding-constrained beam search, which yields deterministic lower bounds on near-verbatim extraction risk at a cost comparable to ${\approx}\,20$ MC samples per sequence.
Across experiments, our approach surfaces information invisible to verbatim methods:
many more extractable sequences, substantially larger per-sequence extraction mass, and patterns in how near-verbatim extraction risk manifests across model sizes and types of text.\looseness=-1
\end{abstract}

\vspace{-.2cm}
\section{Introduction}\label{sec:intro}
\vspace{-.1cm}

LLMs memorize portions of their training data, and it is sometimes possible to extract those memorized sequences in outputs~\citep{carlini2021extracting,carlini2023quantifying}.
Understanding memorization is important for assessing model quality, as it may indicate overfitting rather than generalization.
Extraction of memorized training data also raises other concerns:
it can expose private information~\citep{brown2022privacy, nasr2023scalable, nolte2025machinelearnersacknowledgelegal} or reproduce copyrighted material present in the training data~\citep{lee2023talkin, cooper2024files, ahmed_extracting_2026}.\looseness=-1

Most prior work measures extraction through \emph{verbatim} comparisons:
prompt an LLM with a training-data prefix, use greedy decoding to produce an output, and check whether that output \emph{exactly} matches the corresponding training-data suffix~\citep{lee2022dedup}.
This approach undercounts memorization in two important ways.
First, requiring verbatim matches misses \emph{near-verbatim} memorization~\citep{ippolito-etal-2023-preventing}. 
Second, greedy decoding produces a single deterministic output, so it cannot capture how extraction risk varies across sequences. 
Recent work on \emph{probabilistic} extraction---computing the probability of generating the target suffix under a decoding scheme---addresses the latter, but remains tractable only for verbatim memorization (\citet{hayes2025measuringmemorizationlanguagemodels}, Section~\ref{sec:rw}).
Computing \emph{near-verbatim probabilistic} extraction is far more expensive, as it requires estimating  probability mass over a combinatorially large set of near-verbatim suffixes.\looseness=-1

To address this, we introduce a tractable method for estimating near-verbatim extraction risk:
\newterm{decoding-constrained beam search}, a family of algorithms that produce a provably correct deterministic lower bound on a sequence $\seq$'s near-verbatim extraction probability $\pseqe$ for a given distance metric $\dist$ and tolerance $\varepsilon$ (Section~\ref{sec:kcbs}).
While reliable Monte Carlo (MC) estimation can require ${\approx}\,100{,}000$ samples per sequence (Section~\ref{sec:warmup}), our experiments with top-$k$-constrained beam search ($k$-CBS) yield practically useful lower bounds at a cost comparable to ${\approx}\,20$ MC samples per sequence.
We provide variants that integrate distance-based $\varepsilon$-viability pruning into the search, often yielding tighter lower bounds on $\pseqe$ at reduced runtime cost.  
We evaluate $k$-CBS across multiple model families, model sizes, and datasets (Section~\ref{sec:experiments} \& Appendix~\ref{app:sec:experiments}). 
Our experiments show that near-verbatim probabilistic extraction reveals information invisible to verbatim methods:
far more extractable sequences (e.g., $2.57\%$ of sequences for \textsc{OLMo 2 32B} on Wikipedia, compared to $1.42\%$ for verbatim probabilistic extraction);  
substantially larger per-sequence extraction risk (e.g., in some cases, from $0$ verbatim risk to over $0.85$ near-verbatim risk); 
and patterns in how near-verbatim extraction risk manifests across model sizes and types of text.\looseness=-1

\vspace{-.2cm}
\section{Background and related work}\label{sec:rw}
\vspace{-.1cm}

\begin{figure*}[t!]
\centering
\includegraphics[width=\linewidth]{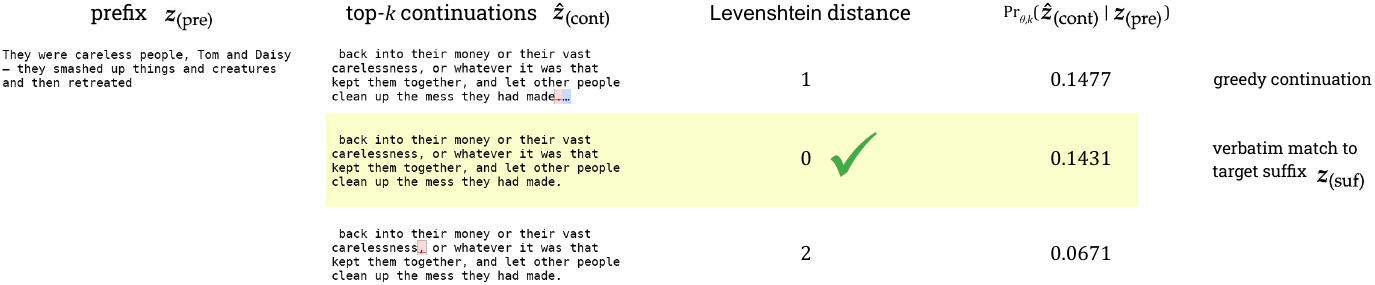}
\vspace{-.5cm}
\caption{\textbf{Probabilistic extraction.} 
For $\model\!=$\textsc{Llama 1 13B} and a training sequence $\seq$ from \emph{The Great Gatsby}, we show 
prefix $\pre\coloneqq\prefix$ and $3$ continuations $\gensuf \coloneqq\gensuffix$ under $\dec\!=$ top-$k\!=\!40$ 
(Equation~\ref{eq:topk:main}) with  conditional probabilities $\Pr_{\model, k}(\gensuf \mid \pre)$ (Equation~\ref{eq:pz:main}). 
We diff each $\gensuf$ with the target suffix 
$\suf\coloneqq\suffix$ (character space: \textcolor{blue}{blue} additions, \textcolor{red}{red} deletions) 
and compute the Levenshtein distance (token space, Equation~\ref{eq:dist:main}).  
We highlight verbatim extraction 
(i.e., $\gensuf\!=\!\suf$, $0.1431\geq\taumin=0.001$), which is \emph{not} the greedy continuation (top row).
All three 
$\gensuf$ are \emph{near-verbatim} matches to  
$\suf$ (Section~\ref{sec:warmup}).\looseness=-1
}
\label{fig:example}
\vspace{-.2cm}
\end{figure*}

Let $\vocab$ denote the token \newterm{vocabulary} for a \newterm{large language model (LLM)} (e.g., $|\vocab|=32{,}000$ for \textsc{Llama 1}).                                                                                         An LLM with \newterm{weights} $\model$ maps a sequence of tokens $\seqb=(\tokb_1, \ldots, \tokb_n) \in \vocab^n$ to a logit vector $\logit \in \R^{|\vocab|}$:                                                                    
$\model: \vocab^n \to \R^{|\vocab|}$.
A \newterm{decoding scheme} $\dec$ defines how logits are mapped to a next-token sampling distribution, $\softmaxsf_\dec: \R^{|\vocab|} \to \mathcal{P}(\vocab)$, where $\mathcal{P}(\vocab)$ is the set of probability distributions over~$\vocab$.
Together, $(\model,\dec)$ define an \newterm{autoregressive generation process}: 
given a \newterm{prompt} $\seqb_{1:i}$, at each step $t > i$ we compute $\logit_t = \model(\seqb_{1:t-1})$ and sample $\tokb_t \sim \softmaxsf_\dec(\logit_t)$; for $j$ steps, this produces a \newterm{continuation} $\hat{\seqb}_{i+1:i+j}$. 

We focus on \newterm{top-$k$ decoding}~\citep{fan-etal-2018-hierarchical}.                                 
At step $t$, let $\vocab_k \subseteq \vocab$ be the set of the $k\!>\!0$ tokens with the largest logits in $\logit_t$.                                                                                         
Top-$k$ zeros out all but the $k$ highest logits and renormalizes:\looseness=-1                               \vspace*{-.1cm}                                                                                                    \begin{align}           
  \label{eq:topk:main}
  \Pr_{\model,k}(\tokb_t{=}v \mid \seqb_{1:t-1}) = \frac{\exp(\logit_t[v])}{\sum_{u \in \vocab_k} \exp(\logit_t[u])} \;\mathbf{1}\{v \in \vocab_k\}. 
\end{align}

\vspace*{-.12cm}\textbf{Verbatim memorization and extraction.} 
\newterm{Memorization} refers to the encoding of specific training sequences in a model's weights~\citep{feldman2020mem, yeom2018privacy,shokri2017membership}, such that the learned distribution assigns very high probability to them. 
For LLMs, this can make \newterm{extraction} possible: 
memorized sequences can sometimes be reproduced in outputs~\citep{carlini2021extracting}. 
The most common way to assess memorization in LLMs is \newterm{discoverable extraction} (\newterm{greedy extraction}): 
split a training sequence $\seq$ into an $\prelen$-length \newterm{prefix} $\prefix$ and a $\suflen$-length \newterm{target suffix} $\suffix$ (i.e., $\seq = \prefix \parallel \suffix$), prompt the LLM with $\prefix$, use \newterm{greedy decoding} as a cheap proxy for deterministically generating a high-probability $\suflen$-length continuation $\gensuffix$, and deem extraction successful if $\gensuffix$ \emph{exactly} matches $\suffix$ (i.e., \(\1[\,\gensuffix = \suffix\,]\))~\citep{lee2022dedup, carlini2023quantifying}. 

\citet{hayes2025measuringmemorizationlanguagemodels} show that greedy extraction greatly underestimates memorization, as memorized continuations that have very high generation probability may not be reachable through greedy (top-$1$) decoding. 
They suggest a definition for \newterm{probabilistic extraction}, which accounts for the non-determinism of more typical decoding schemes $\dec$ (e.g., top-$k$ with $k\!>\!1$) by computing the probability of generating the exact target $\suffix$ given $\prefix$:\looseness=-1 
\vspace*{-.05cm}
\begin{align}
\label{eq:pz:main}
\pseq \;\triangleq\;  \Pr_{\model,\dec}\!\big(\suffix\,\big|\,\prefix\big)
\;=\; \prod_{t=\sufstart}^{\sufend} \Pr_{\model,\dec}\!\big(\tok_t \,\big|\, \seq_{1:t-1}\big)
\;=\; \exp\!\Bigg(
  \sum_{t=\sufstart}^{\sufend}
  \log \Pr_{\model,\dec}\!\big(\tok_t \,\big|\, \seq_{1:t-1}\big)
\Bigg).
\end{align} 
\vspace*{-.05cm}For a threshold $\taumin\in(0,1]$, verbatim probabilistic extraction is successful when the generation probability of $\suffix$ given $\prefix$ is at least $\taumin$, i.e., 
\(
\1[\,\pseq \geq \taumin\,].
\)
\citet{hayes2025measuringmemorizationlanguagemodels} justify studying probabilistic extraction using top-$k$ decoding with $k\!=\!40$, for which \citet{cooper2025books} validate a conservative $\taumin\!=\!0.001$. 
We examine the same settings here.

Figure~\ref{fig:example} provides a concrete example. 
For \textsc{Llama 1 13B} and a training sequence from \emph{The Great Gatsby}, we illustrate the generation of three continuations under top-$k$ decoding (Equation~\ref{eq:topk:main}) for the same prefix. 
Verbatim probabilistic extraction succeeds (middle row): 
$\pseq\!=\!0.1431\!>\!\taumin$. 
Verbatim greedy extraction fails, as the greedy continuation (top row) is not a verbatim match to the target suffix. 
Probabilistic extraction also surfaces nuanced information about \newterm{extraction risk}.
While greedy decoding flattens extraction success to a binary, yes-or-no outcome,  probabilistic extraction under $\dec$ returns an extraction probability $\pseq \in [0,1]$. 
For training sequences $A$ and $B$, if \(p_{\seq^{(A)}} > p_{\seq^{(B)}} \), then in expectation LLM \(\model\) will leak $A$ more often than $B$.
In Figure~\ref{fig:example}, 
$\pseq\!=\!0.1431$ means that \textsc{Llama 1 13B} leaks the exact target suffix about $1$ out of every $7$ times it is prompted with the prefix. 
Finally, probabilistic extraction offers these benefits at comparable cost to greedy decoding.
Verbatim extraction probability $\pseq$ (Equation~\ref{eq:pz:main}) can be computed exactly via \newterm{teacher-forced inference}---a single forward pass over $\seq$, with no sampling required (\citet{cooper2025books}, Appendix~\ref{app:sec:intuition:cost}).\looseness=-1

\vspace{-.2cm}
\section{The challenge of quantifying near-verbatim extraction risk}\label{sec:warmup}
\vspace{-.1cm}

For efficiency reasons, much prior work studies \emph{verbatim} extraction~\citep{lee2022dedup, nasr2023scalable} (Section~\ref{sec:rw}). 
However, as intuited in Figure~\ref{fig:example}, verbatim extraction underestimates memorization.
For greedy extraction, accounting for \emph{near-verbatim} continuations 
of a prefix  
reveals substantially more memorized sequences~\citep{ippolito-etal-2023-preventing}.
(In Figure~\ref{fig:example}, near-verbatim greedy extraction succeeds.) 
For probabilistic extraction, near-verbatim continuations would similarly provide richer information about extraction risk.
(In Figure~\ref{fig:example}, the probabilities of all three continuations contribute to near-verbatim extraction risk:
verbatim risk is $0.1431$, but near-verbatim risk computed over all near-verbatim suffixes is at least $0.1477\!+\!0.1431\!+\!0.0671\!=\!0.3579$.) 
Yet, as \citet{hayes2025measuringmemorizationlanguagemodels} note, straightforward approaches for computing near-verbatim probabilistic extraction---e.g., Monte Carlo (MC) sampling---are computationally expensive. 
In this section, we explain this cost, motivating our algorithm for efficiently estimating near-verbatim probabilistic extraction (Section~\ref{sec:kcbs}).

\custompar{Near-verbatim extraction}
To simplify indexing, 
denote $\prelen$-length 
prefix $\pre \coloneqq \prefix$, 
$\suflen$-length target suffix $\suf \coloneqq\suffix$, and $\suflen$-length generated continuation $\gensuf \coloneqq \gensuffix$. 
For both greedy 
and probabilistic extraction, we 
accommodate near-verbatim 
cases by introducing a distance metric $\dist$ and tolerance $\varepsilon$. 
For greedy extraction, success counts 
\(
    \1[\, \dist(\gensuf, \suf) \le \varepsilon \,]
\). 
(When $\varepsilon\!=\!0$, this reduces to verbatim extraction.) 

For verbatim probabilistic extraction, 
success 
depends on the probability of the exact target, $\pseq$ (Equation~\ref{eq:pz:main}). 
For the near-verbatim case, many continuations---e.g., all three in Figure~\ref{fig:example}---may qualify as extraction success, any of which may reasonably be sampled with a non-deterministic decoding scheme $\dec$.
We therefore define the near-verbatim extraction probability as the aggregate mass of  the set of $\varepsilon$-viable continuations. 
For target suffix $\suf$, denote the \newterm{$\varepsilon$-ball} of near-verbatim matches for distance  
$\dist$ as 
\begin{align}
\label{eq:eball:main}
\ball(\suf)
    \;\triangleq\; \{\,\vv \in \vocab^\suflen: \dist\big(\vv,\suf\big) \le \varepsilon \,\}.  
\end{align}
The \newterm{near-verbatim extraction risk} $\pseqe$ is the total mass on $\ball(\suf)$,\looseness=-1 
\begin{equation}
\label{eq:pze:main}
\pseqe \;\triangleq\; 
\sum_{\vv \in \ball(\suf)} \Pr_{\model,\dec}\big(\vv \mid \pre\big),
\end{equation}
and we count  
success with 
\(
\1\big[\, \pseqe \geq \taumin\,\big] 
\). 
$\ballset^\dist_0(\suf)$ contains only $\suf$, so $p^\dist_{\seq,0}\!=\!\pseq$, i.e., reduces to the verbatim case.) 

We consider two standard token-level distances $\dist: \vocab^\suflen \times \vocab^\suflen \rightarrow \{0, 1, \ldots, \suflen\}$---the \newterm{Hamming} and \newterm{Levenshtein} distances. 
For $\seqb,\seqc \in \vocab^\suflen$,
\vspace*{-.15cm}
\begin{align}
\label{eq:dist:main}
\textstyle
\hamshort(\seqb,\seqc) \;\triangleq\; \sum_{t} \1[\tokb_t \neq \tokc_t]; \qquad 
\levshort(\seqb,\seqc)\;\triangleq\;\min\{d \!: \seqb \xrightarrow{\text{$d$}} \seqc\}.
\end{align}
Hamming ($\hamshort$) counts positional mismatches 
(unit-cost token substitutions). 
Levenshtein ($\levshort$) is more general: 
it is the minimum number of (unit cost) substitution, insertion, or deletion edits $d$ required to transform one sequence into the other. 
For near-verbatim greedy extraction, these metrics are cheap to compute:
just as with verbatim extraction, there is only one suffix to evaluate---the greedy continuation. But to produce more nuanced information about extraction risk, the set of near-verbatim suffixes in $\ball(\suf)$ can be enormous, which can make computing $\pseqe$ enormously expensive.\looseness=-1 

\begin{wrapfigure}{r}{0.45\textwidth}
    \centering
    \vspace{-.3cm}     
    \includegraphics[width=\linewidth]{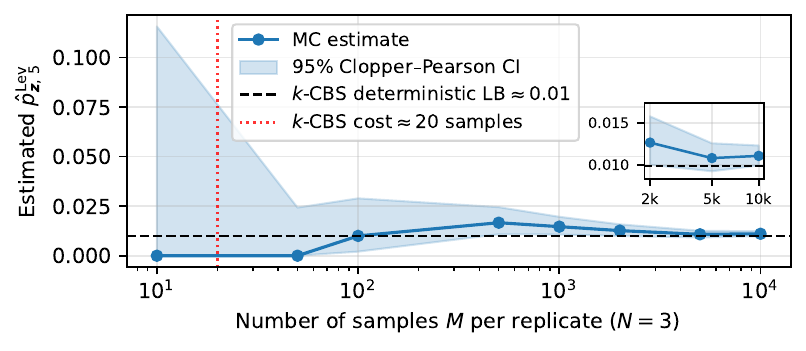}
    \vspace{-.65cm}
    \caption{\textbf{Monte Carlo (MC) estimation.} 
    For Levenshtein distance $\leq\!5$ ($\hat{p}^{\mathsf{Lev}}_{\vz,\, 5}$), we plot convergence for a single sequence $\seq$ from \emph{The Great Gatsby} for \textsc{Llama~2 7B}, showing the pooled MC estimate with a 95\% confidence interval over 3 replicates.
    Our algorithm ($k$-CBS, Section~\ref{sec:kcbs}) produces a deterministic, provably correct lower bound (LB) of $\approx\!0.01$.
    It captures $89.4\%$ of the mean MC estimate at $M\!=\!10^4$ samples, at a cost of $\approx\!20$ MC samples (Appendix~\ref{app:sec:intuition:cost})---a budget at which MC produces no hits.\looseness=-1}
    \label{fig:mc-convergence}
    \vspace{-1.2cm}
\end{wrapfigure}
\custompar{Exactly computing $\pseqe$ is intractable} 
For Hamming, the size of the $\varepsilon$-ball is 
\(
|\hamball(\suf)|
  \!=\! \sum_{r=0}^{\varepsilon} \binom{\suflen}{r}\,(|\vocab|-1)^{r}
\), where 
$\binom{\suflen}{r}$ chooses the $r$ substitution positions and
$(|\vocab|-1)^r$ counts possible substitutions at those positions. 
For $|\vocab|\!=\!32{,}000$, 
$\suflen\!=\!50$, and just \(\varepsilon\!\leq\!2\), \(|\ballset^\hamshort_2(\suf)|\!>\!10^{12}\).
For Levenshtein, the  
$\varepsilon$-ball is even larger 
(Appendix~\ref{app:sec:nv:compute:ham-dominate-lev}). 
It is intractable to teacher force that many suffixes, and it should also be unnecessary. 
Many suffixes in the $\varepsilon$-ball will have $0$ or extremely low probability (e.g., consider a continuation $\gensuf$ that substitutes the token \texttt{\_jazz} 
for \texttt{\_the} in the target suffix $\suf$).\looseness=-1 

\custompar{Estimating $\pseqe$ with MC is expensive}
Alternatively, 
one could prompt $M$ times with the prefix and estimate $\pseqe$ as the proportion of sampled continuations (with decoding scheme $\dec$) lying within the $\varepsilon$-ball of the target suffix. 
This is  statistically unbiased, but also infeasible at scale. 
The probability of never hitting the $\varepsilon$-ball in $M$ i.i.d.\ samples is \(\Pr\big(\textnormal{miss }\ball(\suf)\big) = (1-\pseqe)^M.\)
To guarantee a miss probability of at most $\delta$ requires 
\(
(1 - \pseqe)^M \le \delta
\;\Longleftrightarrow\;
M \ge \ln(1/\delta)/(-\ln(1-\pseqe)).
\) 
Even for a modest miss tolerance of $\delta=0.05$, merely \emph{detecting} an $\varepsilon$-ball with a relatively high mass of $\pseqe=0.01$ requires $M\!\approx\!300$ samples.
Further, detection only guarantees one hit; 
it does not guarantee that the MC estimate of $\pseqe$ is \emph{reliable}.
Figure~\ref{fig:mc-convergence} shows thousands of samples are necessary for an accurate estimate for $p_{\seq, 5}^\levshort\!\approx\!0.01$;
$M\!\approx\!100{,}000$ samples would be needed for reliable estimation of $\pseqe\!=\!\taumin=0.001$ (Appendix~\ref{app:sec:intuition:mc}). 
Compared to teacher-forced inference for verbatim $\pseq$, MC for $\pseqe$ is roughly $M\times$ more expensive (Appendix~\ref{app:sec:intuition:cost}).\looseness=-1 

\vspace{-.2cm}
\section{Decoding-constrained beam search}\label{sec:kcbs}
\vspace{-.1cm}

We show that near-verbatim extraction risk can be estimated at much lower cost than MC sampling, making it feasible to evaluate at scale. 
The overarching intuition is simple:
for a given prefix, beam search is a decoding algorithm that tends to find a set of continuations that are high probability under the model $\model$ (Section~\ref{sec:kcbs:beam}); 
memorized suffixes are especially high probability under the model (Section~\ref{sec:rw}); 
and so, when a sequence is memorized, beam search should return a set of continuations that are near-verbatim matches to the memorized suffix---i.e., continuations in the $\varepsilon$-ball $\ball(\suf)$. 
Following this intuition, we propose algorithms (Sections~\ref{sec:kcbs:baseline} \&~\ref{sec:pruning}) that modify beam search to incorporate decoding scheme $\dec$, which lets us compute an inexpensive, deterministic lower bound on the near-verbatim extraction probability $\pseqe$ (Equation~\ref{eq:pze:main}).
These algorithms have cost comparable to $\bw \ll M$ MC samples (Appendix~\ref{app:sec:intuition:cost}), where $\bw\!=\!20$ works well in practice (Appendix~\ref{app:experiments:width}).\looseness=-1 

\setlength{\textfloatsep}{4pt}
\begin{algorithm}[t]
\caption{Top-$k$ Constrained Beam Search ($k$-CBS)}
\label{alg:kcbs}
{\footnotesize
\KwIn{LLM $\model$; prefix $\pre$; suffix length $\suflen$; beam width $\bw$; top-$k$ parameter $k$; EOS token id; optional $\taumin > 0$}
\KwOut{Set of up to $\bw \cdot k$ pairs $(\genseq,\; \log p)$, where $\genseq = \pre \,\Vert\, \gensuf$ and $\log p = \log \Pr_{\model,\dec}(\gensuf \mid \pre)$; or $\emptyset$ if early termination}
\BlankLine
Compute $\logit_1(\pre)$ via forward pass on $\pre$\tcp*{Prefill}
$\sL_0 \gets \big\{(\pre,\; 0)\big\}$\tcp*{Beam (max. capacity $|\sL_t|=\bw$): pairs $(\genseq,\; \log p)$}
\For{$t = 1, \ldots, \suflen$}{
  $\sC_t \gets \emptyset$\tcp*{Candidate set for step $t$}
  \ForEach{$(\genseq,\; \log p) \in \sL_{t-1}$}{
    $\vocab_{t,k} \gets \TopK_k\big(\logit_t(\genseq)\big)$; \,
    $\vr_t \gets \LogSoftmax\big(\logit_t(\genseq)\big)$; \,
    $Z_t \gets \LogSumExp\big(\vr_t[\vocab_{t,k}]\big)$\;
    \ForEach{$\gentok \in \vocab_{t,k}$}{
      \tcp{append token to partial history, update continuation $\log$ prob}
      $\genseq' \leftarrow \genseq \,\Vert\, \gentok$; \,
      $\log p' \leftarrow \log p + \vr_t[\gentok] - Z_t$; \,
      $\sC_t \gets \sC_t \cup
        \big\{\!\big(\genseq',\;\;
        \log p' \big)\!\big\}$
    }
  }
  \lIf(\tcp*[f]{Return up to $\bw \cdot k$ candidates}){$t = \suflen$}{\Return $\sC_\suflen$}
  Remove from $\sC_t$ any $(\genseq',\, \log p')$ whose last token is EOS\;
  $\sL_t \gets$ top-$\bw$ elements of $\sC_t$ by $\log p'$\tcp*{Prune to beam width}
  \lIf(\tcp*[f]{LB cannot reach $\taumin$}){$\taumin$ and $\max_{(\_,\,\log p) \in \sL_t} \exp(\log p) < \taumin / (\bw \cdot k)$}{\Return $\emptyset$}
  Compute $\logit_{t+1}(\genseq)$ for each $(\genseq, \cdot) \in \sL_t$\;
}
}
\end{algorithm}

\vspace{-.2cm}
\subsection{Beam search}\label{sec:kcbs:beam}
\vspace{-.1cm}

\newterm{Beam search}~\citep{Lowerre1976HARPY} is a deterministic decoding algorithm for autoregressive language models.  
For LLM $\model$ and a prompt $\pre$, at each generation step $t \in 1, \ldots, \suflen$ we maintain a \newterm{beam} of at most $\bw$ $(\prelen+t-1)$-length partial \newterm{histories} $\genseq\coloneqq\pre \,\Vert\,\genseq^{\textnormal{(cont)}}_{<t}$, each with its accumulated \newterm{score}, $\log p(\genseq)\coloneqq\log \Pr_{\model}(\genseq^{\textnormal{(cont)}}_{<t}\mid \pre)$. 
Then, 
\begin{steplist}[leftmargin=.75cm,topsep=.01cm,itemsep=.1cm]    
    \item For each $\genseq$ in the beam, use $\model$ to compute the next-token probability $\Pr_{\model}(\gentok{\mid}\genseq)$ for every vocabulary token $\gentok\in\vocab$.\looseness=-1
    
    \item Expand each history $\genseq$ by every token $\gentok \in \vocab$, yielding a candidate set $\sC_t$ of  
    at most $\bw \cdot |\vocab|$ children $\genseq'\!=\!\genseq \Vert \gentok$ with updated scores
    \(
      \log p(\genseq') = \log p(\genseq) + \log \Pr_{\model}(\gentok{\mid}\genseq)
    \). 
    
    \item Perform an across-beam prune, retaining the $\bw$ highest-scoring $(\prelen+t)$-length histories $\genseq'$ in candidate set $\sC_t$.
    These histories form the $\bw$-sized beam $\sL_t$ for step $t+1$. 
\end{steplist} 
After $\suflen$ steps, the search returns the $\bw$ highest-scoring $\suflen$-length continuations $\gensuf$ that remain in the beam, which 
approximately maximize the conditional probability $\Pr_\model(\gensuf \mid \pre)$. 
Because the joint $\log$-probability of a continuation decomposes across time steps, a continuation must maintain a competitive cumulative score at \emph{every} depth to survive pruning.
Conversely, if at most depths 
a given continuation's 
score is near the top of the model's distribution, 
it is very likely to survive all across-beam prunes and appear in the final beam.
Beam search therefore concentrates computational effort on a thin, high-mass region of candidate continuations rather than exploring the full, exponentially large space (Appendix~\ref{app:sec:intuition:beam}).\looseness=-1

\vspace{-.2cm}
\subsection{Computing a deterministic lower bound on extraction risk}\label{sec:kcbs:baseline}
\vspace{-.1cm}

To quantify near-verbatim probabilistic extraction, we introduce \newterm{decoding-constrained beam search}, which modifies beam search to operate under the distribution induced by decoding scheme $\dec$, and therefore returns continuations with their exact probabilities under $\dec$.
Building on prior work on probabilistic extraction~\citep{hayes2025measuringmemorizationlanguagemodels,cooper2025books}, we focus on top-$k$ decoding (Section~\ref{sec:rw}) to develop \newterm{top-$k$-constrained beam search ($k$-CBS)} (Algorithm~\ref{alg:kcbs}), but our approach 
can be applied with other $\dec$ 
(Appendix~\ref{app:sec:other:nucleus}).\looseness=-1

We highlight key changes to beam search in \textcolor{seabornredmid}{red}.   
For LLM $\model$, decoding scheme \textcolor{seabornredmid}{$\dec=$}\textcolor{seabornredmid}{top-}\textcolor{seabornredmid}{$k$}, and a prompt $\pre$, at each generation step $t \in 1, \ldots, \suflen$ we maintain a beam of at most $\bw$ $(\prelen+t-1)$-length partial histories $\genseq\coloneqq\pre \,\Vert\,\genseq^{\textnormal{(cont)}}_{<t}$, each with  accumulated {score} $\log p(\genseq)\coloneqq\log \Pr_{\model,\textcolor{seabornredmid}{\dec}}(\genseq^{\textnormal{(cont)}}_{<t}\mid \pre)$. 
At each step $t$, 
$\textcolor{seabornredmid}{\vocab_{t,k}(\genseq)}$ 
is the set of \textcolor{seabornredmid}{top-}\textcolor{seabornredmid}{$k$} next tokens with the highest probabilities, given  
history $\genseq$. 
\begin{steplist}[leftmargin=.75cm,topsep=0cm,itemsep=0.1cm]
    \item \label{beam1} For each $\genseq$ in the beam, use $\model$ with \textcolor{seabornredmid}{$\dec$}  to get next-token probabilities $\Pr_{\model,\textcolor{seabornredmid}{\dec}}(\gentok \mid \genseq)$;
    these probabilities are computed only over the \textcolor{seabornredmid}{top-}\textcolor{seabornredmid}{$k$} set 
$\textcolor{seabornredmid}{\vocab_{t,k}(\genseq)}$
(Equation~\ref{eq:topk:main}). 
    
    \item \label{beam2} Expand each history $\genseq$ by the tokens in $\textcolor{seabornredmid}{\vocab_{t,k}(\genseq)}$, 
yielding at most $\bw \cdot \textcolor{seabornredmid}{k}$ 
candidates. 
We update these candidates' scores according to the \textcolor{seabornredmid}{top-}\textcolor{seabornredmid}{$k$} distribution probabilities---i.e., $\log p(\genseq') = \log p(\genseq) + \log \Pr_{\model,\textcolor{seabornredmid}{\dec}}(\gentok \mid \genseq)$. 
    
    \item \label{beam3} Perform an across-beam prune to $\bw$, \textcolor{seabornredmid}{except for at the final step $\suflen$}. 
    At step $\suflen$, return up to $\bw \cdot \textcolor{seabornredmid}{k}$ $\suflen$-length continuations $\gensuf$ and their scores $\log \Pr_{\model,\textcolor{seabornredmid}{\dec}}(\gensuf \mid \pre)$---the exact probabilities under the \textcolor{seabornredmid}{top-}$\textcolor{seabornredmid}{k}$-renormalized decoding distribution.\looseness=-1
\end{steplist}
Notably, beam search's well-known lack of output diversity~\citep{li-etal-2016-diversity,vijayakumar2018diversebeamsearchdecoding} is a \emph{feature} here: 
for memorized sequences, it is useful for the beam to concentrate on the high-mass region near $\suf$, not to explore diverse alternatives.\looseness=-1 

\custompar{From $k$-CBS outputs to a lower bound}
For each returned $\gensuf$, its score is exactly the conditional probability given the prefix under top-$k$ decoding. 
For the sequence from \emph{The Great Gatsby} in Figure~\ref{fig:example}, $\bw=20$, and $k=40$, so $k$-CBS (Algorithm~\ref{alg:kcbs}) returns $\bw \cdot k = 800$ $\suflen$-length continuations $\gensuf$ and their probabilities---including the three near-verbatim $\gensuf$ in the figure. 
We identify the $\gensuf$ that are near-verbatim matches to the target $\suf$ (i.e., are in $\varepsilon$-ball $\ball(\suf)$), and we sum their probabilities to get a lower bound on the near-verbatim extraction probability $\pseqe$ with respect to top-$k$ decoding (Equation~\ref{eq:pze:main}). 

More formally, let $\sF$ denote the set of $(\genseq, \log p)$ pairs returned by $k$-CBS after $\suflen$ steps.
We filter $\sF$ to retain continuations $\gensuf$ that satisfy $\dist(\gensuf, \suf) \leq \varepsilon$, yielding $\sF^{(\leq\varepsilon)} \subseteq \sF$. Then,\looseness=-1 
\begin{align}
\label{eq:pseqe}
\lb \;\triangleq\;  \sum_{(\_, \log p) \in \sF^{(\leq\varepsilon)}} \!\!\!\exp(\log p) \;\leq\; \pseqe.
\end{align}
This is a valid lower bound on $\pseqe$: the candidates in $\sF^{(\leq\varepsilon)}$ are a subset of all continuations within $\varepsilon$ of $\suf$, and their probabilities under top-$k$ are computed exactly 
(Appendix~\ref{app:sec:kcbs}).
Unlike MC sampling (Section~\ref{sec:warmup}), because beam search is deterministic, this bound is deterministic;
it requires no repeated trials.
As shown in Figure~\ref{fig:mc-convergence}, for beam width $\bw=20$, $k$-CBS produces this bound at a cost comparable to just ${\approx}\,20$ MC samples---a budget at which MC produces no hits.
For the sequence in Figure~\ref{fig:example}, 
$\mathrm{LB}_{5, \levshort} = 0.716$---nearly $5\times$ the mass of the verbatim target suffix (i.e., the $0.1431$ mass that verbatim probabilistic extraction captures). 

We use $\hatpseqe = \lb$ to estimate $\pseqe$, keeping in mind that it is an underestimate. 
$k$-CBS may miss near-verbatim continuations that were pruned from the beam, or miss extractable sequences entirely if no near-verbatim continuation survives.
Nevertheless, our experiments (Section~\ref{sec:experiments}) show that $k$-CBS captures rich information about extraction risk in practice.\looseness=-1 

\custompar{Why $k$-CBS produces useful lower bounds} 
The lower bound in Equation~\ref{eq:pseqe} is valid but could  
be vacuous (e.g., $\lb = 0$ if no near-verbatim path survives).
However, for memorized sequences, it captures substantial mass because high-probability continuations are constrained to pick high-ranked tokens at almost every step---precisely the region that $k$-CBS explores. 
A simple counting argument illustrates why.
For continuation $\gensuf$, the conditional probability given prefix $\pre$ is the product of per-token conditional probabilities across $\suflen$ steps (Equation~\ref{eq:pz:main}).
For this product to exceed the extraction minimum $\taumin\!=\!0.001$, per token, the continuation must sustain a geometric-mean probability of at least $\taumin^{1/\suflen}\!=\!(0.001)^{1/50} \!\approx\!0.87$---on average, enormous probability over $\suflen\!=\!50$ steps.

For instance, at most $3$ steps can have per-token probability as low as $0.1$, since $(0.1)^3 \cdot 1^{47} = 0.001 = \taumin$.  
More generally, a token with probability $p$ must be ranked at least $\lceil 1/p \rceil$, since the tokens ranked above it each have probability $\geq p$ and probabilities sum to $1$.
Following similar reasoning, which we prove formally in the appendix,\footnote{A token ranked $R$ or worse contributes probability at most $1/R$.
    At decoding step $t$, let $r_t$ 
    be the rank of the realized token $\subgensuftok{t}$ ($1$ = highest probability).
    We can show that, for any integer $R\ge 2$, the number of steps $t\leq\suflen$ for which $r_t\geq R$ is at most $\lfloor {\ln \taumin}/{\ln(1/R)}\rfloor$ (Appendix~\ref{app:sec:intuition:beam:math}).
}
a continuation can pick a token that falls below the top-$10$ on at most $2$ of $50$ steps;
it can deviate from the distribution's top-$1$ token on at most $9$ steps.
This is exactly the region $k$-CBS explores---continuations that ``hug'' the top of the distribution's ranking at nearly every step.
For memorized sequences, the near-verbatim $\gensuf\in\ball(\suf)$ that dominate $\pseqe$ are therefore very likely to survive across-beam pruning and be captured by $k$-CBS. (See Appendix~\ref{app:sec:intuition:beam:math} for further discussion, including a survival guarantee for heavy-mass paths.)\looseness=-1

\vspace{-.2cm}
\subsection{Constrained beam search with an $\varepsilon$-viable pruning rule}\label{sec:pruning}
\vspace{-.1cm}

Baseline $k$-CBS (Algorithm~\ref{alg:kcbs}) is agnostic to $\dist$: 
it returns $\bw \cdot k$ candidates, which are post-processed to identify those within $\varepsilon$ of the target. 
This generality comes at a cost.  
Throughout the search, beam capacity may be spent on continuations that can never satisfy reasonable $\dist(\gensuf, \suf) \leq \varepsilon$, and the search runs for all $\suflen$ steps even when no viable paths remain. 
We can further modify $k$-CBS to remedy both of these issues with 
\newterm{$\varepsilon$-viability pruning}, which integrates a chosen distance metric $\dist$ and tolerance $\varepsilon$ directly into the search. 
At each step $t$, for \ref{beam2} we check whether each token-expanded partial history $\genseq'$ 
can still produce a final continuation $\gensuf$ within distance $\varepsilon$
of the target---whether there exists \emph{any} possible completion of the continuation that satisfies $\dist(\gensuf, \suf) \leq \varepsilon$. 
If not, the candidate is provably non-$\varepsilon$-viable; 
we do not include it in $\sC_t$, 
freeing beam capacity for $\varepsilon$-viable candidates.\looseness=-1

The incorporation of this $\varepsilon$-viability pruning rule requires only bookkeeping for each item in the beam---no additional model forward passes.
It also enables early termination when no $\varepsilon$-viable candidates remain, often making it cheaper to run than baseline $k$-CBS for non-extractable sequences (Appendix~\ref{app:sec:intuition:cost}).
The trade-off is that, unlike for baseline $k$-CBS (Appendix~\ref{app:sec:bleu}), $\dist$ must be a token-level  metric for which a lower bound on the final distance can be computed incrementally from streamable per-path state (Appendix~\ref{app:sec:prune:invariants:generic-viability-kcbs}), and $\dist$ and $\varepsilon$ 
must be specified upfront.\looseness=-1

\custompar{Hamming distance}
For $\hamshort$ (Equation~\ref{eq:dist:main}), $\varepsilon$-viability is straightforward: 
we maintain a running mismatch count for each candidate in the beam. 
At step $t$, if the generated token $\gentok$ differs from the target token $\suftok{t}$, the count increments. 
Since positional mismatches cannot be undone, this count is monotonically non-decreasing; 
once it exceeds $\varepsilon$, the candidate is permanently non-$\varepsilon$-viable and is removed; it is not expanded further (Appendix~\ref{app:sec:prune:hamming}).\looseness=-1

\custompar{Levenshtein distance}
For $\levshort$ (Equation~\ref{eq:dist:main}), $\varepsilon$-viability is more subtle. 
Later insertions and deletions can repair earlier misalignments, so the minimum achievable distance can \emph{decrease} between steps. 
We therefore track richer per-continuation state in the beam: 
the range of $\levshort$ distances between the partial $t$-length continuation and each $j$-length prefix $\subsuf{1:j}$ of the target suffix $\suf$.\footnote{We propose a streamed version of the \citeauthor{ukkonen}-banded 
    \citeauthor{wagnerfischer1974} dynamic program. 
    Define the dynamic programming 
    table $\dptab[t,j] \triangleq \levshort\big(\subgensuf{1:t},\subsuf{1:j}\big)$, with $\dptab[0,j]\!=\!j$ and $\dptab[t,0]\!=\!t$. 
    At each step $t$, consider generated token $\subgensuftok{t}$ and each suffix token $\suftok{j}$ ($j \!\in\! [\max\{0, t{-}\varepsilon\},\, \min\{\suflen, t{+}\varepsilon\}]$). 
    Update the table to edit cost $\dptab[t,j] = \min\!\big\{\overbrace{\dptab[t{-}1,j]{+}1}^{\text{delete } \subgensuftok{t}}, \overbrace{\dptab[t, j{-}1]{+}1}^{\text{insert } \suftok{j}}, \overbrace{\dptab[t{-}1, j{-}1]{+}\1[\subgensuftok{t}{\neq}\suftok{j}]}^{\text{match/substitute}}\big\}$. 
} 
The continuation remains viable so long as the minimum of this range is at most $\varepsilon$; 
once the minimum exceeds $\varepsilon$, no possible completion can bring the distance back within the $\varepsilon$ budget and the continuation is pruned (Appendix~\ref{app:sec:prune:lev}).\looseness=-1

All returned continuations are $\varepsilon$-viable by construction---no post-processing is needed. 
The 
algorithm also tracks an upper bound $\mathrm{UB} = \mathrm{LB} + \textit{bank}$, where \textit{bank} accumulates the mass of $\varepsilon$-viable continuations removed by the across-beam prune; 
in practice this bound is often loose, so we use $\hatpseqe = \lb$. 
Full 
details 
and invariant proofs are in Appendix~\ref{app:sec:prune}.\looseness=-1

\vspace{-.2cm}
\section{Experiments}\label{sec:experiments}
\vspace{-.1cm}

In the main paper, we show results for $k$-CBS on two model families and datasets, which illustrate different patterns in near-verbatim extraction:
(i)~\textsc{OLMo 2} (7B, 13B, 32B)~\citep{olmo2} on sequences from $10{,}000$ Wikipedia articles from its training data (Dolma 1.7~\citep{dolma}); 
and (ii)~\textsc{Llama 2} (7B, 13B, 70B) on \emph{The Great Gatsby} from the Books3 corpus, which is known to be in \textsc{Llama}'s training data~\citep{touvron2023llamaopenefficientfoundation, kadreyamendedconsolidated}.
In these experiments, we use baseline and $\levshort$-pruned $k$-CBS with $k\!=\!40$, beam width $\bw\!=\!20$, $\varepsilon\!=\!5$, prefix length $\prelen\!=\!50$, suffix length $\suflen\!=\!50$, and minimum extraction probability $\taumin\!=\!0.001$.
Additional details (e.g., runtime, \href{https://github.com/pasta41/probabilistic-extraction-toolkit}{code}, \href{https://afedercooper.info/near-verbatim/}{visualizations}), extraction experiments, 
negative controls, and investigations of $k$-CBS algorithm parameters are in Appendix~\ref{app:sec:experiments}.\looseness=-1

\vspace{-.2cm}
\subsection{Revealing additional extracted sequences and increased risk}\label{sec:experiments:a}
\vspace{-.1cm}

\setlength{\textfloatsep}{10pt}

\begin{figure*}[t!]
\centering
\includegraphics[width=\linewidth]{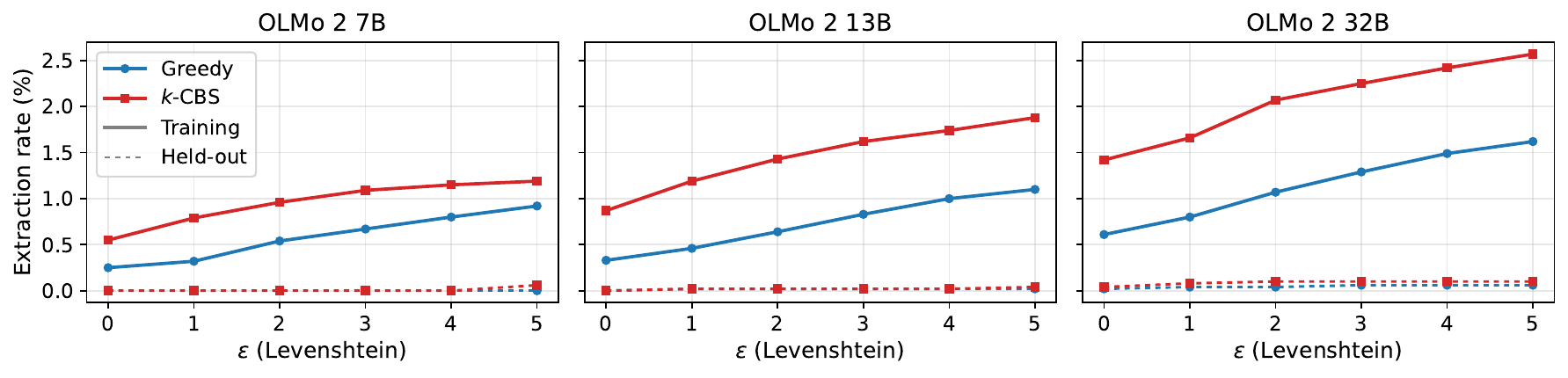}
\vspace{-.6cm}
\caption{\textbf{Comparing extraction rates.} 
For \textsc{OLMo 2 7B}, \textsc{13B}, and \textsc{32B}, we show  
rates for verbatim 
($\varepsilon\!=\!0$) and near-verbatim extraction for $\levshort\, \varepsilon \in \{1,\ldots,5\}$.
For greedy near-verbatim, one generates the single greedy $\gensuf$ and checks if $\levshort(\gensuf,\suf) \leq \varepsilon$.
We use a sample of $10{,}000$ sequences from Wikipedia from \textsc{OLMo 2}'s training data;
to assess validity, we also run negative controls on $5{,}000$ held-out sequences scraped from Wikipedia that post-date  \textsc{OLMo 2}'s training cutoff.
Greedy rates are computed exactly. 
Probabilistic rates are computed with $k$-CBS (Section~\ref{sec:kcbs:baseline});
they may miss some valid instances of extraction, and thus should be interpreted as lower bounds on extraction rates.
We set $B\!=\!20$ for $k$-CBS, so those results cost $\sim\!20\times$ more than greedy extraction (Section~\ref{sec:warmup}).}
\label{fig:olmo:rates:main}
\vspace{-.2cm}
\end{figure*}
\begin{figure*}[t!]
\centering
\hspace{-.49cm}
\includegraphics[width=1.03\linewidth]{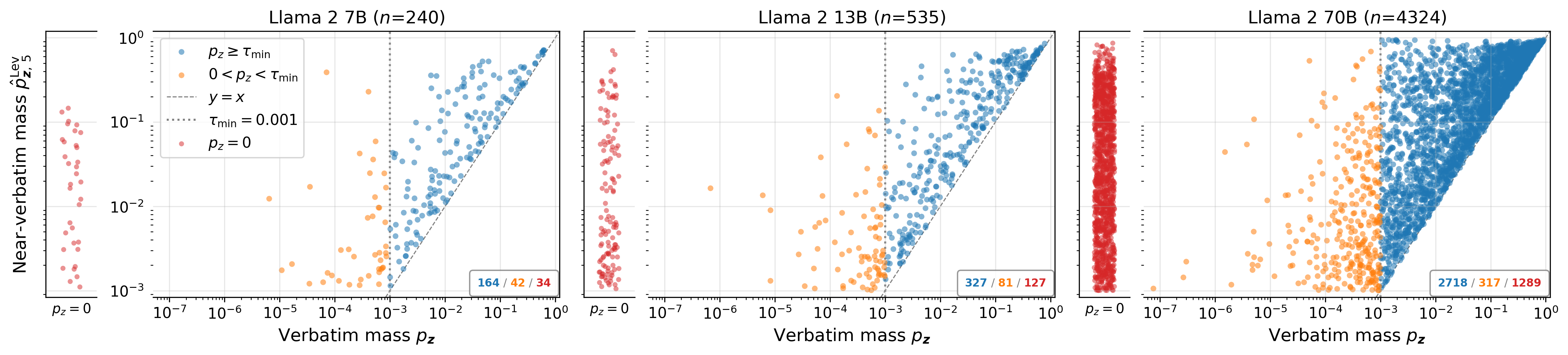}
\vspace{-.4cm}
\caption{\textbf{Near-verbatim mass vs.\ verbatim mass.}  
\textsc{Llama 2} on \emph{The Great Gatsby}; 
each point is one sequence.
Axes show near-verbatim ($p_{\seq,5}^\levshort$, $\levshort\,\varepsilon\!=\!5$) vs.\ verbatim ($\pseq$) extraction mass on a $\log$--$\log$ scale.
\textcolor{seabornredmid}{\textbf{Red}}/\textcolor{seabornorangemid}{\textbf{orange}}
points are ``unlocked'' by near-verbatim extraction (to the left of the $\taumin$ dotted reference line, $\pseq\!<\!\taumin$, but $p_{\seq,5}^\levshort\!\geq\!\taumin$);
\textcolor{seabornbluemid}{\textbf{blue}} points are verbatim-extractable ($\pseq\!\geq\!\taumin$).
Points above the dashed $y\!=\!x$ line show increased extraction risk when near-verbatim mass is accounted for.\looseness=-1}
\label{fig:gatsby:scatter:main}
\end{figure*}

\custompar{Increased extraction rates}
Figure~\ref{fig:olmo:rates:main} compares greedy and baseline $k$-CBS extraction rates (Appendix~\ref{app:sec:experiments:metrics}) for \textsc{OLMo 2} across model sizes and 
$\levshort\, \varepsilon \in \{0,\ldots,5\}$.
$k$-CBS dominates greedy at every $\varepsilon$ and model size, with both rates and gaps growing with model size.
That is, greedy extraction under-counts extraction more for larger models. 
For \textsc{OLMo 2 32B}, the rate at $\varepsilon\!=\!5$ reaches $2.57\%$, compared to $1.62\%$ for greedy near-verbatim, $1.42\%$ for baseline $k$-CBS at $\varepsilon\!=\!0$ (verbatim), and $0.61\%$ for greedy verbatim.
Lines for held-out data (negative controls) 
stay flat near zero, supporting the validity of our extraction procedure (Appendix~\ref{app:experiments:olmo:extraction}).\looseness=-1

\custompar{Analyzing ``unlocked'' instances of extraction}
Figure~\ref{fig:gatsby:scatter:main} plots each sequence's near-verbatim extraction mass ($\hat{p}_{\seq,5}^{\levshort}$) against its verbatim mass ($\pseq$) for \textsc{Llama 2} on \emph{The Great Gatsby}.
The \textcolor{seabornredmid}{\textbf{red}} and \textcolor{seabornorangemid}{\textbf{orange}} points are sequences ``unlocked'' by near-verbatim extraction: 
they fall below the $\taumin$ threshold for verbatim extraction, but are extractable when we account for near-verbatim mass.
\textcolor{seabornredmid}{\textbf{Red}} points have \emph{zero} verbatim mass, 
while \textcolor{seabornorangemid}{\textbf{orange}} points have nonzero but sub-threshold (below $\taumin\!=\!0.001$) verbatim mass.

At 70B, $1{,}606$ sequences are unlocked, of which $1{,}289$ ($80.3\%$) have zero verbatim mass.
These unlocked sequences are not marginal: their mean near-verbatim mass is $0.086$ ($86\times$ the extraction threshold $\taumin$), with a maximum of $0.863$.
Typical edits include 
spacing around punctuation (e.g., ``money---~that'' vs.\ ``money---that'', $\levshort\!=\!1$, mass $0.777$). 
\textcolor{seabornbluemid}{\textbf{Blue}} points are verbatim-extractable; those above the $y\!=\!x$ line show increased extraction risk when near-verbatim mass is included.
The number of unlocked sequences grows from $76$ (7B)  
to $1{,}606$ (70B)---a $21\times$ increase for a $10\times$ increase in parameters.
Interestingly, many unlocked sequences at smaller model sizes are verbatim extractable at larger ones (Appendix~\ref{app:sec:experiments:setup:llama}).\looseness=-1 

\vspace{-.2cm}
\subsection{Patterns in near-verbatim extraction risk}\label{sec:experiments:b}
\vspace{-.1cm}

\begin{figure*}[t!]
\centering
\begin{subfigure}[t]{0.49\textwidth}
    \centering
    \includegraphics[width=\linewidth]{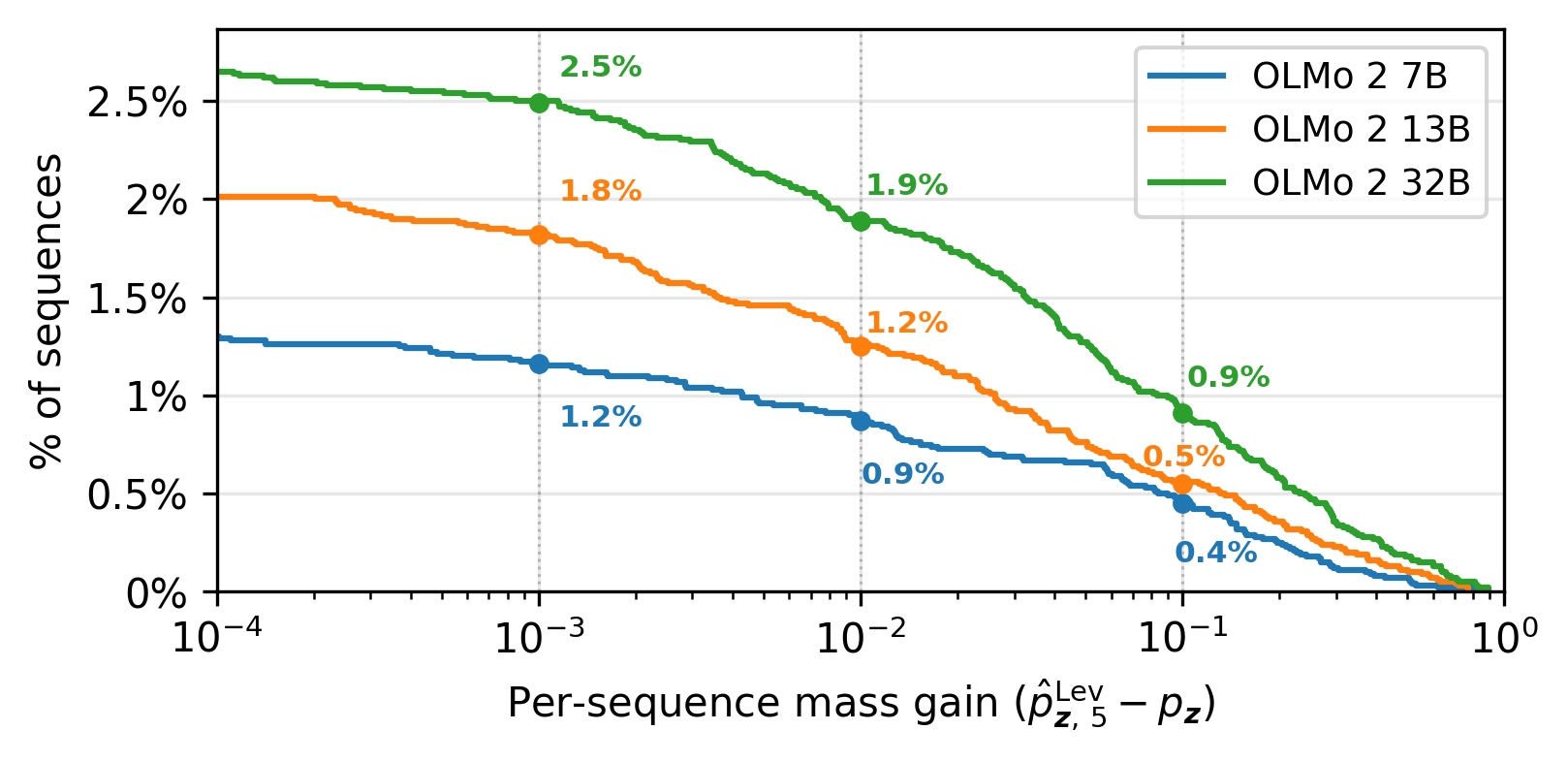}
    \vspace{-.5cm}
    \caption{\textsc{OLMo 2} on Wikipedia}
    \label{fig:ccdf:olmo:main}
\end{subfigure}
\hfill
\begin{subfigure}[t]{0.49\textwidth}
    \centering
    \includegraphics[width=\linewidth]{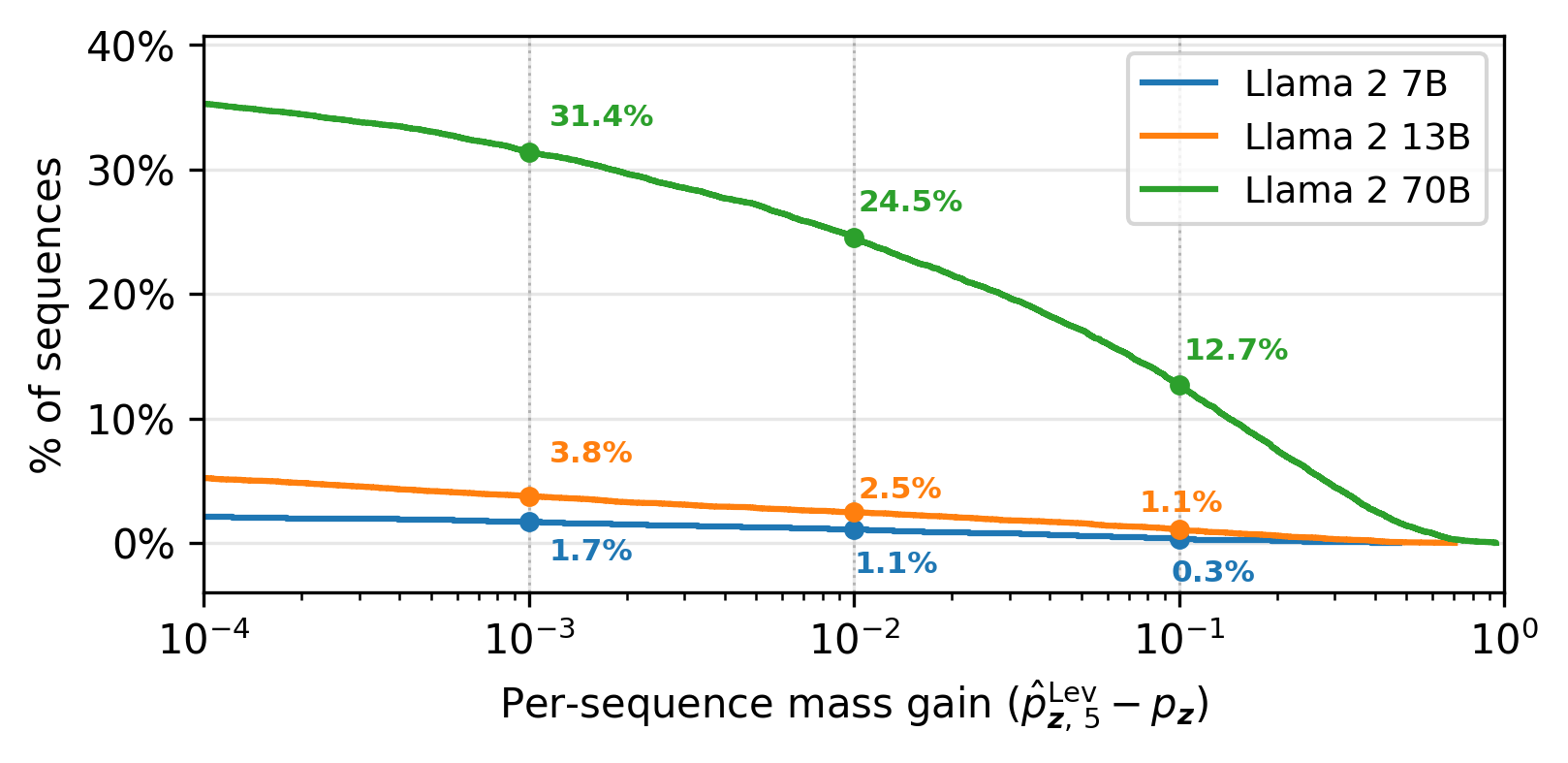}
    \vspace{-.5cm}
    \caption{\textsc{Llama 2} on \emph{The Great Gatsby}}
    \label{fig:ccdf:gatsby:main}
\end{subfigure}
\vspace{-.1cm}
\caption{\textbf{CCDF of per-sequence near-verbatim mass gain.} 
For $\levshort\, \varepsilon\!=\!5$ mass minus verbatim mass ($\hat{p}_{\seq,5}^{\levshort} - \pseq$), a point $(x, y)$ means $y\%$ of sequences have extraction-mass gain $\geq x$.\looseness=-1}
\label{fig:ccdf:main}
\end{figure*}

\custompar{Risk increases scale with model size}
Figure~\ref{fig:gatsby:scatter:main} suggests that larger models have more unlocked sequences with larger near-verbatim mass.
Figure~\ref{fig:ccdf:main} quantifies this directly, showing the complementary CDF 
(CCDF) of absolute per-sequence 
mass gain ($\hat{p}_{\seq,5}^{\levshort} - \pseq$)---i.e., the additional extraction mass from near-verbatim continuations.
Figure~\ref{fig:ccdf:gatsby:main} shows that, for \textsc{Llama 2 70B} and \emph{The Great Gatsby}, $12.7\%$ of all sequences in the book ($32.9\%$ of extractable sequences) exhibit a $\geq\!0.1$ mass gain.
For both CCDFs, 
the curves shift upward and rightward with model size: 
larger models have both \emph{more} sequences with positive extraction risk gain and \emph{larger} absolute per-sequence gains.
Verbatim probabilistic extraction therefore produces increasingly large undercounts of near-verbatim extraction risk at larger model scales.\looseness=-1

\custompar{Per-sequence verbatim mass share varies by model size and text}
The absolute mass gains described above raise a natural question: 
for a given extractable sequence, what relative portion of its total extraction risk comes from the verbatim continuation?
For each extracted sequence $\seq$, 
we compute this as the \newterm{verbatim share}:\looseness=-1 
\vspace{-.15cm}
\begin{align}
\label{eq:verbatim-share}
    \mathsf{verbatim\_share}(\seq) \;\triangleq\;  \tfrac{\pseq}{\hat{p}^\levshort_{\seq,\,5}} \times 100\%. 
\end{align}
This percentage reflects how much the extraction mass concentrates on the verbatim sequence versus on near-verbatim variants (for $\seq$, the near-verbatim only share is $100\% - \mathsf{verbatim\_share}(\seq)$). 
Sequences with $0\%$ verbatim share are those for which the extraction mass comes entirely from near-verbatim continuations ($\pseq\!=\!0$, e.g., the \textcolor{seabornredmid}{\textbf{red}} ``unlocked'' points in Figure~\ref{fig:gatsby:scatter:main}). 
Those with $100\%$ verbatim share exhibit no increase in extraction risk when we account for near-verbatim mass ($\pseq\!=\!\hat{p}^\levshort_{\seq,\,5}$, e.g., the \textcolor{seabornbluemid}{\textbf{blue}} points on $y=x$ in Figure~\ref{fig:gatsby:scatter:main}). 
Most sequences fall somewhere in between these extremes, like the one in Figure~\ref{fig:verbatim-share-ex}, for which verbatim share is $27.56\%$.\looseness=-1

We can compute verbatim share for every extracted sequence and plot the distribution, as in Figure~\ref{fig:violin:main}. 
We indicate the median to summarize the verbatim share of the typical extracted sequence.
For instance, for \textsc{Llama 2 7B} on \emph{The Great Gatsby} (Figure~\ref{fig:violin:gatsby:main}), the median verbatim share is $27.53\%$: 
the median sequence resembles the example in Figure~\ref{fig:verbatim-share-ex}, with respect to the decomposition of verbatim and near-verbatim mass. 

In this setting, the verbatim share decreases substantially with model size. 
Larger \textsc{Llama 2} models spread more of their extraction mass across near-verbatim variants of memorized \emph{Gatsby} text.
This appears to be driven by changes in the composition of the extractable set (e.g., the huge increase in \textcolor{seabornredmid}{\textbf{red}} points in Figure~\ref{fig:gatsby:scatter:main}). 
The set grows dramatically with model size (for \emph{Gatsby}, $240 \to 535 \to 4{,}324$), and newly extractable sequences at larger model sizes tend to have low verbatim share (median verbatim share of $12.9\%$ for sequences first extractable at 70B). 
Through qualitative analysis, we observe near-verbatim edits include spelling, punctuation, and minor syntactic variation---possibly reflecting training on multiple editions or learned variation in natural language syntax (Appendix~\ref{app:experiments:llama2-scale:extraction}).\looseness=-1

The results for \textsc{OLMo 2} on Wikipedia exhibit a different pattern. 
In contrast to \textsc{Llama 2} on \emph{Gatsby}, the median verbatim share \emph{increases} slightly with model size, 
consistent with larger models concentrating more mass on the exact Wikipedia-article training text.
However, verbatim share remains very low overall (Appendix~\ref{app:experiments:olmo}). 
For instance, half of extracted sequences for \textsc{OLMo 2 7B} have verbatim share below the $0.03\%$ median---i.e., nearly half of all extractable sequences have zero verbatim mass ($\pseq\!=\!0$). 

\begin{figure*}[t!]
\centering
\includegraphics[width=\linewidth]{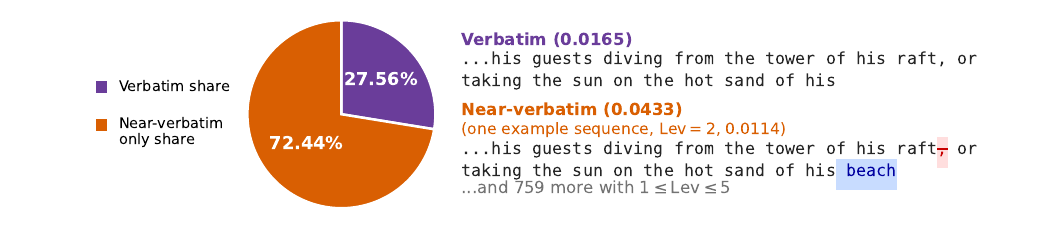}
\vspace{-.8cm}
\caption{\textbf{Verbatim share.}
For \textsc{Llama 2 7B} and one sequence from \emph{The Great Gatsby}, we show the verbatim share (Equation~\ref{eq:verbatim-share}): 
$27.56\%$ of the total near-verbatim mass $p_{\seq,5}^\levshort\!\approx\!0.06$.\looseness=-1}
\label{fig:verbatim-share-ex}
\end{figure*}
\begin{figure*}[t!]
\vspace{-.1cm}
\centering
\begin{subfigure}[t]{0.44\textwidth}
    \centering
    \includegraphics[width=\linewidth]{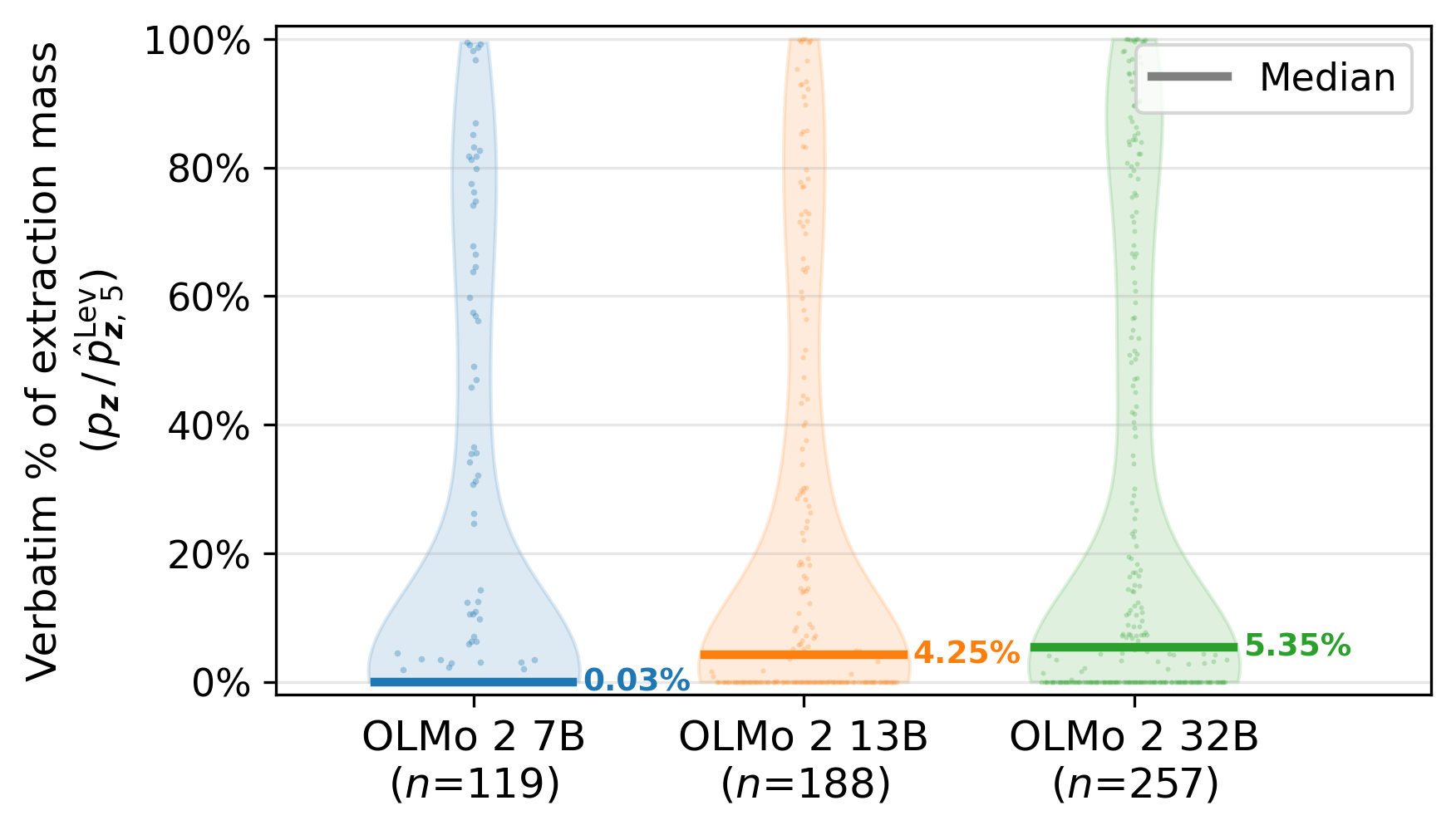}
    \vspace{-.4cm}
    \caption{\textsc{OLMo 2} on Wikipedia}
    \label{fig:violin:olmo:main}
\end{subfigure}
\hfill
\begin{subfigure}[t]{0.44\textwidth}
    \centering
    \includegraphics[width=\linewidth]{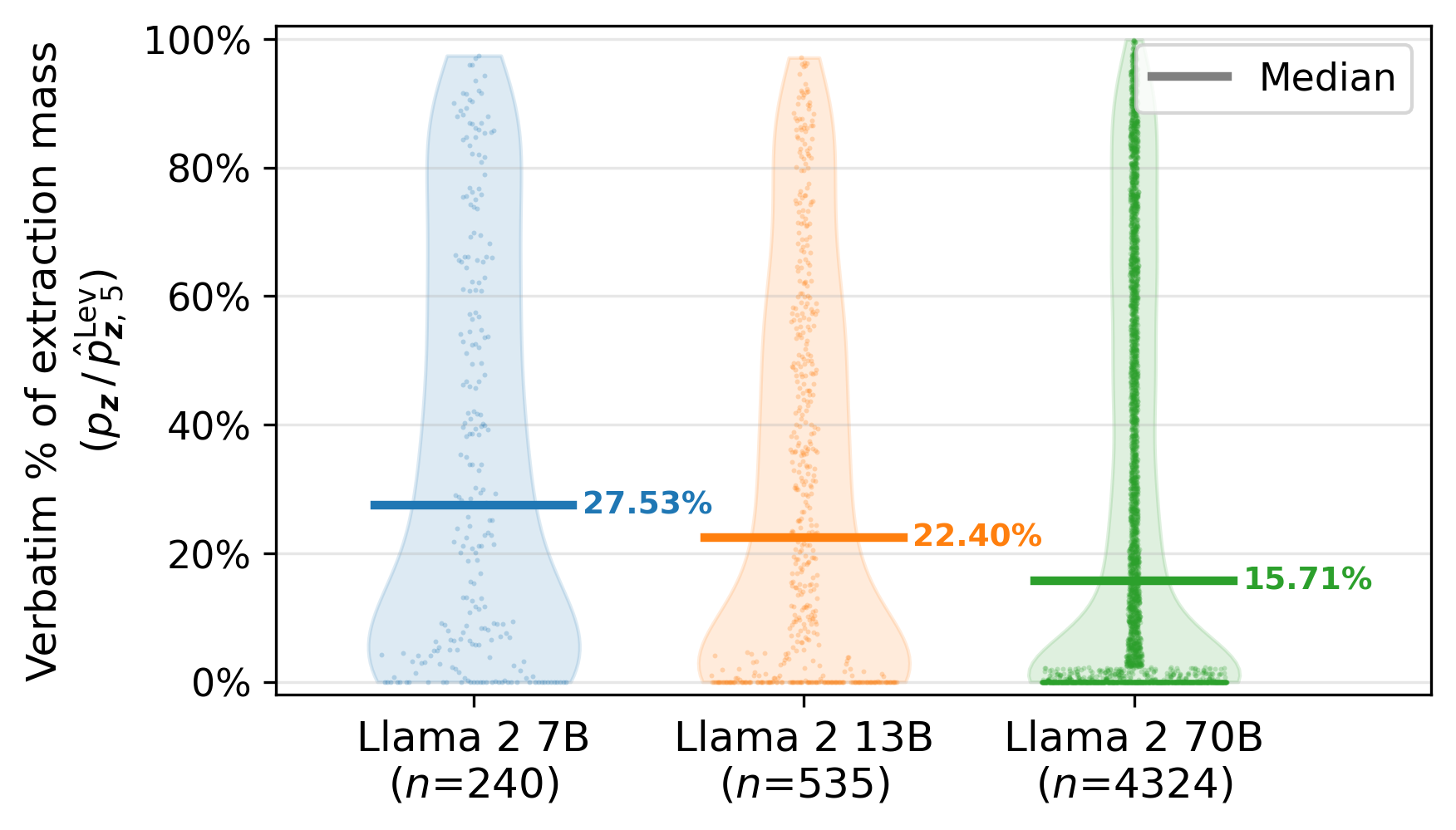}
    \vspace{-.4cm}
    \caption{\textsc{Llama 2} on \emph{The Great Gatsby}}
    \label{fig:violin:gatsby:main}
\end{subfigure}
\vspace{-.1cm}
\caption{\textbf{Distribution of verbatim share.}
We compute verbatim share (Equation~\ref{eq:verbatim-share}) for each extracted sequence and plot the distribution.
Higher values indicate extraction mass is dominated by verbatim memorization;
lower values indicate mass spread across near-verbatim variants.
We annotate the median verbatim-mass shares.\looseness=-1}
\label{fig:violin:main}
\end{figure*}

These patterns illustrate that near-verbatim extraction risk manifests differently depending on the type of text, model, and model size---underscoring the value of analysis beyond verbatim extraction.
We defer investigation of the underlying causes to future work.
Additional results on other datasets and models appear in Appendix~\ref{app:sec:experiments}.\footnote{\textsc{Pythia} models on Enron show a stark increase in verbatim share with model size (Appendix~\ref{app:experiments:pythia}).} 
\vspace{-.2cm}
\section{Conclusion}\label{sec:conclusion}
\vspace{-.1cm}

We introduce decoding-constrained beam search (Section~\ref{sec:kcbs}): 
a family of algorithms that produces deterministic lower bounds on near-verbatim extraction probability at a fraction of the cost of Monte Carlo estimation (Section~\ref{sec:warmup}).
Our experiments show that accounting for near-verbatim probabilistic extraction (Section~\ref{sec:rw}) reveals substantially more memorization and extraction risk, with rich variation across model sizes and types of text (Section~\ref{sec:experiments}).
In future work, we will further investigate patterns in near-verbatim extraction risk, and explore decoding-constrained beam search as a memorization diagnostic tool during model training.\looseness=-1

\section*{Acknowledgments}
We thank Ahmed Ahmed and Abe Hou for feedback on earlier versions of this work.
Katherine Lee joined OpenAI as a full-time employee after contributing to this work.

\bibliographystyle{colm2026_conference}
\bibliography{references}

\clearpage
\newpage
\appendix
\onecolumn

\section*{Appendix Table of Contents}
\startcontents[appendix]
\printcontents[appendix]{l}{1}{\setcounter{tocdepth}{3}}
\clearpage

\section{Language models and verbatim training-data extraction}\label{app:sec:background}
\vspace{-.1cm}

We give basic background and our notation for language models and decoding schemes (Appendix~\ref{app:sec:background:lm}).
We then discuss background and related work concerning the state-of-the-art approach for measuring verbatim extraction (Appendix~\ref{app:sec:background:extraction}).

\subsection{Language models and decoding schemes}\label{app:sec:background:lm}

Let $\vocab$ denote the token \newterm{vocabulary} for a \newterm{large language model (LLM)}.
For example, $|\vocab|=32{,}000$ for \textsc{Llama 1} models. 
An LLM with \newterm{weights} $\model$ maps a sequence of tokens $\seqb=(\tokb_1, \ldots, \tokb_n) \in \vocab^n$ to a logit vector $\logit \in \R^{|\vocab|}$:                                                                    
$\model: \vocab^n \to \R^{|\vocab|}$.
A \newterm{decoding scheme} $\dec$ defines how logits are mapped to a next-token sampling distribution, $\softmaxsf_\dec: \R^{|\vocab|} \to \mathcal{P}(\vocab)$, where $\mathcal{P}(\vocab)$ is the set of probability distributions over~$\vocab$.
For example, temperature scaling or top-$k$ filtering each define such transformations (discussed below). 
Together, $(\model,\dec)$ define an \newterm{autoregressive generation process}: 
given a \newterm{prompt} $\seqb_{1:i}$, at each step $t > i$ we (i) compute $\logit_t = \model(\seqb_{1:t-1})$ and (ii) sample $\tokb_t \sim \softmaxsf_\dec(\logit_t)$, (iii) producing a \newterm{continuation}. 
After $\suflen$ iterations, we obtain a generated \newterm{continuation} $\hat{\seqb}_{\sufinds}$.\looseness=-1 

\paragraph{Logits and probabilities.}
Given a sequence $\seqb_{1:t-1}$, the model $\model$ outputs for every token $v \in \vocab$ a real value $\logit_t[v]$ called a \newterm{logit} associated with the next token $\tokb_t$. 
($\logit_t \in \R^{|\sV|}$ is the entire logit vector.) 
Logits are unnormalized scores, which are mapped to a probability distribution over~$\vocab$ via the \newterm{softmax function}:  
\begin{align}
\label{eq:softmax}
\softmaxsf(\logit_t)[v]
\;=\;
\frac{\exp\!\big(\logit_t[v]\big)}{\sum_{u \in \vocab} \exp\!\big(\logit_t[u]\big)},
\qquad v \in \vocab,
\end{align}
where the model's next-token conditional probability distribution is
\[
\Pr(\cdot \mid \seqb_{1:t-1};\model) \;=\; \softmaxsf(\logit_t),
\]
which is a function $\R^{|\vocab|} \to [0,1]^{|\vocab|}$. 
In particular, the probability of a given token $v$ is
\[
\Pr(\tokb_t = v \mid \seqb_{1:t-1};\model) \;=\; \softmaxsf(\logit_t)[v].
\]

\paragraph{Temperature scaling.}
A \newterm{temperature} parameter $\temp>0$ can rescale the logits before applying softmax:
\begin{align}
\label{eq:temp}
\Pr_\temp(\cdot \mid \seqb_{1:t-1};\model) \;=\; \softmaxsf\!\left(\tfrac{1}{\temp} \logit_t\right).
\end{align}
For a specific token $v$, 
\begin{align}
\label{eq:temp2}
\Pr_\temp(\tokb_t = v \mid \seqb_{1:t-1};\model) 
= \frac{\exp(\frac{\logit_t[v]}{\temp})}{\sum_{u \in \vocab} \exp(\frac{\logit_t[u]}{\temp})}.
\end{align}
Here, $\temp$ is a hyperparameter of the decoding scheme that modifies the distribution used to sample the next token. 
When $\temp=1$, we recover the base distribution of $\model$.  
As $\temp \to 0$, the distribution sharpens toward a point mass on the maximum-logit token, corresponding to deterministic greedy decoding (discussed below).  
As $\temp \to \infty$, the distribution approaches uniform, corresponding to random sampling where each token is equally likely.

\paragraph{Top-$k$ filtering.}
Let $\vocab_k \subseteq \vocab$ be the set of the $k>0$ tokens with largest logits at step $t$. 
\newterm{Top-$k$} filtering~\citep{fan-etal-2018-hierarchical} masks out all other tokens by setting their logits to $-\infty$, which is equivalent to assigning them probability $0$ after softmax (Equation~\ref{eq:softmax}). 
The remaining $k$ logits are normalized to form a distribution:
\begin{align}
\label{eq:topk}
\Pr_k(\tokb_t = v \mid \seqb_{1:t-1};\model) \;=\;
\begin{cases}
\displaystyle \frac{\exp(\logit_t[v])}{\sum_{u \in \vocab_k} \exp(\logit_t[u])}, & v \in \vocab_k, \\[1.2ex]
0, & v \notin \vocab_k.
\end{cases}
\end{align}

Temperature scaling and top-$k$ filtering can be combined: 
first scale the logits by $\tfrac{1}{\temp}$, then mask all but the top-$k$ tokens by setting their logits to $-\infty$, and finally apply softmax to normalize the remaining logits into a probability distribution.
\newterm{Greedy decoding} is the special case $k=1$ (temperature is immaterial):
at each step $t$, the highest-probability (top-$1$) token is chosen deterministically.

\subsection{Verbatim extraction of training data}\label{app:sec:background:extraction}

Models are known to memorize portions, but not all, of their training data.
\newterm{Memorization} refers to the encoding of specific training sequences in a model's weights 
\citep{feldman2020mem, yeom2018privacy,hayes2025strongmia,cooper2023report}, 
such that the learned distribution assigns very high probability to them.
For generative models like LLMs, this can make \newterm{extraction} possible:
memorized sequences can sometimes be reproduced in outputs~\citep{carlini2021extracting, lee2022dedup}.

We denote an LLM $\model$'s \newterm{training dataset} by $\dataset$, which consists of training sequences of tokens. 
A given training sequence $\seq$ can be split into a \newterm{prefix} $\prefix$ and a \newterm{target suffix} $\suffix$, i.e., $\seq = \prefix \parallel \suffix$. 
We quantify \newterm{extraction} of such a training sequence using model $\model$ and decoding scheme $\dec$ as follows.
Given the prefix $\prefix$, we run $\suflen$ autoregressive steps where, at each step, $\model$ produces a distribution over the vocabulary $\vocab$, $\dec$ transforms that distribution, and one token is sampled and appended.
We denote the generated continuation by\looseness=-1
\begin{align*}
\gensuffix \;=\; \gen_{\model,\dec}(\prefix, \suflen).
\end{align*}
In the extraction literature, it is common to set $\prelen=\suflen=50$, i.e., use $100$-token sequences~\citep{carlini2023quantifying, hayes2025measuringmemorizationlanguagemodels, cooper2025books}.
Most prior work studies extraction success in terms of an exact match between the generated continuation $\gensuffix$ and the target suffix $\suffix$.

\subsubsection{Defining verbatim extraction success}\label{app:sec:background:extraction:verbatim-def}

We use a predicate $\suc_{\model,\dec}(\seq)\in\{0,1\}$ to indicate whether verbatim extraction of the target suffix succeeds under the chosen criterion (specified below).

\paragraph{Success with respect to greedy decoding.}
Greedy decoding, which we denote $\dec=\greedy$, is deterministic;
for a given prompt, model, and output length, it always results in the same output. 
As a result, in the greedy case, verbatim extraction success reduces to exact string equality between the generated continuation and the ground-truth target suffix from the training data, i.e.,
\begin{align}
\label{eq:greedy-success}
\suc_{\model,\greedy}(\seq) \;\triangleq\;
\1\big[\,\gensuffix \;=\; \suffix\,\big].
\end{align}
This is the most common approach for quantifying memorization in both research and technical reports for new model releases~\citep{carlini2023quantifying, lee2022dedup, reid2024gemini, gemma2, llama3, pythia}. 
The associated metric is called \newterm{discoverable extraction}~\citep{carlini2023quantifying}.\looseness=-1  

\paragraph{Success with respect to non-deterministic sampling.}
\citet{hayes2025measuringmemorizationlanguagemodels} provide an alternative way of measuring extraction, which accounts for the non-determinism of sampling from distributions produced by most decoding schemes. 
For their \newterm{probabilistic discoverable extraction} metric, one can pick a non-deterministic decoding scheme $\dec$ (e.g., base distribution $\temp=1$, or top-$k$ with $k=40$), where the conditional next-token distribution can be used to compute suffix probabilities $\pseq \in [0,1]$.

In this setting, \citet{cooper2025books} define the success predicate as follows.
For a non-deterministic decoding scheme $\dec$, the probability $\pseq$ of generating the exact $\suflen$-token target suffix given an $\prelen$-token prefix, under $(\model,\dec)$, is
\begin{align}
\label{eq:pz}
\pseq \;\triangleq\;  \Pr_{\model,\dec}\!\big(\suffix\,\big|\,\prefix\big)
\;=\; \prod_{t=\sufstart}^{\sufend} \Pr_{\model,\dec}\!\big(\tok_t \,\big|\, \seq_{1:t-1}\big)
\;=\; \exp\!\Bigg(
  \sum_{t=\sufstart}^{\sufend}
  \log \Pr_{\model,\dec}\!\big(\tok_t \,\big|\, \seq_{1:t-1}\big)
\Bigg).
\end{align}
This is the probability that the model $\model$ generates the exact target suffix $\suffix$, token by token under decoding scheme $\dec$---i.e., the probability that $\gensuffix$ is equal to $\suffix$. 
For a threshold $\taumin\in(0,1]$, verbatim probabilistic extraction success occurs when this probability is at least $\taumin$:\looseness=-1
\begin{align}
\label{eq:prob-success}
\suc_{\model,\dec}(\seq;\taumin) \;\triangleq\; \1\big[\,\pseq \;\geq\; \taumin\,\big].
\end{align}
\citet{cooper2025books} set and validate a very conservative $\taumin=0.001$ for top-$k$ decoding with $\temp=1$ and $k=40$. 
For convenience, we do the same here. 
\looseness=-1

\subsubsection{Observations about probabilistic extraction}\label{app:sec:background:extraction:probabilistic}

Three important observations follow from the above definitions for probabilistic discoverable extraction (which we refer to as \newterm{probabilistic extraction}, for short):

\begin{enumerate}[leftmargin=0.5cm]
\item Greedy-decoded extraction is a special case of probabilistic extraction:
under $\dec=\greedy$, the per-step distribution is a point mass on the $\argmax$ token, so $\pseq=1$ if and only if the generated continuation equals the target suffix (otherwise $\pseq=0$), recovering Equation~\ref{eq:greedy-success}.
\item For stochastic $\dec$, a single-run, one-shot equality indicator $\1[\gen_{\model,\dec}(\prefix,\suflen)=\suffix]$ is a Bernoulli random variable with mean $\pseq$.
\item In practice, it is unnecessary to actually generate any continuations to estimate $\pseq$; 
teacher-forced scoring (described below) computes $\pseq$ directly from the logits $\logit$ in a single forward pass~\citep{cooper2025books}. 
\end{enumerate}

\paragraph{Computing $\pseq$ without sampling (teacher-forced scoring).}
To evaluate $\pseq$ for verbatim extraction, we do not need to generate any tokens. 
Instead, we perform a single forward pass on the \emph{entire} sequence from the training data $\seq=\prefix\!\parallel\!\suffix$ to obtain next-token logit vectors $\logit_t$ for $t=\sufstart,\dots,\sufend$ (each conditioned on $\seq_{1:t-1}$, the prefix and any earlier tokens in the suffix).
We then apply the decoding scheme $\dec$ in \emph{logit space} (e.g., temperature rescaling and/or top-$k$ filtering), and obtain the probability of the ground-truth suffix token $\tok_t$ by applying the softmax:
\begin{align}
\label{eq:teacherforce}
\Pr_{\model,\dec}\!\big(\tok_t \mid \seq_{1:t-1}\big)
\;=\;
\softmaxsf\big(\dec(\logit_t)\big)[\tok_t],
\qquad t=\sufstart,\ldots,\sufend.
\end{align}
Finally,
\(
\pseq
=
\prod_{t=\sufstart}^{\sufend}
\Pr_{\model,\dec}\!\big(\tok_t \mid \seq_{1:t-1}\big),
\)
or equivalently the sum of their $\log$-probabilities (as in Equation~\ref{eq:pz}).  
This procedure, called \newterm{teacher-forced scoring}, computes all $\suflen$ probabilities in one batched pass that reuses the model's internal parallelism.
It avoids $\suflen$ sequential sampling steps from autoregressive generation and the associated overhead.
Even with KV caching, in practice, autoregressive generation is more expensive than teacher forcing (Appendix~\ref{app:sec:intuition:cost}).\looseness=-1

\paragraph{Benefits of probabilistic extraction over discoverable extraction.}
Prior work has shown that probabilistic extraction provides a more nuanced measure of memorization than greedy-decoded discoverable extraction~\citep{hayes2025measuringmemorizationlanguagemodels, cooper2025books}. 
In summary, probabilistic extraction:\looseness=-1
\begin{itemize}[leftmargin=0.5cm]
    \item \textbf{Is more realistic.} Non-deterministic sampling schemes are the norm in practice, so probabilistic extraction more accurately reflects how LLMs are actually used. 
    \item \textbf{Identifies more extractable sequences.} 
    Greedy decoding systematically under-counts valid extraction. 
    High-probability training sequences can be missed simply because greedy decoding always picks the locally highest-probability token at each step. 
    By exploring more of the model's distribution, non-deterministic decoding schemes can surface valid extraction that greedy decoding fails to detect. 
    This gap grows with model size~\citep{cooper2025books, hayes2025measuringmemorizationlanguagemodels}.
    \item \textbf{Quantifies extraction risk.} 
    The probability $\pseq$ is the expected frequency with which the model outputs the suffix when prompted with the prefix. 
    For example, $\pseq=0.5$ means the model will reproduce the suffix about once every two prompts, while $\pseq=0.001$---our chosen $\taumin$---means about once in a thousand prompts. 
    This scalar measure provides a direct notion of \newterm{extraction risk} that can be compared across sequences. 
    Some sequences are more extractable (higher $\pseq$) and therefore more vulnerable. 
    In contrast, since greedy-decoded extraction is deterministic, it collapses this information into a single bit---was the suffix extractable or not?
    With probabilistic extraction, it is possible to analyze how extraction risk varies across different sequences~\citep{cooper2025books}. 
    This makes it especially useful for analyzing how different types of data are at risk of leakage at generation time. 
\end{itemize}

\subsubsection{Computing an extraction rate}\label{app:sec:background:extraction:rate}
Now, let $\seqset$ be a set of sequences $\seq$ drawn from some training dataset $\dataset$ (i.e., $\seqset \subseteq \dataset$).
We can compute an \newterm{extraction rate} over $\seqset$ with 
\begin{align}
\label{eq:extraction-rate}
\mathsf{extraction\_rate}(\seqset; \suc) \;\triangleq\; \frac{1}{|\seqset|}\sum_{\seq\in\seqset} \1[\,\suc(\seq)\,],
\end{align}
where $\suc(\seq)$ is a success predicate defined according to the chosen criterion.
In this work, for verbatim extraction, this predicate can either be the greedy success condition in Equation~\ref{eq:greedy-success} or the more general probabilistic success condition in Equation~\ref{eq:prob-success}.
We will use extraction rates as one of the main metrics in this work.
(We will also visualize distributions over extraction probabilities for a set $\seqset$, in order to convey how risk varies across sequences.)
Equation~\ref{eq:extraction-rate} is general enough to also apply to near-verbatim extraction metrics.\looseness=-1

\vspace{-.1cm}
\section{Problem setup for near-verbatim extraction}\label{app:sec:nv}
\vspace{-.1cm}

The extraction metrics discussed in Appendix~\ref{app:sec:background} all register success only when the generated continuation is an \emph{exact} match to the target suffix.
Even if one token is off---an additional space is inserted, a deleted comma---both success criteria would return $0$. 
Because of this dependence on strict equality with the target suffix, both metrics underestimate memorization~\citep{lee2022dedup,ippolito-etal-2023-preventing}.

To address this, it is possible to extend both greedy-decoded discoverable extraction and probabilistic extraction (Appendix~\ref{app:sec:background:extraction}) to capture instances of near-exact (\newterm{near-verbatim}) memorization. 
In our context of comparing generated suffixes to target suffixes of the same length $\suflen$, we choose a distance metric $\dist: \vocab^\suflen \times \vocab^\suflen \to \R_{\geq 0}$ and a tolerance $\varepsilon \in \R_{\geq0}$. 
Using this metric, if the distance between the generated continuation and the target suffix is at most $\varepsilon$, we count the sequence as extracted. 

In this appendix, we describe the two distance metrics that we consider in this paper (Appendix~\ref{app:sec:nv:distance}), the success criteria for near-verbatim extraction (Appendix~\ref{app:sec:nv:success}), and considerations for computing near-verbatim extraction in practice (Appendix~\ref{app:sec:nv:compute}). 

\subsection{Distance metrics}\label{app:sec:nv:distance}

We consider two standard token-sequence distance measures. 
For $\seqb,\seqc \in \vocab^\suflen$,
\begin{align}
\label{eq:hamming}
\ham(\seqb,\seqc) &\triangleq \sum_{t=1}^\suflen \1[\,\tokb_t \neq \tokc_t\,],
\end{align}
\begin{align}
\label{eq:levenshtein}
\lev(\seqb,\seqc) &\triangleq \min\{\,d : \seqb \xrightarrow{\;\text{$d$ edits}\;} \seqc \,\}.
\end{align}
The Hamming distance counts positional mismatches: 
it computes the distance with respect to token substitutions, where each substitution has unit cost.
The Levenshtein distance is more general: 
it allows substitutions, insertions, and deletions, each with unit cost, and is the minimum number of such edits required to transform one sequence into the other. 
Normalized versions of these metrics are also often useful, e.g., 
\begin{align*}
\normlev(\seqb,\seqc) \triangleq \frac{\lev(\seqb,\seqc)}{\suflen} \in [0,1].
\end{align*}
Since we always consider comparisons for $\suflen$-length sequences, we use unnormalized distance metrics.
We will also often abbreviate $\ham$ as $\hamshort$ and $\lev$ as $\levshort$.\looseness=-1 

\paragraph{Comparative remark.}
For equal-length token sequences, the Hamming and Levenshtein distances need not coincide. 
The Hamming distance counts the number of mismatched positions, while the Levenshtein may use insertions and deletions to find a shorter edit path. 
For example, suppose $\seqb$ is obtained from $\seqc$ by inserting a token at the beginning and deleting one at the end. 
The Levenshtein assigns distance $2$, while the Hamming registers $\suflen$ mismatches, making the two sequences appear more dissimilar than they actually are. 
We will revisit this below in Appendix~\ref{app:sec:nv:compute:challenges}.

\paragraph{Character-based variants.}
Although we compute distances over tokens, the same definitions apply to characters. 
In that case, decoded token sequences may yield different character lengths, but the Levenshtein distance remains well-defined.  
(Hamming requires the two strings being compared to have equal length, so may not apply to character-based distances.)
Since our analysis focuses on tokens, we omit explicit character-based definitions.

\subsection{Success criteria for near-verbatim extraction}\label{app:sec:nv:success}

For both discoverable extraction and probabilistic extraction, we can generalize the success criteria to accommodate near-verbatim memorization.

\paragraph{Discoverable extraction.}
For the greedy-decoded continuation $\gensuffix$ and target suffix $\suffix$, 
we define success under a chosen distance metric $\dist$ and tolerance $\varepsilon \in \R_{\geq0}$ as 
\begin{align}
\label{eq:success-greedy-nv}
    \suc^\dist_{\model,\greedy,\varepsilon}(\seq) \;\triangleq\;
    \1\Big[\, \dist\big(\gensuffix, \suffix\big) \le \varepsilon \,\Big].
\end{align}
When $\varepsilon=0$, this reduces to the verbatim extraction success criterion in Equation~\ref{eq:greedy-success}.

\paragraph{Probabilistic extraction.}
For verbatim extraction, the success criterion depends on the probability of the exact target $\pseq$. 
Now, with tolerance $\varepsilon$, multiple suffixes may satisfy the success criterion---any of which may reasonably be sampled with a non-deterministic decoding scheme $\dec$.
We therefore need to define extraction probability as the aggregate mass over the set of all such $\varepsilon$-viable suffixes.\looseness=-1 

To do so, for target suffix $\suffix$, we define the \newterm{$\varepsilon$-ball of near-verbatim suffixes} for distance metric $\dist$ as
\begin{align}
\label{eq:eball}
\ball\big(\suffix\big)
    \;\triangleq\; \big\{\,\vv \in \vocab^\suflen: \dist\big(\vv,\suffix\big) \le \varepsilon \,\big\}.
\end{align}
For an LLM $\model$ and decoding scheme $\dec$, the total probability mass on this set is 
\begin{equation}
\label{eq:pze}
\begin{aligned}
\pseqe
&\triangleq \Pr_{\model,\dec}\!\big(\vv \in \ball(\suffix) \,\big|\, \prefix\big) \\
&= \expect{{\vv \sim \Pr_{\model,\dec}(\cdot \mid \prefix)}}{
   \,\1\left\{\dist(\vv,\suffix)\le \varepsilon\right\}} \\
&= \sum_{\vv \in \ball(\suffix)} \Pr_{\model,\dec}\big(\vv \mid \prefix\big).
\end{aligned}
\end{equation}
The near-verbatim success criterion is then 
\begin{align}
\label{eq:prob-success-nv}
\suc^\dist_{\model,\dec,\varepsilon}(\seq;\taumin)
\;\triangleq\;
\1\big[\, \pseqe \;\geq\; \taumin\,\big].
\end{align}
When $\varepsilon=0$, $\ballset^\dist_0(\suffix)$ contains only the verbatim suffix, so $p_{\seq,0}^\dist=\pseq$, and Equation~\ref{eq:prob-success-nv} reduces to Equation~\ref{eq:prob-success}.

\subsection{Quantifying near-verbatim extraction success in practice}\label{app:sec:nv:compute}

Having modified the extraction success criteria to allow for near-verbatim matches to the target suffix, we next begin to discuss how one might go about instantiating these criteria in practice.\looseness=-1

\paragraph{Computing greedy-decoded, near-verbatim discoverable extraction is cheap.}
For greedy decoding, near-verbatim success is straightforward to compute in practice:
prompt with the prefix $\prefix$, generate the deterministic greedy continuation $\gensuffix$, compute the distance to the target suffix $\dist(\gensuffix,\suffix)$, and declare extraction success if that distance does not exceed $\varepsilon$.
This requires no additional forward passes through the model beyond those needed to generate the greedy continuation.\footnote{This applies to the setup for computing discoverable extraction that takes a known sequence from the training data, splits it into a prefix and target suffix, and evaluates the generated continuation against the target suffix.
    Other approaches compare the generated continuation against an entire corpus of training data---not just the target suffix.
    For efficiency, such work often uses a suffix array to search for the verbatim generated continuation among the training data 
    \citep{lee2022dedup, nasr2023scalable, nasr2025scalable}.  
    This data structure enables fast exact substring searches. 
}
Because greedy decoding is deterministic, there is only one suffix to evaluate.

In contrast, computing near-verbatim probabilistic extraction exhibits other challenges. 

\subsubsection{Challenges for computing near-verbatim probabilistic extraction}\label{app:sec:nv:compute:challenges}

\paragraph{The set of near-verbatim suffixes can be enormous.}
While, in principle, Equation~\ref{eq:prob-success-nv} only involves a slight modification of the verbatim probabilistic extraction success criterion, in practice, it is significantly more expensive to compute. 

For the Hamming distance (Equation~\ref{eq:hamming}), the number of $\suflen$-length suffixes within radius $\varepsilon$ around a $\suflen$-length target suffix over vocabulary $\vocab$ is
\begin{align}
\label{eq:hamming-ball}
\big|\ballset^{\ham}_\varepsilon(\suffix)\big|
  \;=\; \sum_{r=0}^{\varepsilon} \binom{\suflen}{r}\,(|\vocab|-1)^{r}.
\end{align}
Here $\binom{\suflen}{r}$ chooses the $r$ substitution positions, and
$(|\vocab|-1)^r$ counts all possible substitutions at those positions.
For instance, consider $|\vocab|=32{,}000$ (as is the case for \textsc{Llama 1} models) and $\suflen=50$:
\[
\begin{aligned}
\varepsilon &= 0 &:& \; 1 && (\text{single verbatim match}),\\
\varepsilon &\leq 1 &:& \; 1 + \binom{50}{1}(31{,}999) &&= 1{,}599{,}951,\\
\varepsilon &\leq 2 &:& \; 1 + \binom{50}{1}(31{,}999) + \binom{50}{2}(31{,}999)^2 &&\approx 1.254\times 10^{12}.
\end{aligned}
\]
So, even when just setting $\varepsilon=2$, the number of near-verbatim suffixes is already in
the trillions.

For the Levenshtein distance (Equation~\ref{eq:levenshtein}), as noted in Appendix~\ref{app:sec:nv:distance}, the computed distance can be smaller than the Hamming distance for the same pair of $\suflen$-length sequences, since an insertion-deletion pair may substitute for multiple mismatches. 
Put differently, the Levenshtein distance can effectively ``shift'' a sequence by pairing a deletion with an insertion to realign tokens, whereas the Hamming distance can only compare fixed positions;
as a result, the Levenshtein can need fewer edits to transform one sequence to another.
Therefore, for the same $\varepsilon$, the Levenshtein $\varepsilon$-ball is always at least as large as the Hamming $\varepsilon$-ball, and can be strictly larger (Theorem~\ref{thm:epsballs}).\looseness=-1 

\paragraph{No single-forward-pass, teacher-forced analogue.}
In the verbatim case, the probability $\pseq$ can be computed with a single teacher-forced forward pass through the model (Equation~\ref{eq:teacherforce}).
By contrast, recall that computing $\pseqe$ for near-verbatim extraction requires
\[
  \pseqe \;\triangleq\; \sum_{\vv \in \ballset^{\dist}_\varepsilon(\suffix)}
  \Pr_{\model,\dec}(\vv \mid \prefix).
  \tag{\ref{eq:pze}}
\]
If we were to use the same measurement approach, this would involve one teacher-forced evaluation for each $\vv$ in the
$\varepsilon$-ball.
Given the size of the Hamming ball (Equation~\ref{eq:hamming-ball})---and the
potentially larger Levenshtein ball (Theorem~\ref{thm:epsballs})---this is
prohibitively expensive even for very small $\varepsilon$.
(We discuss cost in more detail in Appendix~\ref{app:sec:intuition:cost}, in the context of providing an intuition for other ways of computing near-verbatim probabilistic extraction.)

\subsubsection{Levenshtein never exceeds Hamming for $\suflen$-length sequences}\label{app:sec:nv:compute:ham-dominate-lev}

We show that, for sequences of the same length, the Levenshtein distance (Equation~\ref{eq:levenshtein}) never exceeds the Hamming distance (Equation~\ref{eq:hamming}). 
As a result, for the same $\varepsilon$ and $\suflen$-length target suffix, the $\varepsilon$-ball of the Hamming distance is a subset of the $\varepsilon$-ball of the Levenshtein distance. 
This follows intuitively by definition:
both metrics allow substitution edits, but the Levenshtein distance may find a cheaper edit path by additionally allowing for insertions and deletions.

\begin{theorem}
\label{thm:epsballs}
Fix a token vocabulary $\vocab$ and $\suflen\in\sN$.
For all $\seqb,\seqc\in \vocab^\suflen$, 
\begin{align}
\label{eq:ham-lev-ineq}
\lev(\seqb,\seqc) \;\leq\; \ham(\seqb,\seqc).
\end{align}
Further: 
\begin{enumerate}[leftmargin=0.5cm]
  \item \textbf{(Inclusion)} For every $\seqb \in \vocab^\suflen$ and $\varepsilon \ge 0$,
        $\ballset^\hamshort_\varepsilon(\seqb) \subseteq \ballset^\levshort_\varepsilon(\seqb)$. 

  \item \textbf{(Equality for $\varepsilon\in\{0,1\}$)} 
        If $\varepsilon\in\{0,1\}$, then, for every $\seqb\in \vocab^\suflen$, 
        $\ballset^\hamshort_\varepsilon(\seqb)=\ballset^\levshort_\varepsilon(\seqb)$.
  
  \item \textbf{(Strictness for $\varepsilon \in \{2,\dots,\suflen-1\}$; existence)} 
        If $\suflen \geq 3$ and $|\vocab| \geq \suflen$, then there exists $\seqb\in\vocab^\suflen$
        such that for every $\varepsilon$ with $2 \le \varepsilon < \suflen$, 
        \[
          \ballset^\hamshort_\varepsilon(\seqb) \;\subsetneq\; \ballset^\levshort_\varepsilon(\seqb).
        \]

  \item \textbf{(Saturation)} For every $\seqb\in\vocab^\suflen$, every $\varepsilon\ge \suflen$, and $\dist\in\{\ham,\lev\}$, 
        $\ballset^{\dist}_\varepsilon(\seqb) = \vocab^\suflen$. 
\end{enumerate}
\end{theorem}

\begin{proof}
First, note that if $\seqb$ and $\seqc$ differ at $r=\ham(\seqb,\seqc)$ positions, then performing those $r$ position-wise substitutions is a valid edit script from $\seqb$ to $\seqc$ of cost $r$ (Equation~\ref{eq:hamming}).
By the definition of the Levenshtein distance (Equation~\ref{eq:levenshtein}), $\lev(\seqb,\seqc) \leq r$.\looseness=-1
We can now take each case in turn.
\begin{enumerate}[leftmargin=0.5cm]
    \item \textbf{(Inclusion)} By Equation~\ref{eq:ham-lev-ineq},  
    if $\seqc\in\ballset^\hamshort_\varepsilon(\seqb)$, then
    $\lev(\seqb,\seqc) \leq \ham(\seqb,\seqc) \leq \varepsilon$. 
    And so $\seqc\in\ballset^\levshort_\varepsilon(\seqb)$---i.e., every member of $\ballset^\hamshort_\varepsilon(\seqb)$ must also be a member of $\ballset^\levshort_\varepsilon(\seqb)$. 

    \item \textbf{(Equality for $\varepsilon\in\{0,1\}$)} For $\varepsilon=0$ both $\varepsilon$-balls are $\{\seqb\}$.  
    
    For $\varepsilon=1$, this corresponds to $\lev(\seqb,\seqc) \leq 1$ with $\seqb,\seqc\in\vocab^\suflen$. 
    In this case, the single edit cannot be an insertion or deletion, as this would change the length to $\suflen+1$ or $\suflen-1$, respectively. 
    So, the only possibilities are either (a) no edit (i.e., $\seqb=\seqc$) or (b) a single substitution. 
    Therefore, it is also the case that $\ham(\seqb,\seqc) \leq 1$. 
    Combining with (1), this means the two $\varepsilon$-balls are equal. 

    \item \textbf{(Strictness for $\varepsilon \in \{2, \ldots, \suflen-1\}$; existence)} 
    The reason $\ballset^\hamshort_\varepsilon(\seqb)$ can be strictly contained in
    $\ballset^\levshort_\varepsilon(\seqb)$ is that Levenshtein distance allows an
    insertion-deletion pair to ``shift'' tokens and realign two sequences.
    For example, if $\suflen\ge 3$ and the vocabulary $\vocab$ is large enough to
    choose pairwise distinct tokens $v_1,\dots,v_\suflen$, consider 
        \[
          \seqb \coloneqq (v_1, v_2, \ldots, v_\suflen),\qquad
          \seqc \coloneqq (v_2, v_3, \ldots, v_\suflen, v_1).
        \]
    Then $\ham(\seqb,\seqc)=\suflen$, because every position mismatches. 
    But $\seqc$ can be obtained from $\seqb$ by deleting $v_1$ and then inserting it at the end, i.e., with two edits
    \[
      \seqb \xrightarrow{\text{delete }v_1} (v_2,\ldots,v_\suflen)
      \xrightarrow{\text{insert }v_1\text{ at end}} \seqc,
    \]
    so $\lev(\seqb,\seqc)\le 2$. 
    Since $\seqb \ne \seqc$ (eliminating $\varepsilon=0$) and a single edit cannot map one $\suflen$-token sequence to another unless it is a substitution, which here can fix at most one of the $\suflen \geq 3$ mismatched positions (eliminating $\varepsilon=1$), we have $\lev(\seqb,\seqc)=2$.
    Thus, for any $\varepsilon \in \{2, \ldots, \suflen-1\}$, $\seqc$ lies in
    $\ballset^\levshort_\varepsilon(\seqb)$ but not in $\ballset^\hamshort_\varepsilon(\seqb)$, proving strict inclusion for this $\seqb$.

    \item \textbf{(Saturation)} For any $\seqc \in \vocab^\suflen$, by definition, $\ham(\seqb,\seqc)\leq\suflen$, since there are at most $\suflen$ mismatched positions. 
    And so, by Equation~\ref{eq:ham-lev-ineq}, $\lev(\seqb,\seqc) \leq \ham(\seqb,\seqc)\leq\suflen$.
    Therefore, once $\varepsilon \geq \suflen$, every $\seqc$ is included; 
    the $\varepsilon$-ball stabilizes to be every sequence in $\vocab^\suflen$. 
\end{enumerate}
\end{proof}

\vspace{-.1cm}
\section{An intuition for more efficient near-verbatim probabilistic extraction}\label{app:sec:intuition}
\vspace{-.1cm}

Even though the set of near-verbatim suffixes can be enormous (Appendix~\ref{app:sec:nv:compute:challenges}), many sequences in that set will be very unlikely (if not $0$-probability) under the model $\model$ and decoding scheme $\dec$.
Imagine taking a particular piece of a famous, memorized text and replacing the token for \texttt{\_the} with the token for \texttt{\_jazz}, \texttt{\_eight}, or \texttt{\_github}, using the \textsc{Llama 1} tokenizer;
those would all be within $\varepsilon\leq{1}$, but are very unlikely to be sequences with any meaningful probability. 
The point is, we should be able to derive a significantly more efficient approach to estimating near-verbatim probabilistic extraction---one that does not enumerate the entire $\ball$, but only sequences with non-zero or non-trivial mass.

In this appendix, we first describe a simple Monte Carlo baseline that samples continuations from $\model$ under $\dec$ (Appendix~\ref{app:sec:intuition:mc}).
While this is an intuitive approach, it turns out to be prohibitively expensive for producing useful and reliable estimates of $\pseqe$.
We then describe the intuition behind the approach that we take in this paper---a modification of beam search that produces a deterministic lower bound on $\pseqe$ at a fraction of the cost (Appendix~\ref{app:sec:intuition:beam}).
We provide quantitative intuition for why this approach succeeds in practice, and compare it with MC sampling.
Finally, we define a common unit of cost---token evaluations---that allows direct comparison across methods (Appendix~\ref{app:sec:intuition:cost}).
This shows why our beam-based approach is significantly cheaper than MC in terms of token evaluations, and enables us to compare with greedy discoverable extraction.
We do not go into formal details about our beam-search-based algorithm here, and instead defer this discussion to Appendices~\ref{app:sec:kcbs} and~\ref{app:sec:prune}.\looseness=-1 

\subsection{Monte Carlo sampling}\label{app:sec:intuition:mc}

Instead of enumerating the entire set $\ball(\suffix)$ and computing the conditional probability for each continuation given the prefix, a straightforward alternative is to estimate $\pseqe$ via sampling.
Denote a training-data sequence $\seq$, composed of the $\prelen$-length prefix $\prefix\equiv\pre$ and the $\suflen$-length target suffix $\suffix\equiv\suf$---i.e., $\seq\coloneqq\pre\Vert\suf$.
Denote a generated $\suflen$-length continuation $\gensuffix\equiv\gensuf$.
Given prefix $\pre$, sample $M$ i.i.d.\ $\suflen$-length continuations
$\{\gensuf^{(1)},\dots,\gensuf^{(M)}\} \sim \Pr_{\model,\dec}(\cdot\mid\pre)$.
Define the empirical, sampling-based estimate for $\pseqe$ as
\begin{align}
\label{eq:hat-pze}
\hatpseqe
   \;=\; \frac{1}{M} \sum_{w=1}^M
      \1\big[\,\dist(\gensuf^{(w)},\,\suf) \le \varepsilon\,\big], \quad \text{where } \gensuf^{(w)}\underset{\text{i.i.d.}}{\sim}\Pr_{\model,\dec}(\cdot \mid \pre).
\end{align}
That is, we prompt the model $M$ times with the prefix $\pre$ and estimate $\pseqe$ as the proportion of sampled continuations (according to decoding scheme $\dec$) lying within the $\varepsilon$-ball of the target suffix.
This estimator is unbiased---$\E[\hatpseqe]=\pseqe$---and straightforward to compute.
However, there are two important reasons why it is expensive for estimating near-verbatim extraction probabilities.

\paragraph{1) There is no guarantee MC will sample from a high-probability $\varepsilon$-ball.}
Each of the $M$ i.i.d.\ samples independently either hits or misses the $\varepsilon$-ball.
The probability of \emph{never} hitting it in $M$ samples is
\begin{align}
\label{eq:mc-miss}
\Pr\big(\textnormal{miss }\ball(\suf)\big) = (1-\pseqe)^M.
\end{align}
To guarantee a miss probability of at most $\delta$, we require
\[
(1 - \pseqe)^M \le \delta
\;\Longleftrightarrow\;
M \ge \frac{\ln(1/\delta)}{-\ln(1-\pseqe)}.
\]
Table~\ref{app:tab:mc-hit} shows the minimum number of samples required for several values of $\pseqe$ and $\delta$, computed using this formula.

\begin{table}[h]
\centering
\caption{\textbf{Samples required to hit at least once.}
Minimum number of MC samples $M$ to hit the $\varepsilon$-ball at least once with probability $\ge 1-\delta$.}
\label{app:tab:mc-hit}
\begin{tabular}{lrrr}
\toprule
$\pseqe$ & \multicolumn{3}{c}{$\ceil{M}$} \\
\cmidrule(lr){2-4}
 & $\delta = 0.005$
 & $\delta = 0.05$ & $\delta = 0.5$ \\
\midrule
 $10^{-1}$ & $51$ & $29$ & $7$  \\
 $10^{-2}$ & $528$ & $299$ & $69$ \\
 $10^{-3}$ & $5{,}296$ & $2{,}995$ & $693$ \\
\bottomrule
\end{tabular}
\end{table}

Even for modest miss tolerance (e.g., $\delta=0.05$), merely \emph{detecting} an $\varepsilon$-ball with mass $\pseqe=10^{-3}$ requires on the order of $3{,}000$ samples.
Note, however, that detection only guarantees \emph{one} hit; it does not guarantee that the MC estimate of $\pseqe$ is accurate.

\paragraph{2) Even when MC hits the $\varepsilon$-ball, the estimate $\hatpseqe$ can be unreliable.}
Each term in Equation~\ref{eq:hat-pze} is a Bernoulli random variable with success probability $\pseqe$.
Because $\hatpseqe$ is the mean of $M$ such i.i.d.\ variables, its variance and standard error are
\[
\Var[\hatpseqe]
= \frac{\pseqe(1-\pseqe)}{M},
\qquad
\standarderror[\hatpseqe]
   \;=\; \sqrt{\frac{\pseqe(1-\pseqe)}{M}}.
\]
To achieve a relative standard error of $\eta$ (i.e., $\standarderror[\hatpseqe] \le \eta \cdot \pseqe$), we need
\[
\sqrt{\frac{\pseqe(1-\pseqe)}{M}} \;\le\; \eta\cdot\pseqe
\quad\Longleftrightarrow\quad
M \;\ge\; \frac{1-\pseqe}{\eta^2\,\pseqe}
\;\approx\; \frac{1}{\eta^2\,\pseqe} \quad \text{when $\pseqe$ is small}.
\]
For the minimum extraction threshold $\pseqe \approx \taumin=10^{-3}$ and $\eta=10\%$, this gives $M \gtrsim 10^{5}$; for $\eta=1\%$, $M \gtrsim 10^{7}$.
In other words, \emph{reliable} estimation of small $\pseqe$ is orders of magnitude more expensive than merely detecting the $\varepsilon$-ball.

\begin{figure}[t!]
    \centering                                    
    \includegraphics[width=.8\linewidth]{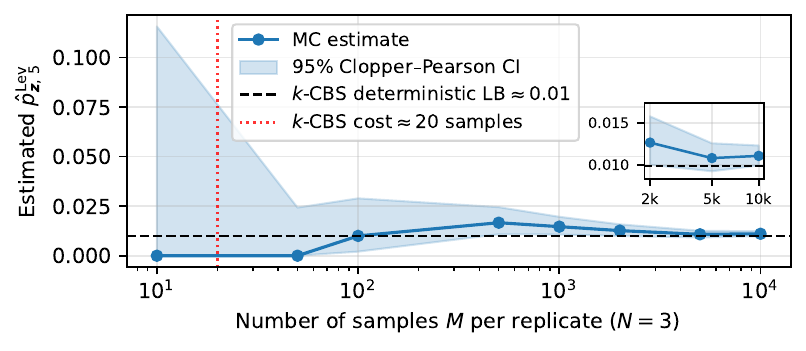}
    \caption{\textbf{Monte Carlo (MC) estimation of near-verbatim extraction probability.}
    For Levenshtein distance $\leq\!5$ ($p^{\mathsf{Lev}}_{\vz,\, 5}$), we plot convergence for a single sequence $\seq$ from \emph{The Great Gatsby} for \textsc{Llama~2 7B}, showing the pooled MC estimate with a 95\% confidence interval over 3 replicates.
    Our algorithm ($k$-CBS) produces a deterministic, provably correct lower bound (LB) of $\approx\!0.01$, which captures $89.4\%$ of the mean MC estimate at $M=10^4$ samples, at a cost of $\approx\!20$ MC samples (see Appendix~\ref{app:sec:intuition:cost} for cost accounting)---a budget at which MC produces no hits.\looseness=-1}
    \label{app:fig:mc-convergence}
    \vspace{-.2cm}
\end{figure}

\paragraph{Empirical illustration.}
We illustrate this cost empirically in Figure~\ref{app:fig:mc-convergence}, which shows MC convergence for a single sequence from \emph{The Great Gatsby} under \textsc{Llama~2~7B} with top-$k$ ($k=40$) decoding.
At $30$ pooled samples (across $N=3$ replicates at $10^1$), the 95\% confidence interval for $\pseqe$ spans from $0$ to over $0.10$---the estimate is consistent with both ``no extraction'' and ``10\% extraction probability.''
The estimate only stabilizes near the true value after thousands of samples.
By contrast, our approach produces a deterministic lower bound on $\pseqe$ at a cost equivalent to roughly $20$ MC samples.

While MC sampling is unbiased and simple, its cost scales as $1/\pseqe$ for detection and roughly $1/(\eta^2\,\pseqe)$ for reliable estimation---prohibitive at scale even when $\pseqe$ is moderately large (e.g., ${\approx}\,0.01$; see Figure~\ref{app:fig:mc-convergence}).
The fundamental limitation is that MC explores the sequence space blindly: most samples land far from the $\varepsilon$-ball and contribute no useful information.
This motivates an alternative approach that concentrates computation on the high-mass region of the sequence space, yielding a deterministic lower bound on $\pseqe$ at a fraction of the cost.\looseness=-1

\subsection{An intuition for using a decoding-constrained beam search approach}\label{app:sec:intuition:beam}

In this work, we propose a method for obtaining a useful, efficient lower bound on near-verbatim extraction probability.
With far fewer evaluations than MC sampling would require, we use a variant of \newterm{beam search} to produce a downward-biased (but still correct) lower bound on $\pseqe$.
We provide brief background on beam search (Appendix~\ref{app:sec:intuition:beam:background}), describe the modifications that turn it into an efficient procedure for lower-bounding $\pseqe$ (Appendix~\ref{app:sec:intuition:beam:summary}), and then provide mathematical intuition for why this approach succeeds in practice (Appendix~\ref{app:sec:intuition:beam:math}).

\subsubsection{Beam search}\label{app:sec:intuition:beam:background}

Beam search~\citep{Lowerre1976HARPY} is a standard decoding algorithm for autoregressive language models that approximately maximizes the conditional probability of a continuation given a prefix~\citep{graves2012sequencetransductionrecurrentneural, BoulangerLewandowski2013AudioCR,sutskever2014sequencesequencelearningneural, bahdanau2016neuralmachinetranslationjointly}.
Denote the length-$(t-1)$ prefix of $\gensuf$ by
$\genseq^{\textnormal{(cont)}}_{<t} \coloneqq (\gentok^{\textnormal{(cont)}}_1,\dots,\gentok^{\textnormal{(cont)}}_{t-1})$.
Then, for $\pre$ and $\gensuf$, the model $\model$ defines
\[
\Pr_{\model}(\gensuf \mid \pre)
\;=\;
\prod_{t=1}^{\suflen}
\Pr_{\model}\!\big(\gentok^{\textnormal{(cont)}}_t \,\big\vert\, \pre \Vert \genseq^{\textnormal{(cont)}}_{<t}\big),
\]
or, in $\log$ space,
\[
\log \Pr_{\model}(\gensuf \mid \pre)
\;=\;
\sum_{t=1}^{\suflen}
\log \Pr_{\model}\!\big(\gentok^{\textnormal{(cont)}}_t \,\big\vert\, \pre \Vert\genseq^{\textnormal{(cont)}}_{<t}\big).
\]
Beam search maintains, at each depth $t$, a \newterm{beam} of at most $\bw$ partial continuations, each with its accumulated $\log$-probability.
At $t=0$ the beam contains only the prefix $\pre$ with score $0$.
At step $t$:
\begin{enumerate}[leftmargin=.5cm]
  \item For each partial history $\genseq$ in the beam ($\pre \Vert\genseq^{\textnormal{(cont)}}_{<t}$, the prefix concatenated with the generated continuation so far), the model produces next-token probabilities
  $\Pr_{\model}(\gentok \mid \genseq)$ over $\gentok\in\vocab$.

  \item Each beam element is expanded by each of these tokens, yielding candidate children $\genseq' = \genseq \Vert \gentok$ with updated scores (the conditional $\log$ probability of the generated continuation so far, given the prefix): 
  \[
  \log p(\genseq') \;=\; \log p(\genseq) + \log \Pr_{\model}(\gentok \mid \genseq).
  \]
  \item For the next iteration ($t+1$), beam search performs an across-beam prune to keep only (at most) the $\bw$ unique highest-scoring partial sequences and discards the rest.
  This is done for efficiency, i.e., to prevent explosive blow-up of the number of sequences under consideration.
\end{enumerate}

After $\suflen$ steps, the algorithm returns the $\bw$ highest-scoring complete continuations in the beam as approximate maximum-probability sequences under the model.

\paragraph{Why beam search favors high-probability sequences.}
The joint $\log$-probability of a continuation decomposes additively across time steps.
Therefore, a high-probability continuation must maintain a high cumulative $\log$-probability at \emph{every} depth in order to remain competitive and stay in the beam.
If at some step $t$ a partial continuation falls far behind many competing sequences, any extension of it will remain disadvantaged when ranked. 

Beam search exploits this structure by pruning low-scoring partial continuations as soon as they fall outside the current top-$\bw$.
Partial continuations that consistently receive high next-token probabilities remain in the top-$\bw$ and are expanded further;
those that accumulate several low-probability token choices are quickly outrun by alternatives and permanently removed via the across-beam prune.

In this sense, beam search concentrates computational effort on a thin, high-mass region of candidate continuations rather than exploring the full, exponentially large space of possible sequences.
Our approach leverages this property: we impose additional constraints to turn beam search into an efficient procedure for lower-bounding the near-verbatim extraction probability $\pseqe$.\looseness=-1

\subsubsection{Modifying beam search with our approach}\label{app:sec:intuition:beam:summary}

The general idea behind our modification is simple.
Rather than scoring candidate sequences under the full model distribution, we implement expansion and scoring to respect the decoding policy $\dec$ that we choose for computing $\pseqe$.
In doing so, beam search (with width $\bw$) will return $\bw$ candidate sequences that are viable under $\dec$, with their associated probabilities under the $\dec$-constrained and renormalized probability distribution.
(In fact, we return $\bw \cdot k$ such candidates by not performing the last across-beam prune at step $\suflen$.)

In particular, for top-$k$ decoding, we keep a beam of width $\bw$.
For a partial history $\genseq$ in the beam at step $t$, we denote
\[
\sS_t(\genseq)\coloneqq \TopK_{k}\big(\logit_t(\genseq)\big)
\]
to be the set of $k$ tokens in $\vocab$ with the largest logits in $\logit_t(\genseq)$ (renaming $\vocab_{t,k}$ from Equation~\ref{eq:topk} and Algorithm~\ref{alg:kcbs}).
At each step $t$, maintain a beam of at most $\bw$ partial continuations, each with its accumulated $\log$-probability.
At $t=0$ the beam contains only the prefix $\pre$ with score $0$.
At step $t$:
\begin{enumerate}[leftmargin=0.5cm]
    \item For each partial history $\genseq$ in the beam, the model $\model$ and top-$k$ decoding $\dec$ produce next-token probabilities
    $\Pr_{\model,\dec}(\gentok \mid \genseq)$ over only the top-$k$-token set, $\gentok\in\sS_t(\genseq)$.

    \item Each beam element is expanded by each of these $k$ tokens, yielding $\bw \cdot k$ candidate children $\genseq' = \genseq \Vert \gentok$ with updated scores
  \[
  \log p(\genseq') \;=\; \log p(\genseq) + \log \Pr_{\model,\dec}(\gentok \mid \genseq).
  \]

     \item For the next iteration ($t+1$), beam search performs an across-beam prune to keep only the $\bw$ unique highest-scoring partial sequences and discards the rest.
\end{enumerate}
After $\suflen$ steps, the algorithm returns all $\bw \cdot k$ complete continuations (i.e., without performing a final across-beam prune).
Similar to traditional beam search, these continuations approximate the maximum-probability sequences under the truncated distribution induced by the model and the top-$k$ decoding policy. 
We refer to this procedure as \newterm{top-$k$ constrained beam search}, or $k$-CBS for short, and we refine it to attempt to get tighter bounds on $\pseqe$ throughout this paper.
Note, also, that we could use other decoding policies, though we focus on top-$k$ in this work (Appendix~\ref{app:sec:other:nucleus}). 

\paragraph{Connection to memorization.}
Intuitively, $k$-CBS should recover near-verbatim (and possibly the exact verbatim) suffixes whenever the sequence is memorized.
When memorized, the training-data suffix $\suf$ and small perturbations of it in $\ball(\suf)$ typically carry unusually high conditional probability under $\Pr_{\model,\dec}(\cdot\mid\pre)$, so the partial continuations corresponding to those suffixes remain among the highest-scoring across depths.
As a result, for a comfortably sized beam width $\bw$, these near-verbatim paths are unlikely to be pruned;
they are likely to survive to the final beam as $\bw$ high-probability candidates.
We can then post-process the returned continuations by applying a distance filter $\dist(\gensuf,\suf)\le\varepsilon$.
That is, let $\sF^{(\leq\varepsilon)}$ denote the set of filtered continuations;
we can sum their masses,
\[
\sum_{\gensuf \in \sF^{(\le\varepsilon)}} \Pr_{\model,\dec}(\gensuf\mid\pre),
\]
yielding a deterministic, rigorous (but biased-downward) lower bound on $\pseqe$ (with respect to $\model$ and top-$k$ for $\dec$). 
We will miss (potentially many) near-verbatim suffixes with this approach, but if the sequence is memorized and therefore high probability, its (near-)verbatim variants should appear in the final beam.
(Again, we actually return all $\bw \cdot k$ candidates from the last iteration, rather than performing the final across-beam prune to width $\bw$, as these are all $\suflen$-length continuations under the decoding policy.) 

Moreover, rather than applying the distance filter only at the end, we can incorporate $\varepsilon$-viability filtering \emph{during} the search: at each depth, prune any partial continuation that can no longer yield a final suffix within distance $\varepsilon$ of $\suf$.
This further focuses the beam on near-verbatim paths, freeing beam capacity and potentially yielding tighter lower bounds (we develop these pruned variants in Appendix~\ref{app:sec:prune}).\looseness=-1

While it is typically a drawback that beam search does not always produce diverse outputs~\citep{li-etal-2016-diversity,vijayakumar2018diversebeamsearchdecoding}, in our setting this is a \emph{feature}:
we explicitly focus on a small set of high-probability continuations that lie close to the training suffix, rather than exploring a wide variety of semantically different but low-probability alternatives.\looseness=-1

\paragraph{Working assumptions (informal).}
Informally, $k$-CBS should produce a useful lower bound on $\pseqe$, given that two conditions hold (which are both realistic for a memorized sequence):
\begin{enumerate}[leftmargin=0.5cm]
    \item \textbf{Token rank.}
    At most steps $t$, the true suffix token $z_t$ lies in the model's top-$k$ set, and deviations are rare enough that the full continuation remains within $\ball(\suf)$.

    \item \textbf{Beam dominance.}
    The verbatim path and/or its near-verbatim variants maintain cumulative $\log$-probabilities that stay within the beam's top-$\bw$ at each depth, so they survive all across-beam prunes.
\end{enumerate}
We formalize these conditions quantitatively in Appendix~\ref{app:sec:intuition:beam:math} below, and show that they follow naturally from the structure of high-mass continuations without distributional assumptions.

\paragraph{Cost comparison to MC.}
Even though $k$-CBS provides a downward-biased lower bound on $\pseqe$, it tends to recover more probability mass than MC sampling at smaller compute budgets because it concentrates computation on the (greedily) highest-mass continuations.
We defer a detailed cost comparison to Appendix~\ref{app:sec:intuition:cost}.

\subsubsection{Mathematical intuition for good performance in practice}\label{app:sec:intuition:beam:math}

The algorithm (Appendices~\ref{app:sec:kcbs} \& \ref{app:sec:prune}) guarantees that the downward-biased lower bound is \emph{valid}---it never exceeds the true $\pseqe$. 
But a lower bound of $0$ is also valid.
So, why should we expect the lower bound to be \emph{useful}---to capture a substantial fraction of $\pseqe$?
In this appendix, we provide quantitative intuition.
The key observation is that any continuation with probability $\ge\taumin$ (our extraction threshold)---or even slightly smaller---must pick high-rank tokens at almost every step.
As a result, such continuations take up spots in the beam;
it is likely that our algorithm retains them. 

\paragraph{Setup.}
Consider a complete history $\genseq\coloneqq(\gentok_1, \ldots, \gentok_\suflen)$.
We write $\genseq\coloneqq \pre \Vert \gensuf$, where $\pre$ is the training-data prefix and $\gensuf$ is a generated continuation given $\pre$.
The probability of such a continuation decomposes as
\[
\Pr_{\model}(\gensuf \mid \pre) = \prod_{t=1}^{\suflen} \Pr_{\model}(\gentok^{\textnormal{(cont)}}_t \;\big|\; \pre\Vert\genseq^{\textnormal{(cont)}}_{<t}),
\]
where $\genseq^{\textnormal{(cont)}}_{<t}$ denotes the length-$(t-1)$ prefix of $\gensuf$.
Abbreviate 
\[
p_t \coloneqq \Pr_{\model}(\gentok^{\textnormal{(cont)}}_t \;\big|\; \pre\Vert\genseq^{\textnormal{(cont)}}_{<t})
\] 
for the conditional probability of the realized token at step $t$.

\begin{lemma}[Low-probability budget from a sequence-mass floor $\tau$]
\label{lem:low-prob-bucket}
Fix a minimum mass threshold $\tau\in(0,1]$.
Let $\gensuf\in\vocab^\suflen$ be a continuation with
$\Pr_{\model}(\gensuf \mid \pre) \geq \tau$.
Then for any $\alpha\in(0,1)$,
\[
\#\{\,t\le\suflen:\ p_t \le \alpha\,\}
\ \le\
\Big\lfloor \frac{\ln \tau}{\ln \alpha} \Big\rfloor.
\]
\end{lemma}

\begin{proof}
If exactly $c$ of the $p_t$ are $\le\alpha$ and on all other steps we use the bound $p_t\leq1$, then 
\[
\Pr_{\model}(\gensuf\mid\pre)=\prod_{t=1}^{\suflen} p_t \le \alpha^c.
\]
The assumption that $\Pr_{\model}(\gensuf\mid\pre)\ge\tau$ forces 
$\alpha^c \ge \tau$. 
Taking the $\ln$ (noting that $\ln\alpha<0$ for $\alpha\in(0,1)$),
\[
\ln \tau \;\le\; c \ln \alpha
\quad\Longrightarrow\quad
c \;\le\; \frac{\ln \tau}{\ln \alpha},
\]
and the integer bound follows by applying $\lfloor\cdot\rfloor$.
\end{proof}

We can translate this point to the rank-ordering of tokens at every step.

\begin{corollary}[Rank-bucket consequences]
\label{cor:rank-bucket}
At step $t$, let $r_t\in\{1,2,\dots,|\vocab|\}$ be the rank of the realized token ($1$ = highest), and let $p_t$ be its conditional probability.
Then for any integer $R\ge 2$ and any $\alpha\in(0,1)$:
\begin{enumerate}[leftmargin=.75cm]
\item[\textnormal{(a)}] \textbf{Outside-head budget.}
Since $r_t \ge R \Rightarrow p_t \le 1/R$ (by the pigeonhole principle on the probability simplex),
\[
\#\{\,t\le\suflen:\ r_t \ge R\,\}
\ \le\
\Big\lfloor \frac{\ln \tau}{\ln(1/R)} \Big\rfloor.
\]
\item[\textnormal{(b)}] \textbf{Head coverage via a probability floor.}
\[
\#\{\,t\le\suflen:\ p_t > \alpha\,\}
\ \ge\
\suflen\ -\ \Big\lfloor \frac{\ln \tau}{\ln \alpha} \Big\rfloor, 
\quad\text{and on those steps}\quad
r_t \le \Big\lfloor \tfrac{1}{\alpha} \Big\rfloor.
\]
\end{enumerate}
\end{corollary}

\begin{proof}
\textbf{(a) Outside-head budget.}
Fix a step $t$ and sort the step-$t$ probabilities in nonincreasing order as
$p^{(t)}_{(1)} \ge p^{(t)}_{(2)} \ge \ldots \ge p^{(t)}_{(|\sV|)}$.
Because the probabilities at step $t$ sum to $1$,
\[
\sum_{i=1}^{R} p^{(t)}_{(i)} \ \le\ 1
\quad\Longrightarrow\quad
R\cdot p^{(t)}_{(R)} \ \le\ 1
\quad\Longrightarrow\quad
p^{(t)}_{(R)} \ \le\ \frac{1}{R}.
\]
If the realized token has rank $r_t\ge R$, then by definition
$p_t \le p^{(t)}_{(R)} \le 1/R$.
Let $c \coloneqq \#\{t\le\suflen : r_t \ge R\}$, i.e., $c$ is the number of steps where the rank of the realized token meets or exceeds $R$.
On those $c$ steps, $p_t \le 1/R$.
(Apply Lemma~\ref{lem:low-prob-bucket} with $\alpha = 1/R$.)

\textbf{(b) Head coverage.}
This is the complement of Lemma~\ref{lem:low-prob-bucket}. 
That is, set 
\(
c \;=\; \#\{\,t\le\suflen:\ p_t \le \alpha\,\}
\ \le\
\Big\lfloor \frac{\ln \tau}{\ln \alpha} \Big\rfloor 
\).
Since there are exactly $\suflen$ steps, 
\[
\#\{t\le\suflen : p_t > \alpha\}
\ =\
\suflen - c
\ \ge\
\suflen - \Big\lfloor \tfrac{\ln\tau}{\ln\alpha} \Big\rfloor.
\]
On steps with $p_t > \alpha$, at most $\lfloor 1/\alpha \rfloor$ tokens can have probability $\ge \alpha$ (otherwise the total mass at that step would exceed $1$), and so $r_t \le \lfloor 1/\alpha \rfloor$.
\end{proof}

Note that we write the above results in terms of the base model distribution $\model$.
Nevertheless, these results hold for \emph{any} model $\model$ and scoring rule $\dec$ that induces a probability distribution---we deliberately avoid distributional assumptions so that the intuition applies broadly.
To translate these for a decoding rule $\dec$, the ranks of realized tokens are limited to a subset of $\{1,2,\ldots,|\vocab|\}$---i.e., a subset of the size of the vocabulary (e.g, $k$ for top-$k$). 

\paragraph{Instantiations for $\tau=\taumin=10^{-3}$ and $\suflen=50$.}

\begin{itemize}[leftmargin=0.5cm]
\item \textbf{Greedy steps} ($\alpha=0.5$).
At most $\lfloor \ln 10^{-3}/\ln 0.5 \rfloor = 9$ steps can have $p_t \le 0.5$.
On any step, at most $1$ token can have probability $>0.5$;
this is the case for the remaining $\ge 41$ steps---i.e., at each of these steps, it is necessarily the model's top-$1$ (greedy) token.

\item \textbf{Top-$10$ steps} ($\alpha=0.1$).
At most $3$ steps can have $p_t \le 0.1$; on the remaining $\ge 47$ steps, the realized token has rank $r_t \le 10$.

\item \textbf{Top-$k$ steps} ($R=40$).
By Corollary~\ref{cor:rank-bucket}(a), at most $\lfloor \ln 10^{-3}/\ln(1/40) \rfloor = 1$ step can have the realized token at rank $\ge 40$.
Under top-$k$ decoding with $k=40$, this means at most one step where the realized token is at the very bottom of the top-$k$ set; on all other steps it lies strictly above rank $40$.
\end{itemize}

These bounds are deliberately worst-case.
For instance, the $\alpha=0.1$ case allows a highly atypical situation in which ten different tokens all have probability exactly $0.1$ at many steps.
In realistic, extracted model outputs, the distribution is much peakier. 
Therefore, Lemma~\ref{lem:low-prob-bucket} and Corollary~\ref{cor:rank-bucket} are conservative: 
they only require that high-mass continuations use high-ranked tokens on almost all steps.

\paragraph{Why high-mass paths survive the beam.}
The instantiations above show that a continuation with probability $\ge\taumin$ is constrained to pick high-rank tokens on almost every step.
Concretely, the geometric mean of the per-token conditional probabilities for such a continuation is
\[
\big(\taumin\big)^{1/\suflen} = (10^{-3})^{1/50} \approx 0.87,
\]
meaning the typical (geometric-mean) per-step conditional probability is roughly $87\%$.
(For $\tau=10^{-4}$, the geometric mean is still ${\approx}\,83\%$.)

For a language model, sustaining ${\approx}\,87\%$ conditional probability over $50$ consecutive tokens is very high---matching the case of memorization.
We observe empirically that continuations at this mass level are overwhelmingly (near-)verbatim matches to the training data.
(We confirm this with negative-control experiments on non-training data, where $k$-CBS does not detect near-verbatim matches above $\taumin$.)
Such continuations maintain competitive cumulative $\log$-probabilities at every depth, because they consistently pick high-rank tokens.
With beam width $\bw=20$ (what we choose in practice, see Appendix~\ref{app:experiments:width}), the beam has ample capacity to retain these paths.

A path is pruned from the beam only if $\bw$ other paths outrank it at some depth---a point we make precise below.
Because the rank-budget constraints force high-mass paths to stay near the top of the cumulative-probability ranking at each step, these paths naturally maintain their competitive position across depths.
This is the core reason our approach of decoding constrained beam search produces useful (not just valid) lower bounds for memorized sequences, and it is fundamentally different from MC sampling, which has no mechanism to concentrate on these paths.\looseness=-1

\paragraph{A rigorous floor: heavy-mass path survival.}
We can prove a rigorous sufficient condition for very high-mass sequences, requiring no assumptions beyond the mass conservation property of the top-$k$ tree (proven in Appendix~\ref{app:sec:kcbs}).

\begin{lemma}[Heavy-mass path survival under across-beam top-$\bw$]
\label{lem:heavy-survival}
Fix a depth $t\in\{1,\ldots,\suflen\}$. Let $\sC_t$ be the candidate set at depth $t$
(the union of all children formed from the beam at depth $t-1$, possibly after any optional $\varepsilon$-viability filtering).
Let $p'(\cdot)>0$ denote the cumulative path probability used for ranking.

If a length-$t$ partial continuation $\genseq_{1:t}\in\sC_t$ satisfies
\[
p'(\genseq_{1:t})\;>\;\frac{1}{\bw+1},
\]
then $\genseq_{1:t}$ is kept by the across-beam top-$\bw$ selection at depth $t$ (with deterministic tie-breaking).
Therefore, if a full path $\genseq_{1:\suflen}$ has $p'(\genseq_{1:t})>\frac{1}{\bw+1}$ for every $t=1,\ldots,\suflen$, then it survives all across-beam prunes and is returned by $k$-CBS (or a pruned variant of $k$-CBS).
\end{lemma}

\begin{proof}
Write $\vx \coloneqq \genseq_{1:t}$ for the partial continuation in question, and suppose for contradiction that $p'(\vx) > \frac{1}{\bw+1}$ but $\vx$ is not kept by the top-$\bw$ selection.
Then there exist $\bw$ distinct candidates $\vw_1, \ldots, \vw_{\bw} \in \sC_t$ that each outrank $\vx$, meaning
\[
p'(\vw_i) \;\ge\; p'(\vx) \quad \text{for each } i = 1, \ldots, \bw.
\]
Summing over these $\bw$ candidates and adding $p'(\vx)$:
\[
\sum_{i=1}^{\bw} p'(\vw_i) \;+\; p'(\vx)
\;\ge\; \bw \cdot p'(\vx) \;+\; p'(\vx)
\;=\; (\bw + 1) \cdot p'(\vx)
\;>\; (\bw + 1) \cdot \frac{1}{\bw+1}
\;=\; 1.
\]
But $\vw_1, \ldots, \vw_{\bw}$ and $\vx$ are $\bw + 1$ distinct elements of $\sC_t$, so their probabilities are a subset of the total mass at depth $t$.
Under top-$k$ renormalization, all depth-$t$ candidates sum to exactly $1$ (at most $1$ with viability filtering; see Appendix~\ref{app:sec:kcbs}), so their total probability cannot exceed $1$---a contradiction.

Applying this argument at each depth $t = 1, \ldots, \suflen$ establishes the second claim by induction.
\end{proof}

This lemma applies to both the baseline and $\varepsilon$-pruned $k$-CBS variants, provided the path passes the viability filter at every depth.
For our setting ($\bw=20$), the threshold is $1/(\bw+~1) \approx 0.048$---any continuation maintaining cumulative probability above ${\approx}\,5\%$ at every depth is guaranteed to appear in the final beam.
There can be at most $\bw$ such paths (since $\bw+1$ of them would violate mass conservation), so the beam has capacity for all of them.\looseness=-1

\paragraph{Comparison with MC at the heavy-mass floor.}
Even this high-probability floor dominates MC sampling on both detection and reliability.
Consider a continuation with probability $p > 1/(\bw+1)$.
From Table~\ref{app:tab:mc-hit}, MC requires $m \ge \ln(1/\delta) / (-\ln(1-p))$ samples merely to detect this continuation with confidence $1{-}\delta$.
Using $-\ln(1-p) \approx p$ for moderate $p$, this is approximately $m \ge \ln(1/\delta)/p$.
For $\bw=20$, $p = 1/21$, and $\delta=0.05$, this gives $m \gtrsim 21 \cdot \ln(20) \approx 63$ samples just for detection.
And reliable estimation of its mass is even more expensive (Appendix~\ref{app:sec:intuition:mc}).
By contrast, $k$-CBS recovers this continuation \emph{deterministically}---with certainty and with its exact probability under the scoring distribution.

\paragraph{Beyond the floor.}
The $1/(\bw+1)$ threshold is a sufficient condition for survival of the search, not a necessary one.
In practice, $k$-CBS (and the pruned variants) captures continuations with probability far below this floor.
For most extractable sequences ($\pseqe \ge \taumin$), no individual continuation exceeds $1/(\bw+1)$; those that survive the beam have individual probability closer to $\taumin$.
We cannot provide a comparable guarantee for these lower-probability continuations without distributional assumptions. 

Nevertheless, the results in this appendix show why it is likely that such paths survive the search.
As the geometric-mean argument above shows, even at $\tau=\taumin=10^{-3}$, it is intuitively the case that high-mass paths maintain very high per-token probabilities and stay competitive in the beam.

We deliberately avoid imposing distributional assumptions that would let us formalize this observation further. 
The rank-budget and heavy-survival results hold for any model and any scoring rule that induces a valid probability distribution, and the practical tightness of the lower bound is an empirical finding that we document in our experiments (Section~\ref{sec:experiments} and Appendix~\ref{app:sec:experiments}).\looseness=-1

\subsection{Comparing method costs with token evaluations}\label{app:sec:intuition:cost}

We compare the cost of different extraction methods in terms of \newterm{token evaluations}---the number of tokens processed by the model in forward passes.
These costs show why our beam-search-based approach is far more efficient than Monte Carlo sampling for estimating $\pseqe$.
We also include greedy decoding and teacher-forced likelihood as reference points, since we report these metrics in our experiments.
Throughout, we consider a single prefix--suffix pair with prefix length $\prelen$ and target suffix length $\suflen$.\looseness=-1

\paragraph{Greedy decoding (verbatim and near-verbatim discoverable extraction).}
Given a prefix of length $\prelen$, greedy decoding produces exactly one (deterministic) continuation of length $\suflen$.
The prefill processes the $\prelen$ prefix tokens, producing logits that predict the first suffix token.
Then $\suflen - 1$ autoregressive decode steps each process one token (the last decode step predicts the $\suflen$th suffix token), for a total of\looseness=-1
\[
\underbrace{\prelen}_{\text{prefill}} \;+\; \underbrace{(\suflen - 1)}_{\text{decode steps}}
\quad\text{token evaluations.}
\]
The downstream check---whether the generated continuation matches the target suffix verbatim or near-verbatim (i.e., within distance $\varepsilon$ for a chosen metric $\dist$)---is post-processing that does not involve the model, so the cost is the same for both verbatim and near-verbatim discoverable extraction.
For $\prelen = 50$ and $\suflen = 50$, this is $99$ token evaluations.
(With EOS tokens allowed, decoding may terminate early, producing fewer than $\suflen$ tokens.)
This is the standard way to evaluate extraction, originating from \citet{lee2022dedup} and \citet{carlini2023quantifying}. 

\paragraph{Teacher-forced likelihood (verbatim probabilistic extraction).}
To compute the verbatim probability of a fixed suffix of length $\suflen$ under a given prefix, we evaluate the next-token distributions along the full prefix $+$ suffix via teacher forcing---a single forward pass over $\prelen + \suflen$ positions:\looseness=-1
\[
\underbrace{\prelen + \suflen}_{\text{one pass over prefix $+$ suffix}}
\quad\text{token evaluations.}
\]
From this single pass, we obtain logits at every position, which can then be post-processed with respect to any decoding scheme $\dec$ to get the suffix probability under that scheme, at no additional token-evaluation cost (Appendix~\ref{app:sec:background:extraction:probabilistic}).
Verbatim probabilistic (discoverable) extraction originates from \citet{hayes2025measuringmemorizationlanguagemodels}, and was made more efficient in practice in \citet{cooper2025books}.
Greedy decoding and teacher forcing thus require essentially the same number of token evaluations ($\prelen + \suflen - 1$ vs.\ $\prelen + \suflen$).

\paragraph{Monte Carlo sampling (near-verbatim probabilistic extraction).}
Under the MC baseline (Appendix~\ref{app:sec:intuition:mc}), we draw $M$ independent $\suflen$-length continuations from the prefix and check each for near-verbatim similarity to the target.
With KV-cache reuse of the shared prefix across all $M$ samples,%
\footnote{This is a lower bound on cost: in principle, the prefix KV cache can be computed once and cloned for each sample.
In practice, standard generation APIs (e.g., HuggingFace \texttt{model.generate()}) re-process the prefix for each batch chunk, yielding a per-sample cost closer to $\prelen + \suflen - 1$ and a total closer to $M(\prelen + \suflen - 1)$.
We present the shared-prefill cost as the theoretical minimum; either way, the dominant cost is the $(\suflen - 1) \cdot M$ decode term.}
the cost is
\begin{equation}
\label{eq:mc-cost}
\underbrace{\prelen}_{\text{one prefill}} \;+\; \underbrace{(\suflen - 1)\,M}_{\text{$M$ samples $\times$ $(\suflen-1)$ decode steps each}}
\quad\text{token evaluations.}
\end{equation}
For $\prelen = 50$, $\suflen = 50$, and $M = 3{,}000$ (the order of magnitude required merely to detect an $\varepsilon$-ball with mass $\pseqe = 10^{-3}$ at $95\%$ confidence; Table~\ref{app:tab:mc-hit}), this gives $50 + 49 \cdot 3{,}000 = 147{,}050$ token evaluations---roughly $1{,}500{\times}$ the cost of teacher-forced verbatim extraction for the same pair.
For reliable estimation with $10\%$ relative standard error, the required sample size grows to $M \gtrsim 10^5$ (Appendix~\ref{app:sec:intuition:mc}), and the cost scales proportionally.
(This is the worst case among extractable sequences: $M$ scales as $1/\pseqe$, so sequences with larger $\pseqe$ require fewer samples.)\looseness=-1

\paragraph{$k$-CBS (near-verbatim probabilistic extraction).}
Our beam-search-based approach maintains a beam of width $\bw$ under top-$k$ decoding.
Write $\sL_t$ for the beam at depth $t$ (partial continuations of length $t$).
The cost breaks down as follows:
\begin{enumerate}[leftmargin=0.5cm]
    \item \textbf{Prefill.}
    Process the $\prelen$ prefix tokens ($\sL_0$, a single element) in one forward pass, producing logits for the first suffix position.
    Select the top-$\bw$ candidates from the top-$k$ tokens $\to$ $\sL_1$ (no additional forward pass needed).
    Cost: $\prelen$ token evaluations.

    \item \textbf{Decode steps $t = 1, \ldots, \suflen - 1$.}
    Compute logits for each element of $\sL_t$ ($|\sL_t|$ token evaluations), expand each by the top-$k$ tokens ($\bw \cdot k$ candidates), and prune to the top-$\bw$ $\to$ $\sL_{t+1}$.
    At the final step ($t = \suflen - 1$), return all $\bw \cdot k$ candidates without pruning (Appendix~\ref{app:sec:kcbs}).
    Cost: $\sum_{t=1}^{\suflen - 1} |\sL_t|$ token evaluations; without viability pruning, $|\sL_t| = \bw$ at every depth, giving $(\suflen - 1) \cdot \bw$.
\end{enumerate}
The total is therefore
\begin{equation}
\label{eq:kcbs-cost}
\underbrace{\prelen}_{\text{prefill}} \;+\; \underbrace{(\suflen - 1)\,\bw}_{\text{$(\suflen-1)$ steps $\times$ $\bw$ beams}}
\quad\text{token evaluations,}
\end{equation}
returning up to $\bw \cdot k$ unique continuations.

For $\prelen = 50$, $\suflen = 50$, $\bw = 20$, and $k = 40$, this is $50 + 49 \cdot 20 = 1{,}030$ token evaluations, returning up to $\bw \cdot k = 800$ candidates.
This is roughly $10{\times}$ the cost of greedy extraction.
By comparison, MC requires ${\approx}\,147{,}000$ token evaluations merely for detection ($m = 3{,}000$; Table~\ref{app:tab:mc-hit}) and ${\approx}\,4{,}900{,}000$ for reliable estimation ($m = 10^5$)---roughly $140{\times}$ and $4{,}800{\times}$ more than $k$-CBS, respectively.

\paragraph{Equal-budget comparison.}
Setting $M = \bw$ in Equation~\ref{eq:mc-cost} gives the same token-evaluation cost as Equation~\ref{eq:kcbs-cost}: both equal $\prelen + (\suflen - 1)\,\bw$.
For $\bw = 20$:
MC returns $20$ random samples, while $k$-CBS returns up to $800$ deterministic candidates to evaluate for near-verbatim extraction.
At this budget ($M = \bw = 20$), MC is overwhelmingly unlikely to produce even a single hit for an $\varepsilon$-ball with mass $\pseqe = 10^{-3}$ (Table~\ref{app:tab:mc-hit}; see also Figure~\ref{app:fig:mc-convergence}).
Even for $\pseqe = 10^{-2}$, MC at $M = 20$ has only an ${\approx}\,18\%$ chance of producing a single hit ($0.99^{20} \approx 0.82$ miss probability).
Moreover, $k$-CBS concentrates its budget on high-mass continuations by greedily retaining the highest-scoring partial sequences at each depth (Appendix~\ref{app:sec:intuition:beam:math}), rather than sampling uniformly from the full distribution as MC does.

\paragraph{$\varepsilon$-pruned variants and early stopping.}
The $\varepsilon$-viability-pruned variants of $k$-CBS (Appendix~\ref{app:sec:prune}) can be cheaper still.
At each step, any partial continuation that can no longer produce a final suffix within distance $\varepsilon$ of the target is pruned, potentially reducing the beam size below $\bw$.
Writing $|\sL_t|$ for the number of surviving beam elements at depth $t$, the cost becomes
\[
\underbrace{\prelen}_{\text{one prefill}}
\;+\;
\underbrace{\sum_{t=1}^{\suflen - 1}\!|\sL_t|}_{\text{one eval per beam element per depth}}
\ \le\
\prelen \;+\; (\suflen - 1)\,\bw \quad \text{token evaluations.}
\]
If the viable beam empties entirely at some step $t^\star < \suflen$, decoding halts after
\[
\prelen \;+\; \sum_{t=1}^{t^\star}\!|\sL_t|
\ \le\
\prelen \;+\; t^\star \cdot \bw
\quad \text{token evaluations.}
\]
For example, with $\prelen = 50$, $\bw = 20$, and $t^\star = 3$, a rough upper bound is $50 + 20 \cdot 3 = 110$ token evaluations---comparable to the cost of greedy extraction.
By contrast, MC has no analogous early stopping and costs $\prelen + (\suflen - 1) \cdot m$ regardless of whether the sequence is extractable.
The same cost reduction applies when using the minimum-extraction-probability early termination criterion (Appendix~\ref{app:sec:prune}), which halts a sequence once the best beam's cumulative probability falls below $\taumin / (\bw \cdot k)$.
In practice, with batched processing, the wall-clock time for a batch is determined by the slowest-to-terminate sequence; early-terminated sequences are padded until the batch completes, so the wall-clock savings depend on the fraction of sequences that terminate early.

\paragraph{Token evaluations as a cost metric.}
Each token evaluation involves passing a token through the model's feed-forward and projection layers (whose per-token cost is fixed, since these layers apply the same weight matrices to each token's representation without attending to other positions).
In contrast, the attention component costs scale with the KV-cache length (i.e., the number of preceding tokens that each new token must attend to).
At the sequence lengths in our setting ($\prelen + \suflen = 100$), the per-token feed-forward and projection cost overwhelmingly dominates:
the crossover where per-step attention FLOPs match per-token feed-forward FLOPs occurs at KV-cache lengths that are orders of magnitude beyond our setting.%
\footnote{For \textsc{Llama~2~7B} ($d=4096$, $d_{\text{ff}}=11008$), one token evaluation costs ${\approx}\,25{,}000{\times}$ more FLOPs than one attention position lookup;
for \textsc{Llama~2~70B} ($d=8192$, $d_{\text{ff}}=28672$), the ratio is ${\approx}\,59{,}000{\times}$.
At a KV-cache length of $100$, the total attention cost per token is therefore $100/25{,}000 \approx 0.4\%$ of the feed-forward cost at 7B scale, and even less at 70B.}
Token-evaluation ratios are therefore a reasonable way to approximate FLOP ratios for our experimental conditions.

However, token evaluations do not capture all factors that affect wall-clock runtime.
Teacher forcing processes all $\prelen + \suflen$ tokens in a single parallel forward pass---effectively one large matrix multiplication with one kernel launch and excellent GPU utilization (using causal masking rather than a KV cache, which also frees memory for larger batch sizes).
Autoregressive methods (greedy, MC, and $k$-CBS) instead require $\suflen - 1$ sequential decode steps, each incurring a separate kernel launch plus per-step KV-cache read/write overhead.
As a result, teacher forcing is substantially faster in wall-clock time than greedy decoding despite requiring essentially the same number of token evaluations.
For one prefix--suffix pair, greedy processes $1$ token per decode step, MC processes $M$, and $k$-CBS processes $\bw$.
Batching across multiple sequences improves GPU utilization for all autoregressive methods, though $k$-CBS requires ${\approx}\,\bw{\times}$ more KV-cache memory per sequence than greedy, limiting batch sizes accordingly.
$k$-CBS also incurs per-step KV-cache gather operations (reordering cache rows to match the pruned beam) that are not reflected in the token-evaluation count.

Overall, token evaluations provide a fair common unit for comparing the computational work performed by each method, but the wall-clock advantage of teacher forcing over autoregressive methods is larger than the token-evaluation ratio alone would suggest.
In practice, for larger models (e.g., 70B), the per-token compute cost dominates fixed per-step overhead, and observed wall-clock ratios approach the token-evaluation ratios.
For smaller models (e.g., 7B), kernel-launch and cache-management overhead represent a larger fraction of total time, widening the gap.
Viability pruning and early termination help offset this overhead by reducing the number of decode steps.
\looseness=-1

\paragraph{Summary for our experimental conditions.}
Table~\ref{app:tab:cost-comparison} summarizes the token-evaluation costs at our experimental settings.

\begin{table}[h]
\centering
\caption{\textbf{Token-evaluation cost comparison} for a single prefix--suffix pair ($\prelen = 50$, $\suflen = 50$, $\bw = 20$, $k = 40$).}
\label{app:tab:cost-comparison}
\begin{tabular}{lrrl}
\toprule
Method & Token evals & vs.\ greedy & Candidates returned \\
\midrule
Greedy / Teacher forcing & $99$ / $100$ & $1{\times}$ & $1$ \\
$k$-CBS ($\bw = 20$) & $1{,}030$ & ${\approx}\,10{\times}$ & up to $800$ (deterministic) \\
MC ($M = 20$, same budget) & $1{,}030$ & ${\approx}\,10{\times}$ & $20$ (random) \\
MC ($M = 3{,}000$, detection) & $147{,}050$ & ${\approx}\,1{,}500{\times}$ & $3{,}000$ (random) \\
MC ($M = 10^5$, estimation) & $4{,}900{,}050$ & ${\approx}\,49{,}500{\times}$ & $10^5$ (random) \\
\bottomrule
\end{tabular}
\end{table}\looseness=-1

\vspace{-.1cm}
\section{Decoding-constrained beam search}\label{app:sec:kcbs}
\vspace{-.1cm}

The prior appendix gives a mathematical intuition and cost analysis for our approach: \newterm{decoding-constrained beam search}. 
Here, we provide more details on our baseline algorithmic approach and invariants concerning the lower bound (and trivial upper bound) on $\pseqe$ that this approach returns. 
Our main focus in this paper is to implement decoding-constrained beam search for top-$k$, but this is not a strict requirement;
any decoding policy that produces a valid probability distribution could be adapted accordingly. 
Therefore, we focus here on presenting details on the top-$k$ version of our algorithm, starting with notation (Appendix~\ref{app:sec:kcbs:notation}) before describing the concrete algorithm (\newterm{baseline top-$k$ constrained beam search (CBS)}) in detail (Appendix~\ref{app:sec:kcbs:baseline:details}) and its invariants (Appendix~\ref{app:sec:kcbs:baseline:invariants}).

In brief, this algorithm is beam search, replacing the full softmax distribution with the renormalized top-$k$ distribution and disabling the last across-beam prune.
For beam width $\bw$ and $\suflen$ iterations, this algorithm returns up to $\bw \cdot k$ $\suflen$-length continuations of the input prefix, each with the associated conditional probability computed with respect to top-$k$ decoding.
These continuations can then be post-processed to check for set membership in the near-verbatim $\varepsilon$-ball $\ball$, for a chosen distance $\dist$ and tolerance $\varepsilon$; 
we sum over the probabilities of those continuations to produce a deterministic lower bound on $\pseqe$ and (often very loose) upper bound on $\pseqe$.

We close this appendix with a discussion of how one might also apply decoding-constrained beam search for nucleus sampling as the decoding policy (Appendix~\ref{app:sec:other:nucleus}).

As we show in Appendix~\ref{app:sec:prune}, we can generally improve on this approach with the addition of $\varepsilon$-viability pruning rule ($\ham$ or $\lev$) to the search procedure, which prunes non-$\varepsilon$-viable continuations from the beam rather than retaining and returning them at step $\suflen$.
This approach, of course, trades off the generality that comes from post-processing in the baseline approach we discuss in this appendix, as it bakes in a distance metric to the search process.
However, this comes with benefits:
without any additional token evaluations---just some additional bookkeeping---the pruned variants often return tighter lower and upper bounds for near-verbatim extractable sequences, also often at lower cost.
The cost of all of these methods is discussed in Appendix~\ref{app:sec:intuition:cost}.
Decoding-constrained beam search is significantly cheaper than MC sampling, and (under a matched compute budget) returns a significantly more useful (and deterministic) estimate of $\pseqe$ (Appendix~\ref{app:sec:intuition:mc}).

\subsection{Notation}\label{app:sec:kcbs:notation}

We describe some common notation for our algorithms and proofs.

\paragraph{Token sequences.} 

\begin{itemize}[leftmargin=0.5cm]
    \item $\pre$: the $\prelen$-token prefix from the training data.
    $\suf$: the corresponding $\suflen$-token ground-truth target suffix.
    The training-data sequence is $\seq = \pre \Vert \suf$.

    \item $\cont$: an arbitrary $\suflen$-token continuation of $\pre$. 
    
    \item $\gensuf$: an actual, generated $\suflen$-token continuation of $\pre$.
    A history is the concatenation of a prefix and a generated continuation: $\genseq = \pre \Vert \gensuf$.

    \item $\subgensuf{1:t}$: the first $t$ tokens of a generated continuation ($\subgensuf{1:0} = \emptyset$, the empty token sequence).
    At step $t$ of Algorithm~\ref{app:alg:beam-topk}, each input beam element holds a history $\genseq = \pre \Vert \subgensuf{1:t-1}$;
    after expansion, the updated history is $\genseq' = \pre \Vert \subgensuf{1:t}$.
\end{itemize}
This notation lets us refer to prefixes, suffixes, and partial continuations directly, avoiding index arithmetic on the full history~$\genseq$.

\paragraph{Candidate sets and beam operations.} 
At depth $t\in\{0,\ldots,\suflen\}$, denote
\begin{itemize}[leftmargin=0.5cm]
    \item $\sL_t$: the beam at depth $t$ (at most $\bw$ elements), obtained by expanding $\sL_{t-1}$, applying across-beam pruning.
    \item $\sC_t$: all children formed from $\sL_{t-1}$ by one-step top-$k$ expansion (exactly $|\sL_{t-1}| \cdot k$ before EOS removal; EOS candidates are removed and recorded but do not enter the beam).
    \item $\sU_t$: the (at most) top-$\bw$ of $\sC_t$ by cumulative probability (across-beam prune); $\sL_t \gets \sU_t$.
    \item $\sF$, the set of returned finals.
\end{itemize}

\paragraph{Per-step $t$ notation.}  
Given a beam element with current history $\genseq$ (the prefix $\pre$ concatenated with a continuation of length $t{-}1$), 
{\small
\begin{align*}
\logit_t(\genseq)\in\R^{|\vocab|} \quad &\triangleq \quad \text{logit vector from $\model$ for the next token given the history $\genseq$};\\
\sS_t(\genseq)\leftarrow \TopK_{k}\big(\logit_t(\genseq)\big) \quad &\triangleq \quad \text{the set of $k$ tokens in $\vocab$ with the largest logits in $\logit_t(\genseq)$};\\
\vr_t(\genseq) \in\R^{|\vocab|}\leftarrow \LogSoftmax\big(\logit_t(\genseq)\big) \quad &\triangleq \quad \text{$\log$ probs vector over $\vocab$, obtained from $\log \softmaxsf$ on $\logit_t(\genseq)$},\\
\quad & \quad\quad\;\, \vr_t(\genseq)[v]=\log \Pr_{\model}(\tok=v\mid\genseq) \quad \forall v\in\vocab\text{ (full $\model$ distribution, before top-$k$)};\\
Z_t(\genseq) \leftarrow \LogSumExp\big(\vr_t(\genseq)[\sS_t(\genseq)]\big) \quad &\triangleq \quad \text{$\log$ of the total probability mass on $\sS_t(\genseq)$ (normalizing constant for top-$k$)},\\
\quad & \quad\quad\;\,Z_t(\genseq) \;=\; \log\sum_{u \in \sS_t(\genseq)} \exp\big(\vr_t(\genseq)[u]\big) \;=\; \log\sum_{u \in \sS_t(\genseq)} \Pr_{\model}(u\mid\genseq).
\end{align*}
}%
Then, for the update rule, selecting any $\gentok\in\sS_t(\genseq)$ under $\dec$ (top-$k$ decoding) has probability
\[
\Pr_{\model,\dec}(\gentok\mid\genseq)=\exp(\vr_t(\genseq)[\gentok]-Z_t(\genseq)),
\]
and the $\log$-prob update is
\[
\vr_t(\genseq)[\gentok] - Z_t(\genseq) = \log \Pr_{\model,\dec}(\gentok \mid \genseq).
\]

In all of our algorithms, we focus on top-$k$ decoding and set temperature $\temp=1$.
However, we could also use temperature scaling with different settings in the decoding policy $\dec$.
If temperature $\temp>0$ is applied, logits are rescaled by $1/\temp$ before the softmax:
\begin{align}
\label{app:eq:temp-scaling}
{\logit}^{(\temp)}_t(\genseq) \;=\; \tfrac{1}{\temp}\,\logit_t(\genseq),
\qquad
\vr_t^{(\temp)}(\genseq) \;=\; \log \softmaxsf\big({\logit}^{(\temp)}_t(\genseq)\big)
  \;=\; \log \softmaxsf\big(\tfrac{1}{\temp}\,\logit_t(\genseq)\big).
\end{align}
Because scaling by a positive $\temp$ preserves order, the top-$k$ set is unchanged:
\[
\sS_t(\genseq)\;=\;\TopK_k\big({\logit}^{(\temp)}_t(\genseq)\big)
  \;=\;\TopK_k\big(\logit_t(\genseq)\big).
\] 
For $\dec=\temp$, we define the temperature-aware $\log$-normalizer
over this constrained set from the $\log$-probs:
\[
Z_t^{(\temp)}(\genseq)
\;\triangleq\;
\LogSumExp\big(\vr_t^{(\temp)}(\genseq)[\sS_t(\genseq)]\big)
\;=\;
\log\sum_{u\in\sS_t(\genseq)} \Pr_{\model,\temp}(u\mid\genseq).
\]
The per-step $t$, per-next-token $\gentok$ top-$k$ $\log$-probability update is then
\begin{equation}
\label{eq:temp-update}
\vr_t^{(\temp)}(\genseq)[\gentok] - Z_t^{(\temp)}(\genseq)
\;=\;
\log \Pr_{\model,(\temp,k)}(\gentok\mid\genseq),
\qquad
\gentok\in\sS_t(\genseq).
\end{equation}
To see that this simplifies to a $\log$ $\softmaxsf$ (Equation~\ref{eq:softmax}) over the top-$k$ scaled logits,
let 
\begin{align*}
C_t \triangleq \log\sum_{w\in\vocab}\exp(\tfrac{1}{\temp}\,\logit_t(\genseq)[w])
\end{align*}
denote the full-vocabulary $\log$-partition function.
For token $\gentok$, we can expand $\vr_t^{(\temp)}$ from its definition (Equation~\ref{app:eq:temp-scaling}):
\begin{align*}
\vr_t^{(\temp)}(\genseq)[\gentok]
&\;=\;
\log\softmaxsf\big(\tfrac{1}{\temp}\,\logit_t(\genseq)\big)[\gentok]
\\
&\;=\;
\log\frac{\exp\!\big(\tfrac{1}{\temp}\,\logit_t(\genseq)[\gentok]\big)}
         {\sum_{w\in\vocab}\exp\!\big(\tfrac{1}{\temp}\,\logit_t(\genseq)[w]\big)}
\\
&\;=\;
\log\exp\!\big(\tfrac{1}{\temp}\,\logit_t(\genseq)[\gentok]\big)
\;-\;
\log\sum_{w\in\vocab}\exp\!\big(\tfrac{1}{\temp}\,\logit_t(\genseq)[w]\big)
\\
&\;=\;
\frac{1}{\temp}\,\logit_t(\genseq)[\gentok] \;-\; C_t.
\end{align*}
Now, expanding $Z_t^{(\temp)}$:
\begin{align*}
Z_t^{(\temp)}(\genseq)
\;=\;
\log\sum_{u\in\sS_t(\genseq)} \exp\big(\vr_t^{(\temp)}(\genseq)[u]\big)
&\;=\;
\log\sum_{u\in\sS_t(\genseq)} \exp\Big(\frac{1}{\temp}\,\logit_t(\genseq)[u] - C_t\Big)\\
&\;=\;
\log\sum_{u\in\sS_t(\genseq)} \Big[ \exp\Big(\frac{1}{\temp}\,\logit_t(\genseq)[u]\Big) \;\cdot\; \exp(-C_t) \Big]\\
&\;=\; 
\log\sum_{u\in\sS_t(\genseq)} \exp\Big(\frac{1}{\temp}\,\logit_t(\genseq)[u]\Big) \;-\; C_t.
\end{align*}

Subtracting:
\begin{align*}
\vr_t^{(\temp)}(\genseq)[\gentok] - Z_t^{(\temp)}(\genseq)
&\;=\;
\Big(\frac{1}{\temp}\,\logit_t(\genseq)[\gentok] - C_t\Big)
\;-\;
\Big(\log\sum_{u\in\sS_t(\genseq)} \exp\Big(\frac{1}{\temp}\,\logit_t(\genseq)[u]\Big) - C_t\Big)\\
&\;=\;
\frac{1}{\temp}\,\logit_t(\genseq)[\gentok]
\;-\;
\log\sum_{u\in\sS_t(\genseq)}
   \exp\Big(\tfrac{1}{\temp}\,\logit_t(\genseq)[u]\Big),
\end{align*}
That is, the update reduces to a $\log$ $\softmaxsf$ over the top-$k$ scaled logits.
Setting $\temp=1$ recovers the update that we use throughout:
$\vr_t(\genseq)[\gentok]-Z_t(\genseq)
 = \log \Pr_{\model,\dec}(\gentok\mid\genseq)$ (and for $\temp=1$, $\dec$ is just top-$k$ without temperature).
(We will generally refer to the decoding strategy as $\dec$, but this means in practice, with respect to the notation introduced here, we set $\temp=1$ with top-$k$, i.e., $(\temp, k) = (1, k)$.)

\subsection{Top-$k$ constrained beam search}\label{app:sec:kcbs:baseline}

In Algorithm~\ref{app:alg:beam-topk}, we describe the baseline approach for the top-$k$ constrained beam search algorithm:
a slight variation on beam search that replaces the full softmax distribution with the renormalized top-$k$ distribution and omits the last step's across-beam prune to width $\bw$.
This approach returns a deterministic lower bound on $\pseqe$ and (almost always very loose) upper bound on $\pseqe$, computed over an $\varepsilon$-viable filtered subset of the returned sequences (for a chosen distance metric and $\varepsilon$).

We describe the algorithm in detail (Appendix~\ref{app:sec:kcbs:baseline:details}) and then prove several invariants (Appendix~\ref{app:sec:kcbs:baseline:invariants}). 
These invariants are straightforward, but are nevertheless important because they justify the soundness of our approach more formally than the intuition provided in Section~\ref{sec:kcbs} and Appendix~\ref{app:sec:intuition}.
The novelty of our work comes from connecting these previously disconnected observations in the service of making near-verbatim probabilistic extraction computationally feasible. 

\begin{algorithm}[t]
\caption{Top-$k$ Constrained Beam Search ($k$-CBS)}\label{app:alg:beam-topk}
\KwIn{LLM $\model$; prefix $\pre$ of length $\prelen$; suffix length $\suflen$; beam width $\bw$;
  top-$k$ parameter $k$ for decoding policy $\dec$ ($\bw \leq k^2$, since each of at most $k$ beams expands to $k$ candidates; $k \ll |\vocab|$);
  EOS token id; optional extraction threshold $\taumin > 0$}
\KwOut{Set $\sF$ of up to $\bw \cdot k$ pairs
  $(\genseq,\; \log p)$,
  where $\genseq = \pre \,\Vert\, \gensuf$ is the full history and $\log p = \log \Pr_{\model,\dec}(\gensuf \mid \pre)$;
  or $\sF = \emptyset$ if early termination is triggered}
\BlankLine
\textbf{Notation (per step $t$ for partial sequence $\genseq$):}
Given a beam element with current history $\genseq$ (the prefix $\pre$ concatenated with a continuation of length $t{-}1$), let:
\begin{itemize}[topsep=1pt,leftmargin=0.75cm, itemsep=1pt]
  \item $\logit_t(\genseq)\in\R^{|\vocab|}$: logits from $\model$ for the next token given $\genseq$;
  \item $\sS_t(\genseq) \subset \vocab$, $|\sS_t(\genseq)| = k$:
  $k$ tokens in $\vocab$ with the largest logits in $\logit_t(\genseq)$;
  \item $\vr_t(\genseq) \in\R^{|\vocab|}$: $\log$ probs over $\vocab$, $\vr_t(\genseq)[v]=\log \Pr_{\model}(v\mid\genseq)$ $\forall v\in\vocab$;
  \item $Z_t(\genseq) \in \R$: normalizing constant for top-$k$,
    \(
      Z_t(\genseq) \;=\; \log\sum_{u \in \sS_t(\genseq)} \exp\big(\vr_t(\genseq)[u]\big).
    \)
\end{itemize}
\BlankLine
\textbf{Update rule (top-$k$ decoding):}
Selecting any $\gentok\in\sS_t(\genseq)$ under top-$k$ has probability \\$\Pr_{\model,\dec}(\gentok\mid\genseq)=\exp(\vr_t(\genseq)[\gentok]-Z_t(\genseq))$, i.e., $\log \Pr_{\model,\dec}(\gentok \mid \genseq)=
\vr_t(\genseq)[\gentok] - Z_t(\genseq)$.\looseness=-1
\BlankLine

\textbf{Beam state.}
Maintain $\sL_t$ as pairs $(\genseq,\log p)$, where $\log p$ is the accumulated top-$k$ decoding $\log$-probability of the partial sequence $\genseq$ so far.
Max. beam capacity $|\sL_t|=\bw$.
\BlankLine
Compute $\logit_1(\pre)$ via forward pass on $\pre$\tcp*{Prefill: $\prelen$ token evals}
$\sL_0 \gets \big\{(\pre,\; 0)\big\}$\tcp*{Beam: pairs $(\genseq,\; \log p)$}
\BlankLine
\For{$t = 1, \ldots, \suflen$}{
  $\sC_t \gets \emptyset$\tcp*{Candidate set for step $t$}
  \ForEach{$(\genseq,\; \log p) \in \sL_{t-1}$}{
    $\sS_t(\genseq) \gets \TopK_k\big(\logit_t(\genseq)\big)$\;
    $\vr_t(\genseq) \gets \LogSoftmax\big(\logit_t(\genseq)\big)$\;
    $Z_t(\genseq) \gets \LogSumExp\big(\vr_t(\genseq)[\sS_t(\genseq)]\big)$\;
    \ForEach{$\gentok \in \sS_t(\genseq)$}{
      $\genseq' \leftarrow \genseq \,\Vert\, \gentok$ \tcp*{append token to partial history}
      $\log p' \leftarrow \log p + \vr_t(\genseq)[\gentok] - Z_t(\genseq)$ \tcp*{update continuation $\log$ prob}
      $\sC_t \gets \sC_t \cup
        \big\{\!\big(\genseq',\;\;
        \log p' \big)\!\big\}$\;
    }
  }
  \BlankLine
  \If(\tcp*[f]{Final step: return all candidates (up to $\bw \cdot k$)}){$t = \suflen$}{
    \Return $\sF \gets \sC_\suflen$\;
  }
  \BlankLine
  \tcp{Non-final step ($t < \suflen$): EOS handling and beam pruning}
  \ForEach{$(\genseq',\; \log p') \in \sC_t$ with latest $\gentok = \textup{EOS}$}{
    Record $(\genseq',\; \log p',\; t)$ as early-termination path;
      remove from $\sC_t$\;
  }
  \lIf{$\sC_t = \emptyset$}{\Return $\emptyset$}
  $\sU_t \gets$ top-$\bw$ elements of $\sC_t$ ranked by $\log p'$\;
  $\sL_t \gets \sU_t$ \tcp*{Prune to beam width}
  \If{$\taumin$ and $\max_{(\_,\,\log p) \in \sL_t} \exp(\log p) < \taumin/(\bw \cdot k)$}{
    \Return $\sF \gets \emptyset$\tcp*{Early termination}
  }
  Compute $\logit_{t+1}(\genseq)$ for each $(\genseq, \cdot) \in \sL_t$\tcp*{$\bw$ token evals}
}
\end{algorithm}

\subsubsection{Detailed description of baseline $k$-CBS}\label{app:sec:kcbs:baseline:details}

Algorithm~\ref{app:alg:beam-topk} implements our \newterm{baseline top-$k$ constrained beam search ($k$-CBS)} algorithm 
for identifying candidates for high-probability continuations of the prefix under top-$k$ decoding. 
Because memorized sequences from the training data are by definition very high probability under $\model$ and $\dec$, if a sequence $\suf$ is memorized, we expect this procedure to return a set of candidates $\gensuf$ that contains near-verbatim matches (and possibly the verbatim match) to the target suffix $\suf$ (Appendix~\ref{app:sec:intuition:beam}).
Per sequence tested, this type of search procedure provides a deterministic, correct lower bound on $\pseqe$;
for beam width $\bw$, it is approximately $\frac{\bw}{2}{\times}$ more expensive than greedy-decoded discoverable extraction, but orders of magnitude cheaper than an unbiased Monte Carlo estimate of $\pseqe$ (Appendix~\ref{app:sec:intuition:cost}). 
A more condensed version of Algorithm~\ref{app:alg:beam-topk} is in the main paper in Algorithm~\ref{alg:kcbs}.

To summarize, the algorithm takes a prefix $\pre$ and greedily (i.e., not optimally) searches for the highest-probability length-$\suflen$ continuations under top-$k$ constrained decoding.
It begins with a single prefill forward pass that processes $\pre$ and produces logits $\logit_1(\pre)$ for the first suffix position.
The beam $\sL_0$ is initialized with the prefix and a cumulative $\log$-probability of zero.
The main loop then iterates over suffix positions $t = 1, \ldots, \suflen$.
At each step, every partial history $\genseq$ in the beam is expanded:
the top-$k$ token set $\sS_t(\genseq)$ is produced from the current logits, 
the $\log$-probability vector $\vr_t(\genseq)$ and normalizing constant $Z_t(\genseq)$ are computed, 
and each candidate token $\gentok \in \sS_t(\genseq)$ is appended to form an extended history $\genseq' = \genseq \,\Vert\, \gentok$ with updated cumulative $\log$-probability (with respect to top-$k$) $\log p' = \log p + \vr_t(\genseq)[\gentok] - Z_t(\genseq)$.
These candidates are collected into the set $\sC_t$.
(At $t = 1$, the beam $\sL_0$ contains a single element (the prefix), so $|\sC_1| = k$:
the prefix paired with each of the $k$ tokens in $\sS_1(\pre)$.
For all subsequent steps $t \geq 2$, $|\sL_{t-1}| = \bw$, so $|\sC_t| = \bw \cdot k$.)
On the final step ($t = \suflen$), up to $\bw \cdot k$ candidate $\suflen$-length histories are returned in the output set $\sF$.\footnote{The reason why this is ``up to $\bw \cdot k$'' and not exactly that number is that, on non-final steps, any candidate whose last token is EOS is recorded as an
early-termination path and removed from $\sC_t$; 
it is possible, in degenerate cases with many EOS sequences, that there are fewer than the maximum possible, though we do not observe this in practice.
}
The remaining candidates are pruned to the top $\bw$ by cumulative $\log$-probability
to form the new beam $\sL_t$, and a forward pass on the $\bw$ selected tokens
produces logits for the next position.

\paragraph{Optional early termination.}
When the extraction threshold $\taumin$ is provided, the algorithm can terminate before completing all $\suflen$ steps if it becomes impossible for the final output to accumulate at least $\taumin$ total mass within $\ball(\suf)$.
The key observation is that cumulative $\log$-probabilities are monotonically non-increasing along any path: 
at each step $t$, the update adds $\log \Pr_{\model,\dec}(\gentok \mid \genseq) \leq 0$, so no descendant of a beam element $(\genseq, \log p) \in \sL_t$ can have cumulative probability exceeding $\exp(\log p)$.
Therefore, $\max_{(\_, \log p) \in \sL_t} \exp(\log p)$ upper-bounds the probability of every final candidate at depth $\suflen$.
Even if all $\bw \cdot k$ final candidates achieved this maximum and all fell within $\ball(\suf)$, the total near-verbatim mass would be at most $\bw \cdot k \cdot \max_{(\_, \log p) \in \sL_t} \exp(\log p)$.
When this quantity is strictly less than $\taumin$---equivalently, $\max_{(\_, \log p) \in \sL_t} \exp(\log p) < \taumin / (\bw \cdot k)$---the lower bound from this prefix can never reach $\taumin$.
If we provide a minimum (validated) extraction probability $\taumin$ (Appendix~\ref{app:sec:experiments}), we can then allow the algorithm to terminate early returning $\sF = \emptyset$, as this sequence cannot be near-verbatim extractable with respect to $\taumin$.
When early termination is not triggered, the algorithm behaves identically to the case without $\taumin$.\looseness=-1

The algorithm performs $1$ prefill forward pass ($\prelen$ token evaluations) followed by at most $\suflen - 1$ decoding forward passes ($\bw$ token evaluations each), for a total of at most $\prelen + (\suflen - 1) \cdot \bw$ token evaluations per sequence (Appendix~\ref{app:sec:intuition:cost}).
We follow the notation in Appendix~\ref{app:sec:kcbs:notation}.

\paragraph{Returned suffix candidates.}
There are no duplicate sequences in the output candidates.  
This follows directly from the semantics of Algorithm~\ref{app:alg:beam-topk}, but is worth making explicit for clarity:  
every output \emph{token} sequence is guaranteed to be unique (Lemma~\ref{lem:unique}).
However, it is possible that in \emph{character} space there are duplicates---that unique token sequences decode to the same string of characters.
Because our probability mass computations are done in token space, we do not need to account for character-space duplicates in a special way.
Every unique token sequence contributes to our extraction mass calculations.  
We return up to $\bw \cdot k$ such unique sequences. 
Unlike traditional beam search, which prunes to $\bw$ at the final step, we omit this last across-beam prune: all $\suflen$-length candidates already have valid probabilities with respect to top-$k$ decoding, so retaining all of them in $\sF$ is free and preserves mass that would otherwise be discarded (Lemma~\ref{lem:mass-preservation}).
Because these $\bw \cdot k$ sequences contain the top-$\bw$ (i.e., those that would survive the last prune), the mass of these $\bw \cdot k$ sequences dominates (Corollary~\ref{cor:prune-inequality}).\looseness=-1

The HuggingFace API makes it very simple to implement changes to beam search to respect our top-$k$ scoring rule using the \texttt{LogitProcessor} abstraction. 
(Using beam search with top-$k$ enabled does not do this by default; 
it is a stochastic variant of beam-search with top-$k$ sampling.)
However, this is not the case for skipping the last iteration across-beam prune, or how we handle EOS.
We need to implement our own search, batching, and KV caching to achieve this.\looseness=-1 

\paragraph{Computing the near-verbatim extraction bounds from suffix candidates.} 
Finally, we can apply our chosen distance metric---in this paper, either $\ham$ (Equation~\ref{eq:hamming}) or $\lev$ (Equation~\ref{eq:levenshtein})---with tolerance $\varepsilon$ to the returned suffix candidates, in order to test which are near-verbatim matches to the target suffix $\suf$. 
(One could in principle apply any edit-distance or semantic similarity metric.)
This provides a rigorous lower bound on $\pseqe$ (Equation~\ref{eq:pze}): 
instead of estimating the near-verbatim suffix probability via Monte Carlo sampling (Appendix~\ref{app:sec:intuition:mc}), we sum the probabilities of the suffix candidates that are in the $\varepsilon$-ball $\ball(\suf)$. 

As noted above, $\sF$ is the set of final returned tuples (up to $\bw \cdot k$ of them) of $\suflen$-length candidates and their respective $\log$ probabilities under top-$k$ (Algorithm~\ref{app:alg:beam-topk}). 
Because the algorithm's semantics respect top-$k$ decoding, the search explores a subset of the full probability distribution of top-$k$ continuations (i.e., mass summing to $1$; see Lemma~\ref{lem:frontier-mass}). 
Therefore, the returned mass of the final continuations in $\sF$ is ${\leq}1$ (in practice, almost certainly $<1$). 
We therefore define \newterm{covered mass} as the total mass of all returned continuations: 
\begin{align}
\label{eq:covered-mass}
\cover(\sF) &\;\coloneq\; \sum_{(\_, \log p) \in \sF} \exp(\log p). 
\end{align}
Of course, since the total probability is $1$, the \emph{un}covered mass---i.e., mass that is not captured by our algorithm---is equivalent to $1-\cover(\sF)$. 

To estimate near-verbatim extraction, we then filter the returned suffixes according to the chosen distance metric $\dist$ and tolerance $\varepsilon$.
That is, on the outputs of Algorithm~\ref{app:alg:beam-topk}, for each $\genseq = \pre \Vert \gensuf$, we evaluate  
\begin{align}
\label{eq:nv-filter}
    \sF^{(\leq \varepsilon)} &\coloneqq \{ (\genseq, \log p) :  (\genseq, \log p) \in \sF \text{ and }\, \gensuf \in \ball(\suf) \}.
\end{align} 

After this filtering operation, we can compute the (deterministic) lower bound of $\pseqe$ as 
\begin{align}
\label{hat-pze:baseline-lb}
\pseqe \;&\geq\; \lb \;\triangleq\; 
\sum_{(\_, \log p) \in \sF^{(\leq\varepsilon)}} \exp(\log p). 
\end{align}

From the above covered mass (Equation~\ref{eq:covered-mass}) and lower bound (Equation~\ref{hat-pze:baseline-lb}), we can also compute a (typically very) loose upper bound on $\pseqe$:
\begin{align}
\label{eq:baseline-ub}
\pseqe \;\leq\; \ub \;\triangleq\; \lb \;+\; \big(1-\cover(\sF)\big),
\end{align}
which we elaborate on in Proposition~\ref{prop:loose-ub}. 
Note that $\ub \leq 1$ always holds, just as $\lb \geq 0$ trivially holds; the formula captures the fact that any uncovered mass could in principle all be $\varepsilon$-viable.
For our baseline algorithm for $k$-CBS, we use a subset of the covered mass to produce an $\varepsilon$-viable lower bound;
a simple upper bound, then, is any potentially viable mass that is \emph{not} covered by the algorithm's output continuations. 
Of course, this can be very loose, if this uncovered mass is quite large.

\paragraph{A simple illustrative example.}
To give a sense of what the algorithm returns, we use the same simple illustrative example from \citet{cooper2025books}--- 
a famous quote from \emph{The Great Gatsby}:\\

{\small
\noindent $\pre$: \verb|They were careless people, Tom and Daisy |\textendash\verb| they smashed up things and creatures |\\
 \verb|and then retreated|

\vspace{0.1cm}

\noindent $\suf$: \verb| back into their money or their vast carelessness, or whatever it was that kept |\\
\verb|them together, and let other people clean up the mess they had made.|\\
}

We use $\pre$ as the prompt and will compare generated continuations $\gensuf$ against the target suffix $\suf$ to determine near-verbatim extraction success.
We use \textsc{Llama 1 13B}, which was trained on the \texttt{Books3} corpus~\citep{touvron2023llamaopenefficientfoundation, lee2023talkin}.
\emph{The Great Gatsby} is included in \texttt{Books3}.
With the \textsc{Llama 1} tokenizer, the prefix is $25$ tokens and the suffix is $32$ tokens. 
We run baseline $k$-CBS (Algorithm~\ref{app:alg:beam-topk}) to get a lower bound on the near-verbatim extraction probability, $\pseqe$, with $\temp=1$, $\bw=20, k=40$.
Table~\ref{app:tab:diffs-ex-gatsby} shows results for the $10$ highest-probability continuations that the algorithm returns.

There are several interesting observations about these results:
\begin{itemize}[leftmargin=0.5cm]
    \item \textbf{Verbatim discoverable extraction would fail, returning $0$.} 
    The first row corresponds to the greedy-decoded continuation, which is \emph{not} a verbatim match to the target suffix. 
    If we were to perform traditional discoverable extraction with greedy decoding, we would output the first row's $\gensuf$, and we would not identify this output as successful extraction, since it fails strict equality with $\suf$.  
    
    \item \textbf{Near-verbatim discoverable extraction succeeds, returning $1$.} 
    Using $\varepsilon = 1$ near-verbatim discoverable extraction (either $\ham$ or $\lev$), we would identify the generation in row 1 as successful near-verbatim extraction.
    
    \item \textbf{Verbatim probabilistic extraction succeeds, returning $\pseq=0.1431$.} 
    Probabilistic extraction shows that the verbatim suffix has $\pseq=0.1431$---a very high probability that we would count as valid extraction (e.g., exceeding a reasonable $\taumin$).
    Note that the top-ranked continuation (row 1) is also the greedy-decoded continuation.
    This provides a notion of extraction risk, absent from greedy discoverable extraction. 

    \item \textbf{Near-verbatim probabilistic extraction reveals higher extraction risk.} 
    In this paper, we consider maximum distances of $\varepsilon=5$ for both Levenshtein and Hamming. 
    For this sequence, $p_{\seq,5}^\levshort \geq 0.7155$ (over all $\bw \cdot k = 20 \cdot 40 = 800$ candidates). 
    With even just $\varepsilon=1$ for the Levenshtein distance, $p_{\seq,1}^\levshort \geq 0.4681$, over $3\times$ the extraction risk compared to verbatim $\pseq$. 
    This shows that verbatim probabilistic extraction can greatly underestimate extraction risk.
    All of the continuations above are effectively the same text, with only slight variations in punctuation. 
\end{itemize}

{
\begin{table*}[b!]
    \caption{\textbf{Example $k$-CBS output.} 
    Top-$10$ generations ($|\gensuf|=32$ tokens) for \textsc{Llama 1 13B}  baseline $k$-CBS ($\bw=20,\, k=40$) for the $25$-token prefix $\pre$:
    \texttt{They were careless people, Tom and Daisy - they smashed up things and creatures and then retreated}.
    This is a quote from \emph{The Great Gatsby}~\citep{The_Great_Gatsby}, a book contained in \textsc{Llama 1}'s training data. 
    We diff each $\gensuf$ with $\suf$, showing deletions in red and crossed out, and additions in blue.
    All exact-matched text is in black.
    Row 2 (highlighted in yellow) shows the verbatim generation of the target suffix, i.e., $\gensuf=\suf$; there is no diff highlighting.}
    \label{app:tab:diffs-ex-gatsby}
    \centering
    \scriptsize
    \begin{tabular}{c c c c c}
        \toprule
         & $\gensuf$ & $\Pr_{\model,\dec}(\gensuf \mid \pre)$ & $\levshort$ & $\hamshort$ \\
        \midrule
        1 & \adjustbox{valign=m}{\includegraphics[width=0.35\linewidth]{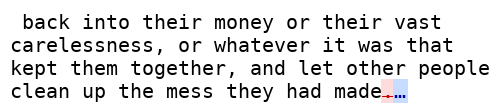}} & 0.1477 & 1 & 1 \\
        \rowcolor{gtrow}
        2 & \adjustbox{valign=m}{\includegraphics[width=0.35\linewidth]{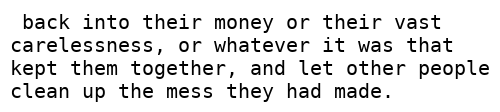}} & 0.1431 & 0 & 0 \\
        3 & \adjustbox{valign=m}{\includegraphics[width=0.35\linewidth]{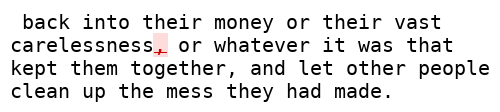}} & 0.0671 & 2 & 22 \\
        4 & \adjustbox{valign=m}{\includegraphics[width=0.35\linewidth]{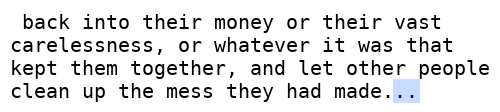}} & 0.0535 & 1 & 1 \\
        5 & \adjustbox{valign=m}{\includegraphics[width=0.35\linewidth]{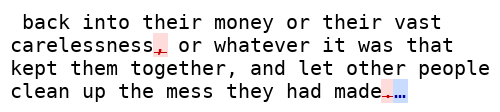}} & 0.0409 & 3 & 22 \\
        6 & \adjustbox{valign=m}{\includegraphics[width=0.35\linewidth]{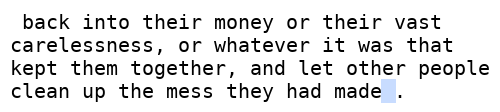}} & 0.0385 & 1 & 1 \\
        7 & \adjustbox{valign=m}{\includegraphics[width=0.35\linewidth]{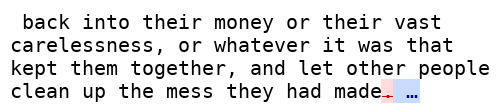}} & 0.0286 & 1 & 1 \\
        8 & \adjustbox{valign=m}{\includegraphics[width=0.35\linewidth]{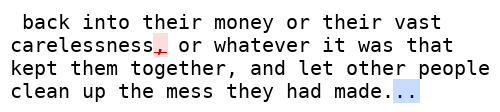}} & 0.0214 & 3 & 22 \\
        9 & \adjustbox{valign=m}{\includegraphics[width=0.35\linewidth]{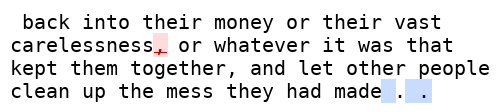}} & 0.0186 & 3 & 22 \\
        10 & \adjustbox{valign=m}{\includegraphics[width=0.35\linewidth]{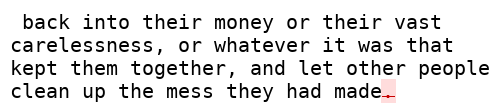}} & 0.0119 & 1 & 1 \\
        \bottomrule
    \end{tabular}
\end{table*}
}
\clearpage 

\subsubsection{Invariants for baseline $k$-CBS}\label{app:sec:kcbs:baseline:invariants}

Using the notation in Appendix~\ref{app:sec:kcbs:notation}, we first show that the semantics of Algorithm~\ref{app:alg:beam-topk} ensure that
every sequence $\genseq$ returned in $\sF$---the collection of pairs $(\genseq,\log p)$---is unique.

\begin{lemma}[No token-level duplicates]
\label{lem:unique}
For a set of sequence-$\log$ prob pairs $\sA \subseteq \vocab^{i}\times\R$ (for a fixed length $i$), define the projection onto sequences by
\[
\pi_1(\sA)\;\coloneqq\;\{\;\genseq\in\vocab^{i}:\exists\,\ell\in\R\text{ with }(\genseq,\ell)\in\sA\;\}.
\] 
As in Algorithm~\ref{app:alg:beam-topk}, assume that we do not maintain partial histories with an EOS until possibly at the final step $\suflen$, so all child sequences at depth $t$ have the same length $\prelen{+}t$ (the $\prelen$-length prefix and $t$ generated tokens so far).
Let $\sL_{t-1}\subseteq \vocab^{\prelen+t-1}\times\R$ be the beam at depth $t{-}1$, and
$\sF\subseteq \vocab^{\prelen+\suflen}\times\R$ the set of returned final sequence-$\log$ prob pairs.
Denote $\sL_{t-1}^{\textnormal{seq}} \coloneqq \pi_1(\sL_{t-1})$ and $\sF^{\textnormal{seq}} \coloneqq \pi_1(\sF)$---i.e., the set of partial histories in the beam at $t-1$ and the set of full sequences at $\suflen$, respectively.
Define the set of unpruned child sequences at depth $t$ as
\[
\sC_t^{\textnormal{seq}} \;=\; \big\{\, \genseq\Vert\gentok \;:\; \genseq \in \sL^{\textnormal{seq}}_{t-1},\ \gentok\in\sS_t(\genseq)\,\big\}.
\]
Then $\sC_t^{\textnormal{seq}}$ contains no duplicates for every $t$, and the returned final sequences satisfy $\sF^{\textnormal{seq}} = \sC_\suflen^{\textnormal{seq}}$ and contain no duplicates.
\end{lemma}

\begin{proof}
Consider the map $f:(\genseq,\gentok)\mapsto \genseq\Vert\gentok$, from
$\bigcup_{\genseq\in \sL_{t-1}^{\textnormal{seq}}}\big(\{\genseq\}\times\sS_t(\genseq)\big)$ into $\vocab^{\prelen+t}$.
Given that we do not maintain histories with an EOS until possibly at $\suflen$, all $\genseq\in \sL_{t-1}^{\textnormal{seq}}$ have the same length $\prelen{+}t{-}1$. 
If $\genseq_1\neq \genseq_2$, let $i$ be the first index where they differ; then
$(\genseq_1\Vert\gentok_1)_i \neq (\genseq_2\Vert\gentok_2)_i$, so $\genseq_1\Vert\gentok_1 \neq \genseq_2\Vert\gentok_2$.
If $\genseq_1=\genseq_2$ but $\gentok_1{\neq}\gentok_2$ (which is guaranteed since $\sS_t(\genseq)$ is a set), the last token differs, so $\genseq_1\Vert\gentok_1 \neq \genseq_2\Vert\gentok_2$. 
Therefore, $f$ is injective on its domain, and $\sC_t^{\textnormal{seq}}$ has no duplicates at all steps $t$.
Since the final sequences at depth $\suflen$ are the unpruned children at $\suflen$, $\sF^{\textnormal{seq}} = \sC_\suflen^{\textnormal{seq}}$ and is also duplicate-free.
\end{proof}

Next, we show that Algorithm~\ref{app:alg:beam-topk} has a local optimality guarantee among the unique child histories (Lemma~\ref{lem:unique}) enumerated at each depth.
(Finite-width beam search need not recover the globally highest-probability sequences among all possible $k^\suflen$ paths.
Importantly, our lower bound result does not rely on global optimality.) 

\begin{proposition}[Local optimality of the beam prunes]
\label{prop:beam-local-topk}
Let $\sL_{t-1}$ be the beam at depth $t{-}1$ in Algorithm~\ref{app:alg:beam-topk}. 
Form the set $\sC_t$ of all child sequences $\genseq'$ and their respective $\log$ probabilities produced at depth $t$ prior to pruning:
\[
\sC_t \coloneqq \big\{\,(\genseq',\log p'):\ (\genseq,\log p)\in\sL_{t-1},\ \gentok\in\sS_t(\genseq),\
\genseq'=\genseq\;\Vert\;\gentok,\ \log p'=\log p+\vr_t(\genseq)[\gentok]-Z_t(\genseq)\,\big\},
\]
where $\log p$ is the running $\log$-probability for the partial history $\genseq$ with a continuation of length $t{-}1$, and $\log p'$ is the running $\log$-probability for the new partial history $\genseq'$ with a continuation of length $t$.
For $t < \suflen$, any candidate whose last token is EOS is removed from $\sC_t$ before pruning.
There are at most $\bw \cdot k$ such pairs in $\sC_t$ (exactly $k$ at $t=1$ since $|\sL_0|=1$, and at most $\bw \cdot k$ for $t \geq 2$).
After the intermediate prune to width $\bw$, the new beam $\sL_t$ consists of the $\bw$ pairs $(\genseq',\log p')$ from $\sC_t$ with the largest values of $\log p'$.
(In practice, ties are broken by a fixed, deterministic rule; such ties are extremely rare.)
At $t=\suflen$, we do not perform the final across-beam prune;
the returned set is exactly the unpruned set $\sF=\sC_\suflen$, which contains the top-$\bw$ pairs by $\log p'$ as a subset.\looseness=-1

\end{proposition}

\begin{proof}
By construction, at each depth $t < \suflen$ the algorithm enumerates all children $\sC_t$ of the current beam $\sL_{t-1}$ (removing EOS candidates), sorts $\sC_t$ by descending $\log p'$, and keeps the top $\bw$ as the set $\sU_t$.
Therefore, $\sL_t = \sU_t$ is exactly the set of the $\bw$ pairs in $\sC_t$ with the largest $\log p'$.
At $t=\suflen$, no prune is applied: $\sF = \sC_\suflen$ contains up to $\bw \cdot k$ candidates, which trivially includes whatever the top-$\bw$ subset would have been.
\end{proof}

Next, we show that the per-step $\log$ probability that Algorithm~\ref{app:alg:beam-topk} accumulates for each sequence is equivalent to that sequence's top-$k$ decoding $\log$ probability. 

\begin{lemma}[Per-step $\log$ probability equals top-$k$ decoding $\log$ probability]
\label{lem:topk-step-update}
For any history $\genseq$---consisting of the prefix $\pre$ and tokens generated so far---and any token $\gentok\in\sS_t(\genseq)$,
\[
\vr_t(\genseq)[\gentok]-Z_t(\genseq)=\log \Pr_{\model,\dec}(\gentok\mid\genseq),
\]
where $\dec$ is top-$k$ renormalization, i.e., 
\[
\Pr_{\model,\dec}(\gentok \mid \genseq)
\;=\;
\frac{\Pr_{\model}(\gentok \mid \genseq)}{\sum_{u\in \sS_t(\genseq)} \Pr_{\model}(u \mid \genseq)}
\quad\text{for }\gentok\in \sS_t(\genseq),
\qquad
\Pr_{\model,\dec}(v \mid \genseq)=0 \text{ for } v\notin \sS_t(\genseq).
\]
\end{lemma}

\begin{proof}
By definition, 
\begin{align*}
\vr_t(\genseq)=\log \softmaxsf(\logit_t(\genseq)) = \log \Pr_{\model}(\cdot | \genseq).
\end{align*}
So, 
\begin{align*}
\exp\big(\vr_t(\genseq)[v]\big)=\Pr_{\model}(v \mid \genseq) \quad \forall v\in\vocab,
\end{align*}
and 
\begin{align*}
Z_t(\genseq)  =\log\!\!\!\sum_{u\in \sS_t(\genseq)} \exp\big(\vr_t(\genseq)[u]\big) =\log\!\!\!\sum_{u\in \sS_t(\genseq)} \Pr_{\model}(u \mid \genseq). 
\end{align*}
Therefore, for any $\gentok\in \sS_t(\genseq)$,
\begin{align*}
\vr_t(\genseq)[\gentok]-Z_t(\genseq)
&= \log \Pr_{\model}(\gentok \mid \genseq) - \log\!\!\!\sum_{u\in \sS_t(\genseq)} \Pr_{\model}(u \mid \genseq)\\
&= \log \frac{\Pr_{\model}(\gentok \mid \genseq)}{\sum_{u\in \sS_t(\genseq)} \Pr_{\model}(u \mid \genseq)}\\
&= \log \Pr_{\model,\dec}(\gentok \mid \genseq). \qedhere
\end{align*}
\end{proof}

\begin{corollary}[Path $\log$-probability equals sum of top-$k$ increments]
\label{cor:topk-path}
Let $\pre\in\vocab^{\prelen}$ be a prefix of length $\prelen$, and let
$\gensuf = (\gentok_1,\ldots,\gentok_{\suflen})$
be any length-$\suflen$ continuation produced by Algorithm~\ref{app:alg:beam-topk}.
For $t=1,\ldots,\suflen$, define the step-$t$ history
\(
\genseq \;\coloneqq\; \pre \,\Vert\, \subgensuf{1:t-1} =\genseq_{1:\,\prelen+t-1}
\) 
Let $\vr_t(\genseq)\in\R^{|\vocab|}$ denote the $\log$-probability vector with entries $\vr_t(\genseq)[u] \coloneqq \log \Pr_{\model}(u\mid \genseq)$.
Let $\sS_t(\genseq)$ be the step-$t$ top-$k$ set and
\[
Z_t(\genseq) \;\coloneqq\; \LogSumExp\big(\,\vr_t(\genseq)[\sS_t(\genseq)]\,\big)
\;=\; \log\!\!\!\sum_{u\in\sS_t(\genseq)} \exp\big(\vr_t(\genseq)[u]\big).
\]
Then the $\log$-probability of $\gensuf$ under the top-$k$ policy $\dec$ decomposes as
\[
\log \Pr_{\model,\dec}\big(\gensuf \mid \pre\big)
\;=\;
\sum_{t=1}^{\suflen} \Big(\vr_t(\genseq)[\gentok_t] \;-\; Z_t(\genseq)\Big).
\]
\end{corollary}

\begin{proof}
By Lemma~\ref{lem:topk-step-update}, at step $t$ the $\log$-increment from choosing
$\gentok_t\in\sS_t(\genseq)$ given history $\genseq$ is
$\vr_t(\genseq)[\gentok_t]-Z_t(\genseq)=\log\Pr_{\model,\dec}(\gentok_t\mid \genseq)$.
Summing over $t=1,\ldots,\suflen$ yields
$\log \Pr_{\model,\dec}(\gensuf\mid\pre)
=\sum_{t=1}^{\suflen}\log\Pr_{\model,\dec}(\gentok_t\mid \genseq)$,
the standard autoregressive factorization.
\end{proof}

Following from above, we can show that the entire mass is conserved under top-$k$ normalization.
We use this below to show that final-step mass is preserved when the last pruning step is disabled.\looseness=-1

\begin{lemma}[Frontier mass identity under top-$k$ renormalization]
\label{lem:frontier-mass}
Fix a prefix $\pre$ and decoding policy $\dec$ (top-$k$).
Let $\vu$ be a partial continuation of length at most $\suflen$, and for brevity let
\[
\Pr_{\model,\dec}(\vu)\;\coloneqq\;\Pr_{\model,\dec}(\vu \mid \pre)
\]
denote its $\dec$-computed conditional probability.
At any step $t<\suflen$, the children of a length-$t$ continuation $\vu$ are the one-token extensions
\[
\{\,\vu\Vert v:\ v\in \sS_{t+1}(\pre\Vert \vu)\,\}
\]
where $\sS_{t+1}(\pre\Vert \vu)$ is the top-$k$ set at step $t{+}1$ given history $\pre\Vert \vu$.

A set $\sA$ of \emph{pairs} $(\vw,\log p)$---$\vw$ is a partial continuation and $\log p=\log \Pr_{\model,\dec}(\vw\mid\pre)$---is a \newterm{frontier} if the sequence components form a cut: 
(i) no $\vw$ is a prefix of another (i.e., no ancestor-descendant pairs), and 
(ii) every length-$\suflen$ continuation has exactly one $\vw\in\sA$ as a prefix (i.e., initial segment of the continuation). 
Then the total mass of any frontier equals $1$:
\[
\sum_{(\_,\log p)\in\sA} \exp(\log p) \;=\; 1.
\]

\paragraph{Returned finals and pruned nodes.}
Let $\sF$ be the set of returned finals and their $\log$ probabilities, pairs $(\gensuf,\log p)$ at depth $\suflen$.
Let $\sR_{\textnormal{prune}}$ be the set of all pruned nodes and their $\log$ probabilities: 
these are pairs $(\vu,\log p)$ for any depth at which the node $\vu$ (and its $\log$ probability) was removed by pruning. 
(By construction, once a node is pruned, none of its descendants are ever constructed, so they cannot be pruned later.  
Siblings and nodes on different branches can both be in $\sR_{\textnormal{prune}}$, but there are no ancestor-descendant pairs inside $\sR_{\textnormal{prune}}$.)

The returned set $\sF$ contains all depth-$\suflen$ children that survived earlier pruning, so $\sR_{\textnormal{prune}}$ contains prunes only from depths $<\suflen$. 
Then $\sF\cup\sR_{\textnormal{prune}}$ is a frontier and
\[
\underbrace{\sum_{(\_,\log p)\in\sF} \exp(\log p)}_{\text{\emph{computed from} } \gensuf}
\;+\;
\underbrace{\sum_{(\vu,\log p)\in\sR_{\textnormal{prune}}} \exp(\log p)}_{\text{\emph{leftover computed from }} \vu}
\;=\; 1.
\]
That is, $\sF\cup\sR_{\textnormal{prune}}$ is a frontier/cut because each root-to-$\suflen$ path intersects it exactly once---either it is kept as a final in $\sF$ or at the first node on that path that was pruned in $\sR_{\textnormal{prune}}$. 
\end{lemma}

\begin{proof}
For any length-$t$ continuation $\vu$ with $t<\suflen$, top-$k$ renormalization ensures that the conditional probabilities of the next token sum to $1$:
\[
\sum_{v\in\sS_{t+1}(\pre\Vert \vu)} \Pr_{\model,\dec}(v \;\big|\; \pre\Vert \vu)=1.
\]
Combining this with the autoregressive chain rule,
\[\Pr_{\model,\dec}(\vu\Vert v \mid \pre) = \Pr_{\model,\dec}(\vu \mid \pre)\cdot\Pr_{\model,\dec}(v \mid \pre\Vert \vu),
\]
we get one-step conservation:
\begin{align}
\label{eq:onestep}
\sum_{v\in\sS_{t+1}(\pre\Vert \vu)} \Pr_{\model,\dec}(\vu\Vert v \;\big|\; \pre)
&=\sum_{v\in\sS_{t+1}(\pre\Vert \vu)} \Pr_{\model,\dec}(\vu \;\big|\; \pre)\,\Pr_{\model,\dec}(v \;\big|\; \pre\Vert \vu) \nonumber\\
&= \Pr_{\model,\dec}(\vu \;\big|\; \pre)\cdot\underbrace{\sum_{v\in\sS_{t+1}(\pre\Vert \vu)}\Pr_{\model,\dec}(v \;\big|\; \pre\Vert \vu)}_{=\,1}
\;=\; \Pr_{\model,\dec}(\vu \;\big|\; \pre),
\end{align}
i.e., the mass of a parent node $\vu$ equals the total mass of its $k$ children (before pruning).

Now let $\mathrm{Desc}_\suflen(\vw) = \{\vx\in\vocab^\suflen :\vw \text{ is a prefix of }\vx\}$, where here ``prefix of'' refers to an initial segment of the continuation, not the input prefix $\pre$).
This is the set of depth-$\suflen$ descendants of $\vw$ (the full-length-$\suflen$ continuations beginning with $\vw$) under the top-$k$ policy. 
By repeatedly applying the one-step conservation along the subtree rooted at $\vw$ (i.e., push mass all the way down to depth $\suflen$), we get the \newterm{descendant sum identity}:
\begin{equation}
\label{eq:descendant-sum}
\sum_{\vx\in \mathrm{Desc}_\suflen(\vw)} \Pr_{\model,\dec}(\vx \;\big|\; \pre)
= \Pr_{\model,\dec}(\vw\;\big|\;\pre).
\end{equation}
Simply put, this is per-node conservation: 
the probabilities of all depth-$\suflen$ descendants of node $\vw$ (i.e., the subtree with $\vw$ as the root) sum to the probability of $\vw$ under top-$k$.
By induction on depth: the base case is Equation~\ref{eq:onestep} (one-step conservation); for the inductive step, apply one-step conservation to each leaf of the current frontier of the subtree, replacing each leaf's mass with the sum of its children's masses, until all leaves are at depth $\suflen$.

If $\sA$ is a frontier, the sets of descendants $\{\mathrm{Desc}_\suflen(\vw): (\vw,\log p)\in\sA\}$ for different $\vw$ are disjoint and their union is the full set of all depth-$\suflen$ continuations (every path hits the frontier). 
From per-node conservation above (descendant sum identity, Equation~\ref{eq:descendant-sum}), summing over $(\vw,\log p)\in\sA$ gives
\begin{align*}
\sum_{(\_,\log p)\in\sA} \exp(\log p)
&= \sum_{(\vw,\log p)\in\sA} \Pr_{\model,\dec}(\vw\;\big|\;\pre) \\
&= \sum_{(\vw,\log p)\in\sA}\;\sum_{\vx\in\mathrm{Desc}_\suflen(\vw)} \Pr_{\model,\dec}(\vx\;\big|\;\pre)
\tag{descendant sum, Eq.~\ref{eq:descendant-sum}} \\
&= \sum_{\text{all }\vx\text{ of length }\suflen} \Pr_{\model,\dec}(\vx\;\big|\;\pre)
\tag{disjoint partition} \\
&= 1.
\tag{descendant sum at the root, which has mass $1$}
\end{align*}
Finally, by construction, every root-to-length-$\suflen$ path ends either at a returned final in $\sF$ or at its first pruned ancestor in $\sR_{\textnormal{prune}}$ (the step where it left the beam), and the identity follows.
\end{proof}

Since Lemma~\ref{lem:frontier-mass} holds for any frontier---regardless of the depths at which its nodes appear---it holds in particular when the frontier includes all depth-$\suflen$ nodes that survived pruning.
At $\suflen$, we simply keep more in $\sF$ rather than allocating mass to $\sR_{\textnormal{prune}}$, since we do not perform the final across-beam prune.
This is also why we can track the mass of partial paths before $\suflen$ when we hit an EOS;
we can effectively treat those paths as parents that contain the mass of children that we do not explore further.

With the conservation of mass, per-step probabilities, and overall path probabilities shown above, we quantify how probability mass behaves at the final expansion when the last pruning operation is disabled.

\begin{lemma}[Final-step mass preservation without pruning]
\label{lem:mass-preservation}
Let $\sL_{\suflen-1} \subseteq \vocab^{\prelen+\suflen-1}\times\R$ be the beam at depth $\suflen{-}1$ in Algorithm~\ref{app:alg:beam-topk}, and let
\[
\sC_\suflen \;=\; \big\{\,(\genseq\,\Vert\,\gentok,\;\log p+\vr_\suflen(\genseq)[\gentok]-Z_\suflen(\genseq)) \;:\; (\genseq,\log p)\in\sL_{\suflen-1},\ \gentok\in\sS_\suflen(\genseq)\,\big\}
\]
be the set of up to $\bw \cdot k$ unpruned children at depth $\suflen$ paired with their associated $\log$ probabilities. 
With the final prune disabled, the total probability associated with the returned final sequences equals the beam mass at depth $\suflen{-}1$:
\[
\sum_{(\_,\log p')\in\sC_\suflen} \exp(\log p')
\;=\;
\sum_{(\_,\log p)\in\sL_{\suflen-1}} \exp(\log p).
\]
\end{lemma}

\begin{proof}
By Lemma~\ref{lem:topk-step-update} (and Corollary~\ref{cor:topk-path}), each child from parent $(\genseq,\log p)$ has
\[
\exp(\log p')=\exp(\log p)\,\cdot\,\exp(\vr_\suflen(\genseq)[\gentok]-Z_\suflen(\genseq))
=\exp(\log p)\,\cdot\,\Pr_{\model,\dec}(\gentok\mid\genseq).
\]
Therefore,
\[
\sum_{(\_,\log p')\in\sC_\suflen} \exp(\log p')
=\sum_{(\genseq,\log p)\in\sL_{\suflen-1}} \bigg[\exp(\log p)\,\cdot\,
\sum_{\gentok\in\sS_\suflen(\genseq)} \Pr_{\model,\dec}(\gentok\mid\genseq)\bigg]
=\sum_{(\genseq,\log p)\in\sL_{\suflen-1}} \exp(\log p),
\]
since $\sum_{\gentok\in\sS_\suflen(\genseq)} \Pr_{\model,\dec}(\gentok\mid\genseq)=1$ by top-$k$ renormalization.
\end{proof}

As an immediate consequence, we capture more mass by not pruning at step $\suflen$. 

\begin{corollary}[Pruned vs. unpruned final sequences]
\label{cor:prune-inequality}
Let $\sC_\suflen$ be the unpruned set of up to $\bw \cdot k$ child sequences and their $\log$ probabilities at depth $\suflen$. 
Let $\sF_{(\bw)}\subseteq\sC_\suflen$ be the $\bw$ final sequences and their $\log$ probabilities that would be kept if we performed the last prune.
Then
\[
\sum_{(\_,\log p')\in\sF_{(\bw)}} \exp(\log p')
\;\le\;
\sum_{(\_,\log p')\in\sC_\suflen} \exp(\log p')
\;=\;
\sum_{(\_,\log p)\in\sL_{\suflen-1}} \exp(\log p).
\]
\end{corollary}

\begin{proof}
Since $\sF_{(\bw)}\subseteq\sC_\suflen$ and all terms are non-negative, the first inequality is monotonicity of finite sums.
The second equality is by Lemma~\ref{lem:mass-preservation}.
\end{proof}

From the above, we now turn from probability mass accounting to the lower bound of the near-verbatim probability: 
summing probabilities of \emph{any} subset of the $\varepsilon$-ball yields a certified lower bound by monotonicity.

\begin{theorem}[Lower bound on near-verbatim mass]
\label{thm:beam-lower-bound}
We denote the $\varepsilon$-ball around the $\suflen$-length target suffix $\suf$ for distance metric $\dist\in \{\ham, \lev\}$ as $\ball(\suf)=\{\vv\in\vocab^{\suflen}:\dist(\vv,\suf)\le\varepsilon\}$ (Equation~\ref{eq:eball}). Let $\sF_{(\bw)}$ be the set of up to $\bw$ pairs of
final $\suflen$-length continuations and their $\log$ probabilities $(\gensuf,\log p)$ that would be returned by Algorithm~\ref{app:alg:beam-topk} with the final prune enabled. 
Define the beam-based estimated mass lower bound to be 
\[
\mathrm{LB}^{(\bw)}_{\varepsilon,\dist}
\;\coloneqq\;
\sum_{(\gensuf,\,\log p)\in \sF_{(\bw)}\,:\ \gensuf \in \ball(\suf)} \exp (\log p).
\]
Then
\[
\mathrm{LB}^{(\bw)}_{\varepsilon,\dist}
\;\le\;
\sum_{\vv\in\ball(\suf)} \Pr_{\model,\dec}(\vv\mid\pre)
\;=\;
\pseqe \quad \text{(by Equation~\ref{eq:pze}).}
\]
\end{theorem}

\begin{proof}
By Lemma~\ref{lem:topk-step-update} and Corollary~\ref{cor:topk-path},
for each $(\gensuf,\log p)\in\sF_{(\bw)}$ we have $\exp(\log p)=\Pr_{\model,\dec}(\gensuf\mid\pre)$.
So 
\[
\mathrm{LB}^{(\bw)}_{\varepsilon,\dist}
=\sum_{(\gensuf,\,\log p)\in \sF_{(\bw)}\,:\ \gensuf \in \ball(\suf)} \Pr_{\model,\dec}(\gensuf\mid\pre).
\]
By Lemma~\ref{lem:unique}, each sequence appears in $\sF_{(\bw)}$ exactly once (no duplicates), so 
$\{\,\gensuf : (\gensuf,\log p)\in\sF_{(\bw)} \text{ and } \gensuf\in\ball(\suf)\,\} \subseteq \ball(\suf)$. 
Monotonicity of finite sums then yields
\[
\mathrm{LB}^{(\bw)}_{\varepsilon,\dist}
\le \sum_{\vv\in\ball(\suf)} \Pr_{\model,\dec}(\vv\mid\pre)
= \pseqe\ \quad \text{(Equation~\ref{eq:pze}).}
\]
\end{proof}

The same monotonicity argument applies if we keep up to $\bw \cdot k$ final pairs of continuations and their $\log$ probs at depth $\suflen$ (i.e., no final across-beam prune to size $\bw$).

\begin{corollary}[Lower bound for the no-final-prune]
\label{cor:beam-lower-bound-noprune}
Let $\sC_\suflen$ be the set of up to $\bw \cdot k$ child pairs $(\gensuf,\,\log p)$ at depth $\suflen$ before the final prune to $\bw$ sequences (as in Proposition~\ref{prop:beam-local-topk} with $t=\suflen$). 
Define
\[
\mathrm{LB}^{(\bw \cdot k)}_{\varepsilon,\dist}
\;\coloneqq\;
\sum_{(\gensuf,\,\log p)\in \sC_\suflen\,:\ \gensuf\in \ball(\suf)} \exp(\log p).
\]
Then 
\[
\mathrm{LB}^{(\bw \cdot k)}_{\varepsilon,\dist}
\;\le\;
\sum_{\vv\in\ball(\suf)} \Pr_{\model,\dec}(\vv\mid\pre)
\;=\;
\pseqe\ \; \text{(Equation~\ref{eq:pze})},
\quad\text{and}\quad
\mathrm{LB}^{(\bw)}_{\varepsilon,\dist}
\;\le\;
\mathrm{LB}^{(\bw \cdot k)}_{\varepsilon,\dist}.
\]
\end{corollary}

\begin{proof}
By Lemma~\ref{lem:unique}, $\sC_\suflen$ has no token-level duplicates, so Theorem~\ref{thm:beam-lower-bound} applies exactly with $\sF_{(\bw)}$ replaced by $\sC_\suflen$, yielding $\mathrm{LB}^{(\bw \cdot k)}_{\varepsilon,\dist}\le \pseqe$. 
Since $\sF_{(\bw)}\subseteq\sC_\suflen$ (final prune keeps the top $\bw$ from the $\bw \cdot k$ candidates in $\sC_\suflen$), we have
$\{\gensuf:(\gensuf,\log p)\in\sF_{(\bw)},\ \gensuf\in\ball(\suf)\}\subseteq
\{\gensuf:(\gensuf,\log p)\in\sC_\suflen,\ \gensuf\in\ball(\suf)\}$. 
So, by monotonicity of finite sums, $\mathrm{LB}^{(\bw)}_{\varepsilon,\dist}\le \mathrm{LB}^{(\bw \cdot k)}_{\varepsilon,\dist}$. 
Lemma~\ref{lem:mass-preservation} and Corollary~\ref{cor:prune-inequality} additionally show that total mass (not just near-verbatim mass) is preserved when the final prune is disabled. 
\end{proof}

We can also show that Algorithm~\ref{app:alg:beam-topk} produces a loose (almost always very loose) upper bound from the covered mass it returns. 

\begin{proposition}[Loose upper bound from covered mass]
\label{prop:loose-ub}
Let $\mathrm{LB}_{\varepsilon,\dist}$ be the beam-based lower bound from Theorem~\ref{thm:beam-lower-bound} or Corollary~\ref{cor:beam-lower-bound-noprune} (the superscript $(\bw)$ or $(\bw \cdot k)$ is suppressed since the result holds for either), and (as in Equation~\ref{eq:covered-mass}) define the (unfiltered) covered mass
\[
\cover(\sF) \;\coloneqq\; \sum_{(\_,\log p)\in\sF} \exp(\log p), 
\]
and define the upper bound 
\[
\mathrm{UB}_{\varepsilon,\dist} \;\coloneqq\; \mathrm{LB}_{\varepsilon,\dist} \;+\; (1-\cover(\sF)).
\]
Then for any distance $\dist\in\{\ham,\lev\}$, the near-verbatim mass satisfies 
\[
\pseqe \;\le\; \mathrm{UB}_{\varepsilon,\dist}.
\]
\end{proposition}

\begin{proof}
By Lemma~\ref{lem:frontier-mass}, for the frontier $\sF\cup\sR_{\textnormal{prune}}$ we have
\begin{align}
\label{eq:frontier-one}
1 \;=\; \underbrace{\sum_{(\_,\log p)\in\sF} \exp(\log p)}_{\cover(\sF)}
\;+\; \sum_{(\_,\log p)\in\sR_{\textnormal{prune}}} \exp(\log p).
\end{align}
As in Lemma~\ref{lem:frontier-mass}, let $\mathrm{Desc}_\suflen(\vu)$ denote the set of depth-$\suflen$ descendants of $\vu$ (i.e., full-length continuations with prefix $\vu$; $\vu$ is a prefix of all these continuations).  
Then, we can decompose $\pseqe$ into contributions from the returned finals and pruned subtrees: 
\begin{align*}
\pseqe
&= \underbrace{\sum_{(\gensuf,\log p)\in\sF} \exp(\log p)\,\1[\gensuf\in\ball(\suf)]}_{=\;\lb}\\
  &\qquad \;+\;
  \sum_{(\vu,\log p)\in\sR_{\textnormal{prune}}}\;\sum_{\vx \in \mathrm{Desc}_\suflen(\vu)}
     \Pr_{\model,\dec}(\vx\mid\pre)\,\1[\vx\in\ball(\suf)].
\end{align*}
Dropping the indicator in the second term yields an upper bound (i.e., allows for the possibility that each $\vx$ is a valid suffix in $\ball$):
\begin{align*}
\sum_{(\vu,\log p)\in\sR_{\textnormal{prune}}}\;\sum_{\vx \in \mathrm{Desc}_\suflen(\vu)}
     \Pr_{\model,\dec}(\vx\mid\pre)
&\;=\; \sum_{(\vu,\log p)\in\sR_{\textnormal{prune}}} \Pr_{\model,\dec}(\vu\mid\pre)\\
&\;=\; \sum_{(\_,\log p)\in\sR_{\textnormal{prune}}} \exp(\log p)\\
&\;=\; 1-\cover(\sF),
\end{align*}
where the first equality uses Equation \ref{eq:descendant-sum} with $\vw\leftarrow\vu$, summed over $(\vu,\log p)\in\sR_{\textnormal{prune}}$; 
the second equality is by definition; and the third equality is by Equation~\ref{eq:frontier-one}.
Therefore, $\pseqe\le \lb + (1-\cover(\sF)) = \ub$.
In Algorithm~\ref{app:alg:beam-topk}, $\sF = \sC_\suflen$ (all $\bw \cdot k$ finals), but the reasoning applies also if $\sF$ contains only the top-$\bw$ finals retained by pruning. 
\end{proof}

\paragraph{Validity of early termination.}
The optional early termination in Algorithm~\ref{app:alg:beam-topk} is justified by the results above.
At any intermediate step $t < \suflen$, cumulative $\log$-probabilities are non-increasing along paths (each step adds $\log \Pr_{\model,\dec}(\gentok \mid \genseq) \leq 0$), so the highest-probability beam element in $\sL_t$ upper-bounds the probability of every depth-$\suflen$ descendant.
By Theorem~\ref{thm:beam-lower-bound} and Corollary~\ref{cor:beam-lower-bound-noprune}, the lower bound $\mathrm{LB}^{(\bw \cdot k)}_{\varepsilon,\dist}$ is a sum of at most $\bw \cdot k$ such descendant probabilities, each bounded by $\max_{(\_, \log p) \in \sL_t} \exp(\log p)$.
Therefore, if $\bw \cdot k \cdot \max_{(\_, \log p) \in \sL_t} \exp(\log p) < \taumin$, then $\mathrm{LB}^{(\bw \cdot k)}_{\varepsilon,\dist} < \taumin$ regardless of the distance metric $\dist$ or threshold $\varepsilon$.
This means that $k$-CBS does not identify mass that indicates the sequence is extractable at level $\taumin$.
(Since $k$-CBS is a greedy search, it may miss such mass, so this is not a guarantee, but nevertheless reflects what occurs in practice.) 
In this case, the algorithm returns $\sF = \emptyset$ (yielding $\mathrm{LB} = 0$) without completing the remaining $\suflen - t$ steps.

\subsection{Nucleus sampling constrained beam search ($p$-CBS)}\label{app:sec:other:nucleus}

Throughout, we discuss all of our algorithms in relation to top-$k$ decoding. 
We note that we could also incorporate temperature into our update rule (Appendix~\ref{app:sec:kcbs:notation}), but choose to focus on $\temp=1$.
This reflects the base distribution of the LLM $\model$, which we (and prior work, namely~\citet{hayes2025measuringmemorizationlanguagemodels} and \citet{cooper2025books}) believe is a useful regime to study for extraction and memorization. 
However, any probability-distribution-preserving scoring rule could also be incorporated into our constrained beam search algorithm, such as \newterm{nucleus (top-$p$) sampling}~\citep{holtzman2020curiouscaseneuraltext}.
(We could also implement pruning with this change in scoring rule, though we do not discuss this in detail and instead talk about the pruned variants of $k$-CBS that we run in Section~\ref{sec:pruning} and Appendix~\ref{app:sec:prune}.) 

This approach is biased in the same sense as our top-$k$ constrained algorithms (deterministic and correct, but downward biased lower bound). 
The scoring rule is changed to be valid under the truncated and renormalized nucleus chain. 
We refer to this algorithm as \newterm{nucleus-constrained beam search ($p$-CBS)}. 

\paragraph{Setup and notation.}
Fix a prefix $\pre$, a suffix length $\suflen \in \sN$, and an EOS policy identical to the baseline (remove from the candidate set prior to across-beam pruning at steps $1{:}\suflen-1$, allow at $t=\suflen$). 
Given a beam element with current history $\genseq$ (which includes the $\prelen$-length prefix) at step $t$, let $\logit_t(\genseq)\in\R^{|\vocab|}$ be the logits from $\model$ and let
\[
  p_t(u \mid \genseq) \;\coloneqq\; \softmaxsf(\logit)[u]
\]
denote the base next-token probability over $\vocab$. 
Ties are broken by a fixed deterministic rule throughout.
For consistency with our notation for top-$k$ sampling (Equation~\ref{eq:topk}), below we write $\Pr(u \mid \genseq)$ for the same base conditional, i.e., $\Pr(u \mid \genseq) \equiv p_t(u \mid \genseq)$;
we introduce this notation only to emphasize $t$. 

\paragraph{Nucleus (top-$p$) sampling.}
Given a threshold $p \in (0,1]$, nucleus sampling retains the smallest set of next tokens whose total base probability covers at least a $p$ fraction of the mass; 
$p$ acts as a coverage knob (small $p$ narrows support, $p=1$ leaves the support unchanged from the base distribution).

At step $t$ for a parent history $\genseq$, consider the ordering of the entire vocabulary by the base probabilities conditioned on $\genseq$.
Let $\sigma_t(\genseq)$ be a permutation of $\vocab$ that lists tokens in nonincreasing order of $\Pr(\cdot \mid \genseq)$, so $\sigma_t(\genseq)[1]$ is the highest-probability token given $\genseq$, $\sigma_t(\genseq)[2]$ is the second-highest, etc.
The permutation $\sigma_t(\genseq)$ is recomputed per parent and per step because the base probabilities $\Pr(\cdot\mid\genseq)$ depend on the parent sequence. 
Define the smallest rank index at step $t$ 
\[
  R_t(\genseq;p) \;\coloneqq\; 
  \min\!\Big\{ r \in \{1,\dots,|\vocab|\} \;:\; \sum_{i=1}^{r} \Pr\!\big(\sigma_t(\genseq)[i]\mid \genseq\big) \;\ge\; p \Big\}.
\]
This makes the \newterm{nucleus set} 
\begin{align}
  \sC_t(\genseq;p)
  \;\triangleq\;
  \big\{\, \sigma_t(\genseq)[i] \;:\; 1 \le i \le R_t(\genseq;p) \,\big\}
\end{align}
uniquely defined.
(We slightly overload the notation $\sC_t$ from Appendix~\ref{app:sec:kcbs:notation}, disambiguating here with $p$ as an argument.)
The corresponding per-step nucleus normalization constant is
\begin{align}
  Z_t(\genseq;p) \;\triangleq\; \sum_{u \in \sC_t(\genseq;p)} \Pr(u \mid \genseq).
\end{align}
Analogous to Equation~\ref{eq:topk}, define the per-step, renormalized nucleus probability by
\begin{align}
\label{eq:nucleus-pr}
  \Pr_{p}(u \mid \genseq)
  \;\coloneqq\;
  \begin{cases}
    \dfrac{\Pr(u \mid \genseq)}{Z_t(\genseq;p)}, & u \in \sC_t(\genseq;p),\\[0.6em]
    0, & u \notin \sC_t(\genseq;p).
  \end{cases}
\end{align}

\paragraph{Nucleus-constrained beam search.}
A baseline approach for $p$-CBS is similar to the baseline $k$-CBS (Algorithm~\ref{app:alg:beam-topk}), with two substitutions:
(i) replace the fixed top-$k$ rule by the dynamic nucleus rule $u \in \sC_t(\genseq;p)$; 
(ii) replace the scoring increment with $\log \Pr_{p}(u \mid \genseq)$.
We need 
\begin{itemize}[leftmargin=.5cm]
  \item \textbf{Nucleus mass parameter} $p\!\in\!(0,1]$ (defines $\sC_t(\genseq;p)$).
  \item \textbf{Beam cap} $\bw$ applied after pooling children at each step $t<\suflen$.
  Setting $\bw=\infty$ yields the variable-width beam (keep all children) at each step. 
\end{itemize} 

Then, at each step $t$ and for each beam prefix $\genseq$:
\begin{enumerate}[leftmargin=.5cm]
  \item \textbf{Child set.} 
  Expand $\genseq$ only to tokens $u\in\sC_t(\genseq;p)$.
  \item \textbf{Scoring rule.} 
  For a child $\genseq\Vert u$, update its cumulative $\log$-score by adding
  \begin{align}
    \log \Pr_{p}(u \mid \genseq)
    \;=\; \log \Pr(u \mid \genseq) \;-\; \log Z_t(\genseq;p).
  \end{align} 
  And so, a length-$t$ partial path has cumulative score $\sum_{s=1}^{t} \log \Pr_{p}(\gentok_s \mid \pre \Vert \subgensuf{1:s-1})$.
  \item \textbf{Pruning.} 
  Pool all children from the current beam and (aside from the same EOS policy as $k$-CBS) retain them according to one of the following rules:\looseness=-1
    \begin{itemize}[leftmargin=.5cm]
      \item \emph{Variable-width beam (full nucleus expansion).} 
      Keep \emph{all} children $\bigcup_{\genseq \in \text{beam}} \sC_t(\genseq;p)$; 
      in variable-width mode ($\bw=\infty$), the step-$t$ beam cardinality is 
      \[
        \sum_{\genseq\ \text{in the step-}(t-1)\ \text{beam}} \big|\,\sC_t(\genseq;p)\,\big|.
      \]
      (This can be quite expensive, depending on $|\sC_t(\genseq;p)|$; see discussion below about computational considerations.)
      \item \emph{Capped beam.} 
      Keep the \emph{top} $\bw$ children by cumulative $\log$-score, exactly as in Algorithm~\ref{app:alg:beam-topk}.
    \end{itemize}
\end{enumerate}

Within a fixed parent $\genseq$, the term $-\log Z_t(\genseq;p)$ is constant across its children, so that parent's local ranking can use $\log \Pr(u \mid \genseq)$.
When pooling children across \emph{different} parents, $Z_t(\genseq;p)$ varies with $\genseq$; subtracting $\log Z_t(\genseq;p)$ aligns scales so that global sorting reflects $\log \Pr_{p}(\cdot \mid \cdot)$.
(The same cross-parent adjustment appears in the $k$-CBS baseline, using its per-step top-$k$ normalizer.)\looseness=-1

\paragraph{Returned results and bias.}
Let $\sF \subseteq \vocab^{\suflen} \times \R$ be the set of continuations and their $\log$ probabilities returned by $p$-CBS.
We return all step-$\suflen$ children generated from the step-$(\suflen-1)$ beam:
\[
|\sF| \;=\; \sum_{\genseq\ \text{in the step-}(\suflen-1)\ \text{beam}} \big|\,\sC_{\suflen}(\genseq;p)\,\big|.
\]
The covered mass is
\begin{align}
  \mathrm{covered\_mass}_{p}(\sF)
  \;\coloneqq\;
  \sum_{(\gensuf,\, \log p) \in \sF}
  \prod_{t=1}^{\suflen} \Pr_{p}\!\big(\gentok_t \,\big|\, \pre\Vert\subgensuf{1:t-1}\big),
\end{align}
where $\subgensuf{1:0}$ is the empty token sequence. 
Because across-beam pruning discards feasible paths, the covered mass is a downward-biased (but correct) lower bound on the total mass under the nucleus filtering rule (directly analogous to the top-$k$ baseline).

\paragraph{Computational considerations (adaptive branching and mitigation).}
The size of $\sC_t(\genseq;p)$ depends on the shape of the base distribution $\Pr(\cdot \mid \genseq)$ at step $t$: 
when the distribution is sharp, few tokens are needed to reach total mass $p$; when it is flatter (many tokens with similar probability), $\sC_t(\genseq;p)$ can be very large, as opposed to the constant size in $k$-CBS.

\paragraph{Why we do not explore $p$-CBS further in this work.}
We defer this to future work, as we find that top-$k$ is sufficient for showing the value of our method.
This choice also aligns with the main configurations reported by prior work on probabilistic extraction~\citep{hayes2025measuringmemorizationlanguagemodels, cooper2025books}.

\paragraph{Extensions mirror the top-$k$ variants.}
Exactly as for the top-$k$ constrained beams, one can (i) bake in a near-verbatim distance budget at decode time (an $\varepsilon$-viable filtering rule), and/or (ii) add tail closure beyond the beam. 
We omit these details for brevity.

\vspace{-.1cm}
\section{$\varepsilon$-pruned $k$-CBS}\label{app:sec:prune}
\vspace{-.1cm}

The baseline $k$-CBS algorithm (Algorithm~\ref{app:alg:beam-topk}) is a good starting place, but we can potentially do better in terms of compute cost.
Rather than post-processing outputs to respect a chosen distance metric and budget $\varepsilon$, we can bake them into the search procedure and prune the beam as the algorithm evolves to only explore nodes that respect these settings.
This only involves additional bookkeeping; 
it is free, with respect to the dominating cost of forward passes.
In doing so, we can also possibly achieve a better upper bound (beyond the trivial one that baseline $k$-CBS algorithm provides) and a potentially better lower bound, as we can concentrate the search on $\varepsilon$-viable nodes only.
The trade-off is that we do not get to run one algorithm, and then post-process for whichever $\dist$ and $\varepsilon$ budget we might want to use;
we have to specify these up front.
But, an added benefit is that we will terminate the search procedure early (potentially very early) if no $\varepsilon$-viable continuations remain in the beam.

We describe these improvements below, first for the Hamming (Appendix~\ref{app:sec:prune:hamming}) and then for the Levenshtein distance (Appendix~\ref{app:sec:prune:lev}). 
The former is significantly simpler, as Hamming is monotone non-decreasing and the Levenshtein is not. 
While both methods require additional bookkeeping (and therefore slight additional overhead compared to the baseline algorithm), they also can both terminate early  if there are no continuations left in the beam that satisfy the chosen distance-based pruning rule. 
Overall, this can make them significantly cheaper to run than $k$-CBS, in terms of overall wall-clock time (Appendix~\ref{app:sec:intuition:cost}). 
We provide unified proofs of the invariants of both algorithms in Appendix~\ref{app:sec:prune:invariants}. 

\paragraph{Note.} 
This approach does not provably provide a stricter lower bound, because baking $\varepsilon$-viability into the pruning rule changes the candidate set of children that get expanded.
In particular, a child that was not part of the baseline $k$-CBS candidate set at a given iteration (because it surfaced due to non-$\varepsilon$-viable children ranked above it being pruned) may be retained because it is $\varepsilon$-viable;
that child may out-compete an $\varepsilon$-viable baseline $k$-CBS child at a later iteration, and then later get pruned if it is no longer viable---in which case neither child contributes to the final lower bound.
\looseness=-1 

\subsection{Hamming-$\varepsilon$-pruned $k$-CBS}\label{app:sec:prune:hamming}

This variant of $k$-CBS bakes a Hamming-distance $\varepsilon$-viability check into the search. 
That is, for prefix $\pre$, $\suflen$-length target suffix $\suf$, and possible $\suflen$-length continuations $\cont$, we seek bounds on
\begin{align}
\label{eq:pseqe:ham}
\pseqeham &\;\triangleq\; \sum_{\cont \in \hamball(\suf)} \Pr_{\model,\dec}(\cont\mid\pre),\\
\hamball(\suf) &\;\triangleq\;\big\{\cont\in\vocab^{\suflen}:\ham(\cont,\suf)\le\varepsilon\big\}.
\end{align}

At each step of the search, we keep a count for each partial path of the current Hamming distance.
Paths whose running Hamming counter would exceed the budget $\varepsilon$ are never expanded;
they are removed from the beam.
EOS handling is unchanged from Algorithm~\ref{app:alg:beam-topk}: candidates whose last token is EOS are recorded and removed from the beam, but do not contribute to the lower bound.
Only the mass of $\varepsilon$-viable paths removed by the across-beam prune is banked toward the upper bound (see below for details);
non-$\varepsilon$-viable paths are discarded by Hamming monotonicity, and EOS-terminated paths are discarded because they cannot produce $\suflen$-length continuations.
As a result, no returned final continuation lies outside the Hamming $\varepsilon$-ball, and beam search pruning capacity is never spent on paths that are not $\varepsilon$-viable under Hamming distance.
 
In a bit more detail, the Hamming distance is monotone non-decreasing:
once a mismatch occurs, it can never be undone.
A partial path whose Hamming distance exceeds $\varepsilon$ can never end up in the Hamming-distance $\varepsilon$-ball, $\hamball(\suf)$.
This means that we can prune that path immediately. 
We refer to this approach as our \newterm{Hamming-$\varepsilon$-pruned $k$-CBS algorithm} (Algorithm~\ref{app:alg:beam-topk-hamming}).
For this algorithm, unlike Algorithm~\ref{app:alg:beam-topk}, we do not wait until the end to run the final distance check on the returned sequences in $\sF$. 
Further, since we establish ahead of time that we want to use $\ham$ as the distance metric with a chosen  $\varepsilon$, continuing to expand such a node would not only be wasted work:
any descendant sequences that remain in the beam would potentially be taking spots for (lower probability) sequences that would otherwise be included in the beam and may be in $\hamball(\suf)$. 
Pruning such sequences once they are encountered allows for the next-highest-probability candidate sequences (if they exist) to be included in the beam $\sL$, so that we can expand them in the search instead.

Of course, pruning with this strategy is not as flexible as Algorithm~\ref{app:alg:beam-topk}, which can be used with any $\dist$ and choice of $\varepsilon$ as a post-processing operation.
But, it introduces additional benefits: 
namely, with some minimal bookkeeping, at no extra cost in token evaluations over Algorithm~\ref{app:alg:beam-topk}, this approach focuses effort on only viable nodes.
It can terminate early if no such nodes exist and as a result can also potentially produce tighter lower and upper bounds.
(However, this is not guaranteed.) 

In a bit more detail, we keep track of the running count of Hamming mismatches between the $t$-token partial continuation $\subgensuf{1:t}=(\subgensuftok{1},\ldots,\subgensuftok{t})$ and the target suffix $\suf=(\suftok{1},\ldots,\suftok{\suflen})$.
Define the running Hamming counter at depth $t$ by
\begin{align}
\label{eq:ham-counter}
\mismatch_t(\gensuf,\suf)\;\coloneqq\;\sum_{i=1}^{t}\1\big[\subgensuftok{i} \neq \suftok{i}\big].
\end{align}
A partial $\subgensuf{1:t}$ is $\varepsilon$-viable iff $\mismatch_t(\subgensuf{1:t},\subsuf{1:t})\le\varepsilon$.

That is, like in Algorithm~\ref{app:alg:beam-topk}, line 10, at step $t$,  we produce $t$-length candidate partial sequences $\genseq' \leftarrow \genseq{\Vert}\gentok$;
however, we only do so for the (at most $\bw \cdot k$) sequences for which the Hamming count is $\leq \varepsilon$, where each sequence is produced from appending each of the top-$k$ tokens $\gentok \in \sS_t(\genseq)$ to each of the $\bw$ partial sequences (of length $t-1$) in the beam. 
Put differently, we immediately prune paths for which the mismatch counter exceeds $\varepsilon$, and therefore do not include them in the candidates that get ranked.
As a result, the candidate list can be far smaller than $\bw \cdot k$. 
In fact, it is possible that no candidates remain---i.e., all $\bw \cdot k$ expanded candidates exceed $\varepsilon$---in which case we terminate the search procedure.
The resulting lower bound is $0$, and the upper bound is the banked mass we have computed so far (see below for more details). 
See Algorithm~\ref{app:alg:beam-topk-hamming} for more details.

\paragraph{A potentially improved lower bound.}
Because we prune every child whose running Hamming distance exceeds $\varepsilon$, every final sequence (if there are any)
in $\sF$ at $t=\suflen$ already satisfies $\ham(\gensuf,\suf)\le\varepsilon$.
(In contrast, Algorithm~\ref{app:alg:beam-topk} must filter outputs as a post-processing operation.)
Consequently, \emph{all} returned final sequences contribute to the lower bound on the near-verbatim mass:
\[
\mathrm{LB}_{\varepsilon,\hamshort}\;=\;\sum_{(\_,\log p)\in\sF}\exp(\log p).
\]
This can potentially exceed the baseline lower bound in Algorithm~\ref{app:alg:beam-topk}, because beam capacity is never spent on non-$\varepsilon$-viable paths during the search;
however, this is not guaranteed, and in some cases could be worse. 
(See note at the end of Section~\ref{app:sec:prune}.)

\paragraph{A potentially improved upper bound.}
Algorithm~\ref{app:alg:beam-topk} technically produces an upper bound on the probability for $\pseqe$, but it is often too loose to be useful (Equation~\ref{eq:baseline-ub}).
In contrast, the Hamming-pruned version discussed here provides a potentially tighter upper bound.
When the across-beam prune removes $\varepsilon$-viable sequences from the beam in a given iteration, we keep track of this information;
we bank the mass of these viable sequences, which are roots of subtrees that potentially contain viable paths that contribute mass to $\pseqe$, but which we do not explore further in the beam search procedure.
(This is why it is an upper bound; 
some unexplored leaves of that subtree would not be $\varepsilon$-viable if explored.)
We build up the upper bound additively from banked mass of $\varepsilon$-viable paths that leave the beam.

The key intuition for how this works depends on the fact that at a given step $t$, the beam (and the beam candidates that we prune) contains partial paths that are \emph{parents} of any possible \emph{descendant} sequences that the algorithm could produce at future steps (i.e., $\geq t$ and $\leq\suflen$, inclusive). 
For each parent, the possible child sequences conserve their probability mass;
the mass of a parent at $t$ gets perfectly redistributed among all of its possible children (Lemma~\ref{lem:frontier-mass}).\looseness=-1 

To see why, fix a partial continuation $\vu$ at depth $t<\suflen$.
Top-$k$ renormalization makes a distribution over the next-possible tokens 
$\sS_{t+1}(\pre\Vert\vu)$:
\[
\sum_{v\in\sS_{t+1}(\pre\Vert\vu)} \Pr_{\model,\dec}(v \;\big|\; \pre\Vert\vu)\;=\;1.
\]
By the chain rule,
\[
\Pr_{\model,\dec}(\vu\Vert v \;\big|\; \pre)
=\Pr_{\model,\dec}(\vu \;\big|\; \pre)\cdot \Pr_{\model,\dec}(v \;\big|\; \pre\Vert\vu).
\]
Summing over $v\in\sS_{t+1}(\pre\Vert\vu)$ yields the one-step conservation identity (see also Lemma~\ref{lem:frontier-mass}, Equation~\ref{eq:onestep}):
\begin{equation}
\label{eq:onestep-ham}
\sum_{v\in\sS_{t+1}(\pre\Vert \vu)} \Pr_{\model,\dec}(\vu\Vert v \;\big|\; \pre)
= \Pr_{\model,\dec}(\vu \;\big|\; \pre).
\end{equation}
Therefore, before any pruning occurs, the mass at a parent continuation $\vu$ is exactly redistributed across its immediate children. 
Iterating this identity level by level down a subtree shows that the total mass of all depth-$\suflen$ descendants of any partial $\vw$ equals $\Pr_{\model,\dec}(\vw\;\big|\;\pre)$. 
(In general, please refer to Lemma~\ref{lem:frontier-mass}.)

After pruning, we only keep some of the children.
Of course, this means the sum over the kept children is $\leq\Pr_{\model, \dec}(\vu \;\big|\; \pre)$;
the discarded remainder that contains $\varepsilon$-viable paths gets counted toward the upper bound.
Refer to Appendix~\ref{app:sec:prune:invariants:viable-finals} for more details.

\paragraph{How we use the UB bookkeeping.}
We keep a scalar accumulator \textit{bank} that stores the total probability mass of $\varepsilon$-viable candidates that were pruned by the intermediate across-beam prune to the top-$\bw$ paths when $t<\suflen$.
This accumulator is used in two different ways:
\begin{itemize}[leftmargin=0.5cm]
    \item \textbf{Early termination (no viable candidates before $\suflen$).}
    If at some depth $t<\suflen$ the viable set is empty, the algorithm stops; then
    \[
    \mathrm{LB}_{\varepsilon,\hamshort}\;=\;0,
    \qquad
    \mathrm{UB}_{\varepsilon,\hamshort}\;=\;\textit{bank}.
    \]

    \item \textbf{Final step ($t=\suflen$).}
    When we reach $\suflen$, every returned final sequence in $\sF$ is $\varepsilon$-viable, so
    \[
    \mathrm{LB}_{\varepsilon,\hamshort}\;=\;\sum_{(\_,\log p)\in\sF}\exp(\log p),
    \qquad
    \mathrm{UB}_{\varepsilon,\hamshort}\;=\;\mathrm{LB}_{\varepsilon,\hamshort}+\textit{bank}.
    \]
\end{itemize}

To see why $\mathrm{UB}_{\varepsilon,\hamshort}$ is a valid upper bound, observe that every depth-$\suflen$ node in the full top-$k$ tree falls into exactly one of the following categories:
\begin{enumerate}[leftmargin=0.5cm]
    \item \textbf{In $\sF$}: counted in $\mathrm{LB}_{\varepsilon,\hamshort}$.
    \item \textbf{Descendant of a banked node}: its mass is bounded by \textit{bank}, since the total mass of all depth-$\suflen$ descendants of a banked ancestor equals that ancestor's mass (Lemma~\ref{lem:frontier-mass}).
    \item \textbf{Descendant of a Hamming-pruned node}: non-$\varepsilon$-viable by Hamming monotonicity, so it cannot contribute to $\pseqeham$.
    \item \textbf{Descendant of an EOS-terminated node}: cannot produce a $\suflen$-length continuation and thus cannot contribute to $\pseqeham$.
\end{enumerate}
Since only categories (a) and (b) can contribute to $\pseqeham$, we have $\pseqeham \leq \mathrm{LB}_{\varepsilon,\hamshort} + \textit{bank} = \mathrm{UB}_{\varepsilon,\hamshort}$.

As with baseline $k$-CBS (Appendix~\ref{app:sec:kcbs:baseline}), the optional $\taumin$-based early termination applies here as well:
if $\max_{(\_, \log p, \_) \in \sL_t} \exp(\log p) < \taumin / (\bw \cdot k)$, then the lower bound for the extraction probability can never reach $\taumin$ regardless of Hamming $\varepsilon$-viability. 
(The same caveat applies: since $k$-CBS is a greedy search, it may miss mass, so this is not an absolute guarantee of non-extractability;
the lower bound is $0$.) 

\paragraph{Filtering the upper bound in practice.}
When a sequence terminates early---whether due to Hamming viability ($\sC^{\leq\varepsilon}_t = \emptyset$) or the $\taumin$ threshold---the upper bound $\mathrm{UB}_{\varepsilon,\hamshort} = \textit{bank}$ is a valid bound on $\pseqeham$, but we find that it is not informative in practice:
the banked mass from pruned $\varepsilon$-viable ancestors can be substantial even when no final $\varepsilon$-viable path exists in the final search results.
In our analysis, we therefore restrict attention to the upper bound only for sequences with $\mathrm{LB}_{\varepsilon,\hamshort} > 0$ (i.e., those for which the algorithm found at least some near-verbatim mass).
This filtering is a postprocessing step and does not affect the algorithm outputs.

\paragraph{Can return fewer than $\bw \cdot k$ suffixes.}
Because we prune any path at step $t$ with $\mismatch_t>\varepsilon$ \emph{before} ranking for the across-beam prune, the candidate set at depth $t$ can be strictly smaller than $\bw \cdot k$ (or even empty, if there are no $\varepsilon$-viable paths remaining). 
As noted above, if at some $t<\suflen$ no $\varepsilon$-viable candidates remain, we terminate; then
$\mathrm{LB}_{\varepsilon,\hamshort}^{(t)}=0$ (EOS-terminated paths do not contribute to the lower bound) and
$\mathrm{UB}_{\varepsilon,\hamshort}^{(t)}=\textit{bank}^{(t)}$ (the banked mass of the $\varepsilon$-viable prunes).
Because of this within-search pruning, the procedure often performs \emph{far fewer} token evaluations than $k$-CBS, which always expands $\bw$ parents into up to $\bw \cdot k$ candidates at each depth (except in degenerate cases of many EOS tokens).
(See discussion of token-evaluation cost comparisons in Appendix~\ref{app:sec:intuition:cost}.)

\paragraph{Tighter bounds at the cost of more forward passes.}
We could run a slight variation of this algorithm that, in addition to the minor bookkeeping described here, runs (potentially many) additional forward passes through the model to tighten the bounds.
The basic idea is that, when the across-beam prune to $\bw$ candidates prunes paths that have \emph{exactly} cost $\varepsilon$, those paths each have \emph{exactly} $1$ viable full suffix in the $\varepsilon$-ball:
the candidate path generated thus far at step $t$, plus the remaining $\suflen-t$ verbatim tokens in the target suffix.
We can gather these \newterm{tail-closed sequences}, and teacher force them at the end of our search algorithm in order to adjust the bounds to be tighter.
We can also do the same for the Levenshtein distance.
Given the cost of this approach, we do not pursue it in our experiments, especially since we observe useful estimates of the near-verbatim extraction probability without doing so. 

\clearpage

\begin{algorithm}[t]
\caption{Hamming-$\varepsilon$-pruned Top-$k$ Constrained Beam Search}\label{app:alg:beam-topk-hamming}
\KwIn{LLM $\model$; prefix $\pre$ of length $\prelen$; target suffix $\suf=(\suftok{1},\ldots,\suftok{\suflen}) \in \vocab^\suflen$; beam width $\bw$;
  top-$k$ parameter $k$ for decoding policy $\dec$ ($\bw \leq k^2$, $k \ll |\vocab|$);
  Hamming distance budget $\varepsilon \ll \suflen$;
  EOS token id; optional $\taumin > 0$}
\KwOut{Set $\sF$ of (at most $\bw \cdot k$) triples
  $(\genseq,\; \log p,\; \mismatch)$,
  where $\genseq = \pre \,\Vert\, \gensuf$ is a full history with length $\prelen + \suflen$, $\log p = \log \Pr_{\model,\dec}(\gensuf \mid \pre)$, and $\mismatch \leq \varepsilon$;
  lower bound $\mathrm{LB}_{\varepsilon,\hamshort}$; upper bound $\mathrm{UB}_{\varepsilon,\hamshort}$}
\BlankLine
\textbf{Notation.} Same as Algorithm~\ref{app:alg:beam-topk}, and $\suftok{t}$ denotes the $t$-th token of the target suffix $\suf$.
\BlankLine
\textbf{Beam state.}
Maintain $\sL_t$ as triples $(\genseq,\log p, \mismatch)$, where $\log p$ is the accumulated top-$k$ decoding $\log$-probability and $\mismatch$ is the running Hamming mismatch count 
$\textstyle
\sum_{i=1}^{t}\1\big[\gentok_i \neq \suftok{i}\big]
$ (Equation~\ref{eq:ham-counter}).
Max.\ beam capacity $|\sL_t|=\bw$.
\BlankLine
\textbf{Viability test (monotone in $t$).} 
$
\mismatch' \leftarrow \mismatch + \1\big[\gentok_t \ne \suftok{t}\big].
$
Keep child path only if $\mismatch'\le \varepsilon$; otherwise drop it permanently (Hamming mismatches cannot be undone).\looseness=-1
\BlankLine
Compute $\logit_1(\pre)$ via forward pass on $\pre$\tcp*{Prefill: $\prelen$ token evals}
$\sL_0 \gets \big\{(\pre,\; 0,\; 0)\big\}$\tcp*{Beam: triples $(\genseq,\; \log p,\; \mismatch)$}
$\textit{bank} \gets 0$\tcp*{UB mass accumulator}
\BlankLine
\For{$t = 1, \ldots, \suflen$}{
  $\sC^{\leq\varepsilon}_t \gets \emptyset$\tcp*{$\varepsilon$-viable candidate set for step $t$}
  \ForEach{$(\genseq,\; \log p,\; \mismatch) \in \sL_{t-1}$}{
    $\sS_t(\genseq) \gets \TopK_k\big(\logit_t(\genseq)\big)$\;
    $\vr_t(\genseq) \gets \LogSoftmax\big(\logit_t(\genseq)\big)$\;
    $Z_t(\genseq) \gets \LogSumExp\big(\vr_t(\genseq)[\sS_t(\genseq)]\big)$\;
    \ForEach{$\gentok \in \sS_t(\genseq)$}{
      $\mismatch' \leftarrow \mismatch + \1[\gentok \neq \suftok{t}]$ \tcp*{Update Hamming counter}
      \If(\tcp*[f]{$\varepsilon$-viable: keep}){$\mismatch' \leq \varepsilon$}{
        $\genseq' \leftarrow \genseq \,\Vert\, \gentok$ \tcp*{Append token to partial history}
        $\log p' \leftarrow \log p + \vr_t(\genseq)[\gentok] - Z_t(\genseq)$ \tcp*{Update continuation $\log$ prob}
        $\sC^{\leq\varepsilon}_t \gets \sC^{\leq\varepsilon}_t \cup
          \big\{\!\big(\genseq',\;\;
          \log p',\;\; \mismatch' \big)\!\big\}$\;
      }
      \tcp{If $\mismatch' > \varepsilon$: discard (non-$\varepsilon$-viable by Hamming monotonicity)}
    }
  }
  \BlankLine
  \If(\tcp*[f]{Final step: return all $\varepsilon$-viable candidates}){$t = \suflen$}{
    $\sF \gets \sC^{\leq\varepsilon}_\suflen$\;
    $\mathrm{LB}_{\varepsilon,\hamshort} \gets \sum_{(\_,\,\log p,\,\_)\in\sF}\exp(\log p)$\;
    $\mathrm{UB}_{\varepsilon,\hamshort} \gets \mathrm{LB}_{\varepsilon,\hamshort} + \textit{bank}$\;
    \Return $\sF,\; \mathrm{LB}_{\varepsilon,\hamshort},\; \mathrm{UB}_{\varepsilon,\hamshort}$\;
  }
  \BlankLine
  \tcp{Non-final step ($t < \suflen$): EOS handling, early termination, beam pruning}
  \ForEach{$(\genseq',\; \log p',\; \mismatch') \in \sC^{\leq\varepsilon}_t$ with latest $\gentok = \textup{EOS}$}{
    Record $(\genseq',\; \log p',\; t)$ as early-termination path;
      remove from $\sC^{\leq\varepsilon}_t$\;
  }
  \BlankLine
  \If(\tcp*[f]{Early termination: no $\varepsilon$-viable candidates remain}){$\sC^{\leq\varepsilon}_t = \emptyset$}{
    \Return $\sF \gets \emptyset,\; \mathrm{LB}_{\varepsilon,\hamshort} \gets 0,\; \mathrm{UB}_{\varepsilon,\hamshort} \gets \textit{bank}$\;
  }
  \BlankLine
  $\sU_t \gets$ top-$\bw$ elements of $\sC^{\leq\varepsilon}_t$ ranked by $\log p'$\;
  $\textit{bank} \gets \textit{bank} + \sum_{(\_,\,\log p',\,\_)\in\sC^{\leq\varepsilon}_t\setminus\sU_t}\exp(\log p')$\tcp*{Bank pruned $\varepsilon$-viable mass}
  $\sL_t \gets \sU_t$ \tcp*{Prune to beam width}
  \If{$\taumin$ and $\max_{(\_,\,\log p,\,\_) \in \sL_t} \exp(\log p) < \taumin / (\bw \cdot k)$}{
    \Return $\sF \gets \emptyset,\; \mathrm{LB}_{\varepsilon,\hamshort} \gets 0,\; \mathrm{UB}_{\varepsilon,\hamshort} \gets \textit{bank}$\;
  }
  Compute $\logit_{t+1}(\genseq)$ for each $(\genseq, \cdot, \cdot) \in \sL_t$\tcp*{$|\sL_t|$ token evals}
}
\end{algorithm}
\clearpage

\subsection{Levenshtein-$\varepsilon$-pruned $k$-CBS}\label{app:sec:prune:lev}

We also design a $k$-CBS variant that uses the Levenshtein distance as its viability test.
Fix the prefix $\pre$ and the target suffix $\suf=(\suftok{1},\ldots,\suftok{\suflen})$ of length $\suflen$.
For possible $\suflen$-length continuations $\cont$ under decoding policy $\dec$ (again, we use top-$k$), we seek tight bounds on
\begin{align}
\label{eq:pseqe:lev}
\pseqelev &\;\triangleq\; \sum_{\cont \in \levball(\suf)} \Pr_{\model,\dec}(\cont\mid\pre),\nonumber\\
\levball(\suf) &\;\triangleq\; \big\{\cont\in\vocab^{\suflen}:\lev(\cont,\suf)\le\varepsilon\big\}.
\end{align}
For equal-length sequences, $\lev\le \ham$, so Levenshtein-pruned beam search can be (strictly) less conservative than Hamming pruning (Appendix~\ref{app:sec:nv:distance}, Theorem~\ref{thm:epsballs}).

\paragraph{Why alignment state is needed (differs from Hamming).}
Unlike Hamming viability (Section~\ref{sec:pruning} \& Appendix~\ref{app:sec:prune:hamming}), Levenshtein $\varepsilon$-viability is \emph{not} monotone because insertions and deletions can repair alignment later, when more tokens are added to the partial path:
a partial path that currently ``looks bad'' by positional mismatches might still finish within $\varepsilon$ edits after realignment---i.e., the minimum achievable Levenshtein distance for the completed continuation can decrease at later steps $t$.

Therefore, a running mismatch counter---sufficient for $\ham$-based pruning---is insufficient for $\lev$-based pruning. 
Instead, we must maintain \newterm{alignment state} that encodes the best attainable edit cost observed so far, as we generate (i.e., \newterm{stream}) new tokens, one at a time, for candidate paths. 
For this, \newterm{dynamic programming}\footnote{We will never shorthand dynamic programming as ``DP'' because in ML privacy/security this refers to \newterm{differential privacy}.\looseness=-1
} 
can help us keep track of state in an efficient way.
We carry forward alignment state across search iterations, not just a counter of mismatched positions, in order to inform a safe $\lev$-pruning rule.

In the following subsections, we discuss the intuition. 
We provide background on the Wagner-Fischer dynamic program solution for the Levenshtein distance (Appendix~\ref{app:sec:prune:lev:wf}), discuss how the Ukkonen's band refinement of this program applies to our $\varepsilon$ setting (Appendix~\ref{app:sec:prune:lev:ukkonen}), and explain how the pieces all fit together for a $\lev$-pruned version of $k$-CBS (Appendix~\ref{app:sec:prune:lev:alg-details}). 
We provide a complete algorithm statement in Algorithm~\ref{app:alg:beam-topk-lev}.\looseness=-1 

\subsubsection{Wagner-Fischer dynamic program}\label{app:sec:prune:lev:wf}

We first need to provide some background on the dynamic program used to solve the Levenshtein distance.
We will also include a simple worked example to give an intuition.

Let $\subgensuf{1:i}=(\subgensuftok{1},\ldots,\subgensuftok{i})$ be the generated partial path at depth $i$ (i.e., does not include $\pre$), and
$\subsuf{1:j}=(\suftok{1},\ldots,\suftok{j})$ be the $j$-length prefix of the target $\suf$.
Define the table\looseness=-1
\begin{equation}
\label{eq:lev-dp-table}
\dptab[i,j]\;\triangleq\; \lev\big(\subgensuf{1:i},\,\subsuf{1:j}\big),
\qquad
i,j\in\{0,\ldots,\suflen\}.
\end{equation}
(Note that $D$ has dimensions $(\suflen + 1) \times (\suflen + 1)$.) 
The row index $i$ is the number of generated tokens, corresponding to $\subgensuf{1:i}$.
The column index $j$ is the number of target tokens, corresponding to the prefix $\subsuf{1:j}$ of the suffix $\suf$ that we align to.
Therefore, the cell $\dptab[i,j]$ is the minimum number of edits needed to turn the current $i$-length prefix of the (ultimately) $\suflen$-length generated continuation $\subgensuf{1:i}$ into the first $j$ tokens of the target suffix $\subsuf{1:j}$.

This table is an \newterm{edit graph}, a $(\suflen\!+\!1)\times(\suflen\!+\!1)$ grid in which each interior cell $\dptab[i,j]$ (with $i,j\ge 1$) can be reached from three predecessors, corresponding to three edit operations:

\begin{itemize}[leftmargin=0.5cm]
    \item move $\downarrow$ (from $\dptab[i\!-\!1,j]$ to $\dptab[i,j]$) = \textbf{delete} $\subgensuftok{i}$ (cost $1$):
    the row advances but the column stays---the new generated token is not aligned to any target token, so it is deleted.
    This reduces to the subproblem $\subgensuf{1:i-1}\!\to\!\subsuf{1:j}$, plus $+1$ for the deletion.
    
    \item move $\rightarrow$ (from $\dptab[i,j\!-\!1]$ to $\dptab[i,j]$) = \textbf{insert} $\suftok{j}$ (cost $1$):
    the column advances but the row stays---target token $\suftok{j}$ is not aligned to any generated token, so it is inserted.
    This reduces to the subproblem $\subgensuf{1:i}\!\to\!\subsuf{1:j-1}$, plus $+1$ for the insertion.
    
    \item move $\searrow$ (from $\dptab[i\!-\!1,j\!-\!1]$ to $\dptab[i,j]$) = \textbf{match/substitute} (cost $0$ if $\subgensuftok{i}\!=\!\suftok{j}$, else $1$):
    both indices advance. 
    This reduces to the subproblem $\subgensuf{1:i-1}\!\to\!\subsuf{1:j-1}$, plus the match/substitute cost.
\end{itemize}

\paragraph{Boundary conditions.}
The boundaries of the table are
\begin{align}
\label{eq:wf:dp:boundary}
\dptab[0,0]&=0,\nonumber\\ 
\dptab[i,0]&=i\ \ (\text{delete all $\subgensuftok{1},\ldots,\subgensuftok{i}$}),\nonumber\\ 
\dptab[0,j]&=j\ \ (\text{insert all $\suftok{1},\ldots,\suftok{j}$}).
\end{align}
Any optimal edit script for $\subgensuf{1:i} \to \subsuf{1:j}$ must end with one of these three operations, each building on the optimal cost of the corresponding smaller subproblem. 
This gives the standard recurrence:  
\begin{equation}
\label{eq:wf:dp:recurrence}
\dptab[i,j] \;=\; \min\{
\underbrace{\dptab[i-1,j]+1}_{\text{delete }\subgensuftok{i}},\
\underbrace{\dptab[i, j-1]+1}_{\text{insert }\suftok{j}},\
\underbrace{\dptab[i-1, j-1] + \1[\subgensuftok{i}\neq \suftok{j}]}_{\text{match / substitute}}\}.
\end{equation}
Each term takes the optimal cost for the smaller subproblem, then pays for the last operation; 
the underbraces correspond to the three edit-graph moves above.

\paragraph{How to read a row in $\dptab$ and why the minimum over $j$ matters.}
A row $i$ summarizes all the ways to align the length-$i$ partial path candidate sequence to the different possible prefixes of the target suffix.
That is, at decoding depth $i\!=\!t$, the entire row $\dptab[t, 0], \dptab[t, 1], \ldots , \dptab[t, \suflen]$ summarizes all possible alignments of the current $t$-length partial path to every length-$j$ prefix of the target suffix $\suf$.
The row minimum $\min_j \dptab[t,j]$ answers: ``what is the best we can possibly do right now (for cost) if later insertions/deletions are allowed to realign before reaching length $\suflen$?''
If that best cost already exceeds $\varepsilon$, then \emph{every} full alignment path from $(0,0)$ to $(\suflen,\suflen)$---which must pass through some column $j^\star$ on row $t$---incurs partial cost $>\varepsilon$ at row $t$ and cannot recover (because remaining edit costs are non-negative).
Therefore, no continuation can end within $\varepsilon$.

\paragraph{A fully worked $2{\times}3$ example (token-by-token).}
Let $\suf=(a,b,c)$ and $\subgensuf{1:2}=(b,c)$.
We compute the entire table using the boundary conditions (Equation~\ref{eq:wf:dp:boundary}) and recurrence (Equation~\ref{eq:wf:dp:recurrence}).
See Figure~\ref{fig:ex1} for a summary.

\begin{itemize}[leftmargin=0.5cm]
\item \textbf{Row $i=1$ ($\subgensuftok{1} = b$)}:
    \begin{itemize}[leftmargin=0.5cm]
        \item $\dptab[1,1]=\min\{\dptab[0,1]\!+\!1,\; \dptab[1,0]\!+\!1,\; \textcolor{teal}{\dptab[0,0]\!+\!\1[b{\neq}a]}\}=\min\{2,\;2,\;\textcolor{teal}{1}\}=1$
        \quad (\textcolor{teal}{substitute})
        \item $\dptab[1,2]=\min\{\dptab[0,2]\!+\!1,\; \dptab[1,1]\!+\!1,\; \textcolor{teal}{\dptab[0,1]\!+\!\1[b\!=\!b]}\}=\min\{3,\;2,\;\textcolor{teal}{1}\}=1$
        \quad (\textcolor{teal}{match})
        \item $\dptab[1,3]=\min\{\dptab[0,3]\!+\!1,\; \textcolor{violet}{\dptab[1,2]\!+\!1},\; \dptab[0,2]\!+\!\1[b{\neq}c]\}=\min\{4,\;\textcolor{violet}{2},\;3\}=2$
        \quad (\textcolor{violet}{insert})
    \end{itemize}
\item \textbf{Row $i=2$ ($\subgensuftok{1}\!=\!b,\; \subgensuftok{2} = c$)}:
    \begin{itemize}[leftmargin=0.5cm]
        \item $\dptab[2,1]=\min\{\textcolor{red}{\dptab[1,1]\!+\!1},\; \dptab[2,0]\!+\!1,\; \textcolor{teal}{\dptab[1,0]\!+\!\1[c{\neq}a]}\}=\min\{\textcolor{red}{2},\;3,\;\textcolor{teal}{2}\}=2$
        \quad (\textcolor{red}{delete}/\textcolor{teal}{substitute} tie)
        \item $\dptab[2,2]=\min\{\textcolor{red}{\dptab[1,2]\!+\!1},\; \dptab[2,1]\!+\!1,\; \textcolor{teal}{\dptab[1,1]\!+\!\1[c{\neq}b]}\}=\min\{\textcolor{red}{2},\;3,\;\textcolor{teal}{2}\}=2$
        \quad (\textcolor{red}{delete}/\textcolor{teal}{substitute} tie)
        \item $\dptab[2,3]=\min\{\dptab[1,3]\!+\!1,\; \dptab[2,2]\!+\!1,\; \textcolor{teal}{\dptab[1,2]\!+\!\1[c\!=\!c]}\}=\min\{3,\;3,\;\textcolor{teal}{1}\}=1$
        \quad (\textcolor{teal}{match})
    \end{itemize}
\end{itemize}

\noindent The last row is $[2,2,2,1]$, so $\min_j \dptab[2,j]=1$ at $j\!=\!3$: ``insert $a$ at the front (cost $1$), then match $b$, then match $c$.''
Note that the diagonal value $\dptab[2,2]=2$ is \emph{not} the best cost---showing why we must minimize over all $j$.

\begin{figure}[t]
\centering
{\footnotesize
\begin{tabular}{l|llll}
$\dptab[i,j]$ & $j\!=\!0$ ($\emptyset$) & $j\!=\!1$ ($a$) & $j\!=\!2$ ($ab$) & $j\!=\!3$ ($abc$)\\\hline
$i\!=\!0$ ($\emptyset$) & $0$ & $1$ & $2$ & $3$ \\[2pt]
$i\!=\!1$ ($b$) & $1$ & $1$\;\textcolor{teal}{$\searrow$} & $1$\;\textcolor{teal}{$\searrow$} & $2$\;\textcolor{violet}{$\rightarrow$} \\[2pt]
$i\!=\!2$ ($bc$) & $2$ & $2$\;\textcolor{red}{$\downarrow$}\!\textcolor{teal}{$\searrow$} & $2$\;\textcolor{red}{$\downarrow$}\!\textcolor{teal}{$\searrow$} & $1$\;\textcolor{teal}{$\searrow$}
\end{tabular}
}%
\caption{\textbf{Fully worked example.} 
Wagner-Fischer table for $\suf=(a,b,c)$ and $\subgensuf{1:2}=(b,c)$.
Colored arrows show the winning operation at each cell:
\textcolor{teal}{$\searrow$}~match/substitute,
\textcolor{red}{$\downarrow$}~delete,
\textcolor{violet}{$\rightarrow$}~insert.
The last row is $[2,2,2,1]$ with $\min_j \dptab[2,j]=1$ at $j\!=\!3$;
the diagonal $\dptab[2,2]\!=\!2$ is not the minimum, illustrating why we minimize over all $j$.}
\label{fig:ex1}
\end{figure}

\paragraph{A streaming example (row minima can go down, then exceed a threshold).}
Let $\suf=(a,b,c,d)$ and consider the stream
$\subgensuf{1:1}=(b)$, $\subgensuf{1:2}=(b,c)$, $\subgensuf{1:3}=(b,c,a)$, $\subgensuf{1:4}=(b,c,a,d)$.
The row minima evolve as
\[
\min_j \dptab[1,j]=1,\quad
\min_j \dptab[2,j]=1,\quad
\min_j \dptab[3,j]=2,\quad
\min_j \dptab[4,j]=2.
\]
At $i\!=\!t\!=\!3$, inserting $a$ breaks the earlier ``insert-$a$-then-match'' alignment and the best partial cost becomes $2$;
at $i\!=\!t\!=\!4$ the best cost remains $2$ via a match on $d$.
Therefore, with $\varepsilon\!=\!1$, we would prune at $i\!=\!t\!=\!3$;
in contrast, with $\varepsilon\!=\!2$, the path remains viable through $i\!=\!t\!=\!4$ (and ends at distance $2$).
See Figure~\ref{fig:ex2}.

\begin{figure}[t]
\centering
{\footnotesize
\begin{tabular}{l|lllllc}
$\dptab[i,j]$ & $j\!=\!0$ ($\emptyset$) & $j\!=\!1$ ($a$) & $j\!=\!2$ ($ab$) & $j\!=\!3$ ($abc$) & $j\!=\!4$ ($abcd$) & \\\hline
$i\!=\!0$ ($\emptyset$) & $0$ & $1$ & $2$ & $3$ & $4$ \\[2pt]
$i\!=\!1$ ($b$) & $1$ & $1$\;\textcolor{teal}{$\searrow$} & $1$\;\textcolor{teal}{$\searrow$} & $2$\;\textcolor{violet}{$\rightarrow$} & $3$\;\textcolor{violet}{$\rightarrow$}
& $\min\!=\!1$ \\[2pt]
$i\!=\!2$ ($bc$) & $2$ & $2$\;\textcolor{red}{$\downarrow$}\!\textcolor{teal}{$\searrow$} & $2$\;\textcolor{red}{$\downarrow$}\!\textcolor{teal}{$\searrow$} & $1$\;\textcolor{teal}{$\searrow$} & $2$\;\textcolor{violet}{$\rightarrow$}
& $\min\!=\!1$ \\[2pt]
$i\!=\!3$ ($bca$) & $3$ & $2$\;\textcolor{teal}{$\searrow$} & $3$\;\textcolor{red}{$\downarrow$}\!\textcolor{violet}{$\rightarrow$}\!\textcolor{teal}{$\searrow$} & $2$\;\textcolor{red}{$\downarrow$} & $2$\;\textcolor{teal}{$\searrow$}
& $\min\!=\!2$ \\[2pt]
$i\!=\!4$ ($bcad$) & $4$ & $3$\;\textcolor{red}{$\downarrow$} & $3$\;\textcolor{teal}{$\searrow$} & $3$\;\textcolor{red}{$\downarrow$} & $2$\;\textcolor{teal}{$\searrow$}
& $\min\!=\!2$
\end{tabular}
}%
\caption{\textbf{Streaming example.} 
$\suf=(a,b,c,d)$ and $\subgensuf{1:4}=(b,c,a,d)$.
Row minima (shown at right) evolve as $1,1,2,2$.
With $\varepsilon\!=\!1$, pruning triggers at $t\!=\!3$;
with $\varepsilon\!=\!2$, the path remains viable and ends at $\dptab[4,4]\!=\!2$.
Colors as in Figure~\ref{fig:ex1}.}
\label{fig:ex2}
\end{figure}

\paragraph{Two facts we use later.}
First, for any $i,j$:
\begin{equation}
\label{eq:length-gap-lb}
\dptab[i,j]\ \ge\ |i-j|,
\end{equation}
i.e., turning a length-$i$ string into a length-$j$ string requires at least $|i\!-\!j|$ insertions or deletions. 
Second, any full alignment path from $(0,0)$ to $(\suflen,\suflen)$ must pass through some column $j^\star$ on each row $i$, and $\dptab[i,j^\star]$ is the minimum partial edit cost to reach $(i,j^\star)$.
Any completion from $(i,j^\star)$ to $(\suflen,\suflen)$ adds a non-negative number of edits, so $\dptab[i,j^\star]$ lower-bounds the total cost of any complete alignment path through $(i,j^\star)$.
This gives our pruning rule: if $\min_j \dptab[i,j] > \varepsilon$, then no extension of that partial path can achieve total cost $\leq \varepsilon$, and we can safely prune it.\looseness=-1

\paragraph{Bookkeeping complexity.}
Maintaining the full (unbanded) row for each partial path requires $O(\suflen)$ time and $O(\suflen)$ space per decoding step (only the current and previous rows are needed).
In Appendix~\ref{app:sec:prune:lev:ukkonen} we reduce both to $O(\varepsilon)$ per step per partial path via banding, which is the regime we use in practice.

\subsubsection{Ukkonen's band and a viable pruning rule}\label{app:sec:prune:lev:ukkonen}

For our purposes, we only need to retain partial paths for which $\dptab[i,j] \leq \varepsilon$ could hold.
By \Eqref{eq:length-gap-lb}, $\dptab[i,j] \geq |i\!-\!j|$, so any cell with $|i\!-\!j| > \varepsilon$ satisfies $\dptab[i,j] > \varepsilon$ and cannot contribute to the row minimum that we use to determine $\varepsilon$-viability pruning.
Concretely, $\dptab[i,j] \leq \varepsilon$ requires $|i\!-\!j| \leq \varepsilon$, i.e., $i\!-\!\varepsilon \leq j \leq i\!+\!\varepsilon$.
Clamping to valid column indices $j \in \{0,\ldots,\suflen\}$, at row $i$ it suffices to maintain only the \newterm{band}
\begin{equation}
\label{eq:ukkonen-band}
j \in [\,j_{\min}(i),\,j_{\max}(i)\,]
\;\equiv\;
[\,\max\{0,\,i-\varepsilon\},\ \min\{\suflen,\,i+\varepsilon\}\,],
\end{equation}
a diagonal strip centered on the main diagonal $j{=}i$.
The unclamped interval $[i\!-\!\varepsilon,\, i\!+\!\varepsilon]$ has width $2\varepsilon\!+\!1$; the band is its intersection with $[0, \suflen]$, so it contains at most $2\varepsilon\!+\!1$ entries for every row $i$.
We treat out-of-band entries as $+\infty$ (never chosen by the $\min$ in the recurrence), reducing the per-row cost from $O(\suflen)$ to $O(\varepsilon)$---a strict improvement when $\varepsilon \ll \suflen$, which is the regime of our near-verbatim extraction setting.
This is exactly the restriction used in \newterm{Ukkonen's thresholded dynamic program}~\citep{ukkonen, navarro2001string}.

The banded pruning rule for a partial path at decoding depth $i$ is then:
\begin{align}
\label{eq:ukkonen-test}
\min_{\,j\in [j_{\min}(i),\,j_{\max}(i)]}\ \dptab[i,j]\;>\;\varepsilon
\quad\implies\quad \text{prune (no viable completion exists).}
\end{align}
This is sound because the banded test implies the \emph{full} row minimum exceeds $\varepsilon$:
every in-band column satisfies $\dptab[i,j] > \varepsilon$ (checked directly by the test),
and every out-of-band column satisfies $\dptab[i,j] \geq |i\!-\!j| > \varepsilon$ (by \Eqref{eq:length-gap-lb}).
Since $\min_j \dptab[i,j] > \varepsilon$ over all columns, no completion of this partial path can achieve total cost $\leq \varepsilon$, by the pruning argument in Appendix~\ref{app:sec:prune:lev:wf} (Lemma~\ref{lem:ukkonen-prune-and-safety}).

\paragraph{Streaming, banded updates (how we maintain state).}
For each active beam item at depth $t$, we store the banded row $\dptab[t,j]$ for $j \in [j_{\min}(t),\, j_{\max}(t)]$ (Equation~\ref{eq:ukkonen-band}), initialized at $t{=}0$ with $\dptab[0,j]=j$ for $j \in [0,\, \min\{\suflen, \varepsilon\}]$.
This single row is sufficient to compute the next: the recurrence (Equation~\ref{eq:wf:dp:recurrence}) for $\dptab[t{+}1,j]$ reads only from row $t$ (the $\textcolor{red}{\downarrow}$~delete and $\textcolor{teal}{\searrow}$~match/substitute predecessors) and from earlier entries of row $t{+}1$ (the $\textcolor{violet}{\rightarrow}$~insert predecessor, already filled left to right).
In practice, we pre-allocate two buffers of size $2\varepsilon{+}1$ and alternate between them, avoiding per-step allocation.
When we append a token $\subgensuftok{t+1}$:
\begin{enumerate}[leftmargin=0.5cm]
\item \textbf{Compute} the next band limits $[j_{\min}(t\!+\!1),\,j_{\max}(t\!+\!1)]$ via \Eqref{eq:ukkonen-band}.

\item \textbf{Fill} $\dptab[t\!+\!1,j]$ for $j = j_{\min}(t\!+\!1),\ldots,j_{\max}(t\!+\!1)$ (left to right) using Equation~\ref{eq:wf:dp:recurrence}, treating any out-of-band predecessor as $+\infty$.
This keeps the recurrence exactly consistent with Wagner-Fischer.

\begin{itemize}[leftmargin=0.5cm]
    \item[-] For instance, when $j{=}j_{\min}(t\!+\!1){=}0$, the $\textcolor{violet}{\rightarrow}$~insert predecessor $\dptab[t\!+\!1,{-}1]$ and the $\textcolor{teal}{\searrow}$~match/substitute predecessor $\dptab[t,{-}1]$ are out-of-band ($+\infty$), so only the $\textcolor{red}{\downarrow}$~delete predecessor remains: $\dptab[t\!+\!1,0]=\dptab[t,0]+1$.

    \item[-] The left-to-right fill order ensures that $\dptab[t\!+\!1,j\!-\!1]$ has already been computed when needed for the $\textcolor{violet}{\rightarrow}$~insert term.
\end{itemize}

\item \textbf{Prune} this partial path immediately if $\min_{j\in [j_{\min}(t\!+\!1),\,j_{\max}(t\!+\!1)]}\dptab[t\!+\!1,j]>\varepsilon$ (Lemma~\ref{lem:ukkonen-prune-and-safety});
otherwise, carry its banded row forward as the new state.
\end{enumerate}
This keeps a beam item's per-child update cost and memory both $O(\varepsilon)$---the same order of bookkeeping overhead as the $\ham$-pruned variant (which maintains a single mismatch counter per beam item).
Note that $\lev$-pruning is less conservative than $\ham$-pruning, so fewer partial paths may be pruned, potentially requiring more forward passes; 
this is a difference in pruning rate, not in per-step overhead.
We discuss the cost comparison further in the next subsection.

\paragraph{A note on further optimization.}
The per-step bookkeeping cost could be reduced further using Myers' bit-vector algorithm~\citep{myers1999bit, navarro2001string}, which encodes the banded Wagner-Fischer recurrence with a handful of machine-word operations per step.
We do not implement this, as the $O(\varepsilon)$ bookkeeping overhead is already negligible relative to the cost of a model forward pass at each decoding step.
\subsubsection{Integrating the Levenshtein-pruning dynamic program into $k$-CBS}\label{app:sec:prune:lev:alg-details}

We incorporate this bookkeeping without additional model calls to form our Levenshtein-pruned $k$-CBS variant.
We use $t$ to index into the row of the dynamic programming table (instead of $i$) to align with the generation step. 

\begin{itemize}[leftmargin=0.25in]
\item \textbf{Per-partial sequence state.} 
Maintain triples $(\genseq,\log p,\mathrm{aux})$ with $\mathrm{aux}=\big(\dptab[t,\cdot],j_{\min}(t),j_{\max}(t)\big)$, i.e., the step-$t$ dynamic programming table row for the step-$t$ Ukkonen band (Equation~\ref{eq:ukkonen-band}).

\item \textbf{EOS policy.}
If an EOS is the latest generated token in a beam candidate at step $t < \suflen$, we discard the path, i.e., do not include it in the next beam.
(At $\suflen$, we allow the last token to be EOS.) 
This enforces that all returned continuations have fixed length $\suflen$.

\item \textbf{Child sequence $\genseq'$ formation and $\varepsilon$-viability.} 
Just as with the Hamming variant (Algorithm~\ref{app:alg:beam-topk-hamming}), all probability bookkeeping ($\log p$, $Z_t$, etc.) matches our baseline $k$-CBS algorithm (Algorithm~\ref{app:alg:beam-topk}).
For history (including the prefix) $\genseq$, for each $\gentok\in\sS_{t+1}(\genseq)$, append $\gentok$ (i.e., $\genseq' = \genseq \Vert \gentok$), update the banded row in $O(\varepsilon)$, and keep the child if and only if the updated band minimum $\le\varepsilon$ (i.e., $\min_{j\in[j_{\min}(t+1),j_{\max}(t+1)]}D[t{+}1,j]\le\varepsilon$) (Lemma~\ref{lem:ukkonen-prune-and-safety}).
This is the streamed Ukkonen viability condition, adapted to our fixed-length-$\suflen$ target and per-row band.

\item \textbf{Across-beam prune.} Among $\varepsilon$-viable children, sort by $\log p'$ and keep up to $\bw$ in the beam, as usual.
(Non-$\varepsilon$-viable partial paths are discarded before ranking.)

\item \textbf{Intermediate UB bookkeeping.}
For $t\!<\!\suflen$ we bank only across-beam pruned $\varepsilon$-viable partial paths.
That is, as with the Hamming-pruned variant (Appendix~\ref{app:sec:prune:hamming}), we can bank the mass of pruned but still $\varepsilon$-viable nodes (band minimum $\le\varepsilon$ at prune time, but $\log p$ did not sort this path into the top-$\bw$) for a potentially tighter upper bound; 
nodes pruned after they become non-$\varepsilon$-viable (band minimum $>\varepsilon$) contribute~$0$ to this banked mass.\looseness=-1

\item \textbf{Final acceptance test.}
At $t\!=\!\suflen$, accept a continuation if and only if $\dptab[\suflen,\suflen]\le\varepsilon$;
that way, every continuation in the final set $\sF$ has $\lev\le\varepsilon$ by construction.
Unlike the $\ham$-pruned version, we need an explicit final acceptance check.
The $\varepsilon$-viability test used during search asks whether the partial path is within $\varepsilon$ edits of \emph{some} prefix of the target (the minimum over $j$), which is the correct criterion for safe pruning.
But a $\suflen$-length continuation can pass that test---by aligning well to a shorter prefix ($j < \suflen$)---while still having $\dptab[\suflen,\suflen] > \varepsilon$, i.e., exceeding $\varepsilon$ edits against the full target.
The explicit check $\dptab[\suflen,\suflen] \leq \varepsilon$ addresses this.
(For Hamming, no separate check is needed: the mismatch 
counter at step $\suflen$ already equals $\ham(\cont, \suf)$.)

\item \textbf{Bounds.} 
We do not perform the last across-beam prune, keeping up to $\bw \cdot k$ finals. 
$\mathrm{LB}_{\varepsilon,\levshort}$ is the sum of $\exp(\log p)$ over all kept final sequences (all are $\levshort\le\varepsilon$, by construction); 
non-$\varepsilon$-viable finals contribute to neither the lower bound nor the upper bound bank.
$\mathrm{UB}_{\varepsilon,\levshort}=\textit{bank}+\mathrm{LB}_{\varepsilon,\levshort}$, with no mass counted twice.

\item \textbf{Early termination.}
If at some $t\!<\!\suflen$ no $\varepsilon$-viable children remain ($\sC^{\leq\varepsilon}_t\!=\!\emptyset$), the algorithm terminates with $\mathrm{LB}_{\varepsilon,\levshort}=0$ and $\mathrm{UB}_{\varepsilon,\levshort}=\textit{bank}$,
i.e., the mass banked from $\varepsilon$-viable nodes pruned so far, which upper-bounds the remaining near-verbatim mass by the descendant-sum/frontier identity, just as in the Hamming-pruned variant (Lemma~\ref{lem:meta-early-ub}).
Additionally, the same minimum-probability early termination used in the Hamming variant applies here: if the highest-probability beam candidate falls below a threshold derived from a user-specified minimum extraction probability, we terminate with $\mathrm{LB}_{\varepsilon,\levshort}=0$ and $\mathrm{UB}_{\varepsilon,\levshort}=\textit{bank}$.
In practice, the upper bound from early-terminated examples is too loose to be informative; we filter these at reporting time.
\end{itemize}

\paragraph{Complexity and practical choices.}
Scalar-banded Wagner-Fischer~\citep{wagnerfischer1974} (with Ukkonen's criterion~\citep{ukkonen}) requires $O(\varepsilon)$ time and $O(\varepsilon)$ memory per partial path per step, totaling $O(\bw\,k\,\suflen\,\varepsilon)$ time and $O(\bw\,k\,\varepsilon)$ peak memory across the beam (reduced to $O(\bw\,\varepsilon)$ after each across-beam prune).
This is purely bookkeeping---no extra forward passes through the model---so the overall cost is dominated by the same number of token evaluations as in baseline $k$-CBS.

\begin{algorithm}[t]
\caption{Levenshtein-$\varepsilon$-pruned Top-$k$ Constrained Beam Search}\label{app:alg:beam-topk-lev}
\small
\KwIn{LLM $\model$; prefix $\pre$ of length $\prelen$; target suffix $\suf=(\suftok{1},\ldots,\suftok{\suflen}) \in \vocab^\suflen$; beam width $\bw$;
  top-$k$ parameter $k$ for decoding policy $\dec$ ($\bw \leq k^2$, $k \ll |\vocab|$);
  Levenshtein distance budget $\varepsilon \ll \suflen$;
  EOS token id; optional $\taumin > 0$}
\KwOut{Set $\sF$ of (at most $\bw \cdot k$) triples
  $(\genseq,\; \log p,\; d)$,
  where $\genseq = \pre \,\Vert\, \gensuf$ is a full history with length $\prelen + \suflen$, $\log p = \log \Pr_{\model,\dec}(\gensuf \mid \pre)$, and $d = \dptab[\suflen,\suflen] = \lev(\gensuf, \suf) \leq \varepsilon$;
  lower bound $\mathrm{LB}_{\varepsilon,\levshort}$; upper bound $\mathrm{UB}_{\varepsilon,\levshort}$}
\BlankLine
\textbf{Notation.} Same as Algorithm~\ref{app:alg:beam-topk}.
Additionally: $\suftok{t}$ denotes the $t$-th token of $\suf$;
$\dptab[t,j] \triangleq \lev(\subgensuf{1:t},\,\subsuf{1:j})$ (Equation~\ref{eq:lev-dp-table});
$[j_{\min}(t),\, j_{\max}(t)]$ is the Ukkonen band at depth $t$ (Equation~\ref{eq:ukkonen-band}).
We write $\dptab[t,\cdot]$ for the banded row $\big(\dptab[t,j]\big)_{j\in[j_{\min}(t),\,j_{\max}(t)]}$.
\BlankLine
\textbf{Beam state.}
Maintain $\sL_t$ as triples $(\genseq,\log p, \dptab[t,\cdot])$, where $\log p$ is the accumulated top-$k$ decoding $\log$-probability and $\dptab[t,\cdot]$ is the banded Wagner-Fischer row for the partial path at depth $t$.
Max.\ beam capacity $|\sL_t|=\bw$.
\BlankLine
\textbf{Viability test (not monotone in $t$; see Appendix~\ref{app:sec:prune:lev:wf}).} 
After appending token $\gentok$ at depth $t$, compute the updated banded row $\dptab[t,j]$ for $j\in[j_{\min}(t),\,j_{\max}(t)]$ via Equation~\ref{eq:wf:dp:recurrence}, treating out-of-band predecessors as $+\infty$.
Keep child only if $\min_{j\in[j_{\min}(t),\,j_{\max}(t)]} \dptab[t,j] \leq \varepsilon$; otherwise prune (Equation~\ref{eq:ukkonen-test}).
\BlankLine
Compute $\logit_1(\pre)$ via forward pass on $\pre$\tcp*{Prefill: $\prelen$ token evals}
$\sL_0 \gets \big\{(\pre,\; 0,\; \dptab[0,\cdot])\big\}$ where $\dptab[0,j]=j$ for $j\in[0,\,\min\{\suflen,\varepsilon\}]$\tcp*{Init beam}
$\textit{bank} \gets 0$\tcp*{UB mass accumulator}
\BlankLine
\For{$t = 1, \ldots, \suflen$}{
  $\sC^{\leq\varepsilon}_t \gets \emptyset$\tcp*{$\varepsilon$-viable candidate set for step $t$}
  \ForEach{$(\genseq,\; \log p,\; \dptab[t{-}1,\cdot]) \in \sL_{t-1}$}{
    $\sS_t(\genseq) \gets \TopK_k\big(\logit_t(\genseq)\big)$\;
    $\vr_t(\genseq) \gets \LogSoftmax\big(\logit_t(\genseq)\big)$\;
    $Z_t(\genseq) \gets \LogSumExp\big(\vr_t(\genseq)[\sS_t(\genseq)]\big)$\;
    \ForEach{$\gentok \in \sS_t(\genseq)$}{
      Compute $\dptab[t,j]$ for $j\in[j_{\min}(t),\,j_{\max}(t)]$ from $\dptab[t{-}1,\cdot]$ and $\gentok$\tcp*{$O(\varepsilon)$ DP update}
      \If(\tcp*[f]{$\varepsilon$-viable: keep}){$\min_{j\in[j_{\min}(t),\,j_{\max}(t)]} \dptab[t,j] \leq \varepsilon$}{
        $\genseq' \leftarrow \genseq \,\Vert\, \gentok$ \tcp*{Append token to partial history}
        $\log p' \leftarrow \log p + \vr_t(\genseq)[\gentok] - Z_t(\genseq)$ \tcp*{Update continuation $\log$ prob}
        $\sC^{\leq\varepsilon}_t \gets \sC^{\leq\varepsilon}_t \cup
          \big\{\!\big(\genseq',\;\;
          \log p',\;\; \dptab[t,\cdot] \big)\!\big\}$\;
      }
      \tcp{If band min $> \varepsilon$: prune (no viable completion; Eq.~\ref{eq:ukkonen-test})}
    }
  }
  \BlankLine
  \If(\tcp*[f]{Final step: acceptance test and bounds}){$t = \suflen$}{
    $\sF \gets \big\{(\genseq',\,\log p',\,\dptab[\suflen,\suflen]) \in \sC^{\leq\varepsilon}_\suflen : \dptab[\suflen,\suflen] \leq \varepsilon\big\}$\tcp*{Accept iff $\lev \leq \varepsilon$}
    $\mathrm{LB}_{\varepsilon,\levshort} \gets \sum_{(\_,\,\log p,\,\_)\in\sF}\exp(\log p)$\;
    $\mathrm{UB}_{\varepsilon,\levshort} \gets \mathrm{LB}_{\varepsilon,\levshort} + \textit{bank}$\;
    \Return $\sF,\; \mathrm{LB}_{\varepsilon,\levshort},\; \mathrm{UB}_{\varepsilon,\levshort}$\;
  }
  \BlankLine
  \tcp{Non-final step ($t < \suflen$): EOS handling, early termination, beam pruning}
  \ForEach{$(\genseq',\; \log p',\; \dptab[t,\cdot]) \in \sC^{\leq\varepsilon}_t$ with latest $\gentok = \textup{EOS}$}{
    Record $(\genseq',\; \log p',\; t)$ as early-termination path;
      remove from $\sC^{\leq\varepsilon}_t$\;
  }
  \BlankLine
  \If(\tcp*[f]{Early termination: no $\varepsilon$-viable candidates remain}){$\sC^{\leq\varepsilon}_t = \emptyset$}{
    \Return $\sF \gets \emptyset,\; \mathrm{LB}_{\varepsilon,\levshort} \gets 0,\; \mathrm{UB}_{\varepsilon,\levshort} \gets \textit{bank}$\;
  }
  \BlankLine
  $\sU_t \gets$ top-$\bw$ elements of $\sC^{\leq\varepsilon}_t$ ranked by $\log p'$\;
  $\textit{bank} \gets \textit{bank} + \sum_{(\_,\,\log p',\,\_)\in\sC^{\leq\varepsilon}_t\setminus\sU_t}\exp(\log p')$\tcp*{Bank pruned $\varepsilon$-viable mass}
  $\sL_t \gets \sU_t$ \tcp*{Prune to beam width}
  \If{$\taumin$ and $\max_{(\_,\,\log p,\,\_) \in \sL_t} \exp(\log p) < \taumin / (\bw \cdot k)$}{
    \Return $\sF \gets \emptyset,\; \mathrm{LB}_{\varepsilon,\levshort} \gets 0,\; \mathrm{UB}_{\varepsilon,\levshort} \gets \textit{bank}$\;
  }
  Compute $\logit_{t+1}(\genseq)$ for each $(\genseq, \cdot, \cdot) \in \sL_t$\tcp*{$|\sL_t|$ token evals}
}
\end{algorithm}
\clearpage

\subsection{Invariants for $\varepsilon$-pruned $k$-CBS}\label{app:sec:prune:invariants}

To show invariants for the $\varepsilon$-pruned $k$-CBS algorithms, we first give some definitions and common facts that we will use in our proofs (Appendix~\ref{app:sec:prune:invariants:generic-viability-kcbs}), and some specific results for each distance metric (Appendix~\ref{app:sec:prune:invariants:specific-prelim}).
Then we present proofs of the respective algorithm invariants, leveraging a unified framework (Appendix~\ref{app:sec:prune:invariants:viable-finals}). 
We show that all returned finals are $\varepsilon$-viable under the chosen distance metric for pruning, observations about the cardinality bounds of the returned finals, and early termination upper bounds. 
It is not the case that the lower and upper bounds returned by $\varepsilon$-pruned variants are guaranteed to be tighter than those identified by $k$-CBS, but we often observe them to be in practice.\looseness=-1  

\subsubsection{Distance-viability-pruned $k$-CBS: Notation and background}
\label{app:sec:prune:invariants:generic-viability-kcbs}

\paragraph{Shared viability semantics.}
The Hamming- and Levenshtein-pruned variants share identical beam-search machinery---expansion, probability bookkeeping, banking, early termination---and differ only in how viability is tracked.
We abstract the distance-specific logic into three operations on a per-path auxiliary state $\mathrm{aux}$, parameterized by a target suffix $\suf$ and distance budget~$\varepsilon$:

\begin{itemize}[leftmargin=0.5cm]
\item $\mathsf{Viable.Init}(\suf,\,\varepsilon) \to \mathrm{aux}_0$\;: initialize per-path viability state.
\item $\mathsf{Viable.Update}(\mathrm{aux},\,\gentok,\,t,\,\suf,\,\varepsilon) \to (\mathrm{aux}',\,\varepsilon_\star)$\;:
update $\mathrm{aux}$ after appending token $\gentok$ at depth~$t$, and return a lower bound $\varepsilon_\star$ on the final distance $\dist(\gensuf,\suf)$ achievable by any completion of this partial path.
If $\varepsilon_\star > \varepsilon$, the path is provably non-$\varepsilon$-viable and is pruned.\looseness=-1 
\item $\mathsf{Viable.IsFinal}(\mathrm{aux},\,\varepsilon) \to \{0,1\}$\;:
at depth~$\suflen$, return $1$ if the completed continuation satisfies $\dist(\gensuf,\suf) \leq \varepsilon$, and $0$ otherwise.
\end{itemize}

We instantiate these semantics for the Hamming and Levenshtein distance as follows: 

\begin{center}
{\small
\renewcommand{\arraystretch}{1.4}
\begin{tabular}{@{}l@{\;\;}p{4.8cm}@{\;\;}p{6.8cm}@{}}
\toprule
& \textbf{Hamming} ($\dist = \ham$) & \textbf{Levenshtein} ($\dist = \lev$) \\
\midrule
$\mathrm{aux}$
  & Mismatch count $\mismatch \in \{0,\ldots,\varepsilon\}$
  & Banded DP row $\dptab[t,\cdot] = \big(\dptab[t,j]\big)_{j\in[j_{\min}(t),\,j_{\max}(t)]}$ \\
\textsf{Init}
  & $\mathrm{aux}_0 = 0$
  & $\dptab[0,j] = j$ for $j \in [0,\,\min\{\suflen,\varepsilon\}]$ \\
\textsf{Update}
  & $\mismatch' \leftarrow \mismatch + \1[\gentok \neq \suftok{t}]$; \newline $\varepsilon_\star \leftarrow \mismatch'$
  & Fill $\dptab[t,\cdot]$ via Eq.~\ref{eq:wf:dp:recurrence} (banded; Eq.~\ref{eq:ukkonen-band}); \newline $\varepsilon_\star \leftarrow \min_{j\in[j_{\min}(t),\,j_{\max}(t)]} \dptab[t,j]$ \\
\textsf{IsFinal}
  & Always $1$ (viability $\Rightarrow$ acceptance)
  & $\1[\dptab[\suflen,\suflen] \leq \varepsilon]$ \\
\midrule
Cost per call & $O(1)$ & $O(\varepsilon)$ \\
\bottomrule
\end{tabular}
}%
\end{center}

For Hamming, the mismatch count is monotonically non-decreasing along any path, so $\varepsilon_\star$ never decreases---and surviving to depth~$\suflen$ with $\mismatch \leq \varepsilon$ already guarantees $\ham(\gensuf, \suf) \leq \varepsilon$, making \textsf{IsFinal} trivially true.
For Levenshtein, $\varepsilon_\star$ can decrease between steps (Appendix~\ref{app:sec:prune:lev:wf}), but remains a valid lower bound at each step;
the explicit final check $\dptab[\suflen,\suflen] \leq \varepsilon$ is needed because the viability test during search only ensures alignment to \emph{some} prefix of~$\suf$, not the full target (Appendix~\ref{app:sec:prune:lev:alg-details}).
We show how this interface can get integrated with $k$-CBS in Algorithm~\ref{alg:kcbs-pruned}. 

This viability interface accommodates any token-level distance metric for which a lower bound on the final distance can be computed incrementally from streamable per-path state, e.g., $\mathsf{LCS}$ distance and weighted edit distances.
For these semantics, we provide common notation, soundness criteria for viability pruning rules, and some basic facts that we use in our proofs.\looseness=-1 

\begin{algorithm}[t]
\caption{$\varepsilon$-pruned Top-$k$ Constrained Beam Search ($\dist$-pruned $k$-CBS)}\label{alg:kcbs-pruned}
\small
\KwIn{LLM $\model$; prefix $\pre$; target suffix $\suf \in \vocab^\suflen$; beam width $\bw$; top-$k$ parameter $k$;
  distance budget $\varepsilon$; viability oracle $(\mathsf{Init},\, \mathsf{Update},\, \mathsf{IsFinal})$;
  EOS token id; optional $\taumin > 0$}
\KwOut{Set $\sF$ of (at most $\bw \cdot k$) triples $(\genseq,\;\log p,\;\mathrm{aux})$ with $\mathsf{IsFinal}(\mathrm{aux},\varepsilon)=1$;
  $\mathrm{LB}_{\varepsilon,\dist}$; $\mathrm{UB}_{\varepsilon,\dist}$}
\BlankLine
Compute $\logit_1(\pre)$ via forward pass on $\pre$\tcp*{Prefill}
$\sL_0 \gets \big\{(\pre,\; 0,\; \mathsf{Viable.Init}(\suf,\varepsilon))\big\}$\tcp*{Beam (max.\ capacity $\bw$): triples $(\genseq,\;\log p,\;\mathrm{aux})$}
$\textit{bank} \gets 0$\;
\BlankLine
\For{$t = 1, \ldots, \suflen$}{
  $\sC^{\leq\varepsilon}_t \gets \emptyset$\;
  \ForEach{$(\genseq,\; \log p,\; \mathrm{aux}) \in \sL_{t-1}$}{
    $\sS_t \gets \TopK_k\big(\logit_t(\genseq)\big)$; \,
    $\vr_t \gets \LogSoftmax\big(\logit_t(\genseq)\big)$; \,
    $Z_t \gets \LogSumExp\big(\vr_t[\sS_t]\big)$\;
    \ForEach{$\gentok \in \sS_t$}{
      $(\mathrm{aux}',\, \varepsilon_\star) \gets \mathsf{Viable.Update}(\mathrm{aux},\, \gentok,\, t,\, \suf,\, \varepsilon)$\;
      \If{$\varepsilon_\star \leq \varepsilon$}{
        $\genseq' \leftarrow \genseq \,\Vert\, \gentok$; \,
        $\log p' \leftarrow \log p + \vr_t[\gentok] - Z_t$; \,
        $\sC^{\leq\varepsilon}_t \gets \sC^{\leq\varepsilon}_t \cup \big\{\!\big(\genseq',\, \log p',\, \mathrm{aux}'\big)\!\big\}$\;
      }
    }
  }
  \BlankLine
  \If(\tcp*[f]{Final step: acceptance test and bounds}){$t = \suflen$}{
    $\sF \gets \big\{(\genseq',\,\log p',\,\mathrm{aux}') \in \sC^{\leq\varepsilon}_\suflen : \mathsf{Viable.IsFinal}(\mathrm{aux}',\,\varepsilon) = 1\big\}$\;
    $\mathrm{LB}_{\varepsilon,\dist} \gets \sum_{(\_,\,\log p',\,\_)\in\sF}\exp(\log p')$; \,
    $\mathrm{UB}_{\varepsilon,\dist} \gets \mathrm{LB}_{\varepsilon,\dist} + \textit{bank}$\;
    \Return $\sF,\; \mathrm{LB}_{\varepsilon,\dist},\; \mathrm{UB}_{\varepsilon,\dist}$\;
  }
  \BlankLine
  Remove from $\sC^{\leq\varepsilon}_t$ any $(\genseq',\, \log p',\, \mathrm{aux}')$ whose last token is EOS\;
  \lIf(\tcp*[f]{No $\varepsilon$-viable candidates remain}){$\sC^{\leq\varepsilon}_t = \emptyset$}{\Return $\emptyset,\; 0,\; \textit{bank}$}
  $\sU_t \gets$ top-$\bw$ of $\sC^{\leq\varepsilon}_t$ by $\log p'$\;
  $\textit{bank} \gets \textit{bank} + \sum_{(\_,\,\log p',\,\_)\in\sC^{\leq\varepsilon}_t\setminus\sU_t}\exp(\log p')$\tcp*{Bank pruned $\varepsilon$-viable mass}
  $\sL_t \gets \sU_t$\;
  \lIf{$\taumin$ and $\max_{(\_,\,\log p,\,\_) \in \sL_t} \exp(\log p) < \taumin / (\bw \cdot k)$}{\Return $\emptyset,\; 0,\; \textit{bank}$}
  Compute $\logit_{t+1}(\genseq)$ for each $(\genseq, \cdot, \cdot) \in \sL_t$\;
}
\end{algorithm}

\paragraph{Token-sequence notation.} We recall the sequence notation from Appendix~\ref{app:sec:kcbs:notation}.
As noted in that section, this notation lets us refer to prefixes, suffixes, and partial continuations directly, avoiding index arithmetic on the full history~$\genseq$.

\paragraph{Candidate sets and beam operations.}
We extend the notation in Appendix~\ref{app:sec:kcbs:notation} to also include $\varepsilon$-viability pruning. 
At depth $t\in\{0,\ldots,\suflen\}$, denote
\begin{itemize}[leftmargin=0.5cm]
    \item $\sL_t$: the beam at depth $t$ (at most $\bw$ elements), obtained by expanding $\sL_{t-1}$, applying $\varepsilon$-viability and across-beam pruning.
    \item $\sC_t$: all children formed from $\sL_{t-1}$ by one-step top-$k$ expansion (at most $|\sL_{t-1}| \cdot k$ candidates, including non-$\varepsilon$-viable ones).
    \item $\sC^{\leq\varepsilon}_t \subseteq \sC_t$: the subset of $\varepsilon$-viable children ($\varepsilon_\star \leq \varepsilon$). Algorithm~\ref{alg:kcbs-pruned} constructs $\sC^{\leq\varepsilon}_t$ directly by filtering during expansion.
    \item $\sU_t$: the (at most) top-$\bw$ of $\sC^{\leq\varepsilon}_t$ by cumulative probability (across-beam prune); $\sL_t \gets \sU_t$.
    \item $\sF$, the set of returned $\varepsilon$-viable finals;
    it is possible that we may return early (at $t < \suflen$) with $\emptyset$, if no such finals exist. 
\end{itemize}

As in the invariants for $k$-CBS, we denote $\mathrm{Desc}_\suflen(\vu)$ the set of depth-$\suflen$ descendants (for any node $\vu$ at depth $t$) reachable by continuing the same per-state top-$k$ policy (ignoring across-beam pruning) from $\vu$.
That is, $\mathrm{Desc}_\suflen(\vu)=\{\vw\in\vocab^\suflen :\vu \text{ is a prefix of }\vw\}$ is the set of depth-$\suflen$ descendants of $\vu$ (i.e., the full-length continuations with prefix $\vu$, under the top-$k$ policy).\looseness=-1 

\paragraph{Generic per-depth $t$ viability predicate.}
Fix a target suffix $\suf$ of length $\suflen$ and a maximum distance (i.e., tolerance) $\varepsilon$ from $\suf$.
A \newterm{viability predicate} is a family $\{\mathsf{Viable}_t(\cdot)\}_{t=0}^{\suflen}$ with the following property. For any generated continuation $\gensuf$ of prefix $\pre$:
\begin{itemize}[leftmargin=0.5cm]
    \item \customtarget{soundness}{\textbf{(Soundness.)}}
    $\mathsf{Viable}_t$ is a necessary condition for $\varepsilon$-viability:
    if $\dist(\gensuf, \suf) \leq \varepsilon$, then
    $\mathsf{Viable}_t(\subgensuf{1:t})$ holds for every
    $t \in \{0, \ldots, \suflen\}$.
    Equivalently (by contrapositive), if
    $\neg\,\mathsf{Viable}_t(\subgensuf{1:t})$, then no
    length-$\suflen$ completion of $\subgensuf{1:t}$ satisfies
    $\dist(\cdot,\, \suf) \leq \varepsilon$, so the path may be
    safely pruned.
\end{itemize}
Viability at step $t$ depends only on the partial $\subgensuf{1:t}$ and the target suffix $\suf$, not on the prefix $\pre$, as the history $\genseq$ and the training sequence $\seq$ share the same $\prelen$-token prefix.

For $\ham$, $\mathsf{Viable}_t$ is ``running mismatch counter $\le\varepsilon$'' (\customlink{soundness}{soundness} via monotonicity; see Lemma~\ref{lem:ham-monotone}).
For $\lev$, $\mathsf{Viable}_t$ is ``banded row minimum $\le\varepsilon$ at row $t$'' (\customlink{soundness}{soundness} via Ukkonen's criterion; see Lemma~\ref{lem:ukkonen-prune-and-safety}).

\paragraph{Common facts used throughout.}
All candidates at a fixed depth are ordered by a strict total order $\succ$ on cumulative log-probabilities $\log p'$,
with a fixed deterministic tie-breaker.
For a finite set $\sX$ and $\vx\in\sX$, let $\mathrm{rank}_{\sX}(\vx)$ be the position of $\vx$ under $\succ$ in $\sX$ ($1$ is highest).
\begin{itemize}[leftmargin=0.5cm]
    \item \customtarget{f1}{F1 (rank under restriction).}
    If $\sY\subseteq \sX$ and $\vx\in \sY$, then
    \[
        \mathrm{rank}_{\sY}(\vx)\le \mathrm{rank}_{\sX}(\vx)
    \quad \text{(restricting to a subset can only remove items that outrank $\vx$)}.
    \]

    \item \customtarget{f2}{F2 (no cross-parent duplicates at a depth).}
    Children formed from distinct parents at the same depth are token-distinct
    (no duplicates across parents);
    see Lemma~\ref{lem:unique}.
    Consequently,
    \[
        |\sC_t| \le k\,|\sL_{t-1}| \le \bw \cdot k, \quad |\sC^{\leq \varepsilon}_t|\le|\sC_t|, \quad \text{and} \quad |\sU_t|\le|\sC^{\leq \varepsilon}_t|.
    \]

    \item \customtarget{f3}{F3 (descendant-sum identity).}
    For any node $\vu$ at depth $t<\suflen$,
    \begin{equation*}
    \sum_{\vx\in \mathrm{Desc}_\suflen(\vu)} \Pr_{\model,\dec}(\vx\mid\pre)
    \;=\;
    \Pr_{\model,\dec}(\vu\mid\pre).
    \end{equation*}
    Under per-state top-$k$ renormalization, a parent's mass equals the sum of its immediate children's masses;
    iterating to depth $\suflen$ yields the identity.
    (See Lemma~\ref{lem:frontier-mass}, Equation~\ref{eq:descendant-sum}.)

    \item \customtarget{f4}{F4 (equal per-path probability across variants).}
    The probability assigned to a path depends only on the model's conditional distributions along that path's history, not on the other beam elements.
    Therefore, any depth-$\suflen$ path $\gensuf$ enumerated by both baseline and viability-pruned $k$-CBS receives the same probability in both:
    \[
    \Pr_{\model,\dec}^{\textnormal{base}}(\gensuf\mid\pre)
    \;=\;
    \Pr_{\model,\dec}^{\textnormal{prune}}(\gensuf\mid\pre).
    \]
    (See Lemma~\ref{lem:topk-step-update} and Corollary~\ref{cor:topk-path}.)
\end{itemize}

\subsubsection{Distance-specific results for streaming generation $\varepsilon$-viability}\label{app:sec:prune:invariants:specific-prelim}

Our proof framework will depend on a distance $\dist$ satisfying the generic per-depth $t$ viability predicate, $\{\mathsf{Viable}_t(\cdot)\}_{t=0}^{\suflen}$, described in the prior subsection.
For $\ham$, $\mathsf{Viable}_t$ is ``running mismatch counter $\le\varepsilon$,'' for which we show \customlink{soundness}{soundness} via monotonicity in Lemma \ref{lem:ham-monotone}.
For $\lev$, $\mathsf{Viable}_t$ is ``banded row minimum $\le\varepsilon$ at row $t$,'' for which we show \customlink{soundness}{soundness} in Lemma~\ref{lem:ukkonen-prune-and-safety}.

We start with Hamming: 
extending a continuation cannot decrease the Hamming distance to the target suffix;
once a partial path becomes non-$\varepsilon$-viable, it remains non-$\varepsilon$-viable.

\begin{lemma}[Monotone Hamming $\varepsilon$-viability]
\label{lem:ham-monotone}
Let $\suf=(\suftok{1},\ldots,\suftok{\suflen})$
and let $\subgensuf{1:t}=(\subgensuftok{1},\ldots,\subgensuftok{t})$.
Define the running Hamming counter
\[
\mismatch_t(\subgensuf{1:t},\suf)\;=\;\sum_{i=1}^{t}\1\!\big[\subgensuftok{i}\neq \suftok{i}\big],
\]
which is the Hamming distance at $t$ between a $t$-length continuation and the first $t$ tokens of the $\suflen$-length suffix (Equation~\ref{eq:ham-counter}).
Then $t\mapsto \mismatch_t(\subgensuf{1:t},\suf)$ is nondecreasing.
In particular, if $\mismatch_t(\subgensuf{1:t},\suf)>\varepsilon$ for some $t<\suflen$, every length-$\suflen$ extension has Hamming distance $>\varepsilon$ and cannot lie in $\hamball(\suf)$.
\end{lemma}

\begin{proof}
$\mismatch_{t+1}=\mismatch_t+\1[\subgensuftok{t+1}\neq \subsuf{t+1}]\ge \mismatch_t$.
So, if $\mismatch_t>\varepsilon$, later edits cannot decrease the Hamming distance, so all descendants remain non-$\varepsilon$-viable. 
It is therefore safe to prune children once they become non-$\varepsilon$-viable;
$\varepsilon$-viable finals must also be $\varepsilon$-viable at all prefixes $t$.\looseness=-1
\end{proof}

The Hamming results rely on monotonicity of the mismatch counter (Lemma~\ref{lem:ham-monotone}), which does not hold for the Levenshtein distance.
Instead, we establish \customlink{soundness}{soundness} via \citeauthor{ukkonen}'s criterion on the banded \citeauthor{wagnerfischer1974} dynamic program.
First, we show that the band has width at most $2\varepsilon+1$.

\begin{lemma}[Band restriction]
\label{lem:band-restriction}
Fix a row $i$.
If $|j-i|>\varepsilon$, then $\dptab[i,j]>\varepsilon$.
Therefore, any column with $\dptab[i,j]\le\varepsilon$ must lie in $[j_{\min}(i),\,j_{\max}(i)]$, whose width is at most $2\varepsilon+1$.
\end{lemma}

\begin{proof}
By \Eqref{eq:length-gap-lb}, $\dptab[i,j]\ge |i-j|$. If $|i-j|>\varepsilon$ then $\dptab[i,j]>\varepsilon$, so any entry with value $\le\varepsilon$ must satisfy $|i-j|\le\varepsilon$.
The condition $|i-j|\le \varepsilon$ is equivalent to
\(
i-\varepsilon \ \le\ j \ \le\ i+\varepsilon
\).
Since indices must also lie in $[0,\suflen]$, the admissible columns at row $i$ lie in
\(
j\ \in\ [\,j_{\min}(i),\,j_{\max}(i)\,]
\ \equiv\
\big[\,\max\{0,\,i-\varepsilon\},\ \min\{\suflen,\,i+\varepsilon\}\,\big],
\)
as in \Eqref{eq:ukkonen-band}.
The interval $[i-\varepsilon,\,i+\varepsilon]$ contains
\[
(i+\varepsilon) - (i-\varepsilon) + 1 \;=\; 2\varepsilon + 1
\]
integer columns.
After intersecting with $[0,\suflen]$, the width can only decrease.
Therefore, the band has width at most $2\varepsilon+1$.
\end{proof}

\begin{lemma}[Ukkonen band $\varepsilon$-viability soundness]
\label{lem:ukkonen-prune-and-safety}
Fix $\varepsilon\ge 0$ and a target suffix $\suf\in\vocab^{\suflen}$.
For each row $t\in\{0,\ldots,\suflen\}$ define the Ukkonen band
\[
  [j_{\min}(t),\,j_{\max}(t)]
  \;\coloneqq\;
  \big[\max\{0,\,t-\varepsilon\},\ \min\{\suflen,\,t+\varepsilon\}\big]
  \quad\text{(Equation~\ref{eq:ukkonen-band})}.
\]
Let $\dptab[t,j]=\lev\big(\subgensuf{1:t},\,\subsuf{1:j}\big)$ be the Wagner-Fischer dynamic program value (Equation~\ref{eq:lev-dp-table}).
The predicate
\[
\mathsf{Viable}_t(\subgensuf{1:t})
\;\coloneqq\;
\Big[\min_{\,j\in[j_{\min}(t),\,j_{\max}(t)]} \dptab[t,j] \;\leq\; \varepsilon\Big]
\]
satisfies \customlink{soundness}{soundness}:
if $\lev(\gensuf,\,\suf) \leq \varepsilon$, then $\mathsf{Viable}_t(\subgensuf{1:t})$ holds for every $t \in \{0,\ldots,\suflen\}$.
Equivalently, if $\mathsf{Viable}_t(\subgensuf{1:t})$ fails at any $t$, no length-$\suflen$ completion of $\subgensuf{1:t}$ satisfies $\lev(\cdot,\,\suf) \leq \varepsilon$, so the path may be safely pruned.
\end{lemma}

\begin{proof}
We prove the pruning direction (the converse follows by contrapositive).
Suppose $\min_{j\in[j_{\min}(t),\,j_{\max}(t)]} \dptab[t,j] > \varepsilon$ for some $t < \suflen$.
By Lemma~\ref{lem:band-restriction}, any column with $\dptab[t,j]\le\varepsilon$ must lie inside the band, so in fact \emph{all} entries in row $t$ exceed $\varepsilon$.
Since edit costs are non-negative, any completion of $\subgensuf{1:t}$ to length $\suflen$ can only add cost, so $\dptab[\suflen,\suflen] > \varepsilon$ and no $\varepsilon$-viable completion exists.

For the boundary case of $t=\suflen$, note that $j=\suflen\in[j_{\min}(\suflen),\,j_{\max}(\suflen)]$ and $\dptab[\suflen,\suflen]=\lev(\gensuf,\suf)\le\varepsilon$, so the band minimum is $\le\varepsilon$.
\end{proof}

\subsubsection{General invariants that follow}\label{app:sec:prune:invariants:viable-finals}

Algorithm~\ref{alg:kcbs-pruned} returns only paths passing the $\mathsf{IsFinal}$ check, so every returned final is $\varepsilon$-viable.

\begin{corollary}[All returned finals are $\varepsilon$-viable]
\label{cor:outputs-in-ball}
Every returned final $(\genseq,\,\log p,\,\mathrm{aux})\in\sF$ from Algorithm~\ref{alg:kcbs-pruned} satisfies $\dist(\gensuf,\,\suf)\le\varepsilon$, i.e., $\gensuf\in\ball(\suf)$.
\end{corollary}

\begin{proof}
Algorithm~\ref{alg:kcbs-pruned} returns only paths where $\mathsf{Viable.IsFinal}(\mathrm{aux},\,\varepsilon) = 1$.
For $\ham$, viability at depth $\suflen$ implies $\mismatch_\suflen(\subgensuf{1:\suflen},\suf) \leq \varepsilon$, and so $\ham(\gensuf,\,\suf) \leq \varepsilon$ (Lemma~\ref{lem:ham-monotone}).
For $\lev$, $\mathsf{IsFinal}$ checks $\dptab[\suflen,\suflen] \leq \varepsilon$, which equals $\lev(\gensuf,\,\suf)$ by definition of the dynamic programming table (Equation~\ref{eq:lev-dp-table}).
\end{proof}

We next show a simple result about the cardinality of the returned finals.

\begin{lemma}[Cardinality of returned finals under top-$k$ and a sound pruning rule]
\label{lem:meta-cardinality}
Assume the standard beam mechanics that we use throughout:
(i) after across-beam pruning, $|\sL_t|\le \bw$ for every step $t$;
(ii) each parent forms at most $k$ children via per-state top-$k$ token expansion;
(iii) no cross-parent duplicates at a given depth $t$ (\customlink{f2}{F2}, Lemma~\ref{lem:unique});
and (iv) $\varepsilon$-viability filtering and final-acceptance checks only remove candidates ($\sC^{\leq\varepsilon}_t \subseteq \sC_t$, $\sF \subseteq \sC^{\leq\varepsilon}_\suflen$).
Then, at depth $\suflen$, $|\sF|\le \bw \cdot k$.  
If there exists $t<\suflen$ with $\sU_t=\emptyset$, the run terminates early with $\sF=\emptyset$.
\end{lemma}

\begin{proof}
At depth $\suflen-1$, $|\sL_{\suflen-1}|\le \bw$ by (i).
Each parent produces exactly $k$ children via top-$k$ expansion (ii), and by (iii) no two parents produce the same child, so
\[
|\sC_{\suflen}| \;=\; k\,|\sL_{\suflen-1}|
\;\le\; \bw \cdot k.
\]
At the final step there is no across-beam prune; by (iv), $\varepsilon$-viability filtering and the $\mathsf{IsFinal}$ check can only remove candidates, so $|\sF| \leq |\sC^{\leq\varepsilon}_\suflen| \leq |\sC_\suflen| \leq \bw \cdot k$.
If $\sU_t = \emptyset$ for some $t < \suflen$, the run halts with $\sF = \emptyset$ by construction.
\end{proof}

Algorithms~\ref{app:alg:beam-topk-hamming} ($\ham$-pruned) and \ref{app:alg:beam-topk-lev} ($\lev$-pruned) satisfy (i)–(iv) in Lemma~\ref{lem:meta-cardinality}, so it applies exactly in both cases.

We now give a unified early-termination upper bound covering both $\ham$- and $\lev$-pruned variants.

\begin{lemma}[Early-termination UB under sound viability pruning]
\label{lem:meta-early-ub}
Fix $\varepsilon$ and a viability predicate satisfying \customlink{soundness}{soundness}.
Run Algorithm~\ref{alg:kcbs-pruned} and let $\sR_{\textnormal{prune}}$ denote the set of nodes removed by the across-beam prune during the run.
Since the across-beam prune acts on $\sC^{\leq\varepsilon}_t$, every node in $\sR_{\textnormal{prune}}$ is $\varepsilon$-viable at the time of pruning.
Suppose the run terminates early at some depth $t < \suflen$ with $\sC^{\leq\varepsilon}_t = \emptyset$, so that $\sF = \emptyset$ and $\lb = 0$.
Then
\[
\pseqe \;\leq\; \textit{bank}
\;=\;
\sum_{(\vu,\,\log p) \in \sR_{\textnormal{prune}}} \exp(\log p).
\]
\end{lemma}

\begin{proof}
Fix any length-$\suflen$ continuation $\gensuf$ with $\dist(\gensuf,\,\suf) \leq \varepsilon$.
By \customlink{soundness}{soundness}, $\subgensuf{1:i}$ is $\mathsf{Viable}_i$ for every $i \in \{0, \ldots, \suflen\}$.
If $\subgensuf{1:t-1}$ were in $\sL_{t-1}$, the expansion at step $t$ would generate $\subgensuf{1:t}$, which is $\varepsilon$-viable by \customlink{soundness}{soundness} and therefore in $\sC^{\leq\varepsilon}_t$---contradicting $\sC^{\leq\varepsilon}_t = \emptyset$.
So $\subgensuf{1:t-1} \notin \sL_{t-1}$, and some earlier prefix of $\gensuf$ must have been removed by the across-beam prune.
Let
\begin{align*}
t^\star(\gensuf)\;&\coloneqq\;\min\big\{\,i\in\{1,\ldots,\suflen-1\}:\ \text{$\subgensuf{1:i}$ was removed by the across-beam prune}\,\big\},\\
\vu(\gensuf)\;&\coloneqq\;\subgensuf{1:t^\star(\gensuf)}.
\end{align*}
That is, $t^\star(\gensuf)$ is the earliest step at which some prefix of $\gensuf$ was discarded by the across-beam prune, and $\vu(\gensuf)$ is that pruned prefix.
By \customlink{soundness}{soundness}, $\vu(\gensuf)$ was $\varepsilon$-viable when pruned, so $\big(\vu(\gensuf),\,\log p\big) \in \sR_{\textnormal{prune}}$.
More generally, once a node is pruned, no descendants are ever constructed, so no element of $\sR_{\textnormal{prune}}$ is an ancestor of another and the descendant sets $\{\mathrm{Desc}_\suflen(\vu) : (\vu, \cdot) \in \sR_{\textnormal{prune}}\}$ are pairwise disjoint.
Every $\varepsilon$-viable $\gensuf$ belongs to exactly one such set, namely $\mathrm{Desc}_\suflen\big(\vu(\gensuf)\big)$.
Therefore,
\[
\pseqe
\;\leq\;
\sum_{(\vu,\,\_) \in \sR_{\textnormal{prune}}}
\;\sum_{\vx \in \mathrm{Desc}_\suflen(\vu)} \Pr_{\model,\dec}(\vx\mid\pre)
\;=\;
\sum_{(\vu,\,\log p) \in \sR_{\textnormal{prune}}} \exp(\log p)
\;=\;
\textit{bank},
\]
where the inequality holds because each $\varepsilon$-viable continuation is a descendant of exactly one pruned node (and non-$\varepsilon$-viable descendants only add mass), and the first equality is the descendant-sum identity (\customlink{f3}{F3}, Equation~\ref{eq:descendant-sum}).
\end{proof}

Both the Hamming (Algorithm~\ref{app:alg:beam-topk-hamming}) and Levenshtein (Algorithm~\ref{app:alg:beam-topk-lev}) variants satisfy the assumptions of Lemma~\ref{lem:meta-early-ub}: \customlink{soundness}{soundness} is established by Lemma~\ref{lem:ham-monotone} and Lemma~\ref{lem:ukkonen-prune-and-safety}, respectively, so $\pseqe \leq \textit{bank}$ upon early termination in both cases.

Last, recall that at the top of this appendix (Appendix~\ref{app:sec:prune}) we noted that $\varepsilon$-pruned $k$-CBS does not provably produce a tighter (dominating) lower bound than baseline $k$-CBS (with post-processing for distance $\leq \varepsilon$).
We now give a concrete example showing how viability pruning can cause the pruned algorithm to miss an $\varepsilon$-viable candidate that the baseline would have retained.
The same mechanism implies that the upper bound $\mathrm{UB}_{\varepsilon,\dist}$ is also not guaranteed to be tighter than the analogous (trivial) baseline upper bound. 

\medskip

\begin{remark}[No dominance guarantee]\label{rem:no-dominance}
There exist inputs for which $\mathrm{LB}_{\varepsilon,\dist}$ from $\varepsilon$-pruned $k$-CBS (Algorithm~\ref{alg:kcbs-pruned}) is strictly less than the lower bound obtained by running baseline $k$-CBS (Algorithm~\ref{app:alg:beam-topk}) and post-processing for distance $\leq\varepsilon$.
The same holds for the upper bound: the pruned $\mathrm{UB}_{\varepsilon,\dist}$ is neither provably larger nor provably smaller than the baseline upper bound, because the two algorithms explore different regions of the search space and account for unexplored mass differently.
This holds for both $\dist = \ham$ and $\dist = \lev$.
\end{remark}

\paragraph{Counterexample.}
Take $\bw = k = 2$.
Suppose that at some step $t$, the two algorithms have already diverged due to earlier viability pruning, and their beams differ. 
\begin{itemize}[leftmargin=0.5cm]
\item \textbf{Baseline beam} $\sL_t = \{(\vx, \log p(\vx)),\, (\vu, \log p(\vu))\}$, where $\vx$ is $\varepsilon$-viable and $\vu$ is not.
(The baseline retains $\vu$ because it does not prune by viability.)

\item \textbf{Pruned beam} $\sL_t = \{(\vx, \log p(\vx)),\, (\vw, \log p(\vw))\}$, where both $\vx$ and $\vw$ are $\varepsilon$-viable.
($\vw$ entered the pruned beam at an earlier step when $\vu$---or another non-$\varepsilon$-viable candidate---was pruned, freeing a beam slot.)
(We drop $\mathrm{aux}$ in the beam below, as we do not need to reference  it to make this point; it would be present in practice.) 
\end{itemize}

\noindent At step $t+1$, both algorithms expand their $\bw = 2$ beams to $\bw \cdot k = 4$ candidates.
Write $p(\cdot) = \exp(\log p(\cdot))$ for the cumulative path probability under the top-$k$ distribution.
Because top-$k$ renormalization ensures the $k$ conditional probabilities sum to~$1$, the $k$ children's cumulative probabilities $p'$ sum to exactly the parent's $p$.

Suppose $p(\vx) = 0.10$, $p(\vu) = 0.50$, $p(\vw) = 0.30$.
(These are three distinct depth-$t$ paths in the top-$k$ tree, so their sum $0.90 \leq 1$.)

\paragraph{Baseline expansion.} (from $\vx$ and $\vu$, with $\vu$ non-viable):
\[
\vx \;\to\; \{\vx^{(1)}\;(p'{=}0.06),\;\; \vx^{(2)}\;(p'{=}0.04)\},
\qquad
\vu \;\to\; \{\vu^{(1)}\;(p'{=}0.45),\;\; \vu^{(2)}\;(p'{=}0.05)\}.
\]
Top-$\bw = 2$ by probability: $\{\vu^{(1)},\, \vx^{(1)}\}$.
Because $\vu$ is non-viable, no descendant of $\vu$ can be $\varepsilon$-viable at depth $\suflen$: for $\ham$, this follows from monotonicity of $\mismatch_t$ (Lemma~\ref{lem:ham-monotone}); for $\lev$, from \customlink{soundness}{soundness} (Lemma~\ref{lem:ukkonen-prune-and-safety}).
So $\vu^{(1)}$ cannot contribute to the baseline lower bound despite occupying a beam slot.
The baseline's surviving $\varepsilon$-viable candidate at step $t+1$ is $\vx^{(1)}$.

\paragraph{Pruned expansion.} (from $\vx$ and $\vw$, both viable):
\[
\vx \;\to\; \{\vx^{(1)}\;(p'{=}0.06),\;\; \vx^{(2)}\;(p'{=}0.04)\},
\qquad
\vw \;\to\; \{\vw^{(1)}\;(p'{=}0.18),\;\; \vw^{(2)}\;(p'{=}0.12)\}.
\]
All four candidates are $\varepsilon$-viable.
Top-$\bw = 2$: $\{\vw^{(1)},\, \vw^{(2)}\}$.
Candidate $\vx^{(1)}$, with $p'(\vx^{(1)}) = 0.06 < 0.12 = p'(\vw^{(2)})$, is \emph{pruned from the beam}.

\paragraph{Consequence.}
The baseline's beam at step $t+1$ contains the $\varepsilon$-viable candidate $\vx^{(1)}$; the pruned beam does not.
If $\vw^{(1)}$ and $\vw^{(2)}$ later become non-$\varepsilon$-viable (their descendants diverge from the reference), they contribute no mass to the pruned lower bound---but $\vx^{(1)}$'s mass has already been lost from the pruned beam.
Meanwhile, if $\vx^{(1)}$ survives to the final step in the baseline, it contributes to the baseline lower bound.
The pruned algorithm traded $\vx^{(1)}$ for $\vw^{(1)}$ and $\vw^{(2)}$, and that trade can be a net loss.

\paragraph{Mechanism.}
Viability pruning reclaims beam slots occupied by non-$\varepsilon$-viable candidates ($\vu$), admitting new $\varepsilon$-viable candidates ($\vw$) whose children ($\vw^{(1)}, \vw^{(2)}$) can then crowd out $\varepsilon$-viable candidates ($\vx^{(1)}$) that the baseline would have retained.
In effect, the pruned algorithm explores a \emph{different} region of the search space, not a strict superset of the baseline's.

\paragraph{Empirical observation.}
Even so, in our experiments we consistently observe $\mathrm{LB}_{\varepsilon,\dist} \geq \mathrm{LB}_{\varepsilon,\dist}^{\textnormal{baseline}}$.
Non-$\varepsilon$-viable beam slots cannot contribute to the final lower bound, so replacing them with $\varepsilon$-viable candidates tends to increase the mass of $\varepsilon$-close completions discovered by the search.
The counterexample above requires the newly-admitted viable candidates' children to be numerous enough and probable enough to displace a previously-reachable viable candidate---a situation that appears to be rare in practice.\looseness=-1

\vspace{-.1cm}
\section{Details on experiments}\label{app:sec:experiments}
\vspace{-.1cm}

We next provide additional results and details on all experiments.
We first describe the metrics and visualizations that we use in our analysis (Appendix~\ref{app:sec:experiments:metrics}), 
then provide information about the specific models, datasets, and code (Appendix~\ref{app:sec:experiments:setup}).
We provide extended results for \textsc{OLMo 2} on Wikipedia (Appendix~\ref{app:experiments:olmo}), \textsc{Pythia} on Enron (Appendix~\ref{app:experiments:pythia}), and \textsc{Llama 2} (Appendix~\ref{app:experiments:llama2-scale}) and \textsc{Llama 3.1 8B} (Appendix~\ref{app:experiments:llama3}) on public domain books.

Following \citet{cooper2025books}, we set the minimum extraction threshold $\taumin=0.001$, which we also validate with our negative controls. 
We defer to \citet{cooper2026firstprinciples} for additional details on how calibrate more fine-grained thresholds. 
In most experiments, we use beam width $\bw=20$, as we observe this setting to work well for top-$k$ with $k=40$.
We make this decision based on the experiments shown in Appendix~\ref{app:experiments:width};
the gains from larger beam widths are marginal.\footnote{Of course, one could start with a wider beam and let $\varepsilon$-pruning naturally narrow it.
But to actually achieve improved efficiency through narrowing, we would need to implement batch compaction, which is not trivial.
In general, we defer improved efficiency, implementation, and detailed exploration of algorithm parameter tuning to future work.}
In Appendix~\ref{app:sec:experiments:prune}, we also do a full analysis comparing baseline $k$-CBS (Section~\ref{sec:kcbs:baseline}) to Hamming- and Levenshtein-pruned (Section~\ref{sec:pruning}) $k$-CBS, in terms of how changing $\varepsilon$ as a constraint impacts results.
Finally, we also run brief analysis (following \citet{ippolito-etal-2023-preventing}) on baseline $k$-CBS outputs for BLEU score (Appendix~\ref{app:sec:bleu}). 

We omit analysis comparing running $k$-CBS with and without the last-iteration across-beam prune.
We find that omitting the last prune leads to small overall gains in identifying extractable sequences, and can lead to large gains in extraction risk for individual sequences.
Since we get the up $\bw \cdot k$ outputs for free (i.e., with no additional model forward passes), we just report results for this configuration.  

We find that the Levenshtein-pruned algorithm performs the best:
it is faster than the baseline method, and in practice generally returns tighter lower bounds. 
We conclude from this analysis that running the $\levshort\, \varepsilon=5$ configuration is the reasonable choice;
the results can be post-processed for smaller $\varepsilon$ with minimal compromise in quality for those lower-$\varepsilon$ runs.\looseness=-1

The project \href{https://afedercooper.info/near-verbatim/}{website} contains additional visualizations.

\subsection{Metrics and visualizations for assessing extraction}\label{app:sec:experiments:metrics}

We detail the different types of metrics and visualizations that we use to analyze near-verbatim extraction: extraction rates (Section~\ref{app:sec:experiments:metrics:rates}), visualizing near-verbatim extractable sequences that verbatim extraction misses and their increased risk (Section~\ref{app:sec:experiments:metrics:unlocked-and-risk}), and per-sequence extraction mass composition (Section~\ref{app:sec:experiments:metrics:composition}).
For books, we also provide heatmaps that visualize where extraction occurs in a book, following \citet{cooper2025books} (Section~\ref{app:sec:experiments:metrics:heatmaps})

\subsubsection{Extraction rates}\label{app:sec:experiments:metrics:rates}

Extraction rates are the most common metric in the literature 
\citep{carlini2021extracting, carlini2023quantifying, nasr2023scalable, nasr2025scalable,hayes2025measuringmemorizationlanguagemodels,lee2022dedup}.
For our purposes, plotting extraction rates (like Figure~\ref{fig:olmo:rates:main}) is useful for showing how the number of extractable sequences changes over different settings---greedy vs.\ probabilistic, verbatim vs.\ near-verbatim.
As discussed in Section~\ref{sec:rw} and Appendix~\ref{app:sec:background:extraction:rate}, for some population of sequences $\seqset$:
\begin{align}
\mathsf{extraction\_rate}(\seqset; \suc) \;\triangleq\; \frac{1}{|\seqset|}\sum_{\seq\in\seqset} \1[\,\suc(\seq)\,], \tag{\ref{eq:extraction-rate}}
\end{align}
where $\suc(\seq)$ is a success predicate defined according to the chosen criterion.
For greedy extraction, the success criterion is
\begin{align}
    \suc^\dist_{\model,\greedy,\varepsilon}(\seq) \;\triangleq\;
    \1\Big[\, \dist\big(\gensuf, \suf\big) \le \varepsilon \,\Big] \tag{\ref{eq:success-greedy-nv}},
\end{align}
where setting $\varepsilon=0$ captures the verbatim case.
For probabilistic extraction and a chosen (validated) minimum extraction probability $\taumin$, the success criterion is
\begin{align}
\suc^\dist_{\model,\dec,\varepsilon}(\seq;\taumin)
\;\triangleq\;
\1\big[\, \pseqe \;\geq\; \taumin\,\big], \tag{\ref{eq:prob-success-nv}}
\end{align}
where similarly setting $\varepsilon=0$ captures the verbatim case. 

Note that extraction rates flatten information about variation in extraction risk surfaced by probabilistic extraction.
Since greedy extraction is deterministic, this flattening has no effect;
it just counts greedy extracted sequences. 
But for probabilistic extraction, any extractable sequence (irrespective of its risk) gets counted equally in the extraction rate.

\subsubsection{Visualizing ``unlocked'' sequences and per-sequence risk increase}\label{app:sec:experiments:metrics:unlocked-and-risk}

We use two types of visualizations for extraction risk.
First, we use scatter plots (like Figure~\ref{fig:gatsby:scatter:main}) to visualize how extractability and risk change per sequence.
This gives a qualitative sense of how risk changes for the set of sequences.
On a $\log$--$\log$ scale, we plot a sequence's near-verbatim mass as a function of its verbatim mass, showing the line $y=x$ as a reference point. 
We highlight three different categories in different colors.
\textcolor{seabornbluemid}{\textbf{Blue}} points are verbatim-extractable ($\pseq\!\geq\!\taumin$).
Points above the dashed $y\!=\!x$ line therefore show increased extraction risk when near-verbatim mass is accounted for.
\textcolor{seabornredmid}{\textbf{Red}}/\textcolor{seabornorangemid}{\textbf{orange}}
points are ``unlocked'' by near-verbatim extraction. 
They are all to the left of the $\taumin$ dotted reference line. 
For \textcolor{seabornorangemid}{\textbf{orange}} points, $0\!<\!\pseq\!<\!\taumin$, but $p_{\seq,5}^\levshort\!\geq\!\taumin$).
That is, these points have some verbatim mass, but insufficient mass to be verbatim extractable; 
they become extractable when the near-verbatim mass is included.
\textcolor{seabornredmid}{\textbf{Red}} points have zero verbatim mass (which is why they are not on the $\log$--$\log$ plot canvas), but are near-verbatim extractable. 

Second, we provide complementary cumulative distribution functions (CCDFs) of per-sequence near-verbatim mass gain.
For each sequence $\seq$ in the training dataset, we compute the per-sequence mass gain 
\begin{align}
\label{app:eq:per-seq-mass-gain}
\Delta_\seq \triangleq \hat{p}^\levshort_{\seq,5} - \pseq,
\end{align}
i.e., the mass for the Levenshtein distance with $\varepsilon=5$ minus the verbatim mass from $k$-CBS results. 
This is the absolute additional extraction risk revealed by relaxing from verbatim to near-verbatim for that specific sequence.
We plot the CCDF of these gains.
The $x$-axis is the per-sequence mass gain threshold $\Delta_\seq$. 
The $y$-axis is the percentage of all sequences in the population (not just extractable ones) with gain $> \Delta_\seq$

Only sequences with strictly positive gain appear in the curve, but the denominator is the full set, so the $y$-intercept gives the fraction of all sequences that have any near-verbatim mass gain.
The curve then steps down as $\Delta_\seq$ increases, showing the tail of the gain distribution.
A point on the CCDF (e.g., and annotation like $1.2\%$ at $\Delta_\seq=0.001$) means $1.2\%$ of all sequences in the population have per-sequence absolute mass gain $\geq0.001$ from the near-verbatim relaxation;
at $\Delta_\seq = 0.1$, it shows the fraction of sequences that gain (absolute) mass gain $\geq 0.1$.\looseness=-1

A curve shifting right/upward for larger models means that larger models have both more sequences with gains and larger gains per sequence---a heavier-tailed distribution of near-verbatim risk increase.

We also provide tables containing values from the CCDF over per-sequence mass gain (Equation~\ref{app:eq:per-seq-mass-gain}), computed with respect to each model/dataset's \emph{fixed} extractable set.
This shows the gain just for extracted sequences, and is not easily comparable across experiments since the extractable set is different (i.e., each is normalized by a different denominator). 

\subsubsection{Analyzing per-sequence risk composition}\label{app:sec:experiments:metrics:composition}

We provide three different ways to visualize how per-sequence extraction mass is allocated at different Levenshtein distances. 
First, for each extractable sequence ($p^\levshort_{\seq,5} \geq \taumin = 10^{-3}$), we
decompose its total near-verbatim mass into incremental \newterm{$\varepsilon$-shells}: 
\begin{align}
\label{app:eq:mass-shells}
\Delta_\seq(\varepsilon) \triangleq \hat{p}^\levshort_{\seq,\,\varepsilon} - \hat{p}^{\levshort}_{\seq,\,\varepsilon-1}, 
\end{align}
where $\Delta_\seq(0) = \pseq$ (verbatim mass).
While Equation~\ref{app:eq:per-seq-mass-gain} computes the mass gain with respect to the near-verbatim mass using the most permissive distance threshold and the verbatim mass, the above equation diffs the mass between adjacent thresholds to identify how much mass is exactly attributable to a specific $\levshort$ distance---a shell in the $\varepsilon$-ball. 

For $\varepsilon \in \{0,\ldots 5\}$, we can compute the \newterm{$\varepsilon$-shell share}:
\begin{align}
\label{app:eq:shell-fraction}
\mathsf{shell\_share}(\seq,\varepsilon)\triangleq \frac{\Delta_\seq(\varepsilon)}{\hat{p}^\levshort_{\seq,\,5}}.
\end{align}
For each sequence, the respective $\varepsilon$-shell shares sum to $1$.
We also focus particular attention on the \newterm{verbatim share} shell fraction (i.e., for $\Delta_\seq(0)$):
\begin{align}
\label{app:eq:verbatim-share}
\mathsf{verbatim\_share}(\vz) = \frac{\pseq}{\hat{p}^\levshort_{\seq,\,5}}. 
\end{align}
From computing the $\varepsilon$-shell share for every $\varepsilon \in \{0,\ldots,5\}$ using Equation~\ref{app:eq:shell-fraction}, we can produce a per-sequence heatmap of the mass contributions at each $\levshort$ distance to a sequence's overall extraction mass (e.g., Figure~\ref{app:fig:olmo:heatmap}).
Each row represents an extracted sequence (with rows sorted by the verbatim mass share, see Equation~\ref{app:eq:verbatim-share}). 
The columns are the $\varepsilon$-shells.
Each cell's color is the intensity of that shell's fraction for that sequence.
For instance, sequences at the top have all of their mass allocated to the verbatim continuation;
the verbatim share is $100\%$, so that shell is indicated to have $100\%$ of the share and the other $5$ shells indicate $0\%$.
The heatmap as a whole visualizes heterogeneity across sequences, with respect to which $\levshort$ distances contribute different relative amounts of mass. 

We also plot distributions over extracted-sequence $\varepsilon$-mass share.
In the main paper, we provide comparative plots across model sizes for the verbatim share---i.e., Equation~\ref{app:eq:verbatim-share}, for the same dataset across a model family (e.g., Figure~\ref{fig:violin:main}). 
We provide this type of plot here as well, and complement it with per-model-and-dataset distribution breakdowns for each $\varepsilon$-shell's mass (e.g., Figure~\ref{app:fig:olmo:incremental}). 
These show the distribution of per-sequence $\varepsilon$-shell fractions, with one violin plot per $\varepsilon$-shell. 
They reveal the spread and central tendency (if there is one) of how mass is distributed across $\levshort$ distances. 

\subsubsection{Heatmaps to visualize book extraction coverage}\label{app:sec:experiments:metrics:heatmaps}

We use the same heatmap visualization as in \citet{cooper2025books} for book memorization. 
For the books experiments, we produce sequences by sliding through the length of the book, taking $100$-token segments every $20$ characters (Appendix~\ref{app:sec:experiments:setup:llama}). 
The sequences therefore overlap (which is intentional, see \citet{cooper2025books}). 
Multiple suffixes cover each character position in the book; 
for each position, we show the maximum extraction probability across covering suffixes, so that the heatmap reflects the highest extraction risk at each location. 
Larger probabilities have darker color/higher intensity (on a $\log$ scale), and everything below $\taumin$ is shown in white (not extractable). 
In \citet{cooper2025books}, the heatmaps are built from verbatim extraction probabilities ($\varepsilon = 0$).
Here, we show near-verbatim extraction probabilities $\hat{p}^\levshort_{\seq, \varepsilon}$ for the chosen distance threshold $\varepsilon$. 
We omit heatmaps for the Hamming distance.\looseness=-1

\subsection{Experimental settings}\label{app:sec:experiments:setup}

Being a \newterm{member} of the training data is a necessary (though not sufficient) condition for extraction: 
by definition, extraction only applies to memorized training data.
Therefore, the simplest setting for studying memorization and extraction is to run experiments on known training data.
For open models, in some cases we have ground-truth knowledge about the training data, and so we leverage these cases to simplify our experimental setup.
For each extraction experiment, we test models with data that was known with certainty to be in the training dataset, and run experiments for three different settings:
Wikipedia entries, public domain books, and emails.

Even when we know if a particular sequence was included in the training data, it is often quite challenging to know with certainty that a particular sequence was \emph{not} in the training data 
\citep{lee2023explainers, cooper2024unlearning}.
This is because models are typically trained on enormous web scrapes, which contain duplicates and near-duplicates of content available from multiple sources~\citep{lee2022dedup}. 
Nevertheless, to assess the validity of an extraction procedure, it is important to make a best-effort to run \newterm{negative controls}:
experiments on non-training data (\newterm{held-out data}), which should register no extraction signal with our measurement procedure.
If our extraction procedure registers such signal on non-training data (where extraction is impossible), then the procedure produces false positives. 
In memorization research, it is more common to tolerate false negatives than false positives~\citep{hayes2025strongmia, cooper2025books};
there is a preference toward conservative claims originating from the \newterm{membership inference} literature in ML security~\citep{carlini2022membership, shokri2017membership}. 
We therefore make a best effort to pair our extraction experiments with appropriate negative controls~\citep{cooper2026firstprinciples}.\looseness=-1 

We detail each model and dataset below, including which datasets we run for which model based on known information about training-set membership.
Altogether, we test open-weight, non-instruction-tuned LLMs from three different families. 
We obtain all model weights from HuggingFace.
All code can be found on \href{https://github.com/pasta41/probabilistic-extraction-toolkit}{GitHub}. 

\subsubsection{\textsc{OLMo 2} on Wikipedia}\label{app:sec:experiments:setup:olmo} 

\paragraph{Models.} 
We run experiments on all three \textsc{OLMo 2} model sizes: \href{https://huggingface.co/allenai/OLMo-2-1124-7B}{7B}, \href{https://huggingface.co/allenai/OLMo-2-1124-13B}{13B}, and \href{https://huggingface.co/allenai/OLMo-2-0325-32B}{32B}. 
        
\paragraph{Training data subset.} 
The training dataset for \textsc{OLMo 2} is publicly available, and the overall training recipe is documented in the release report~\citep{olmo2}.
Wikipedia was included in the training data, and we draw a subset of these entries for training data for our extraction experiments.
Specifically, we draw $10{,}000$ unique Wikipedia entries, which has a cutoff date of December 2023, where each entry is at least $100$ tokens long.  
The data can be found on \href{https://huggingface.co/datasets/allenai/olmo-mix-1124}{HuggingFace}.
We specifically streamed records from \texttt{data/wiki/wiki-0000.json.gz} and \texttt{data/wiki/wiki-0001.json.gz} from the \texttt{allenai/olmo-mix-1124} (used for training \textsc{OLMo 2}).
We take the first $100$ tokens of each entry as a sequence.
For convenience, we provide the subset of the data that we use as the input file for our experiments 
\href{https://drive.google.com/file/d/1K76kfAG0hXxgTiinvWm-0U4k148ALN-n/view?usp=drive_link}{here}.

\paragraph{Negative control (held-out subset).}
We curated a set of $10{,}000$ Wikipedia entries that post-date \textsc{OLMo 2}'s training data.
To obtain these pages, we paginated through Wikipedia's \texttt{logevents} API (\texttt{letype=create}, main namespace) from February 24, 2026 back to January 1, 2024. 
For each batch of $500$ creation events, we checked the page lengths via \texttt{prop=info} and kept pages with $\geq\!8{,}000$ bytes of wikitext. 
We then \newterm{reservoir sampled}\footnote{Reservoir sampling obtains a uniform sample of $N$ items from a stream of unknown size, without loading everything into memory.
We place the first $N$ items into the reservoir; 
then, for every subsequent item we encounter, we assign a shrinking random chance of swapping it into the reservoir, replacing a random existing entry.
By the time we have iterated through every item, each item in the entire stream has had an equal probability of ending up in the final reservoir.} $100{,}000$ pages that fit these conditions from the whole range of Wikipedia pages in this time frame, de-duplicating by page ID.
(This means we only ever hold $100{,}000$ pages in memory at one time.) 
We maintained this full list of pages.
However, since page-fetching can be slow, we trimmed the $100{,}000$ pages to $30{,}000$ (randomly sampled), and filtered redirects and disambiguation pages using \texttt{prop=pageprops|info}.
Then, we verified each page's creation date, fetching the oldest revision (\texttt{rvdir=newer\&rvlimit=1}) and confirming the timestamp was on or after \texttt{2024-01-01}, removing any pages that failed this check. 
We then fetched the plain text for each page one at a time (\texttt{prop=extracts\&explaintext=true}) and filtered by length, dropping any page with fewer than $1{,}500$ characters of plain text.
We then trimmed to $10{,}000$ total pages, shuffled them, and saved them as JSON (both metadata and full article text).
We provide the dataset 
\href{https://drive.google.com/file/d/1UjN_EfzPUlspD9PmTBf_WUfhyKc9idTH/view?usp=sharing}{here}, 
and note that we ran experiments using only the first $5{,}000$ pages.\looseness=-1

This is also an imperfect negative control, as old deleted Wikipedia pages (that were in the training data) may be re-created and posted with new create dates;
these new pages (containing old content) would then not actually post-date the training cutoff, and our process would pull them in.
There are also instances of highly templated text in Wikipedia pages;
new pages can be near-duplicates of older pages.
We document observed instances of this in our negative controls in Appendix~\ref{app:experiments:olmo:extraction}. 

\subsubsection{{\textsc{Pythia} on emails}}\label{app:sec:experiments:setup:pythia} 

\paragraph{{Models}.} 
We run experiments on the four largest \textsc{Pythia} model sizes: \href{https://huggingface.co/EleutherAI/pythia-1b}{1B}, \href{https://huggingface.co/EleutherAI/pythia-2.8b}{2.8B}, \href{https://huggingface.co/EleutherAI/pythia-6.9b}{6.9B}, and \href{https://huggingface.co/EleutherAI/pythia-12b}{12B}~\citep{pythia}. 

\paragraph{{Training data subset}.} 
The Pythia suite was trained on the Pile~\citep{gao2020pile}, which contains multiple copies of the Enron email dataset. 
Similar to \citet{hayes2025measuringmemorizationlanguagemodels}, we therefore use Enron (obtained from the Pile) as our sample of \textsc{Pythia} training data in extraction experiments.
We use $10{,}000$ unique emails, where each email is at least $100$ tokens long. 
We take the first $100$ tokens of each email as a sequence.
We provide the specific version of the Enron dataset we used 
\href{https://drive.google.com/file/d/1WG7GlNSP0ZbsLODzxl7P0Lj9kMmKXbLT/view?usp=sharing}{here}.

\paragraph{Negative control (held-out subset).}
As a corresponding negative control, we run experiments on a subset of \href{https://github.com/imdeepmind/Preprocessed-TREC-2007-Public-Corpus-Dataset}{TREC 2007 Spam}. 
We adapt this experiment from \citet{hayes2025measuringmemorizationlanguagemodels}.
However, we note from manual inspection of the ``ham'' subset that many of these emails contain verbatim news articles (i.e., are news digests), which are very likely to be in Common Crawl and other web scrapes, and thus are likely candidates to be training data~\citep{lee2023explainers}. 
We therefore filter these emails out, and use a resulting set of $2{,}000$ emails that are at least $100$ tokens long for our held-out subset, which we provide 
\href{https://drive.google.com/file/d/1ns5P1CKDKUqDHPUVoVP3wFore9cVR4qW/view?usp=drive_link}{here}. 
We take the first $100$ tokens of each email as a sequence.
    
\subsubsection{{\textsc{Llama} on public domain books}}\label{app:sec:experiments:setup:llama} 

\paragraph{{Models}.} 
We run experiments on \href{https://huggingface.co/huggyllama/llama-13b}{\textsc{Llama 1 13B}}~\citep{touvron2023llamaopenefficientfoundation} (Table~\ref{app:tab:diffs-ex-gatsby}), the \textsc{Llama 2} family (\href{https://huggingface.co/meta-llama/Llama-2-7b}{7B}, \href{https://huggingface.co/meta-llama/Llama-2-13b}{13B}, \href{https://huggingface.co/meta-llama/Llama-2-70b}{70B})~\citep{llama2}, and \href{https://huggingface.co/meta-llama/Llama-3.1-8B}{\textsc{Llama 3.1 8B}}~\citep{llama3}. 
We pick these models based on results from \citet{cooper2025books}:
we use the same illustrative example from that work involving \textsc{Llama 1 13B}; 
we use \textsc{Llama 2} as it has three sizes available  to examine different scales;
and we use \textsc{Llama 3.1 8B} to also examine one model from the more recent series.
For concision, we predominantly omit for \textsc{Llama 3.1 8B}.

\paragraph{{Training data subset}.} 
It is public knowledge that Meta trained all \textsc{Llama} models on books, and that the first three generations of these models (through the last minor versions of \textsc{Llama 3}) were trained on the Books3 corpus~\citep{touvron2023llamaopenefficientfoundation,kadreyamendedconsolidated,cooper2025books}. 
We pull four public domain books from Books3: 
\textit{The Great Gatsby}~\citep{The_Great_Gatsby}, \textit{Winnie the Pooh}~\citep{Winnie_the_Pooh}, \textit{Orlando}~\citep{Orlando}, and \textit{Pride and Prejudice}~\citep{Pride_and_Prejudice}.
We trim any edition-specific front matter (e.g., editor's forward) and back matter (e.g., index) so that the file only contains text from the book. 
We then chunk up the book into $100$-token sequences, starting from the first character and using a stride length of $20$ characters. 
We upload the text files for these trimmed four public domain books 
\href{https://drive.google.com/drive/folders/1CoSFs4Rcv7ijb4j2NZwY6RAS6wCuiSQ_?usp=sharing}{here}. 
To create sequences from books, we follow the procedure in \citet{cooper2025books}, where we take $100$-token sequences every $s$ characters;
we set $s=20$.
For the \textsc{Llama 2} tokenizer, this yields $13{,}390$ sequences for \emph{The Great Gatsby}, $21{,}822$ sequences for \emph{Orlando}, $6{,}152$ sequences for \emph{Winnie the Pooh}, and $10{,}457$ for \emph{The People's Dictator} (first three chapters). 
For the \textsc{Llama 3.1} tokenizer, this yields $34{,}271$ sequences for \emph{Pride and Prejudice} and $10{,}454$ for \emph{The People's Dictator} (first three chapters).

\paragraph{Negative control (held-out subset).}
We use \emph{The People's Dictator} as a negative control. 
This is an openly licensed (CC-BY-SA-NC) book from 2025---well after \textsc{Llama 3}'s training date cutoff.
We provide a link to the full book PDF 
\href{https://www.taylorfrancis.com/books/oa-mono/10.4324/9781032691848/people-dictator-alejandro-quiroga?_gl=1*19hy4lr*_gcl_au*MjExMjI2NzQ0OC4xNzc0NDg1OTk4*_ga*MTcxMDgyNjg2MS4xNzc0NDg2MDAw*_ga_0HYE8YG0M6*czE3NzQ0ODYwMDAkbzEkZzAkdDE3NzQ0ODYwMDAkajYwJGwwJGgw}{here}.
We removed the licensing information from the text file and trimmed to the first three chapters prior to running experiments. 

As \citet{cooper2025books} document, while this is a useful way to select a negative control, it is imperfect, as books may verbatim quote texts from prior to the training cutoff data;
the \emph{document} of a post-cutoff book may not be included in the training data, but \emph{specific text} within that document may be included in the training data from other sources.
A more ideal negative control would use a public domain novel that we know with certainty was not included in \textsc{Llama}'s training data;
however, we do not have sufficient information to identify such a book, and so instead use long-form narrative text from the book noted above.\looseness=-1 

\clearpage
\subsubsection{Summary of reported experiments}\label{app:sec:experiment:summary}

Table~\ref{tab:exp-summary} summarizes all experimental runs and configurations reported in this paper. 

\begin{table*}[h!]
\centering
\scriptsize
\setlength{\tabcolsep}{5pt}
\caption{\textbf{Reported experimental configurations.}
Configurations across datasets, models, and constraints
(where appropriate, beam width $\bw$, distance metric $\dist \in \{\hamshort,\levshort\}$, and tolerance $\varepsilon$).
We report the number of GPUs and batch size for each experiment.
NC = negative control; stride = number of characters used for chunking in book runs.}
\label{tab:exp-summary}
\begin{tabular}{llllll}
\toprule
\textbf{Domain} & \textbf{Dataset} & \textbf{Models} & \textbf{NC} & \textbf{Constraint} & \textbf{Run config} \\
\midrule

\multicolumn{6}{l}{\textbf{Verbatim probabilistic (teacher forcing); greedy generation}} \\
\midrule

Books
& \makecell[l]{\emph{Winnie} \\ \emph{the Pooh}}
& \makecell[l]{\textsc{Llama 2}\\ \{7B, 13B, 70B\}}
& No
& --
& \makecell[l]{batch $200$ (7B,13B), $400$ (70B);\\stride $20$; GPUs $1$/$1$/$4$}
\\

Books
& \makecell[l]{\emph{The Great}\\ \emph{Gatsby}}
& \makecell[l]{\textsc{Llama 2}\\ \{7B, 13B, 70B\}}
& No
& --
& \makecell[l]{batch $200$ (7B,13B), $400$ (70B);\\stride $20$; GPUs $1$/$1$/$4$}
\\

Books
& \makecell[l]{\emph{Orlando}}
& \makecell[l]{\textsc{Llama 2}\\ \{7B, 13B, 70B\}}
& No
& --
& \makecell[l]{batch $200$ (7B,13B), $400$ (70B);\\stride $20$; GPUs $1$/$1$/$4$}
\\

Books
& \makecell[l]{\emph{The People's}\\\emph{Dictator}}
& \makecell[l]{\textsc{Llama 2}\\ \{7B, 13B, 70B\};\\ \textsc{Llama 3.1}\\ 8B}
& Yes
& --
& \makecell[l]{batch $200$ (7B,13B,8B), $400$ (70B);\\stride $20$; GPUs $1$/$1$/$1$/$4$}
\\

Books
& \makecell[l]{\emph{Pride and}\\\emph{Prejudice}}
& \makecell[l]{\textsc{Llama 3.1}\\ 8B}
& No
& --
& \makecell[l]{batch $200$; stride $20$; GPUs $1$}
\\

Emails
& Enron
& \makecell[l]{\textsc{Pythia}\\ \{1B, 2.8B, 6.9B, 12B\}}
& No
& --
& \makecell[l]{batch $200$; stride --; GPUs $1$}
\\

Emails
& \makecell[l]{TREC 2007 Spam\\``ham'' subset}
& \makecell[l]{\textsc{Pythia}\\ \{1B, 2.8B, 6.9B, 12B\}}
& Yes
& --
& \makecell[l]{batch $200$; stride --; GPUs $1$}
\\

Wikipedia
& Train
& \makecell[l]{\textsc{OLMo 2}\\\{7B, 13B, 32B\}}
& No
& --
& \makecell[l]{batch $200$ (7B,13B), $400$ (32B);\\stride --; GPUs $1$/$1$/$4$}
\\

Wikipedia
& Held out
& \makecell[l]{\textsc{OLMo 2}\\\{7B, 13B, 32B\}}
& Yes
& --
& \makecell[l]{batch $200$ (7B,13B), $400$ (32B);\\stride --; GPUs $1$/$1$/$4$}
\\

\midrule
\multicolumn{6}{l}{\textbf{Sampling}} \\
\midrule
MC
& \makecell[l]{{Single sequence,}\\\emph{The Great Gatsby}}
& \makecell[l]{\textsc{Llama 2}\\ 7B}
& No
& --
& \makecell[l]{batch $200$; stride --; GPUs $1$; \\ $M=10{,}000$;\\ seeds \{1000,2000,3000\}}
\\

\midrule
\multicolumn{6}{l}{\textbf{$k$-CBS}} \\
\midrule

\textbf{baseline} & & & & &\\
\cmidrule{0-1}

One-off
& \makecell[l]{{Single sequence,}\\\emph{The Great Gatsby}}
& \makecell[l]{\textsc{Llama 1}\\ 13B}
& No
& \makecell[l]{$\bw \in \{20,30,40\}$}
& \makecell[l]{batch $1$; stride --; GPUs $1$}
\\

Emails
& Enron
& \makecell[l]{\textsc{Pythia}\\ \{1B, 2.8B, 6.9B, 12B\}}
& No
& \makecell[l]{$\bw=20$}
& \makecell[l]{batch $10$; stride --; GPUs $1$}
\\

Emails
& \makecell[l]{TREC 2007 Spam\\``ham'' subset}
& \makecell[l]{\textsc{Pythia}\\ \{1B, 2.8B, 6.9B, 12B\}}
& Yes
& \makecell[l]{$\bw=20$}
& \makecell[l]{batch $10$; stride --; GPUs $1$}
\\

Wikipedia
& Train
& \makecell[l]{\textsc{OLMo 2}\\\{7B, 13B, 32B\}}
& No
& \makecell[l]{$\bw=20$}
& \makecell[l]{batch $10$ (7B,13B), $20$ (32B);\\stride --; GPUs $1$/$1$/$4$}
\\

Wikipedia
& Held out
& \makecell[l]{\textsc{OLMo 2}\\\{7B, 13B, 32B\}}
& Yes
& \makecell[l]{$\bw=20$}
& \makecell[l]{batch $10$ (7B,13B), $20$ (32B);\\stride --; GPUs $1$/$1$/$4$}
\\

Books
& \makecell[l]{\emph{Winnie} \\ \emph{the Pooh}}
& \makecell[l]{\textsc{Llama 2}\\ \{7B, 13B, 70B\}}
& No
& \makecell[l]{$\bw=20$}
& \makecell[l]{batch $10$ (7B,13B), $20$ (70B); \\stride $20$; GPUs $1$/$1$/$4$}
\\

Books
& \makecell[l]{\emph{The People's}\\\emph{Dictator}}
& \makecell[l]{\textsc{Llama 2}\\ \{7B, 13B, 70B\}}
& No
& \makecell[l]{$\bw=20$}
& \makecell[l]{batch $10$ (7B,13B), $20$ (70B); \\stride $20$; GPUs $1$/$1$/$4$}
\\

\cmidrule{0-1}
\textbf{$\varepsilon$-viability pruned} & & & & &\\
\cmidrule{0-1}

Books
& \makecell[l]{\emph{Winnie} \\ \emph{the Pooh}}
& \makecell[l]{\textsc{Llama 2}\\ \{7B, 13B, 70B\}}
& No
& \makecell[l]{$\bw=20$;\\ $\hamshort,\, \varepsilon \in \{0,\ldots,5\}$}
& \makecell[l]{batch $10$ (7B,13B), $20$ (70B); \\stride $20$; GPUs $1$/$1$/$4$}
\\

Books
& \makecell[l]{\emph{Winnie} \\ \emph{the Pooh}}
& \makecell[l]{\textsc{Llama 2}\\ \{7B, 13B, 70B\}}
& No
& \makecell[l]{$\bw=20$;\\ $\levshort,\, \varepsilon \in \{0,\ldots,5\}$}
& \makecell[l]{batch $10$ (7B,13B), $20$ (70B); \\stride $20$; GPUs $1$/$1$/$4$}
\\

Books
& \makecell[l]{\emph{The People's}\\\emph{Dictator}}
& \makecell[l]{\textsc{Llama 2}\\ \{7B, 13B, 70B\}}
& Yes
& \makecell[l]{$\bw=20$;\\ $\hamshort,\, \varepsilon \in \{0,\ldots,5\}$}
& \makecell[l]{batch $10$ (7B,13B), $20$ (70B); \\stride $20$; GPUs $1$/$1$/$4$}
\\

Books
& \makecell[l]{\emph{The People's}\\\emph{Dictator}}
& \makecell[l]{\textsc{Llama 2}\\ \{7B, 13B, 70B\}}
& Yes
& \makecell[l]{$\bw=20$;\\ $\levshort,\, \varepsilon \in \{0,\ldots,5\}$}
& \makecell[l]{batch $10$ (7B,13B), $20$ (70B); \\stride $20$; GPUs $1$/$1$/$4$}
\\

Books
& \makecell[l]{\emph{The Great}\\ \emph{Gatsby}}
& \makecell[l]{\textsc{Llama 2}\\ \{7B, 13B, 70B\}}
& No
& \makecell[l]{$\bw=20$;\\ $\levshort,\, \varepsilon=5$}
& \makecell[l]{batch $10$ (7B,13B), $20$ (70B); \\stride $20$; GPUs $1$/$1$/$4$}
\\

Books
& \makecell[l]{\emph{Orlando}}
& \makecell[l]{\textsc{Llama 2}\\ \{7B, 13B, 70B\}}
& No
& \makecell[l]{$\bw=20$;\\ $\levshort,\, \varepsilon=5$}
& \makecell[l]{batch $10$ (7B,13B), $20$ (70B); \\stride $20$; GPUs $1$/$1$/$4$}
\\

Books
& \makecell[l]{\emph{Pride and}\\\emph{Prejudice}}
& \makecell[l]{\textsc{Llama 3.1}\\ 8B}
& No
& \makecell[l]{$\bw=20$;\\ $\levshort,\, \varepsilon=5$}
& \makecell[l]{batch $10$; stride $20$; GPUs $1$}
\\

Books
& \makecell[l]{\emph{Winnie} \\ \emph{the Pooh}}
& \makecell[l]{\textsc{Llama 2}\\ \{7B, 13B, 70B\}}
& No
& \makecell[l]{$\bw\in\{30,40\}$;\\ $\levshort,\, \varepsilon=5$}
& \makecell[l]{$\bw=30$: batch $7$/$7$/$13$;\\$\bw=40$: batch $5$/$5$/$10$;\\stride $20$; GPUs $1$/$1$/$4$}
\\

\bottomrule
\end{tabular}
\end{table*}
\FloatBarrier
\subsection{Extended results for \textsc{OLMo 2} on Wikipedia}\label{app:experiments:olmo}

This appendix provides additional analysis for the \textsc{OLMo 2} experiments presented in Section~\ref{sec:experiments}.
For these settings, we run experiments with the baseline $k$-CBS method in Section~\ref{sec:kcbs:baseline}, and do not include results for the pruned variants in Section~\ref{sec:pruning}.
We only post-process the $k$-CBS results in this section for Levenshtein distance.

\subsubsection{Extracted sequences and  rates}\label{app:experiments:olmo:extraction}

\paragraph{Verbatim greedy extraction  and verbatim probabilistic extraction.} 
As a point of reference, we provide verbatim extraction rates from two different pipelines in Table~\ref{tab:olmo:verbatim:rates}: 
verbatim greedy extraction (via generation) and verbatim probabilistic extraction (via teacher forcing).
While these methods involve roughly equivalent token evaluations, wall clock time for teacher forcing is almost always faster (for reasons discussed in Appendix~\ref{app:sec:intuition:cost}). 

\begin{table}[h]
\centering
\caption{\textbf{Verbatim extraction rates for \textsc{OLMo 2} on Wikipedia.}
Verbatim extraction rates via greedy (deterministic generation), teacher forcing, and $k$-CBS
on training data ($10{,}000$ sequences) and held-out negative controls ($5{,}000$ sequences).
Probabilistic and $k$-CBS rates use threshold $\pseq \geq \taumin = 0.001$.}
\label{tab:olmo:verbatim:rates}
\scriptsize
\begin{tabular}{l cc cc cc}
\toprule
& \multicolumn{2}{c}{\textsc{OLMo 2 7B}} & \multicolumn{2}{c}{\textsc{OLMo 2 13B}} & \multicolumn{2}{c}{\textsc{OLMo 2 32B}} \\
\cmidrule(lr){2-3}
\cmidrule(lr){4-5}
\cmidrule(lr){6-7}
Pipeline & Train & Held-out & Train & Held-out & Train & Held-out \\
\midrule
Verbatim greedy & 0.25\% & 0.00\% & 0.33\% & 0.00\% & 0.61\% & 0.02\% \\
Verbatim probabilistic & 0.69\% & 0.00\% & 1.02\% & 0.00\% & 1.61\% & 0.04\% \\
Verbatim $k$-CBS & 0.55\% & 0.00\% & 0.87\% & 0.00\% & 1.42\% & 0.04\% \\
\bottomrule
\end{tabular}
\end{table}

\begin{table}[h]
\centering
\caption{\textbf{Verbatim probabilistic vs.\ verbatim $k$-CBS on \textsc{OLMo 2} (training data).}
$k$-CBS with beam width $\bw\!=\!20$ may under-count verbatim extraction.
We report rates ($\taumin=0.001$) for $\bw\!=\!20$, the $k$-CBS recovery ratio, and the number of sequences missed by $k$-CBS.}
\label{tab:olmo:kcbs:verbatim:gap}
\scriptsize
\begin{tabular}{l r r r}
\toprule
& \textsc{OLMo 2 7B} & \textsc{OLMo 2 13B} & \textsc{OLMo 2 32B} \\
\midrule
Verbatim probabilistic & 0.69\% & 1.02\% & 1.61\% \\
Verbatim $k$-CBS & 0.55\% & 0.87\% & 1.42\% \\
\midrule
Recovery ratio & 0.80 & 0.85 & 0.88 \\
Missed by $k$-CBS & 14 & 16 & 19 \\
\bottomrule
\end{tabular}
\end{table}

Note that, as underscored by Table~\ref{tab:olmo:kcbs:verbatim:gap}, the lower bound extraction probabilities provided by baseline $k$-CBS under-count the verbatim probabilistic extraction rate produced by teacher forcing.
Part of the reason is that baseline $k$-CBS provides a lower bound due to pruning paths each iteration, and therefore can miss some mass.
We still recover the larger majority of the mass; the extraction rates are fairly similar, with a decrease in (relative) under-counting as model size increases. 
This is because the missed sequences are relatively low mass, and we find that larger model sizes often exhibit higher per-sequence extraction risk.

Further, there are subtle differences between logits in teacher forcing (which does not involve KV caching) and generation (which does), which can slightly change exchange results for the same sequences due to floating point rounding.
Given the low cost of the verbatim probabilistic pipeline, one could also just run both and take the union over results.
However, it would be necessary to note that the noise from logit differences may impact estimates---i.e., from mixing verbatim probabilities from the verbatim probabilistic pipeline with $k$-CBS.
Alternatively, one could also widen the beam width used for $k$-CBS to capture more paths.\looseness=-1

Experiments on pruned variants recover much of this missed mass, as detailed in our other experiments, further suggesting that those methods (for edit-distance-based analysis) are superior to the baseline method (Appendices~\ref{app:experiments:llama2-scale} \&~\ref{app:sec:experiments:prune}). 

\paragraph{Comparing verbatim and near-verbatim extraction rates.}
Figure~\ref{app:fig:olmo:rates} (which replicates Figure~\ref{fig:olmo:rates:main}) compares greedy and baseline $k$-CBS extraction rates for \textsc{OLMo 2} across model sizes and Levenshtein thresholds $\varepsilon \in \{0,\ldots,5\}$.
$k$-CBS dominates greedy at every $\varepsilon$ and model size, with both rates and gaps growing with model size.
That is, greedy extraction under-counts extraction more for larger models.
Held-out (negative-control) lines stay flat near zero, supporting the validity of our extraction procedure.
We address held-out sequences that register as extractable in more detail below.

\begin{figure*}[t!]
\centering
\includegraphics[width=\linewidth]{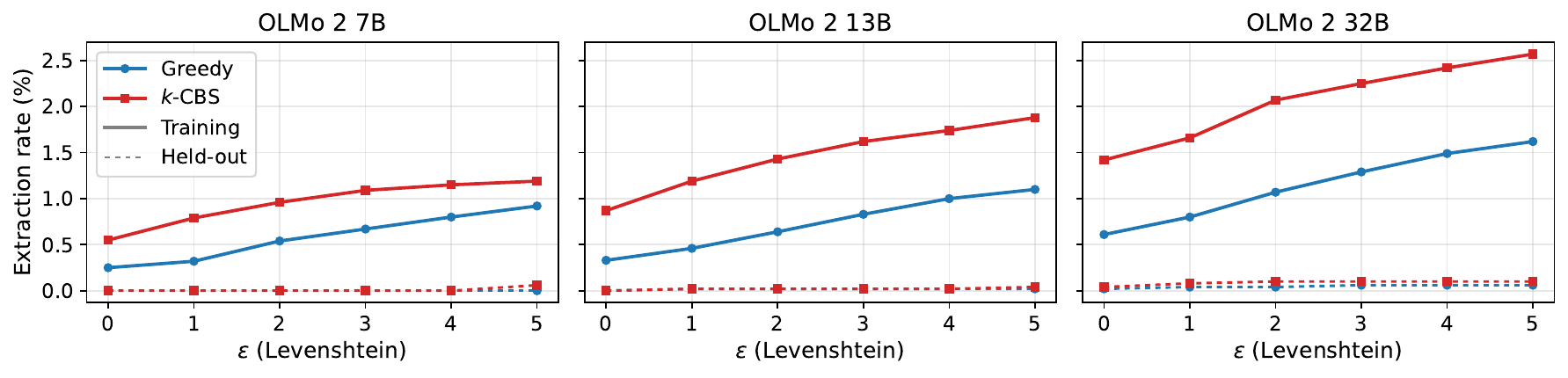}
\caption{\textbf{Comparing extraction rates for \textsc{OLMo 2}.} 
For \textsc{OLMo 2 7B}, \textsc{13B}, and \textsc{32B}, we show greedy and $k$-CBS probabilistic rates for verbatim extraction ($\varepsilon\!=\!0$) and near-verbatim extraction for $\varepsilon\, \in \{1,\ldots,5\}$.
We use a sample of $10{,}000$ sequences from Wikipedia from \textsc{OLMo 2}'s training data;
to assess validity, we also run analogous negative controls on $5{,}000$ sequences scraped from Wikipedia that post-date   \textsc{OLMo 2}'s training cutoff.
The greedy rates are exact. 
The probabilistic rates are computed with $k$-CBS (Section~\ref{sec:kcbs:baseline});
they may miss some valid instances of extraction, and thus should be interpreted as lower bounds on extraction rates.\looseness=-1}
\label{app:fig:olmo:rates}
\end{figure*}

\paragraph{Negative control results.} 
As noted in Appendix~\ref{app:sec:experiments:setup:olmo}, establishing a clean negative control on Internet-scraped data like Wikipedia is challenging.
Using page create dates for determining whether text post-dates a training cutoff is imperfect;
text on a newly posted Wikipedia page may be copied from an older, earlier page from elsewhere on the Internet (and included in Common Crawl); a Wikipedia page may itself be created, deleted, and the re-created as a new page years later; multiple Wikipedia pages share text/templated structure; and more.  
All held-out examples flagged as extractable (at $\varepsilon \leq 5$, by either greedy or probabilistic extraction) fall into these categories: 
they trace back to boilerplate or duplicated text that existed on Wikipedia well before \textsc{OLMo 2}'s training cutoff. 
They are instances of extraction (true positives) rather than false positives. 

Across the three models, we observe a total of $5$ such sequences. 
We discuss each below, along with the model, extraction method, and near-verbatim tolerance for which it is extractable.
We give a few examples of near-verbatim extracted text.

\begin{table}[h]
\centering
{\renewcommand{\arraystretch}{1.2}}
\caption{\textbf{Held-out sequences in negative controls for \textsc{OLMo 2}.} 
Sequences that are extractable verbatim (v.) and near-verbatim (nv.) under both greedy and probabilistic extraction.
A verbatim-extractable sequence is necessarily near-verbatim extractable.
For probabilistic extraction, the near-verbatim mass often significantly exceeds the verbatim mass.\looseness=-1}
\scriptsize
\begin{tabular}{l *{12}{c}}
\toprule
& \multicolumn{4}{c}{\textbf{7B}} & \multicolumn{4}{c}{\textbf{13B}} & \multicolumn{4}{c}{\textbf{32B}} \\
\cmidrule(lr){2-5} \cmidrule(lr){6-9} \cmidrule(lr){10-13}
& \multicolumn{2}{c}{\textbf{greedy}} & \multicolumn{2}{c}{\textbf{prob.}} & \multicolumn{2}{c}{\textbf{greedy}} & \multicolumn{2}{c}{\textbf{prob.}} & \multicolumn{2}{c}{\textbf{greedy}} & \multicolumn{2}{c}{\textbf{prob.}}\\
\cmidrule(lr){2-3} \cmidrule(lr){4-5} \cmidrule(lr){6-7} \cmidrule(lr){8-9} \cmidrule(lr){10-11} \cmidrule(lr){12-13}
& \makecell{\textbf{v.}} & \makecell{\textbf{nv.}} & \makecell{\textbf{v.}} & \makecell{\textbf{nv.}} & \makecell{\textbf{v.}} & \makecell{\textbf{nv.}} & \makecell{\textbf{v.}} & \makecell{\textbf{nv.}} & \makecell{\textbf{v.}} & \makecell{\textbf{nv.}} & \makecell{\textbf{v.}} & \makecell{\textbf{nv.}}\\
\midrule
\textbf{1. Provincial Board} &  &  & & \checkmark  &  &  &  & \checkmark  & \checkmark & \checkmark & \checkmark  & \checkmark  \\
\textbf{2. 98th Academy Awards} &  &  &  & \checkmark  &  & \checkmark  &  & \checkmark  &  & \checkmark  &  & \checkmark  \\
\textbf{3. 2015 WinStar} &  &  &  & \checkmark  &  &  &  &  &  &  &  & \checkmark  \\
\textbf{4. Batman} &  &  &  &  &  &  &  &  &  & \checkmark  & \checkmark  & \checkmark  \\
\textbf{5. Finnmark} &  &  &  &  &  &  &  &  &  &  &  & \checkmark  \\
\bottomrule
\end{tabular}
\end{table}

\begin{enumerate}[leftmargin=0.75cm,topsep=0cm]
    \item \textbf{Maguindanao del Sur Provincial Board.} 
    %
    This page was created in 2025 (i.e., makes sense it would be pulled in during curation of held-out data), but the extracted text is a boilerplate template describing Philippine provincial board elections. 
    From manual investigation using the Wikipedia API, the same template appears verbatim on $\sim\!80$ provincial board pages on Wikipedia, the vast majority created in August 2020 (with some dating back to 2011–2013). 

    {\footnotesize
    \noindent $\pre$: \verb|The Maguindanao del Sur Provincial Board is the Sangguniang Panlalawigan |\\
    \verb|(provincial legislature) of the Philippine province of Maguindanao del Sur.\nThe |\\
    \verb|members are elected via plurality-at-large voting: the|
    
    \vspace{0.1cm}
    
    \noindent $\suf$: \verb| province is divided into two districts, each having five seats. A voter votes|\\
    \verb|up to five names, with the top five candidates per district being elected. The vice|\\
    \verb|governor is the ex officio presiding officer, and only votes to break ties. The|
    }


    \item \textbf{98th Academy Awards -- Best International Feature Film submissions.}
    %
    This page was created January 31, 2026---well after training cutoff. 
    However, the extracted text is identical or near-identical to boilerplate shared across all ``List of submissions to the Nth Academy Awards'' pages. 
    From manual inspection using the Wikipedia API, these pages date back to at least the 85th Academy Awards (created August 2012), and the boilerplate was present from their very first revisions.  

    {\footnotesize
    \noindent $\pre$: \verb|This is a list of submissions to the 98th Academy Awards for Best International |\\ 
    \verb|Feature Film. The Academy of Motion Picture Arts and Sciences (AMPAS) has invited the |\\
    \verb|film industries of various countries to submit their best film for the Academy Award for|
    
    \vspace{0.1cm}
    
    \noindent $\suf$: \verb| Best International Feature Film every year since the award was created in 1956. |\\
    \verb|The award is presented annually by the Academy to a feature-length motion picture |\\
    \verb|produced outside the United States that contains primarily non-English dialogue. The |\\
    \verb|International Feature Film |\textcolor{red}{\texttt{Executive }}\verb|Committee oversees|
    }

    \vspace{0.1cm}

    Sample near-verbatim $\gensuf$ (\textsc{OLMo 2 13B}, greedy $\levshort=1$; in character space, deletions in \textcolor{red}{red} in $\suf$, above, and additions in \textcolor{blue}{blue} in $\gensuf$, below):\looseness=-1

    {\footnotesize
    \noindent $\gensuf$: \verb| Best International Feature Film every year since the award was created in 1956. |\\
    \verb|The award is presented annually by the Academy to a feature-length motion picture |\\
    \verb|produced outside the United States that contains primarily non-English dialogue. The |\\
    \verb|International Feature Film |\textcolor{blue}{\texttt{Award }}\verb|Committee oversees|
    }


    \item \textbf{2015 WinStar World Casino \& Resort 350.} 
    This NASCAR Truck Series race page was created January 10, 2024 (after the cutoff), but the extracted text follows the same template used across hundreds of NASCAR race articles, many of which predate the training cutoff (e.g., the 2022 NextEra Energy 250 was created February 2022). 
    The template structure and stock phrases (track descriptions, race numbering) are widely duplicated.\looseness=-1

    {\footnotesize
    \noindent $\pre$: \verb|The 2015 WinStar World Casino & Resort 350 was the 21st stock car race of the |\\
    \verb|2015 NASCAR Camping World Truck Series, and the 17th iteration of the event. The race |\\
    \verb|was held on Friday, November|
    
    \vspace{0.1cm}
    
    \noindent $\suf$: \verb| 6, 2015, in Fort Worth, Texas at Texas Motor Speedway, a 1.5 mi (2.4 km) permanent |\\
    \verb|tri-oval shaped racetrack. The race took the scheduled 147 laps to complete. Erik|
    }


    \item \textbf{Batman in popular culture.} 
    This page was created in December 2025, but the extracted text exists on multiple fan wikis dating to 2019, and thus plausibly was included in Common Crawl (and therefore the training data).
    It also contains a \emph{Guardian} quote about Batman that existed verbatim in the main Batman article as of at least November 2023 (pre-cutoff)---and of course online at the \emph{Guardian} itself. \looseness=-1 

    {\footnotesize
    \noindent $\pre$: \verb|The DC Comics character Batman has become a popular culture icon, recognized around |\\
    \verb|the world. The character's presence has extended beyond his comic book origins; events |\\
    \verb|such as the release of the 1989 Batman film and its accompanying merchandising "b|
    
    \vspace{0.1cm}
    
    \noindent $\suf$: \verb|rought the Batman to the forefront of public consciousness| \texttt{\textcolor{red}{".}} \verb|In an article |\\
    \verb|commemorating the sixtieth anniversary of the character, The Guardian wrote, "Batman |\\
    \verb|is a figure blurred by the endless reinvention that is modern mass culture. He is at |\\
    \verb|once an|
    }

    Sample near-verbatim $\gensuf$ (\textsc{OLMo 2 32B}, greedy $\levshort=1$; in character space, deletions in \textcolor{red}{red} in $\suf$, above, and additions in \textcolor{blue}{blue} in $\gensuf$, below):\looseness=-1

    {\footnotesize
    \noindent $\gensuf$: \verb|rought the Batman to the forefront of public consciousness| \texttt{\textcolor{blue}{."}} \verb|In an article |\\
    \verb|commemorating the sixtieth anniversary of the character, The Guardian wrote, "Batman |\\
    \verb|is a figure blurred by the endless reinvention that is modern mass culture. He is at |\\
    \verb|once an|
    }


    \item \textbf{Finnmark (Storting constituency).} 
    This page was created September 29, 2024. The extracted text is boilerplate shared across all 19 Norwegian Storting constituency pages (that we found through manual Wikipedia API search). 
    The earliest of these (Oslo) was created August 2021, with others following in 2021–2022 (all pre-cutoff).

    {\footnotesize
    \noindent $\pre$: \verb|Finnmark (Northern Sami: Finnmárku; Kven: Finmarkku) is one of the 19 multi-|\\
    \verb|member constituencies of the Storting, which is the national legislature of |\\
    \verb|Norway. The constituency was established in|
    
    \vspace{0.1cm}
    
    \noindent $\suf$: \verb|1921 following the introduction of proportional representation for elections to |\\
    \verb|the Storting. It is conterminous with the county of Finnmark. The constituency |\\
    \verb|currently elects eight of the 169 members of the Storting using the open party-list |\\
    \verb|proportional representation electoral|
    }
\end{enumerate}

\paragraph{``Unlocked'' extracted sequences.} 
Most \textcolor{seabornbluemid}{\textbf{blue}} points sit well above $y$=$x$, indicating increased extraction risk for verbatim-extractable sequences. 
Approximately $\sim\!45–55\%$ of extractable sequences are unlocked (\textcolor{seabornredmid}{\textbf{red}} + \textcolor{seabornorangemid}{\textbf{orange}}) across all three sizes. 
The unlocked population has sequences with substantial mass, but many sequences are below $0.01$. 
As shown below (Appendix~\ref{app:experiments:olmo:risk}), the median verbatim share increases with model size and the \textcolor{seabornbluemid}{\textbf{blue}} fraction grows (from $\sim\!46\%$ for 7B and 13B $\to 55\%$ for 32B), but even at 32B nearly half of extractable sequences have zero or sub-threshold verbatim mass.
\textcolor{seabornorangemid}{\textbf{Orange}} points are relatively sparse; 
sequences tend to be either clearly verbatim-extractable or have zero verbatim mass.
For comparative notes to \textsc{Pythia} on Enron, see Appendix~\ref{app:fig:pythia:rates}.

\begin{figure*}[t!]
\centering
\includegraphics[width=.95\linewidth]{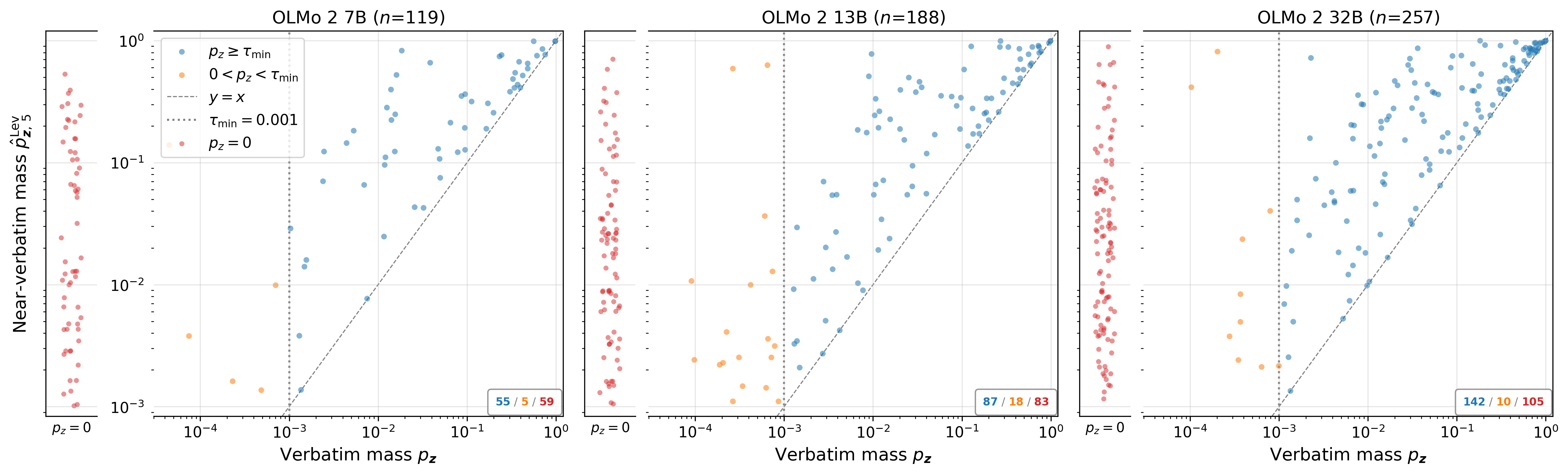}
\caption{\textbf{Near-verbatim mass vs.\ verbatim mass for \textsc{OLMo 2}.} 
\textsc{OLMo 2} on Wikipedia (training subset); 
each point is one sequence.
Axes show near-verbatim ($p_{\seq,5}^\levshort$, $\levshort\,\varepsilon\!=\!5$) vs.\ verbatim ($\pseq$) extraction mass on a $\log$--$\log$ scale.
\textcolor{seabornredmid}{\textbf{Red}}/\textcolor{seabornorangemid}{\textbf{orange}}
points are ``unlocked'' by near-verbatim extraction (to the left of the $\taumin$ dotted reference line, $\pseq\!<\!\taumin$, but $p_{\seq,5}^\levshort\!\geq\!\taumin$);
\textcolor{seabornbluemid}{\textbf{blue}} points are verbatim-extractable ($\pseq\!\geq\!\taumin$).
Points above the dashed $y\!=\!x$ line show increased extraction risk when near-verbatim mass is accounted for.\looseness=-1}
\label{app:fig:olmo:scatter}
\end{figure*}

\subsubsection{Extraction risk}\label{app:experiments:olmo:risk}

\paragraph{CCDF over near-verbatim risk gain.}
We show two views of per-sequence near-verbatim risk gain (Equation~\ref{app:eq:per-seq-mass-gain}):
(1) the CCDF over the population of training sequences, including both extractable and non-extractable (Figure~\ref{app:fig:olmo:ccdf}, which duplicates Figure~\ref{fig:ccdf:main}); 
(2) a table containing points the CCDF of per-sequence mass gain on the fixed extractable set only,  where the fixed set differs by model (Table~\ref{app:tab:olmo:rel-ccdf}).
In the whole-population CCDF, 32B strictly dominates at every gain threshold; 
larger models produce more sequences with near-verbatim mass beyond verbatim.
The curve shifts upward and rightward, indicating that the absolute mass gains are also larger.
In the fixed-set CCDF, there is not a clear dominance pattern with respect to relative mass gains computed according to each model's extractable set.
Since these numbers show \emph{gains}, note that it is possible for sequences that have high verbatim risk to have smaller relative gains in near-verbatim risk.
(For high-risk sequences, there is less possible risk increase, both relatively and absolutely.)

\begin{figure}[t]
\centering
\begin{minipage}[t]{0.48\linewidth}
\vspace{0pt}
\centering
\includegraphics[width=.95\linewidth]{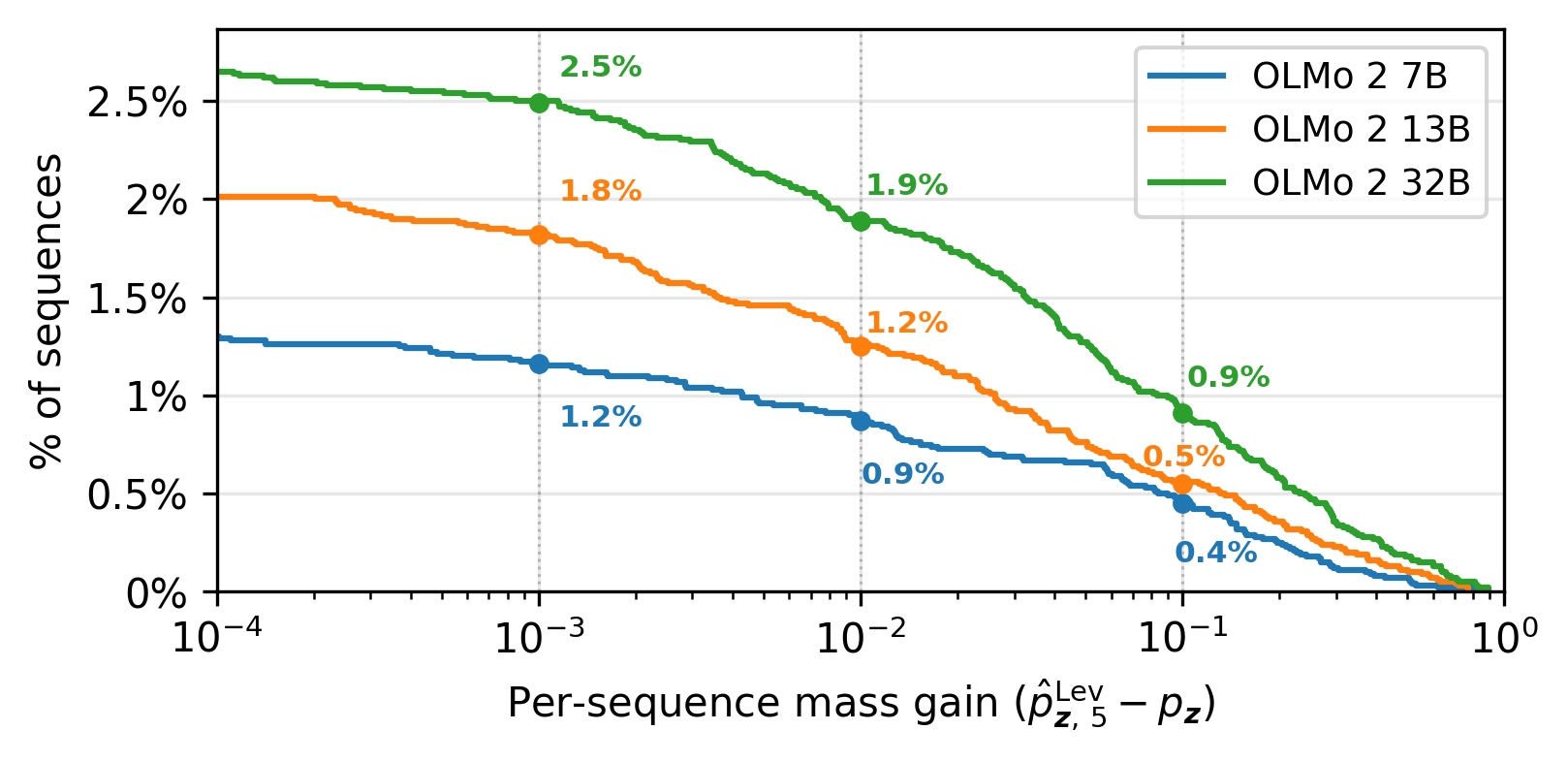}
\caption{\textbf{CCDF of population per-sequence near-verbatim mass gain for \textsc{OLMo 2}.} 
For $\levshort\, \varepsilon\!=\!5$ mass minus verbatim mass ($\hat{p}_{\seq,5}^{\levshort} - \pseq$), a point $(x, y)$ means $y\%$ of sequences have extraction-mass gain $\geq x$.
Plotted over the whole training set sample ($10{,}000$ Wikipedia sequences).\looseness=-1}
\label{app:fig:olmo:ccdf}
\end{minipage}
\hfill
\begin{minipage}[t]{0.48\linewidth}
\vspace{0pt}
\captionof{table}{\textbf{Points on the CCDF of \emph{extracted} per-sequence mass gain for \textsc{OLMo 2}.}
We provide specific values from the CCDF over the per-sequence mass gain for extracted sequences only.
For this CCDF, the maximum $y$-value is $100\%$. 
The denominators are different for each, given different counts in the fixed extractable set.}
\label{app:tab:olmo:rel-ccdf}
\centering
\begin{tabular}{llll}
\toprule
& \textbf{7B} & \textbf{13B} & \textbf{32B} \\
\midrule
$\bf \geq10^{-3}$ & $97.5\%$ & $96.8\%$  & $96.9\%$  \\
$\bf \geq10^{-2}$ & $73.1\%$ & $66.5\%$  & $73.5\%$  \\
$\bf \geq10^{-1}$ & $37.8\%$ & $29.3\%$ & $35.4\%$ \\
\bottomrule
\end{tabular}
\end{minipage}
\end{figure}

\paragraph{Per-sequence $\varepsilon$-shell share analysis.}
We refer to Appendix~\ref{app:sec:experiments:metrics} for detailed explanations of the metrics and plot types for this analysis.

For \textsc{OLMo 2} on Wikipedia, sequence-extraction mass is dispersed across different distances.
This is clear from both the heatmap (Figure~\ref{app:fig:olmo:heatmap}) and incremental mass share (Equation~\ref{app:eq:shell-fraction}) distributions for the different $\varepsilon$-shells (Figure~\ref{app:fig:olmo:incremental}). 
Nearly half of all extracted sequences for the 7B model have $0$ verbatim share (Equation~\ref{app:eq:verbatim-share}).  
The median verbatim share is $0.03\%$.
Larger models increasingly concentrate more mass on the verbatim share, as is clear from both the tops of the heatmaps (Figure~\ref{app:fig:olmo:heatmap}) and the increasing median of the verbatim share (Figure~\ref{app:fig:olmo:verbatimshare}, duplicated from Figure~\ref{fig:violin:olmo:main}).

\begin{figure*}[t!]
\centering
\begin{subfigure}[t]{0.54\textwidth}
    \centering
    \includegraphics[width=\linewidth,trim={0 0 0 0.75cm},
clip]{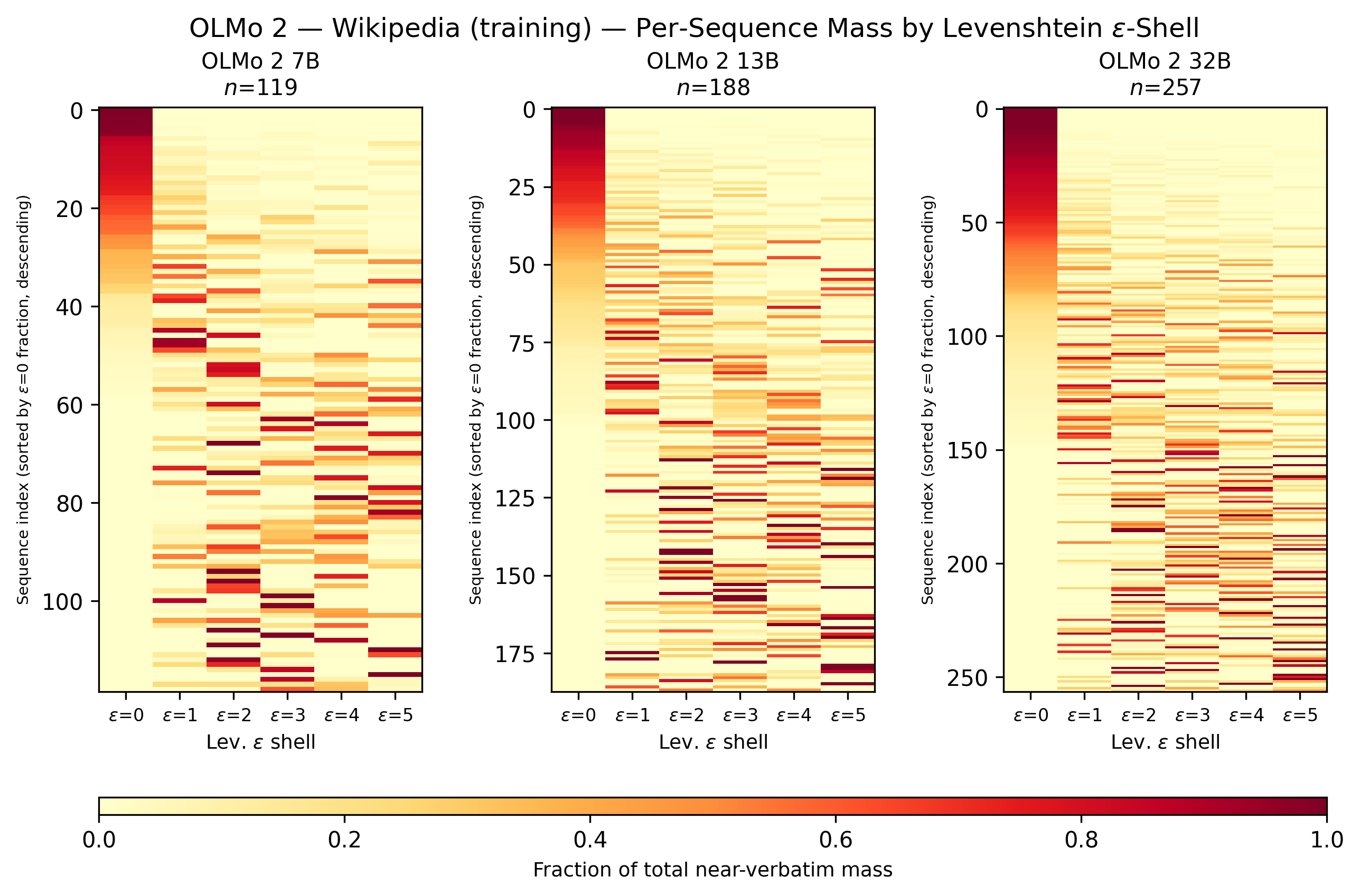}
    \caption{Heatmap of per-sequence mass by $\levshort$ $\varepsilon$-shell share}
    \label{app:fig:olmo:heatmap}
\end{subfigure}
\hfill
\begin{subfigure}[t]{0.45\textwidth}
    \centering
    \includegraphics[width=\linewidth]{paper/figure/olmo/olmo2_wiki_verbatim_share_violin.png}
    \caption{Verbatim mass share distributions}
    \label{app:fig:olmo:verbatimshare}
\end{subfigure}\\
\begin{subfigure}[t]{\textwidth}
    \centering
    \includegraphics[width=\linewidth,trim={0 0 0 0.75cm},clip]{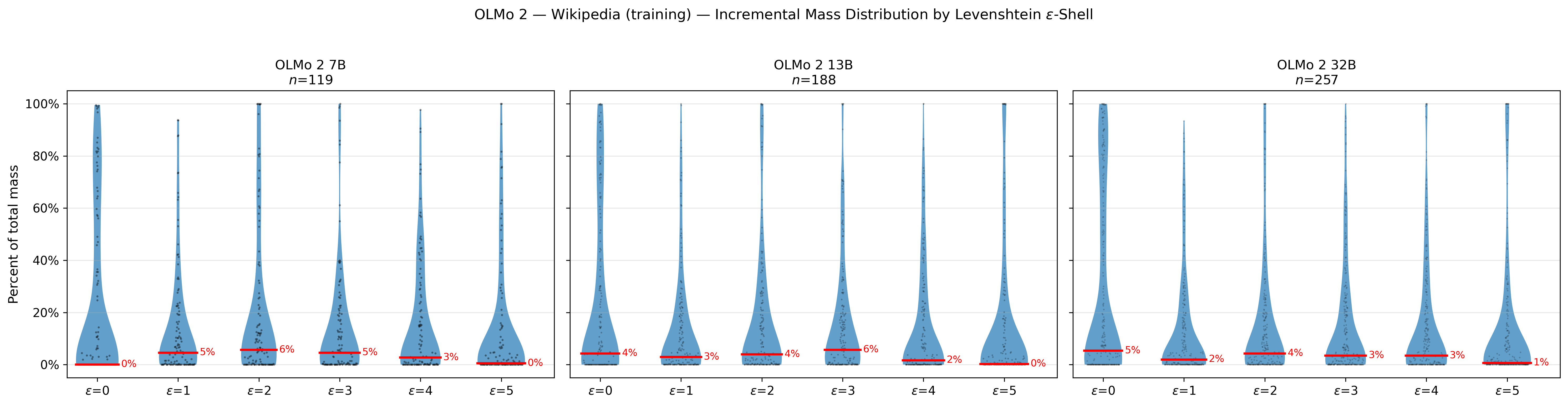}
    \caption{Incremental mass share distributions}
    \label{app:fig:olmo:incremental}
\end{subfigure}
\caption{\textbf{Illustrating different views of $\varepsilon$-shell share for \textsc{OLMo 2} on Wikipedia.}
(\textbf{a}) Heatmaps across model size of per-sequence mass share by $\varepsilon$-shell (Equation~\ref{app:eq:shell-fraction}), sorted by verbatim share (Equation~\ref{app:eq:verbatim-share}).
(\textbf{b}) Violin plots comparing the distribution of per-sequence verbatim share (Equation~\ref{app:eq:verbatim-share}) across model sizes, with the median verbatim share annotated.
(\textbf{c}) Violin plots showing distributions over the per-$\varepsilon$-shell mass share (Equation~\ref{app:eq:shell-fraction}) per model.
Each shell shows the mass share contributed by the given $\lev$ distance. 
Note that each of the three $\varepsilon=0$ violin plots correspond to those in the cross-model comparison in (\textbf{b}).\looseness=-1 
}
\end{figure*}
\FloatBarrier

\subsection{Extended results for \textsc{Pythia} on emails}\label{app:experiments:pythia}

This appendix provides experiments for \textsc{Pythia} on emails---Enron (in the training data) and TREC 2007 Spam ``ham'' (a curated subset) for a negative control. 
For these settings, we run experiments with the baseline $k$-CBS method in Section~\ref{sec:kcbs:baseline}, and do not include results for the pruned variants in Section~\ref{sec:pruning}.
We only post-process the $k$-CBS results in this section for Levenshtein distance.

\subsubsection{Extracted sequences and  rates}\label{app:experiments:pythia:extraction}

\paragraph{Verbatim greedy extraction  and verbatim probabilistic extraction.} 
As a point of reference, we provide verbatim extraction rates from two different pipelines in Table~\ref{tab:pythia:verbatim:rates}: 
verbatim greedy extraction (via generation) and verbatim probabilistic extraction (via teacher forcing).
While these methods involve roughly equivalent token evaluations, wall clock time for teacher forcing is almost always faster (for reasons discussed in Appendix~\ref{app:sec:intuition:cost}). 

\begin{table}[h]
\centering
\caption{\textbf{Verbatim extraction rates for \textsc{Pythia} on Enron emails.}
Verbatim extraction rates via greedy (deterministic generation), teacher forcing, and $k$-CBS
on training data ($10{,}000$ sequences) and held-out negative controls ($2{,}000$ sequences).
Probabilistic and $k$-CBS rates use threshold $\pseq \geq \taumin = 0.001$.}
\label{tab:pythia:verbatim:rates}
\scriptsize
\begin{tabular}{l cc cc cc cc}
\toprule
& \multicolumn{2}{c}{\textsc{Pythia 1B}} & \multicolumn{2}{c}{\textsc{Pythia 2.8B}} & \multicolumn{2}{c}{\textsc{Pythia 6.9B}} & \multicolumn{2}{c}{\textsc{Pythia 12B}} \\
\cmidrule(lr){2-3}
\cmidrule(lr){4-5}
\cmidrule(lr){6-7}
\cmidrule(lr){8-9}
Pipeline & Train & Held-out & Train & Held-out & Train & Held-out & Train & Held-out \\
\midrule
Verbatim greedy & 0.74\% & 0.00\% & 1.81\% & 0.00\% & 2.86\% & 0.00\% & 3.84\% & 0.00\% \\
Verbatim probabilistic & 1.92\% & 0.00\% & 3.53\% & 0.00\% & 5.16\% & 0.00\% & 6.52\% & 0.00\% \\
Verbatim $k$-CBS & 1.54\% & 0.00\% & 3.21\% & 0.00\% & 4.83\% & 0.00\% & 6.10\% & 0.00\% \\
\bottomrule
\end{tabular}
\end{table}

\begin{table}[h]
\centering
\caption{\textbf{Verbatim probabilistic vs.\ verbatim $k$-CBS on \textsc{Pythia} (training data).}
$k$-CBS with beam width $\bw\!=\!20$ may under-count verbatim extraction.
We report rates ($\taumin=0.001$), the $k$-CBS recovery ratio, and the number of sequences missed by $k$-CBS.}
\label{tab:pythia:kcbs:verbatim:gap}
\scriptsize
\begin{tabular}{l r r r r}
\toprule
& \textsc{Pythia 1B} & \textsc{Pythia 2.8B} & \textsc{Pythia 6.9B} & \textsc{Pythia 12B} \\
\midrule
Verbatim probabilistic & 1.92\% & 3.53\% & 5.16\% & 6.52\% \\
Verbatim $k$-CBS & 1.54\% & 3.21\% & 4.83\% & 6.10\% \\
\midrule
Recovery ratio & 0.80 & 0.91 & 0.94 & 0.94 \\
Missed by $k$-CBS & 38 & 32 & 35 & 44 \\
\bottomrule
\end{tabular}
\end{table}

Note that, as underscored by Table~\ref{tab:pythia:kcbs:verbatim:gap}, the lower bound extraction probabilities provided by baseline $k$-CBS under-count the verbatim probabilistic extraction rate produced by teacher forcing.
Part of the reason is that baseline $k$-CBS provides a lower bound due to pruning paths each iteration, and therefore can miss some mass.
We still recover the larger majority of the mass; the extraction rates are fairly similar, with a decrease in (relative) under-counting as model size increases (as with the results in Appendix~\ref{app:experiments:olmo}).  
This is because the missed sequences are relatively low mass, and we find that larger model sizes often exhibit higher per-sequence extraction risk.

Further, there are subtle differences between logits in teacher forcing (which does not involve KV caching) and generation (which does), which can slightly change extraction results for the same sequences due to floating point rounding.
Given the low cost of the verbatim probabilistic pipeline, one could also just run both and take the union over results.
However, it would be necessary to note that the noise from logit differences may impact estimates---i.e., from mixing verbatim probabilities from the verbatim probabilistic pipeline with $k$-CBS.
Alternatively, one could also widen the beam width used for $k$-CBS to capture more paths.\looseness=-1

Experiments on pruned variants recover much of this missed mass, as detailed in our other experiments, further suggesting that those methods (for edit-distance-based analysis) are superior to the baseline method (Appendices~\ref{app:experiments:llama2-scale} \&~\ref{app:sec:experiments:prune}). 

\paragraph{Comparing verbatim and near-verbatim extraction rates.}
Figure~\ref{app:fig:pythia:rates} compares greedy and baseline $k$-CBS extraction rates for \textsc{Pythia} across model sizes and Levenshtein thresholds $\varepsilon \in \{0,\ldots,5\}$.
As with the results for \textsc{OLMo 2} (Section~\ref{sec:experiments} \& Appendix~\ref{app:experiments:olmo}), $k$-CBS dominates greedy at every $\varepsilon$ and model size, with both rates and gaps growing with model size.
That is, greedy extraction under-counts extraction more for larger models.
However, while this is true for comparing gaps between models at the same $\varepsilon$, gaps for the same model shrink with increasing $\varepsilon$. 
In this sense, the greedy rate becomes a better estimate of extraction for a given model at larger $\varepsilon$. 

At $\varepsilon = 5$, \textsc{Pythia 12B} reaches an extraction rate of $7.80\%$ (vs.\ similarly-sized \textsc{OLMo 2 13B} at $1.88\%$ on Wikipedia).
This likely reflects properties of the training data---possibly the data itself (wiki entries vs.\ emails) and training mix.
Notably, the Enron corpus itself contains substantial internal duplication, and the Pile includes it without deduplication. 
Held-out (negative-control) lines stay flat at zero, supporting the validity of our extraction procedure.

\paragraph{Negative control results.} 
Not a single sequence from the held-out ``ham'' email set is flagged at any $\varepsilon$ by either method.

\paragraph{``Unlocked'' extracted sequences.} 
Most \textcolor{seabornbluemid}{\textbf{blue}} points sit well above $y$=$x$, indicating increased extraction risk for verbatim-extractable sequences. 
Compared to the \textsc{OLMo 2} on Wikipedia experiments (Appendix~\ref{app:experiments:olmo:extraction}), \textcolor{seabornbluemid}{\textbf{blue}} points dominate more strongly and increasingly with scale ($55\% \to 78\%$). 
There is a clear dense cluster of \textcolor{seabornbluemid}{\textbf{blue}}  points in the $10^{-2}$ to $10^{-1}$  verbatim mass range whose near-verbatim mass jumps to significantly above $10^{-1}$. 
These are sequences with moderate verbatim mass where near-verbatim extraction roughly goes from $2\times$--$10\times$ the verbatim extraction risk. 
This cluster is visible at all four model sizes, though is slightly less pronounced for the 1B model.
It suggests a population of Enron sequences that the model has memorized well but not perfectly---enough verbatim mass to be extractable, but with a lot of additional probability on close variants.  
The unlocked fraction (\textcolor{seabornredmid}{\textbf{red}} + \textcolor{seabornorangemid}{\textbf{orange}}) shrinks more dramatically ($45\% \to 22\%$). 
\textcolor{seabornorangemid}{\textbf{Orange}} points are relatively sparse; 
sequences tend to be either clearly verbatim-extractable or have zero verbatim mass.

Overall, for both these experiments and those on \textsc{OLMo 2} (Appendix~\ref{app:fig:olmo:rates}), near-verbatim extraction both reveals hidden sequences (\textcolor{seabornredmid}{\textbf{red}}/\textcolor{seabornorangemid}{\textbf{orange}}) and increases measured risk for verbatim-extractable sequences (\textcolor{seabornbluemid}{\textbf{blue}} above $y=x$). 
Both effects are present in both sets of experiments. 
Where they differ is in the balance:
\textsc{OLMo 2} on Wikipedia has a persistently large hidden population with substantial mass; 
the extractable set for \textsc{Pythia} on Enron is more dominated by verbatim memorization. 

\begin{figure*}[t!]
\centering
\includegraphics[width=\linewidth]{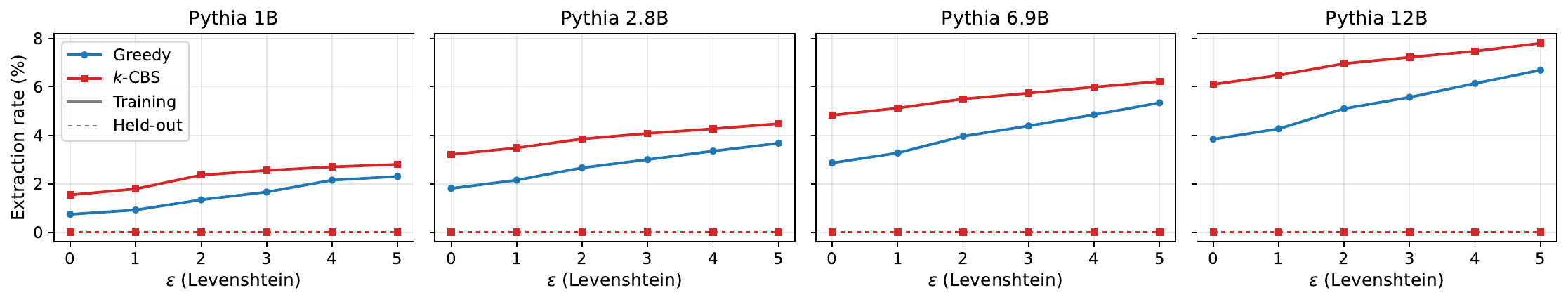}
\caption{\textbf{Comparing extraction rates for \textsc{Pythia}.} 
For \textsc{Pythia 1B}, \textsc{2.8B}, and \textsc{6.9B}, and \textsc{12B} we show greedy and $k$-CBS probabilistic rates for verbatim extraction ($\varepsilon\!=\!0$) and near-verbatim extraction for $\varepsilon\, \in \{1,\ldots,5\}$.
We use a sample of $10{,}000$ sequences from Enron emails from \textsc{Pythia}'s training data;
to assess validity, we also run analogous negative controls on $2{,}000$ sequences from TREC 2007 Spam (``ham'' subset). 
The greedy rates are exact. 
The probabilistic rates are computed with $k$-CBS (Section~\ref{sec:kcbs:baseline});
they may miss some valid instances of extraction, and thus should be interpreted as lower bounds on extraction rates.\looseness=-1}
\label{app:fig:pythia:rates}
\end{figure*}
\begin{figure*}[t!]
\centering
\includegraphics[width=.95\linewidth]{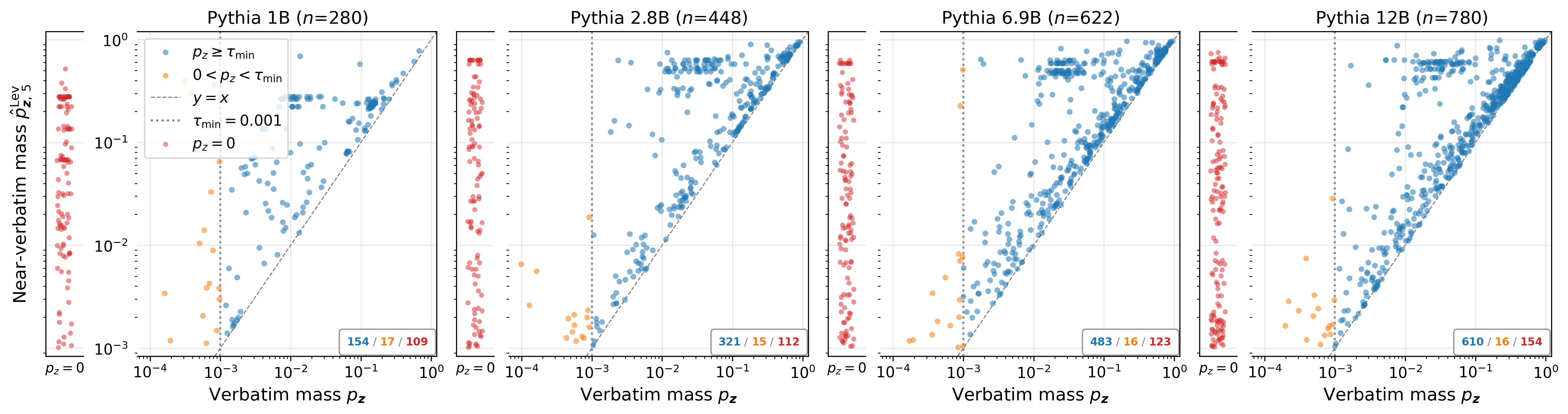}
\vspace{-.3cm}
\caption{\textbf{Near-verbatim mass vs.\ verbatim mass for \textsc{Pythia}.} 
\textsc{Pythia} on Enron emails; 
each point is one sequence.
Axes show near-verbatim ($p_{\seq,5}^\levshort$, $\levshort\,\varepsilon\!=\!5$) vs.\ verbatim ($\pseq$) extraction mass on a $\log$--$\log$ scale.
\textcolor{seabornredmid}{\textbf{Red}}/\textcolor{seabornorangemid}{\textbf{orange}}
points are ``unlocked'' by near-verbatim extraction (to the left of the $\taumin$ dotted reference line, $\pseq\!<\!\taumin$, but $p_{\seq,5}^\levshort\!\geq\!\taumin$);
\textcolor{seabornbluemid}{\textbf{blue}} points are verbatim-extractable ($\pseq\!\geq\!\taumin$).
Points above the dashed $y\!=\!x$ line show increased extraction risk when near-verbatim mass is accounted for.\looseness=-1}
\label{app:fig:pythia:scatter}
\end{figure*}

\subsubsection{Extraction risk}\label{app:experiments:pythia:risk}

\paragraph{CCDF over near-verbatim risk gain.}
We show two views of per-sequence near-verbatim risk gain (Equation~\ref{app:eq:per-seq-mass-gain}):
(1) the CCDF over the population of training sequences, including both extractable and non-extractable (Figure~\ref{app:fig:pythia:ccdf}); 
(2) a table containing points the CCDF of per-sequence mass gain on the fixed extractable set only,  where the fixed set differs by model (Table~\ref{app:tab:pythia:rel-ccdf}).
In the whole-population CCDF, larger models produce more sequences with near-verbatim mass beyond verbatim.
The curve shifts upward and rightward, indicating that the absolute mass gains are also larger, but the differences converge for the 2.8B, 6.9B, and 12B models at around $10^{-1}$.
In the fixed-set CCDF, there is not a clear dominance pattern with respect to relative mass gains computed according to each model's extractable set.
Since these numbers show \emph{gains}, note that it is possible for sequences that have high verbatim risk to smaller relative gains in near-verbatim risk.
(For high-risk sequences, there is less possible risk increase, both relatively and absolutely.)

\begin{figure}[t]
\centering
\begin{minipage}[t]{0.48\linewidth}
\vspace{0pt}
\centering
\includegraphics[width=.95\linewidth]{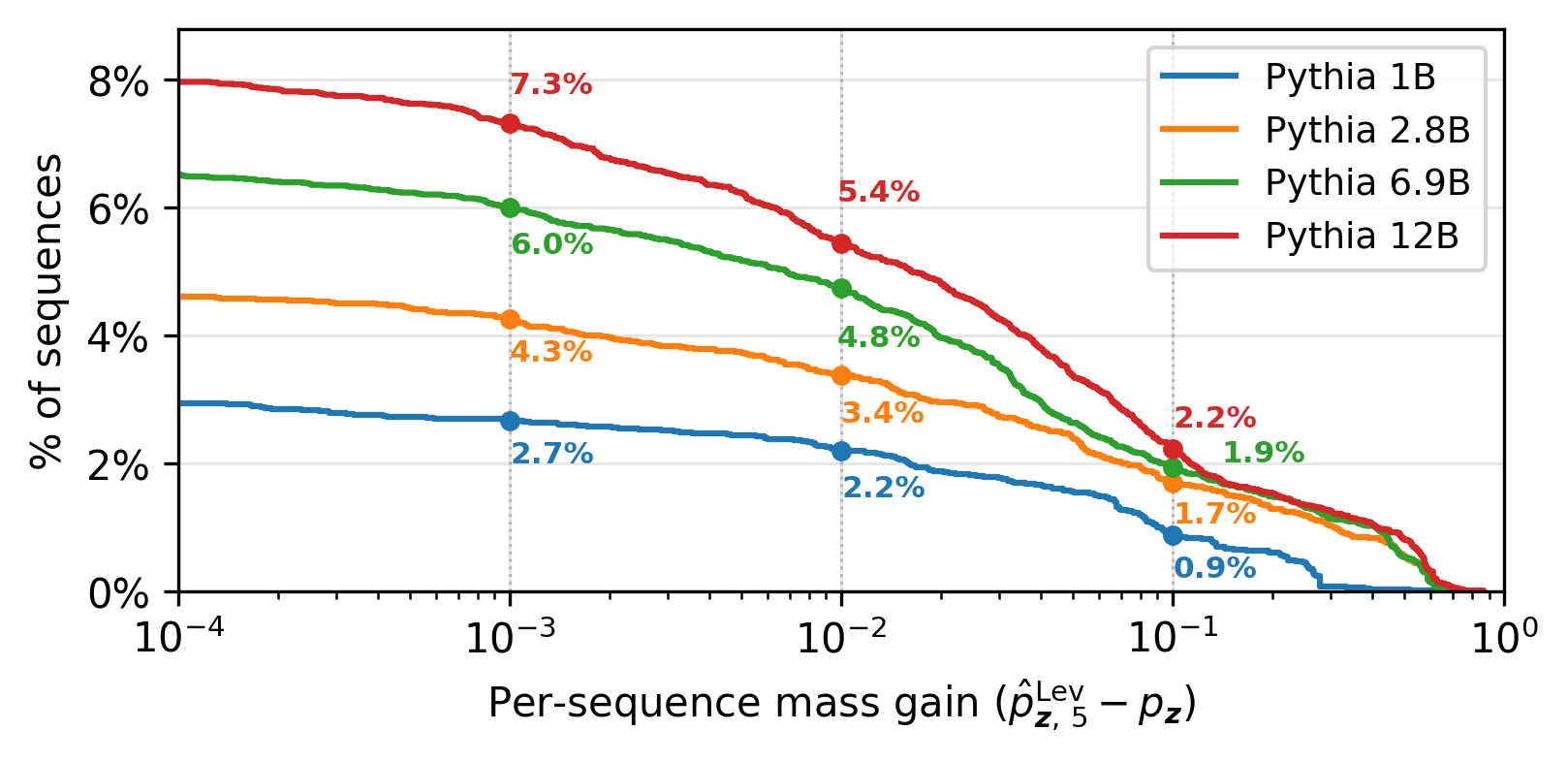}
\caption{\textbf{CCDF of population per-sequence near-verbatim mass gain for \textsc{Pythia}.} 
For $\levshort\, \varepsilon\!=\!5$ mass minus verbatim mass ($\hat{p}_{\seq,5}^{\levshort} - \pseq$), a point $(x, y)$ means $y\%$ of sequences have extraction-mass gain $\geq x$.
Plotted over the whole training set sample ($10{,}000$ Wikipedia sequences).\looseness=-1}
\label{app:fig:pythia:ccdf}
\end{minipage}
\hfill
\begin{minipage}[t]{0.48\linewidth}
\vspace{0pt}
\captionof{table}{\textbf{Points on the CCDF of \emph{extracted} per-sequence mass gain.}
We provide specific values from the CCDF over the per-sequence mass gain for extracted sequences only.
For this CCDF, the maximum $y$-value is $100\%$. 
The denominators are different for each, given different counts in the fixed extractable set for \textsc{Pythia}.}
\label{app:tab:pythia:rel-ccdf}
\centering
\begin{tabular}{lllll}
\toprule
& \textbf{1B} & \textbf{2.8B} & \textbf{6.9B} & \textbf{12B}\\
\midrule
$\bf \geq10^{-3}$ & $95.7\%$ & $95.1\%$ & $96.6\%$ & $93.8\%$  \\
$\bf \geq10^{-2}$ & $78.6\%$ & $75.4\%$ & $76.4\%$ & $69.7\%$ \\
$\bf \geq10^{-1}$ & $31.8\%$ &  $38.2\%$ &  $31.2\%$ &  $28.6\%$ \\
\bottomrule
\end{tabular}
\end{minipage}
\end{figure}

\paragraph{Per-sequence $\varepsilon$-shell share analysis.}
We refer to Appendix~\ref{app:sec:experiments:metrics} for detailed explanations of the metrics and plot types for this analysis.

As is clear in all three groups of plots in Figure~\ref{app:fig:pythia:share}, for \textsc{Pythia} on Enron there is a clear (and increasing) concentration of mass on the verbatim continuation as model size increases.
Mass is more evenly distributed across distances for \textsc{Pythia 1B} (Figure~\ref{app:fig:pythia:incremental}), but in general there is a stark pattern of an increasing median verbatim share (Equation~\ref{app:eq:verbatim-share}) as model size increases from 1B to 12B, with a huge jump in between 1B and 2.8B ($4.69\% \to 38.82\% \to 49.77\% \to 61.94\%$).
For the majority of sequences for the 12B model, verbatim share makes up significantly over half of the extraction mass (Figure~\ref{app:fig:pythia:verbatimshare});
the heatmaps similarly show this, with an enormous number of verbatim-only and verbatim-heavy sequences (Figure~\ref{app:fig:pythia:heatmap}). 
These results are consistent with larger models more consistently memorizing the exact target suffix, with smaller models exhibiting fuzzier memorization of the target suffix.
However, even for the 2.8B model, the verbatim share is substantial. 

These results are also consistent with the scatter plots in Figure~\ref{app:fig:pythia:scatter}, which show substantial numbers of sequences with high verbatim mass, especially at larger model sizes.
(See the \textcolor{seabornbluemid}{\textbf{blue}} points concentrated in the top right, which necessarily are close to the $y=x$ reference line as there is not much more mass that near-verbatim could add, given how high the verbatim mass already is.) 

Even so, at all model sizes, there are still many sequences have $0$ verbatim share (corresponding to the \textcolor{seabornredmid}{\textbf{red}} points in Figure~\ref{app:fig:pythia:scatter}), though the fraction of such zero-verbatim-share sequences decreases per model over each model's respective extractable set ($38.9\%$ for 1B to $19.7\%$ for 12B). 
As shown in Figure~\ref{app:fig:pythia:incremental}, mass from other distances is more evenly dispersed. 

\begin{figure*}[t!]
\centering
\begin{subfigure}[t]{0.54\textwidth}
    \centering
    \includegraphics[width=\linewidth,trim={0 0 0 0.75cm},
clip]{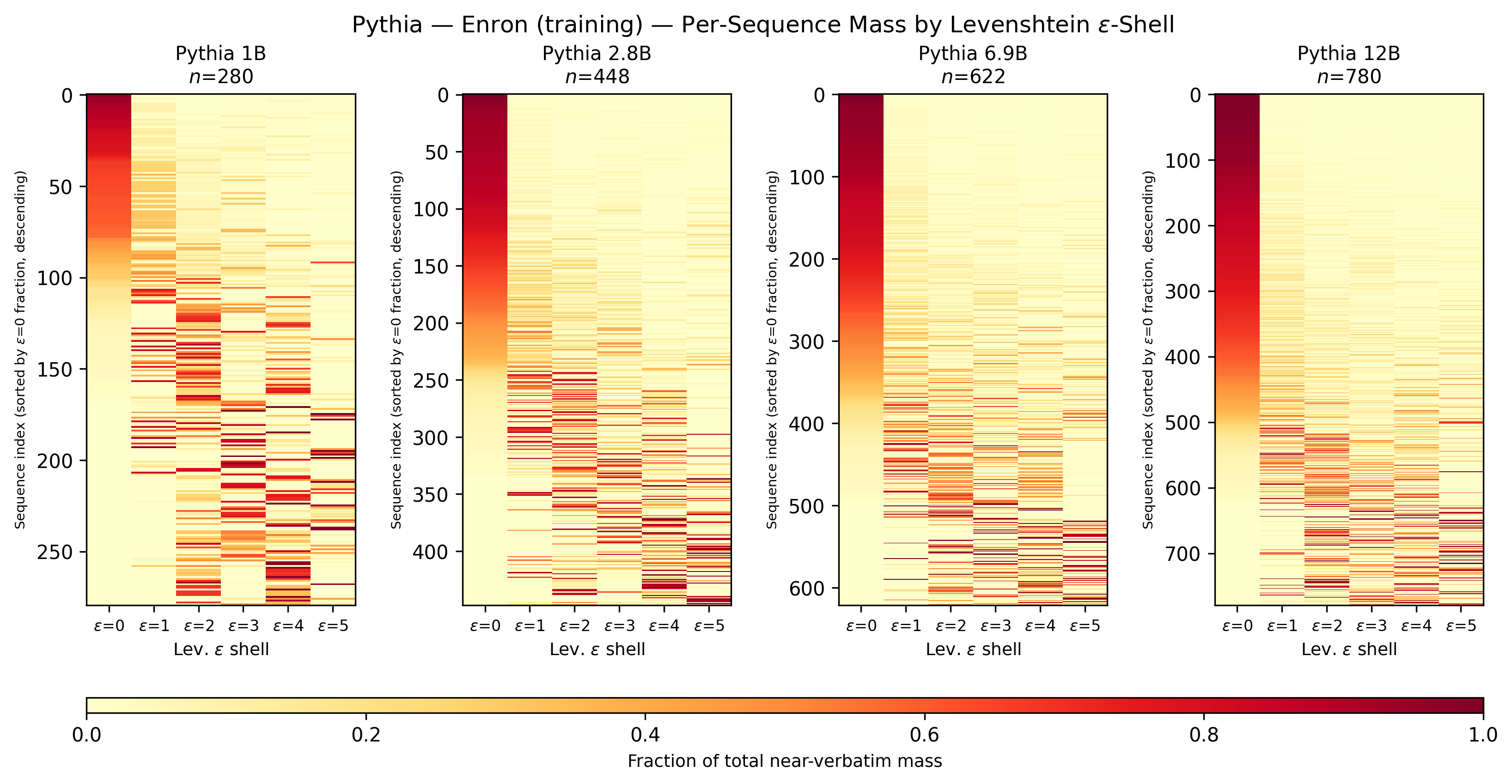}
    \caption{Heatmap of per-sequence mass by $\levshort$ $\varepsilon$-shell share}
    \label{app:fig:pythia:heatmap}
\end{subfigure}
\hfill
\begin{subfigure}[t]{0.45\textwidth}
    \centering
    \includegraphics[width=\linewidth]{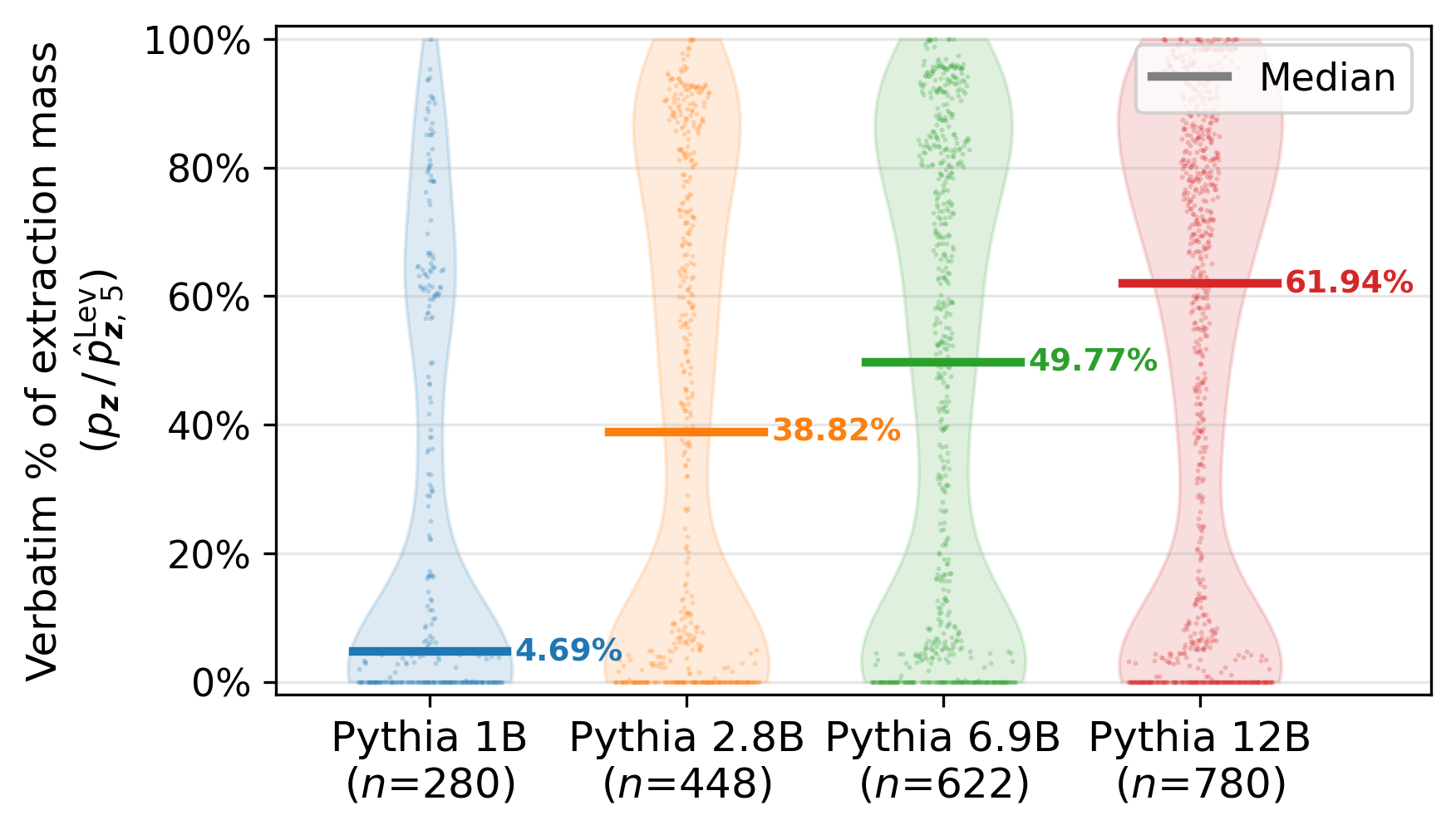}
    \caption{Verbatim mass share distributions}
    \label{app:fig:pythia:verbatimshare}
\end{subfigure}\\
\begin{subfigure}[t]{\textwidth}
    \centering
    \includegraphics[width=\linewidth,trim={0 0 0 0.75cm},clip]{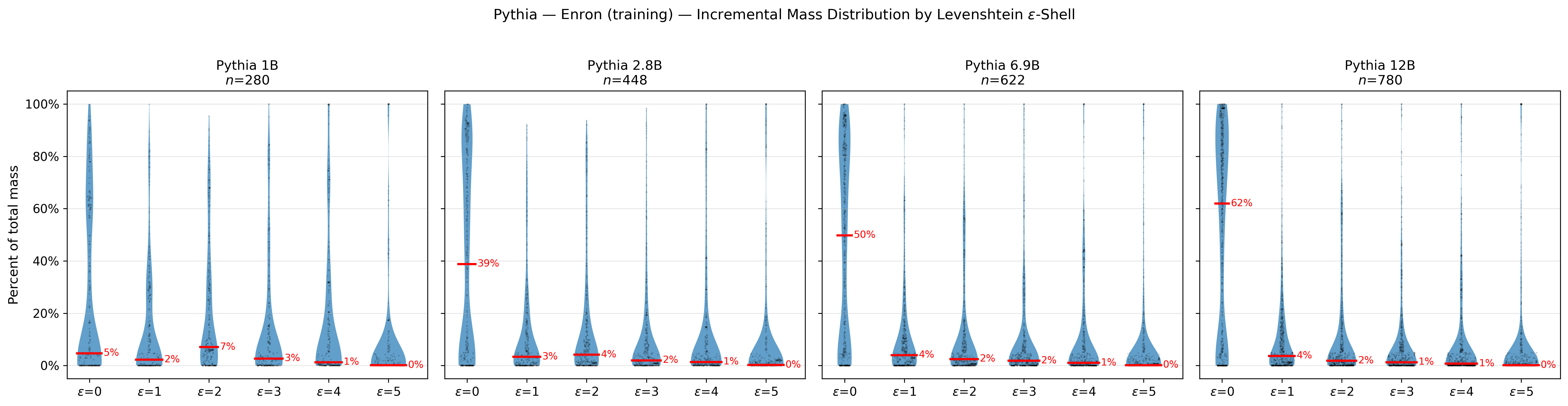}
    \caption{Incremental mass share distributions}
    \label{app:fig:pythia:incremental}
\end{subfigure}
\caption{\textbf{Illustrating different views of $\varepsilon$-shell share for \textsc{Pythia}.}
(\textbf{a}) Heatmaps across model size of per-sequence mass share by $\varepsilon$-shell (Equation~\ref{app:eq:shell-fraction}), sorted by verbatim share (Equation~\ref{app:eq:verbatim-share}).
(\textbf{b}) Violin plots comparing the distribution of per-sequence verbatim share (Equation~\ref{app:eq:verbatim-share}) across model sizes, with the median verbatim share annotated.
(\textbf{c}) Violin plots showing distributions over the per-$\varepsilon$-shell mass share (Equation~\ref{app:eq:shell-fraction}) per model.
Each shell shows the mass share contributed by the given $\lev$ distance. 
Note that each of the three $\varepsilon=0$ violin plots correspond to those in the cross-model comparison in (\textbf{b}). 
}
\label{app:fig:pythia:share}
\end{figure*}
\FloatBarrier

\subsection{Extended results for \textsc{Llama 2} on public domain books}\label{app:experiments:llama2-scale}

We include results for three public domain books (\emph{The Great Gatsby}, \emph{Orlando}, and \emph{Winnie the Pooh}) using the $\levshort$-pruned variant ($\varepsilon=5$). 
The results for \emph{The Great Gatsby} supplement those in Section~\ref{sec:experiments}.
The negative control is the first three chapters of \emph{The People's Dictator}, which has no hits for any of the three \textsc{Llama 2} sizes. 
(It is a perfectly clean negative control.) 

\subsubsection{Extracted sequences and rates}\label{app:experiments:llama2-scale:extraction}

\paragraph{Verbatim greedy extraction  and verbatim probabilistic extraction.} 
As a point of reference, we provide verbatim extraction rates from two different pipelines: 
verbatim greedy extraction (via generation) and verbatim probabilistic extraction (via teacher forcing).
See Table~\ref{tab:llama2:verbatim:rates} for all three books. 
While these methods involve roughly equivalent token evaluations, wall clock time for teacher forcing is almost always faster (for reasons discussed in Appendix~\ref{app:sec:intuition:cost}). 

\begin{table}[t]
\centering
\caption{\textbf{Verbatim extraction rates for \textsc{Llama 2}.}
Verbatim extraction rates via greedy (deterministic generation), teacher forcing, and $k$-CBS ($\levshort \, \varepsilon=5$).
For all experiments, held-out negative controls use the first three chapters of \textit{The People's Dictator} ($10{,}457$ sequences).
Probabilistic and $k$-CBS rates use threshold $\pseq \geq \taumin = 0.001$.}
\label{tab:llama2:verbatim:rates}

\begin{subtable}{\linewidth}
\centering
\caption{\textit{The Great Gatsby} (Train: $13{,}390$ sequences)}
\scriptsize
\begin{tabular}{l cc cc cc}
\toprule
& \multicolumn{2}{c}{\textsc{Llama 2 7B}} & \multicolumn{2}{c}{\textsc{Llama 2 13B}} & \multicolumn{2}{c}{\textsc{Llama 2 70B}} \\
\cmidrule(lr){2-3}
\cmidrule(lr){4-5}
\cmidrule(lr){6-7}
Pipeline & Train & Held-out & Train & Held-out & Train & Held-out \\
\midrule
Verbatim greedy & 0.47\% & 0.00\% & 0.79\% & 0.00\% & 7.24\% & 0.00\% \\
Verbatim probabilistic & 1.21\% & 0.00\% & 2.46\% & 0.00\% & 20.91\% & 0.00\% \\
Verbatim $k$-CBS ($\levshort \, \varepsilon=5$) & 1.22\% & 0.00\% & 2.44\% & 0.00\% & 20.30\% & 0.00\% \\
\bottomrule
\end{tabular}
\end{subtable}

\vspace{0.8em}

\begin{subtable}{\linewidth}
\centering
\caption{\textit{Orlando} (Train: $21{,}822$ sequences)}
\scriptsize
\begin{tabular}{l cc cc cc}
\toprule
& \multicolumn{2}{c}{\textsc{Llama 2 7B}} & \multicolumn{2}{c}{\textsc{Llama 2 13B}} & \multicolumn{2}{c}{\textsc{Llama 2 70B}} \\
\cmidrule(lr){2-3}
\cmidrule(lr){4-5}
\cmidrule(lr){6-7}
Pipeline & Train & Held-out & Train & Held-out & Train & Held-out \\
\midrule
Verbatim greedy & 0.00\% & 0.00\% & 0.00\% & 0.00\% & 0.11\% & 0.00\% \\
Verbatim probabilistic & 0.00\% & 0.00\% & 0.03\% & 0.00\% & 0.31\% & 0.00\% \\
Verbatim $k$-CBS ($\levshort \, \varepsilon=5$) & 0.00\% & 0.00\% & 0.03\% & 0.00\% & 0.29\% & 0.00\% \\
\bottomrule
\end{tabular}
\end{subtable}

\vspace{0.8em}

\begin{subtable}{\linewidth}
\centering
\caption{\textit{Winnie the Pooh} (Train: $6{,}152$ sequences)}
\scriptsize
\begin{tabular}{l cc cc cc}
\toprule
& \multicolumn{2}{c}{\textsc{Llama 2 7B}} & \multicolumn{2}{c}{\textsc{Llama 2 13B}} & \multicolumn{2}{c}{\textsc{Llama 2 70B}} \\
\cmidrule(lr){2-3}
\cmidrule(lr){4-5}
\cmidrule(lr){6-7}
Pipeline & Train & Held-out & Train & Held-out & Train & Held-out \\
\midrule
Verbatim greedy & 0.10\% & 0.00\% & 0.29\% & 0.00\% & 3.75\% & 0.00\% \\
Verbatim probabilistic & 0.39\% & 0.00\% & 0.78\% & 0.00\% & 10.86\% & 0.00\% \\
Verbatim $k$-CBS ($\levshort \, \varepsilon=5$) & 0.39\% & 0.00\% & 0.78\% & 0.00\% & 10.70\% & 0.00\% \\
\bottomrule
\end{tabular}
\end{subtable}

\end{table}

Note that, as underscored by the results in Table~\ref{tab:kcbs:llama2:verbatim:gap}, the lower bound extraction probabilities provided by $\levshort$-pruned $k$-CBS ($\varepsilon=5$) can under-count the verbatim probabilistic extraction rate produced by teacher forcing.
For instance, for \textsc{Llama 2 70B} and \textit{The Great Gatsby} (Table~\ref{tab:kcbs:gatsby:verbatim:gap}), verbatim $k$-CBS misses $86$ instances of extraction ($20.30\%$ extraction rate, compared to the teacher-forced verbatim probabilistic pipeline $20.91\%$ rate; recovery ratio of $97\%$).  
One possible reason is that $k$-CBS provides a lower bound due to pruning paths each iteration, and therefore can miss some mass.
Another reason is that there are subtle differences between logits in teacher forcing (which does not involve KV caching) and generation (which does), which can slightly change exchange results for the same sequences due to floating point rounding.
At a given token in a suffix, there may be multiple tokens with very similar probabilities;
their ranks may flip due to such rounding differences. 
This can also lead to $k$-CBS returning a \emph{higher} verbatim extraction rate than the one from teacher-forced inference.
See, for instance, \textsc{Llama 2 7B} and \emph{The Great Gatsby} (Table~\ref{tab:kcbs:gatsby:verbatim:gap}), where $k$-CBS has a higher extraction rate ($1.22\%$ compared to $1.21\%$; recovery ratio of $1.01$).
Overall, the results between methods are very similar with respect to overall extraction rates;
both flows capture essentially the same results for verbatim extraction. 
We could also widen the beam, at additional cost, to attempt to capture more with $k$-CBS. 

\begin{table}[t]
\centering
\caption{\textbf{Verbatim probabilistic vs.\ verbatim $k$-CBS on \textsc{Llama 2}.}
We report extraction rates ($\taumin = 0.001$, $\bw=20$, $\levshort\,\varepsilon=5$), the $k$-CBS recovery ratio, and the number of sequences missed by $k$-CBS.}
\label{tab:kcbs:llama2:verbatim:gap}

\begin{subtable}{\linewidth}
\centering
\caption{\textit{The Great Gatsby}}
\label{tab:kcbs:gatsby:verbatim:gap}
\scriptsize
\begin{tabular}{l r r r}
\toprule
& \textsc{Llama 2 7B} & \textsc{Llama 2 13B} & \textsc{Llama 2 70B} \\
\midrule
Verbatim probabilistic & 1.21\% & 2.46\% & 20.91\% \\
Verbatim $k$-CBS ($\levshort\,\varepsilon=5$) & 1.22\% & 2.44\% & 20.30\% \\
\midrule
Recovery ratio & 1.01 & 0.99 & 0.97 \\
Missed by $k$-CBS ($\levshort\,\varepsilon=5$) & 0 & 3 & 86 \\
\bottomrule
\end{tabular}
\end{subtable}

\vspace{0.8em}

\begin{subtable}{\linewidth}
\centering
\caption{\textit{Orlando}}
\scriptsize
\begin{tabular}{l r r r}
\toprule
& \textsc{Llama 2 7B} & \textsc{Llama 2 13B} & \textsc{Llama 2 70B} \\
\midrule
Verbatim probabilistic & 0.00\% & 0.03\% & 0.31\% \\
Verbatim $k$-CBS ($\levshort\,\varepsilon=5$) & 0.00\% & 0.03\% & 0.29\% \\
\midrule
Recovery ratio & inf & 1.00 & 0.93 \\
Missed by $k$-CBS & 0 & 0 & 5 \\
\bottomrule
\end{tabular}
\end{subtable}

\vspace{0.8em}

\begin{subtable}{\linewidth}
\centering
\caption{\textit{Winnie the Pooh}}
\scriptsize
\begin{tabular}{l r r r}
\toprule
& \textsc{Llama 2 7B} & \textsc{Llama 2 13B} & \textsc{Llama 2 70B} \\
\midrule
Verbatim probabilistic & 0.39\% & 0.78\% & 10.86\% \\
Verbatim $k$-CBS ($\levshort\,\varepsilon=5$) & 0.39\% & 0.78\% & 10.70\% \\
\midrule
Recovery ratio & 1.00 & 1.00 & 0.99 \\
Missed by $k$-CBS & 0 & 0 & 11 \\
\bottomrule
\end{tabular}
\end{subtable}

\end{table}

\paragraph{Comparing verbatim and near-verbatim extraction rates.}
Figure~\ref{app:fig:llama2:rates} compares greedy and $\levshort\, \varepsilon=5$ $k$-CBS extraction rates for \textsc{Llama 2} across model sizes and Levenshtein thresholds $\varepsilon \in \{0,\ldots,5\}$ for \emph{The Great Gatsby}, \emph{Orlando}, and \emph{Winnie the Pooh}. 
As with the results for \textsc{OLMo 2} (Section~\ref{sec:experiments} \& Appendix~\ref{app:experiments:olmo}) and \textsc{Pythia} (Appendix~\ref{app:experiments:pythia}), $k$-CBS dominates greedy at every $\varepsilon$ and model size, with both rates and gaps growing with model size.
(The exception is \emph{Winnie the Pooh}, where we observe a very slightly increased gap between 7B and 13B; 
it is more pronounced for the other books. See Figure~\ref{app:fig:winnie:rates}.) 
That is, greedy extraction under-counts extraction more for larger models.
The extraction rates for \textsc{Llama 2 70B} are relatively enormous per-book, compared to the smaller models. 
This is true even for \emph{Orlando}, where the overall amount of extraction we observe is very small (the maximum near-verbatim probabilistic rate is just above $0.4\%$, see Figure~\ref{app:fig:orlando:rates}). 
Held-out (negative-control) lines stay flat at zero, supporting the validity of our extraction procedure.

\begin{figure*}[t!]
\centering
\begin{subfigure}{\linewidth}
    \includegraphics[width=\linewidth]{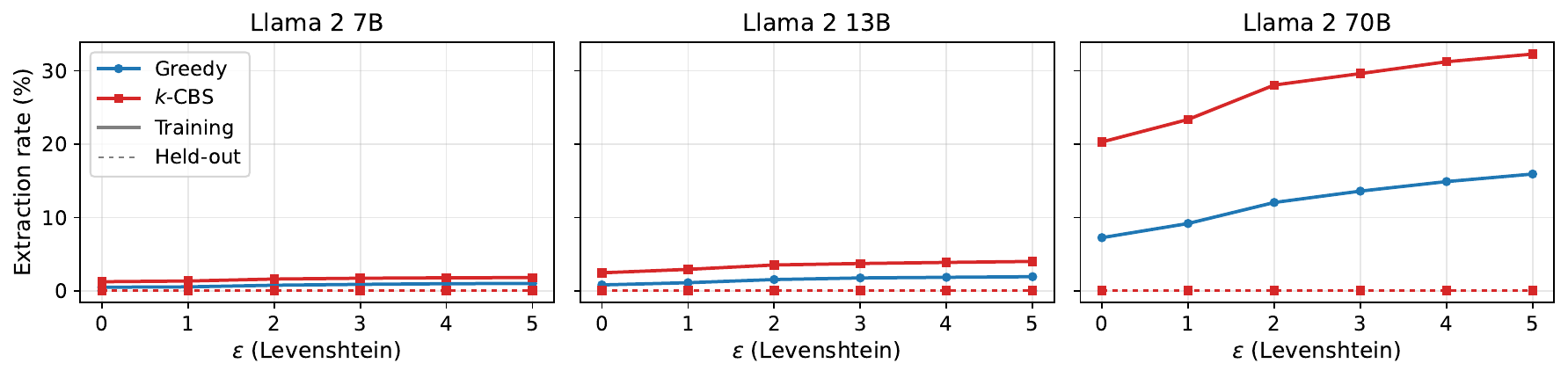}
    \subcaption{\emph{The Great Gatsby} ($13{,}390$ training sequences)}
\end{subfigure}
\begin{subfigure}{\linewidth}
    \includegraphics[width=\linewidth]{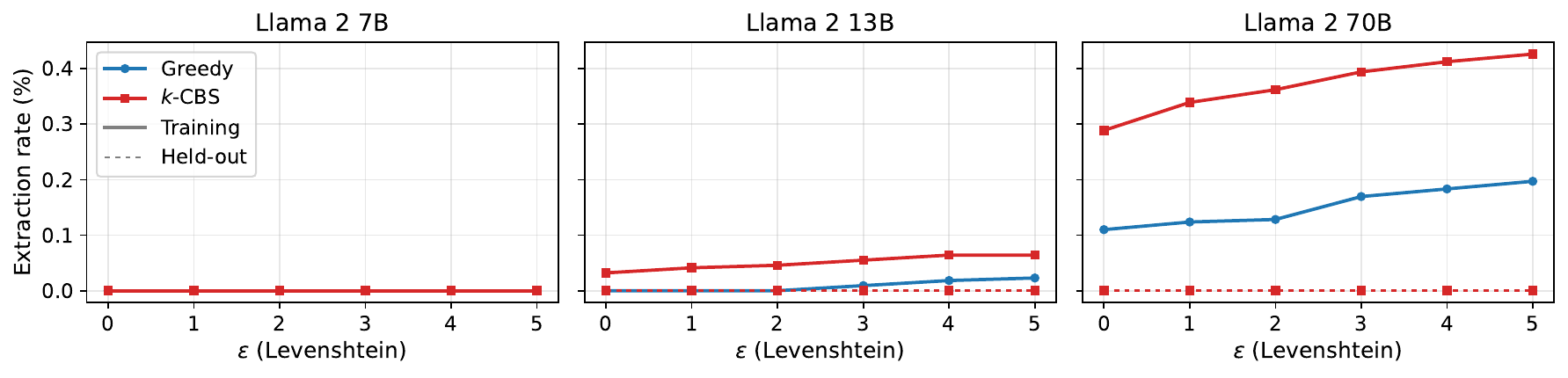}
    \subcaption{\emph{Orlando} ($21{,}822$ training sequences)}
    \label{app:fig:orlando:rates}
\end{subfigure}
\begin{subfigure}{\linewidth}
    \includegraphics[width=\linewidth]{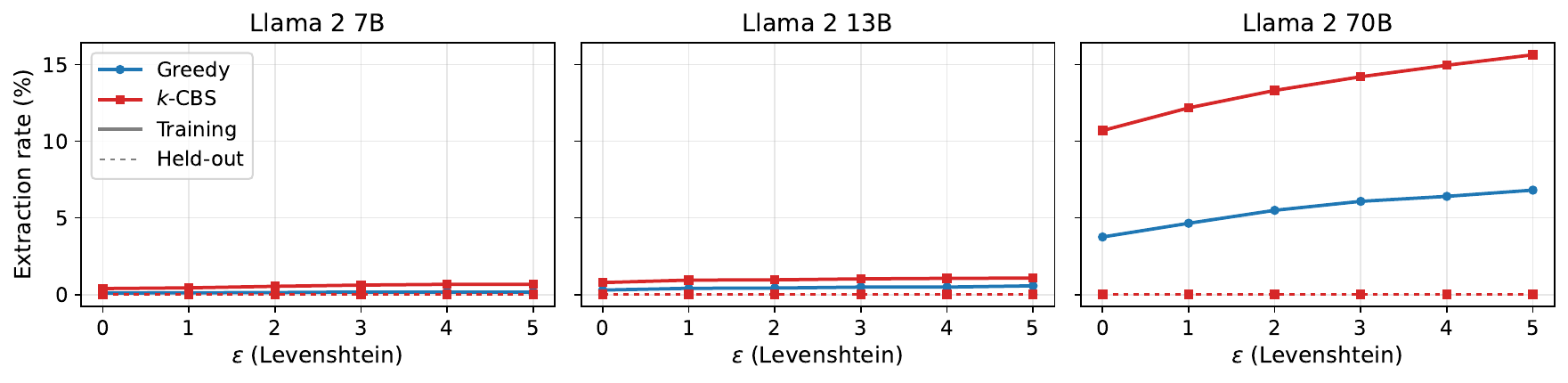}
    \subcaption{\emph{Winnie the Pooh} ($6{,}152$ training sequences)}
    \label{app:fig:winnie:rates}
\end{subfigure}
\caption{\textbf{Comparing extraction rates for \textsc{Llama 2}.} 
For \textsc{Llama 2 7B}, \textsc{13B}, and \textsc{70B} we show greedy and $\levshort\, \varepsilon=5$ $k$-CBS probabilistic rates for verbatim extraction ($\varepsilon\!=\!0$) and near-verbatim extraction for $\varepsilon\, \in \{1,\ldots,5\}$ (based on subsetting the pruned algorithm's results).
To assess validity, we also run analogous negative controls on $10{,}457$ sequences from \emph{The People's Dictator} (first three chapters). 
The greedy rates are exact. 
The probabilistic rates are computed with $\levshort$-pruned $k$-CBS (Section~\ref{sec:pruning});
they may miss some valid instances of extraction, and thus should be interpreted as lower bounds on extraction rates.\looseness=-1}
\label{app:fig:llama2:rates}
\end{figure*}

\paragraph{Negative control results.} 
Not a single sequence from the held-out book (\emph{The People's Dictator}) is flagged at any $\varepsilon$ by either method.

\begin{figure*}[t!]
\centering
\begin{subfigure}{\linewidth}
    \includegraphics[width=\linewidth]{paper/figure/gatsby/gatsby_llama-2_model-size_lev-eps5_verbatim_vs_nearvb_junk.png}
    \subcaption{\emph{The Great Gatsby} (replicated from Figure~\ref{fig:gatsby:scatter:main})}
\end{subfigure}
\begin{subfigure}{\linewidth}
    \includegraphics[width=\linewidth]{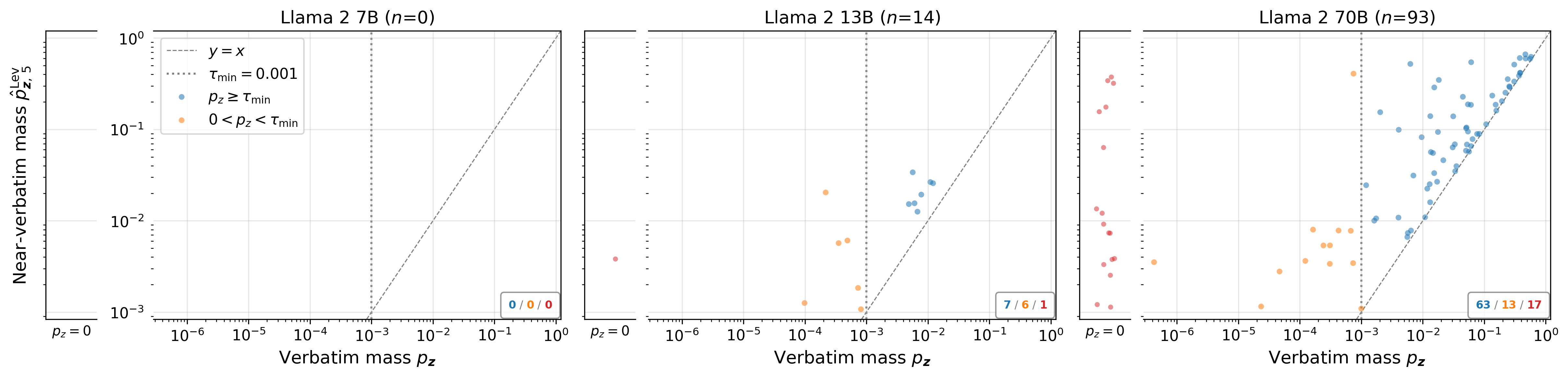}
    \subcaption{\emph{Orlando}}
    \label{app:fig:orlando:scatter}
\end{subfigure}
\begin{subfigure}{\linewidth}
    \includegraphics[width=\linewidth]{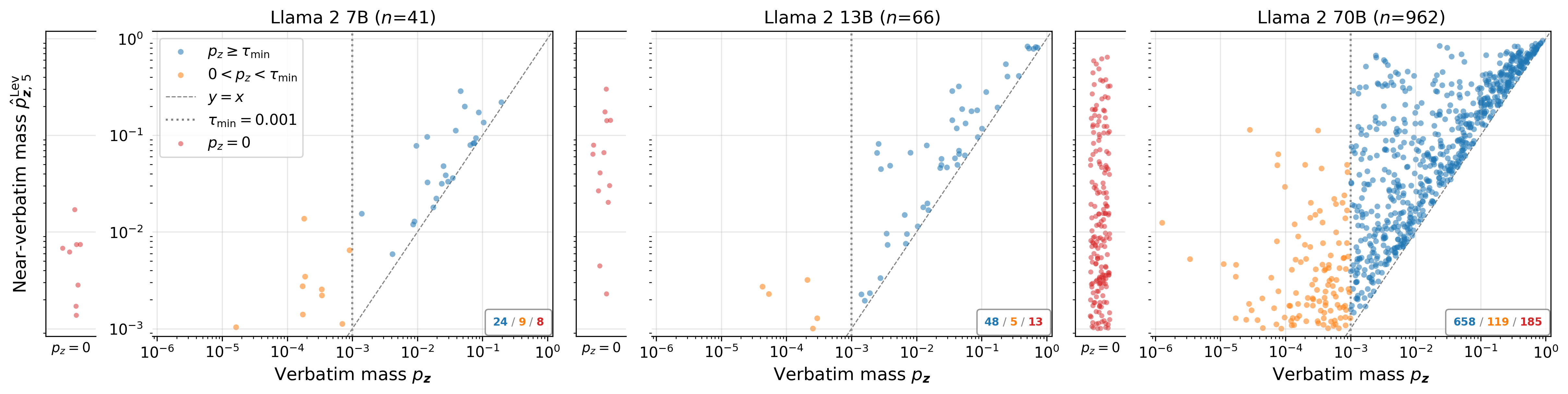}
    \subcaption{\emph{Winnie the Pooh}}
    \label{app:fig:winnie:scatter}
\end{subfigure}
\caption{\textbf{Near-verbatim mass vs.\ verbatim mass for \textsc{Llama 2}.} 
\textsc{Llama 2} on three public domain books; 
each point is one sequence.
Axes show near-verbatim ($p_{\seq,5}^\levshort$, $\levshort\,\varepsilon\!=\!5$) vs.\ verbatim ($\pseq$) extraction mass on a $\log$--$\log$ scale.
\textcolor{seabornredmid}{\textbf{Red}}/\textcolor{seabornorangemid}{\textbf{orange}}
points are ``unlocked'' by near-verbatim extraction (to the left of the $\taumin$ dotted reference line, $\pseq\!<\!\taumin$, but $p_{\seq,5}^\levshort\!\geq\!\taumin$);
\textcolor{seabornbluemid}{\textbf{blue}} points are verbatim-extractable ($\pseq\!\geq\!\taumin$).
Points above the dashed $y\!=\!x$ line show increased extraction risk when near-verbatim mass is accounted for.\looseness=-1}
\label{app:fig:llama2:scatter}
\end{figure*}

\paragraph{``Unlocked'' extracted sequences.} 
The three different books exhibit very different degrees of extraction (in terms of number of instances) and risk (in terms of magnitude and variation). 
With respect to risk, Figure~\ref{app:fig:llama2:scatter} gives a sense of how per-sequence risk changes across model sizes and books. 

\emph{The Great Gatsby} is more significantly memorized in general, and shows larger amounts of verbatim extraction (and increased risk) as well as ``unlocked'' near-verbatim extraction.
\emph{Orlando} is barely memorized by any model, but there is increased memorization (and risk) for \textsc{Llama 2 70B}.
\emph{Winnie the Pooh} shows patterns similar to \emph{The Great Gatsby}, but there is less extraction overall. 

In general, larger models exhibit larger degrees of risk---\textcolor{seabornbluemid}{\textbf{blue}} points sit well above $y$=$x$, indicating increased extraction risk for verbatim-extractable sequences, with higher risk overall for larger models.
Near-verbatim measurements reveal extraction that verbatim methods do not catch (\textcolor{seabornredmid}{\textbf{red}} + \textcolor{seabornorangemid}{\textbf{orange}}). 

We provide some qualitative examples of extracted sequences, to give a sense of what near-verbatim measurements capture (in terms of the diff and risk increase).
We mark missing text from the ground-truth suffix from Books3 in \textcolor{red}{red} and additions in the generation in \textcolor{blue}{blue}.

\begin{itemize}[leftmargin=.75cm]
    \item \textit{The Great Gatsby}, \textsc{Llama 2 70B}
    \begin{itemize}[leftmargin=.75cm]
        \item Verbatim mass: $0$; near-verbatim mass ($\levshort\, \varepsilon=5$): $0.863$.  A typo in the ground-truth text we have from Books3 is repaired by the closest-match generation.

        {\footnotesize
        \noindent $\pre$: \verb|In between time|{\textendash}\verb|"\n\nAs I went over to say goodbye I saw that the |\\ 
        \verb|expression of bewilderment had come back into Gatsby's face, as though a |\\
        \verb|faint doubt had occurred to him as to the|
        
        \vspace{0.1cm}
         
        \noindent $\suf$: \verb| quality of his present happiness. Almost five years! There must have been |\\
        \verb|moments even that afternoon | \textcolor{red}{\texttt{whe }} \verb|Daisy tumbled short of his dreams|\textendash\verb|not |\\
        \verb|through her own fault, but because of the colossal vitality of his illusion|

        \vspace{0.1cm}
    
        \noindent Best near-verbatim $\gensuf$ ($\levshort=1$, mass $0.628$): \verb| quality of his present happiness. |\\
        \verb|Almost five years! There must have been |\verb|moments even that afternoon | \textcolor{blue}{\texttt{when}} \\
        \verb|Daisy tumbled short of his dreams|\textendash\verb|not through her own fault, but because |\\
        \verb|of the colossal vitality of his illusion|
        }

        \item Verbatim mass: $0$; near-verbatim mass ($\levshort\, \varepsilon=5$): $0.822$.  A space in the ground-truth suffix is not included in the closest-match generation.

        {\footnotesize
        \noindent $\pre$: \verb|y anything in his house, old sport."\n\n"She's got an indiscreet voice," I |\\ 
        \verb|remarked. "It's full of|\textendash\verb|" I hesitated.\n\n"Her voice is full of money," he|
        
        \vspace{0.1cm}
         
        \noindent $\suf$: \verb| said suddenly.\n\nThat was it. I'd never understood before. It was full |\\
        \verb|of money|\textendash\colorbox{red!15}{\rule{0pt}{1.2ex}\hspace{0.2em}}\verb|that was the inexhaustible charm that rose and fell in it, the |\\
        \verb|jingle of it, the cymbals|

        \vspace{0.1cm}
    
        \noindent Best near-verbatim $\gensuf$ ($\levshort=1$, mass $0.777$): \verb| said suddenly.\n\nThat was it I'd |\\
        \verb|never understood before. It was full of money|\textendash\verb|that was the inexhaustible charm |\\
        \verb|that rose and fell in it, the jingle of it, the cymbals|
        }
    \end{itemize}
    \item \textit{Orlando}, \textsc{Llama 2 70B}
    \begin{itemize}[leftmargin=.75cm]
        \item Verbatim mass $0.000746$; near-verbatim mass ($\levshort\, \varepsilon=5$): $0.408$.  Punctuation differences.

        {\footnotesize
        \noindent $\pre$: \verb|n. Let us, as he takes his seat, read the following passage from the |\\
        \verb|_Spectator_ :\n\n'I consider woman as a beautiful, romantic animal, that may |\\
        \verb|be adorned with furs and feathers|
        
        \vspace{0.1cm}
         
        \noindent $\suf$: \verb|, pearls and diamonds, ores and silks. The lynx shall cast its skin at |\\
        \verb|her feet to make her a tippet|\textcolor{red}{\texttt{,}} \verb|the peacock, parrot and swan shall pay |\\
        \verb|contributions to her muff|\textcolor{red}{\texttt{;}}

        \vspace{0.1cm}
    
        \noindent Best near-verbatim $\gensuf$ ($\levshort=3$, mass $0.232$): \verb|, pearls and diamonds, ores and |\\
        \verb|silks. The lynx shall cast its skin at her feet to make her a tippet|\textcolor{blue}{\texttt{;}} \verb|the|\\
        \verb|the peacock, parrot, and swan shall pay contributions to her muff|
        }

        \item Verbatim mass $0$; near-verbatim mass ($\levshort\, \varepsilon=5$): $0.176$.  Punctuation differences, and then text shift as a result of tokenization-to-text length variation.

        {\footnotesize
        \noindent $\pre$: \verb|the voyage to the Houyhnhnms:\n\n'I enjoyed perfect Health of Body and |\\
        \verb|Tranquillity of Mind; I did not find the Treachery or Inconstancy of a Friend, |\\
        \verb|nor the Injuries|
        
        \vspace{0.1cm}
         
        \noindent $\suf$: \verb|of a secret or open Enemy. I had no occasion of bribing, flattering |\\
        \verb|or pimping, to procure the Favour of any great Man or of his Minion|\textcolor{red}{\texttt{.}} \\
        \verb|I wanted no Fence against | \textcolor{red}{\texttt{Fraud}}

        \vspace{0.1cm}
    
        \noindent Best near-verbatim $\gensuf$ ($\levshort=5$, mass $0.127$): \verb|of a secret or open Enemy. I had |\\
        \verb|no occasion of bribing, flattering|\textcolor{blue}{\texttt{,}} \verb|or pimping, to procure the Favour of |\\
        \verb|any great Man|\textcolor{blue}{\texttt{,}} \verb|or of his Minion|\textcolor{blue}{\texttt{;}} \verb|I wanted no Fence against|
        }
    \end{itemize}
    \item \textit{Winnie the Pooh}, \textsc{Llama 2 70B}
    \begin{itemize}[leftmargin=.75cm]
        \item Verbatim mass $0$; near-verbatim mass ($\levshort\, \varepsilon=2$): $0.643$.  Hyphenization difference, and then text shift as a result of tokenization-to-text length variation. 

        {\footnotesize
        \noindent $\pre$: \verb|t\n\n## IN WHICH\n\nChristopher Robin Leads an Expotition to the North Pole|\\
        \verb|\n\nONE FINE DAY Pooh had stumped up to the top of the Forest to see if his friend|\\
        
        \vspace{0.1cm}
        \noindent $\suf$: \verb|Christopher Robin was interested in Bears at all. At breakfast that morning |\\
        \verb|(a simple meal of marmalade spread lightly over a honey|\textcolor{red}{\texttt{-}}\verb|comb or two) he had |\\
        \verb|suddenly thought of a new song. It began like this:\n|

        \vspace{0.1cm}
    
        \noindent Best near-verbatim $\gensuf$ ($\levshort=2$, mass $0.405$): \verb|Christopher Robin was interested |\\
        \verb|in Bears at all. At breakfast that morning (a simple meal of marmalade spread |\\
        \verb|lightly over a honeycomb or two) he had suddenly thought of a new |\\
        \verb|song. It began like this:\n|\textcolor{blue}{\texttt{\textbackslash n}}
        }
    \end{itemize}
\end{itemize}

Note that the extracted sequences from \emph{Orlando} included above reflect quotes from other texts, not Woolf's original writing.
The first sequence is from an essay by Joseph Addison, and the second is from \emph{Gulliver's Travels}.

\subsubsection{Extraction risk}\label{app:experiments:llama2-scale:risk}

\textbf{CCDF over near-verbatim risk gain.}
We similarly provide CCDF views of the near-verbatim risk gain, specifically for \emph{The Great Gatsby} and \emph{Winnie the Pooh} (Figure~\ref{app:fig:llama2:ccdfs}).
We omit \emph{Orlando} given the minimal amount of extraction, and instead refer to the scatter plot visualization to assess risk gain for those results (Figure~\ref{app:fig:orlando:scatter}). 
We also provide tables of points on the CCDF of the per-sequence mass gain on the fixed extractable set only, where the fixed set differs by model (Table~\ref{app:tab:llama2:rel-ccdf}).
Overall, we observe the same patterns as in our other results on \textsc{OLMo 2} (Appendix~\ref{app:sec:experiments:setup:olmo}) and \textsc{Pythia} (Appendix~\ref{app:sec:experiments:setup:pythia}), albeit with larger gaps between the 70B model and the smaller models. 
In the fixed-set CCDF results, unlike for the prior results, we do generally observe a dominance pattern with respect to relative mass gains computed according to each model's extractable set for \emph{The Great Gatsby}:
this underscores the enormity of the relative, not only absolute, gains in risk for the 70B model on this book. 

\begin{figure*}[t!]
\centering
\begin{subfigure}{0.45\linewidth}
    \includegraphics[width=\linewidth]{paper/figure/gatsby/gatsby_llama-2_model-size_lev-eps5_per_seq_nv_gain.png}
    \subcaption{\emph{The Great Gatsby} (replicated from Figure~\ref{fig:ccdf:gatsby:main})}
\end{subfigure}
\begin{subfigure}{0.45\linewidth}
    \includegraphics[width=\linewidth]{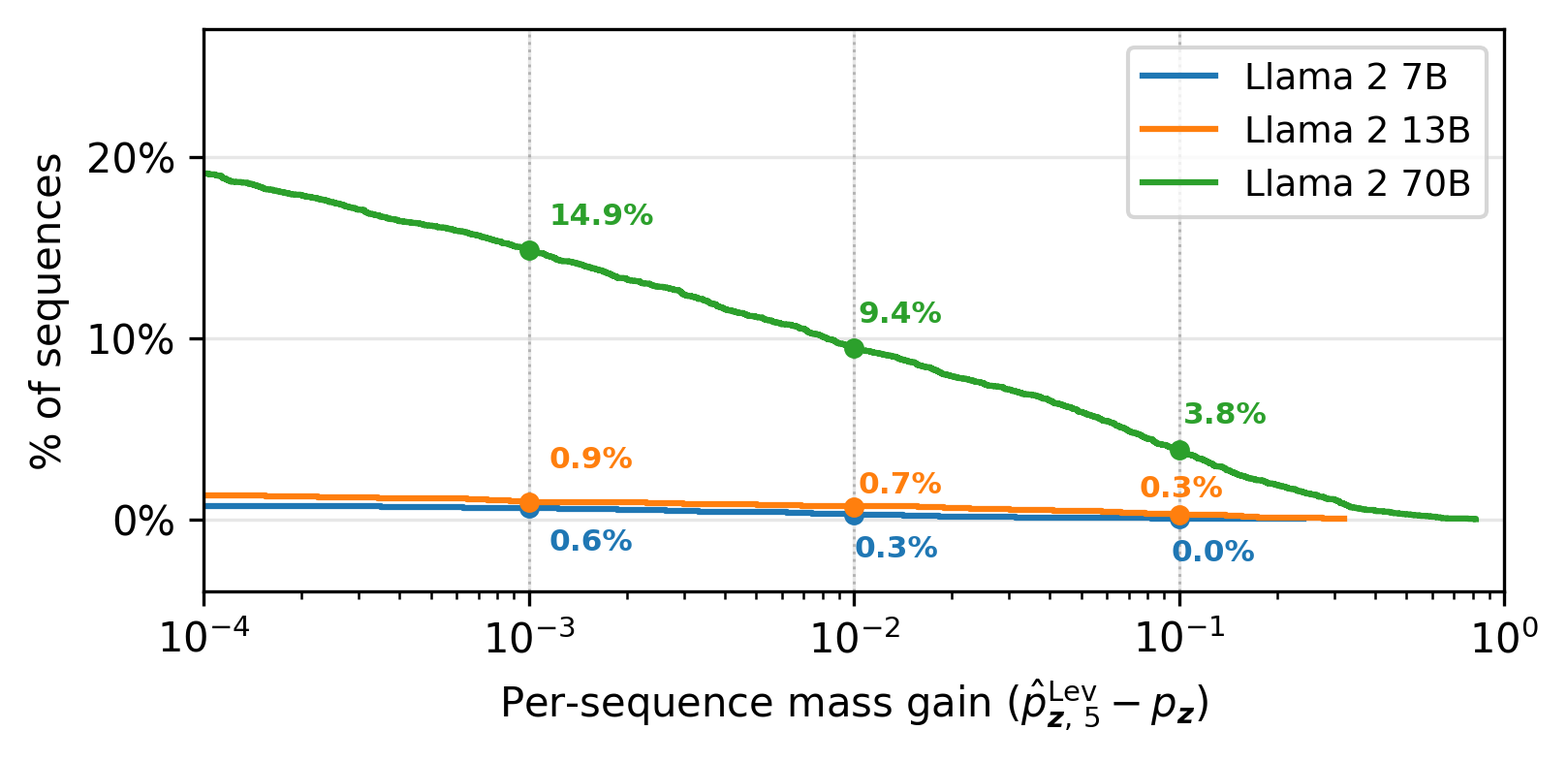}
    \subcaption{\emph{Winnie the Pooh}}
\end{subfigure}
\caption{\textbf{CCDF of population per-sequence near-verbatim mass gain for \textsc{Llama 2}.} 
For $\levshort\, \varepsilon\!=\!5$ mass minus verbatim mass ($\hat{p}_{\seq,5}^{\levshort} - \pseq$), a point $(x, y)$ means $y\%$ of sequences have extraction-mass gain $\geq x$.
Plotted over the whole training set sample for each book ($13{,}390$ sequences from \emph{The Great Gatsby}; $6{,}152$ sequences from \emph{Winnie the Pooh}).\looseness=-1}
\label{app:fig:llama2:ccdfs}
\end{figure*}
\begin{figure}[t]
\centering
\begin{minipage}[t]{\linewidth}
\vspace{0pt}
\captionof{table}{\textbf{Points on the CCDF of \emph{extracted} per-sequence mass gain for \textsc{Llama 2}.}
We provide specific values from the CCDF over the per-sequence mass gain for extracted sequences only.
For this CCDF, the maximum $y$-value is $100\%$. 
The denominators are different for each, given different counts in the fixed extractable set for \textsc{Llama 2}, both on different models and different books.}
\label{app:tab:llama2:rel-ccdf}
\centering
\begin{minipage}[t]{0.48\linewidth}
\centering
\textit{The Great Gatsby}\\[2pt]
\begin{tabular}{lllll}
\toprule
& \textbf{7B} & \textbf{13B} & \textbf{70B}\\
\midrule
$\bf \geq10^{-3}$ & $95.0\%$ & $94.0\%$  & $97.2\%$ \\
$\bf \geq10^{-2}$ & $59.6\%$ & $61.7\%$ & $76.0\%$ \\
$\bf \geq10^{-1}$ & $17.1\%$ & $26.9\%$ & $39.2\%$ \\
\bottomrule
\end{tabular}
\end{minipage}
\hfill
\begin{minipage}[t]{0.48\linewidth}
\centering
\textit{Winnie the Pooh}\\[2pt]
\begin{tabular}{lllll}
\toprule
& \textbf{7B} & \textbf{13B} & \textbf{70B}\\
\midrule
$\bf \geq10^{-3}$ & $95.1\%$ & $87.9\%$  & $95.0\%$ \\
$\bf \geq10^{-2}$ & $45.1\%$ & $65.2\%$ & $24.2\%$\\
$\bf \geq10^{-1}$ & $4.9\%$ & $60.4\%$ & $24.5\%$\\
\bottomrule
\end{tabular}
\end{minipage}
\end{minipage}
\end{figure}

\begin{figure*}[t!]
\centering
\begin{subfigure}{\linewidth}
    \includegraphics[width=\linewidth,trim={0 0 0 1cm}, clip]{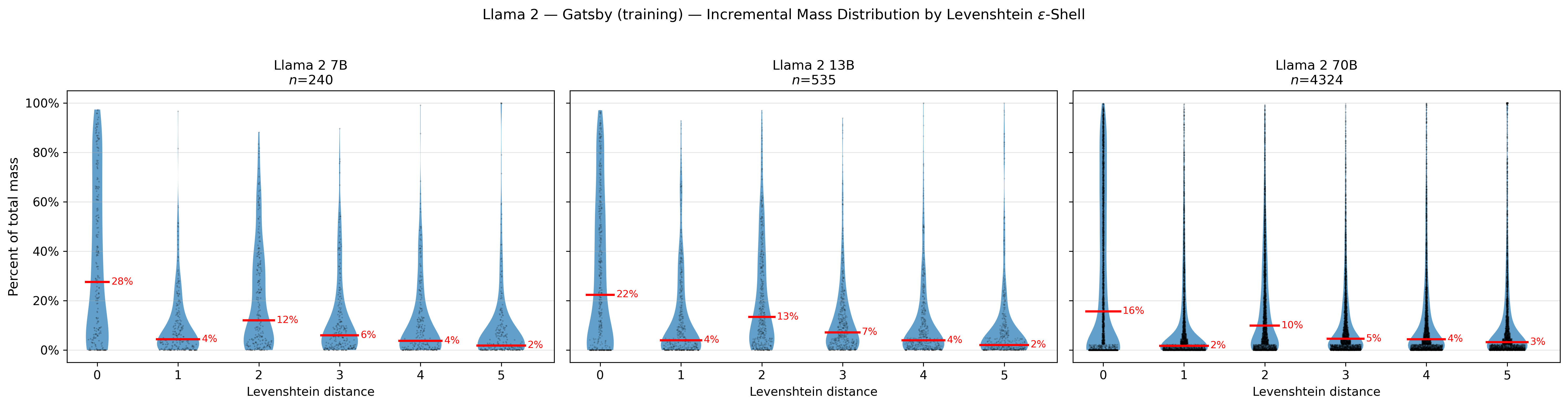}
    \subcaption{\emph{The Great Gatsby}}
    \label{app:fig:gatsby:shell}
\end{subfigure}
\begin{subfigure}{\linewidth}
    \includegraphics[width=\linewidth,trim={0 0 0 1cm}, clip]{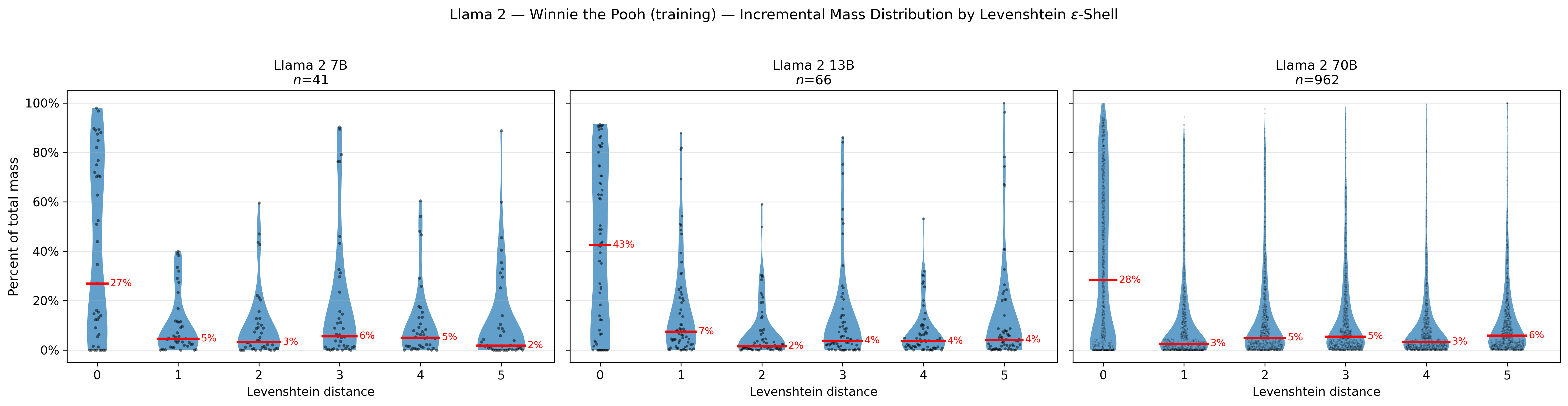}
    \subcaption{\emph{Winnie the Pooh}}
    \label{app:fig:winnie:shell}
\end{subfigure}
\caption{\textbf{Illustrating $\varepsilon$-shell share for \textsc{Llama 2}.}
Violin plots showing distributions over the per-$\varepsilon$-shell mass share (Equation~\ref{app:eq:shell-fraction}) per model and two books.
Each shell shows the mass share contributed by the given $\lev$ distance. 
}
\label{app:fig:llama2:shell}
\end{figure*}

\paragraph{Per-sequence $\varepsilon$-shell share analysis.}
We refer to Appendix~\ref{app:sec:experiments:metrics} for detailed explanations of the metrics and plot types for this analysis.
The dispersal across $\varepsilon$-shells varies by book, with respect to how memorized the book is.
As discussed in the main paper (Section~\ref{sec:experiments:b}), \emph{The Great Gatsby} shows decreasing median verbatim share by model size.
Here, it is also clear that the majority of mass comes from $\varepsilon\leq2$ across all three model sizes, with fairly equal dispersal in the last three distance shells.
There is far less extraction for \emph{Winnie the Pooh} for both \textsc{Llama 2 7B} and \textsc{13B}, so these violins are produced from significantly fewer data points.
We therefore do not make claims about relative mass shifts across model sizes, but we do note that the verbatim shell contains the most mass (with respect to median verbatim share) across all three model sizes. 

\subsubsection{Heatmaps and cross-model sequence analysis}\label{app:experiments:llama2-scale:sequence}

\paragraph{Heatmaps visualizing extracted regions.}
We heatmap draw from visualization strategies in \citet{cooper2025books} to illustrate where memorization manifests in a particular book (character location), enabling comparisons of regions of memorized text across models. 
In Figure~\ref{app:fig:heatmaps} for a given heatmap, each vertical strip corresponds to a character position in the source text; 
color intensity shows the \emph{maximum} extraction probability across all overlapping $50$-token suffixes that cover that position, on a $\log$ scale.
White indicates no extractable suffix ($< \taumin = 0.001$).

The heatmaps showcase how memorization varies across models and books.
Our analysis shows the verbatim heatmaps (as in~\citet{cooper2025books}), compared to the near-verbatim heatmaps produced using $\levshort, \varepsilon=5$. 
There are several points worth highlighting.
The first two underscore our other analysis: 
near-verbatim extraction identifies more unique instances of extraction, and increased extraction risk for sequences that are verbatim extractable.
For instance, consider Figure~\ref{app:fig:gatsby:heatmap}, which shows the heatmaps for \emph{The Great Gatsby}.
For \textsc{Llama 2 7B} at around character $220{,}000$, the near-verbatim heatmap picks up extraction that verbatim extraction misses; 
starting at around character $170{,}000$, the near-verbatim extracted region exhibits an extraction risk increase (darker blue) compared to the verbatim heatmap.\looseness=-1 

The heatmaps also show how near-verbatim extraction at smaller sizes can be an indicator (though not determinative of) verbatim extraction at larger model sizes.
For instance, for \emph{The Great Gatsby} and \textsc{Llama 2 7B}, there are $76$ sequences that are $\levshort\, \varepsilon=5$ near-verbatim extractable (visible on the near-verbatim heatmap) but not verbatim extractable (not visible on the verbatim heatmap).
Of these sequences, $31$ ($40.8\%$) are verbatim extractable at \textsc{Llama 2 13B}.
Similarly, there are $208$ $\levshort\, \varepsilon=5$ near-verbatim extractable (but not verbatim extractable) sequences, of which $106$ ($51.0\%$) are verbatim extractable at \textsc{Llama 2 70B}.
One of the easiest to see examples of this is with the \emph{Orlando} heatmaps (Figure~\ref{app:fig:orlando:heatmap}), as there are much sparser extraction hits to pick apart.
At about $30{,}000$ characters, both \textsc{Llama 2 7B} and \textsc{Llama 2 13B} show a near-verbatim extractable region that is not verbatim extractable; this region is verbatim extractable for \textsc{Llama 2 70B}.

\paragraph{Analyzing mass for common extractable sequences across model sizes.}
The heatmaps (Figure~\ref{app:fig:heatmaps}) show significantly increased instances of near-verbatim extraction for the 70B model, complementing the scatter plots in Figure~\ref{app:fig:llama2:scatter}).
We also and observed decrease in median verbatim share from 7B to 70B for both \emph{The Great Gatsby} and \emph{Winnie the Pooh} (Figure~\ref{app:fig:llama2:shell}). 
To decompose this a bit, we restrict to the sequences extractable at all three model sizes---i.e., hold the set of sequences fixed---and ask: 
how much more probability does a larger model place on the same memorized content?
We show results for \emph{The Great Gatsby} and \emph{Winnie the Pooh} in Figure~\ref{app:fig:common-set};
we omit results for \emph{Orlando} given the minimal amount of overall extraction. 

Both books exhibit similar patterns in their distributions:
from low (near $\taumin$) for 7B to very high extraction risk at 70B. 
At 7B, most sequences cluster near the extraction threshold (for \emph{Gatsby}, median of $0.037$ and mean of $0.119$; for \emph{Winnie the Pooh}, median of $0.053$ and mean of $0.020$).
The typical common-extractable sequence has very low extraction risk. 
At 70B for both books, the bulk of the distribution sits in the $0.4$--$0.9$ range (for \emph{Gatsby} median of $0.693$ and mean of $0.598$; for \emph{Winnie the Pooh},  median of $0.666$ and mean of $0.0564$).
The model assigns majority probability to near-verbatim continuations.
The few sequences that remain low-mass at 70B pull the mean below the median.

Overall, we observe that the shift indicates that thee same sequences carry dramatically more mass at larger model sizes. 
In general, model scale does not just widen the set of memorized content;
it deepens the memorization of content at smaller scales.
Median verbatim share \emph{over the whole population of extracted sequences per model} can still decrease, as larger models surface many more weakly memorized sequences.\looseness=-1

\begin{figure*}[t!]
\centering
\begin{subfigure}{\linewidth}
    \includegraphics[width=\linewidth,trim={0 0 0 1cm}, clip]{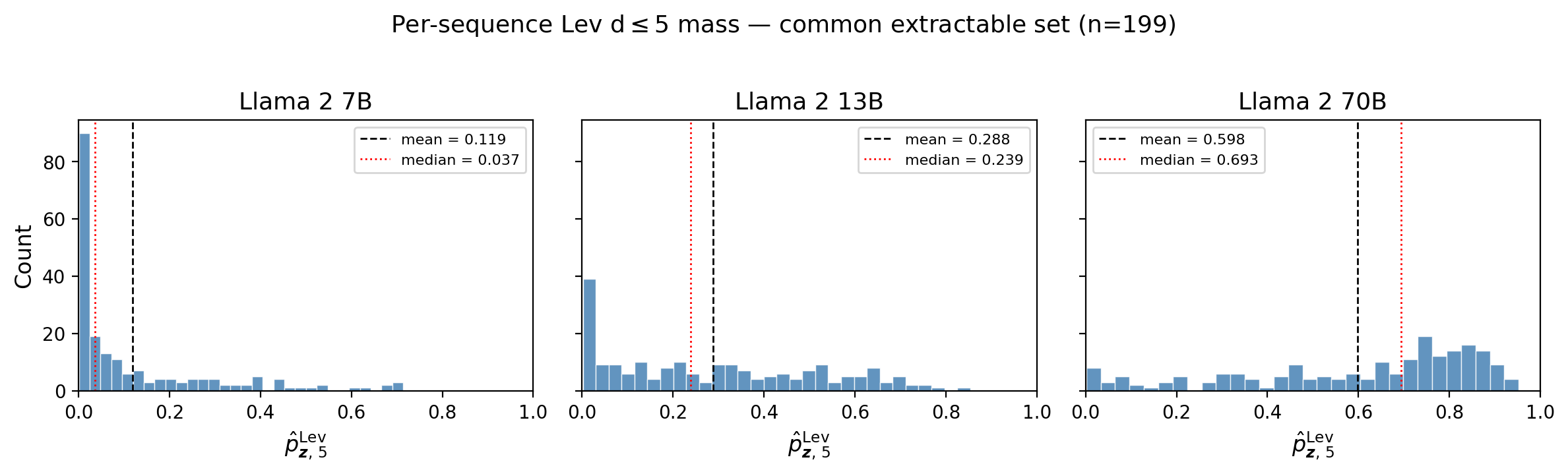}
    \subcaption{\emph{The Great Gatsby} ($n=199$)}
    \label{app:fig:gatsby:common}
\end{subfigure}
\begin{subfigure}{\linewidth}
    \includegraphics[width=\linewidth,trim={0 0 0 1cm}, clip]{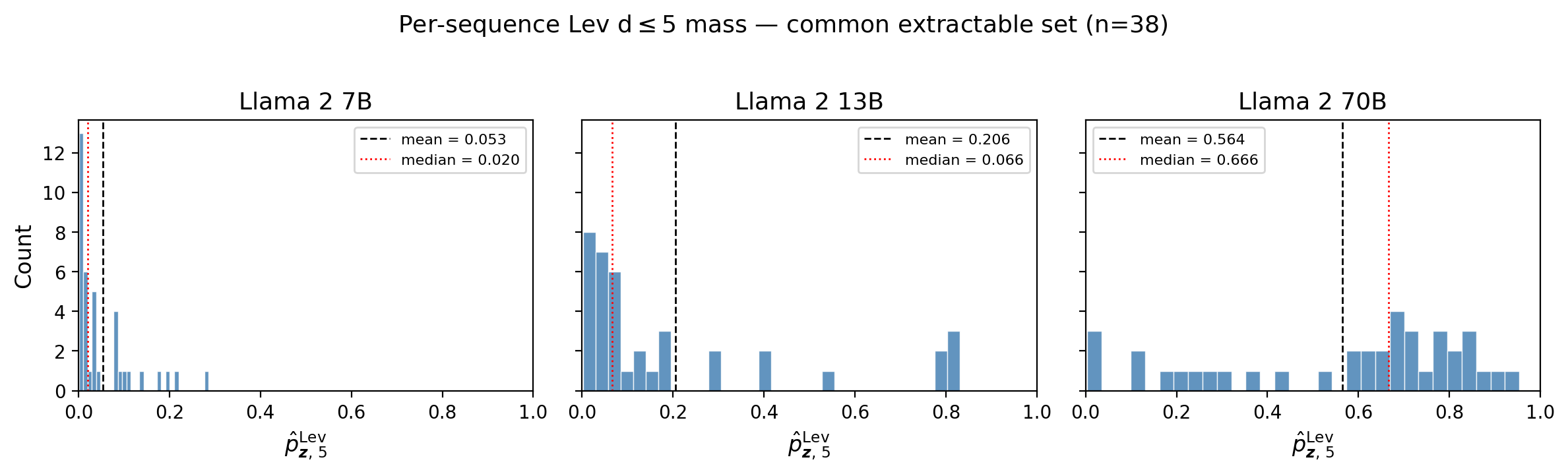}
    \subcaption{\emph{Winnie the Pooh} ($n=38$)}
    \label{app:fig:winnie:common}
\end{subfigure}
\caption{\textbf{Fixed-set mass grows dramatically with model size.}
By restricting to the sequences extractable at all three model sizes, we hold the set of sequences fixed and ask: 
how much more probability does a larger model place on the same memorized content?
}
\label{app:fig:common-set}
\end{figure*}

\begin{figure*}[t!]
\centering
\begin{subfigure}{\linewidth}
    \includegraphics[width=\linewidth]{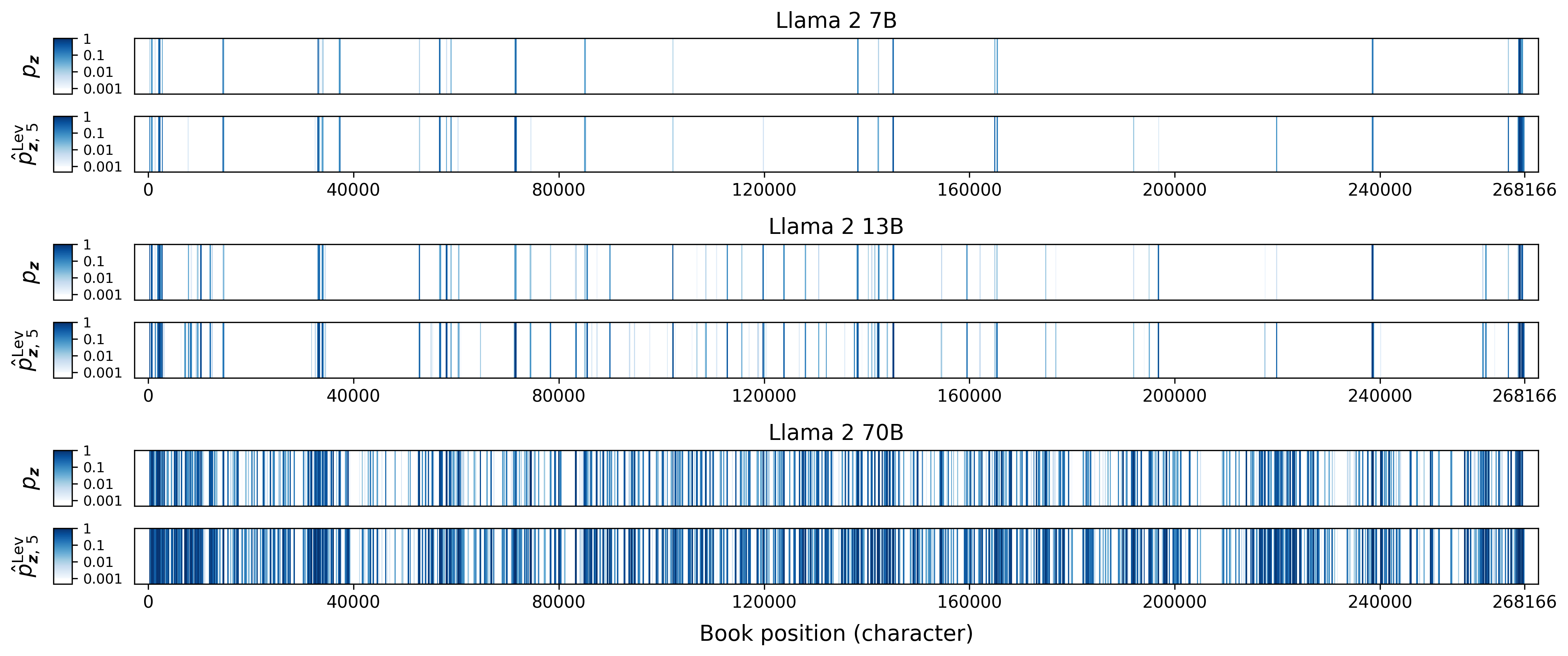}
    \subcaption{\emph{The Great Gatsby}}
    \label{app:fig:gatsby:heatmap}
\end{subfigure}
\begin{subfigure}{\linewidth}
    \includegraphics[width=\linewidth]{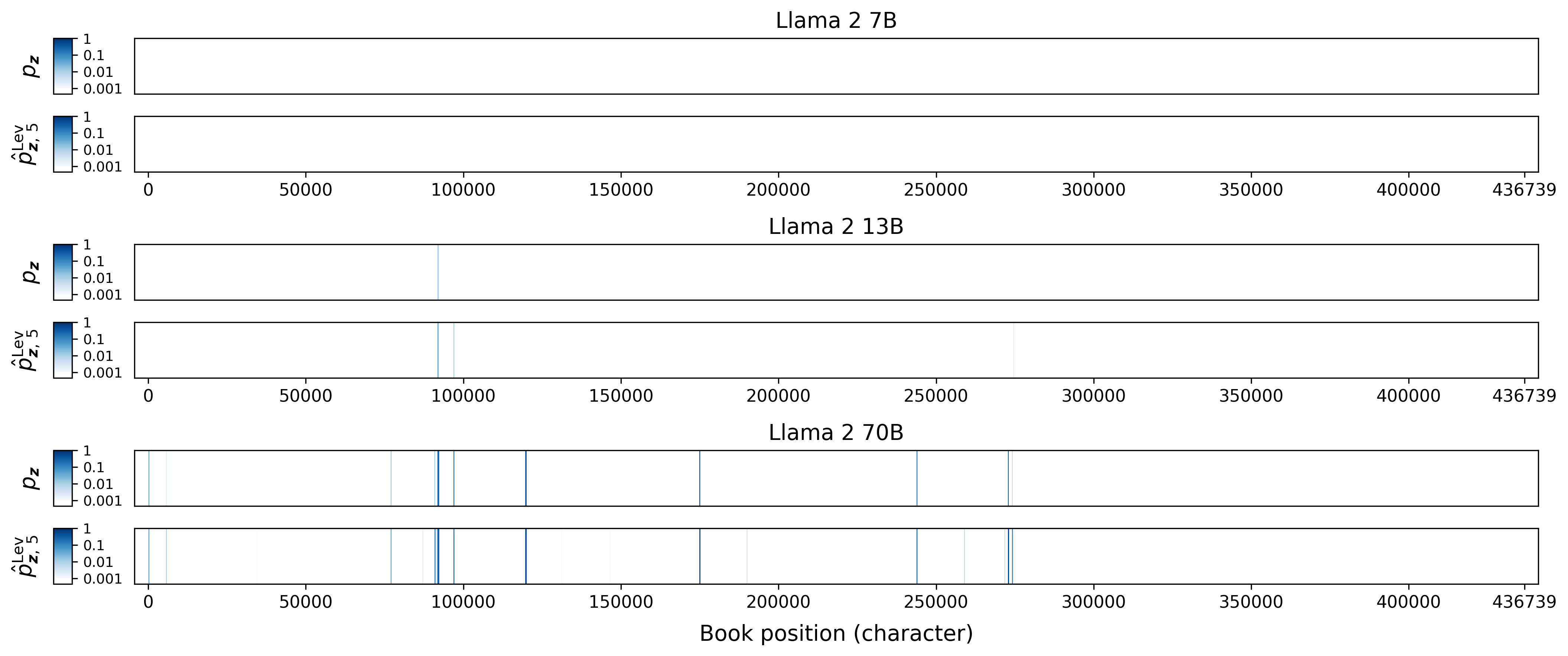}
    \subcaption{\emph{Orlando}}
    \label{app:fig:orlando:heatmap}
\end{subfigure}
\begin{subfigure}{\linewidth}
    \includegraphics[width=\linewidth]{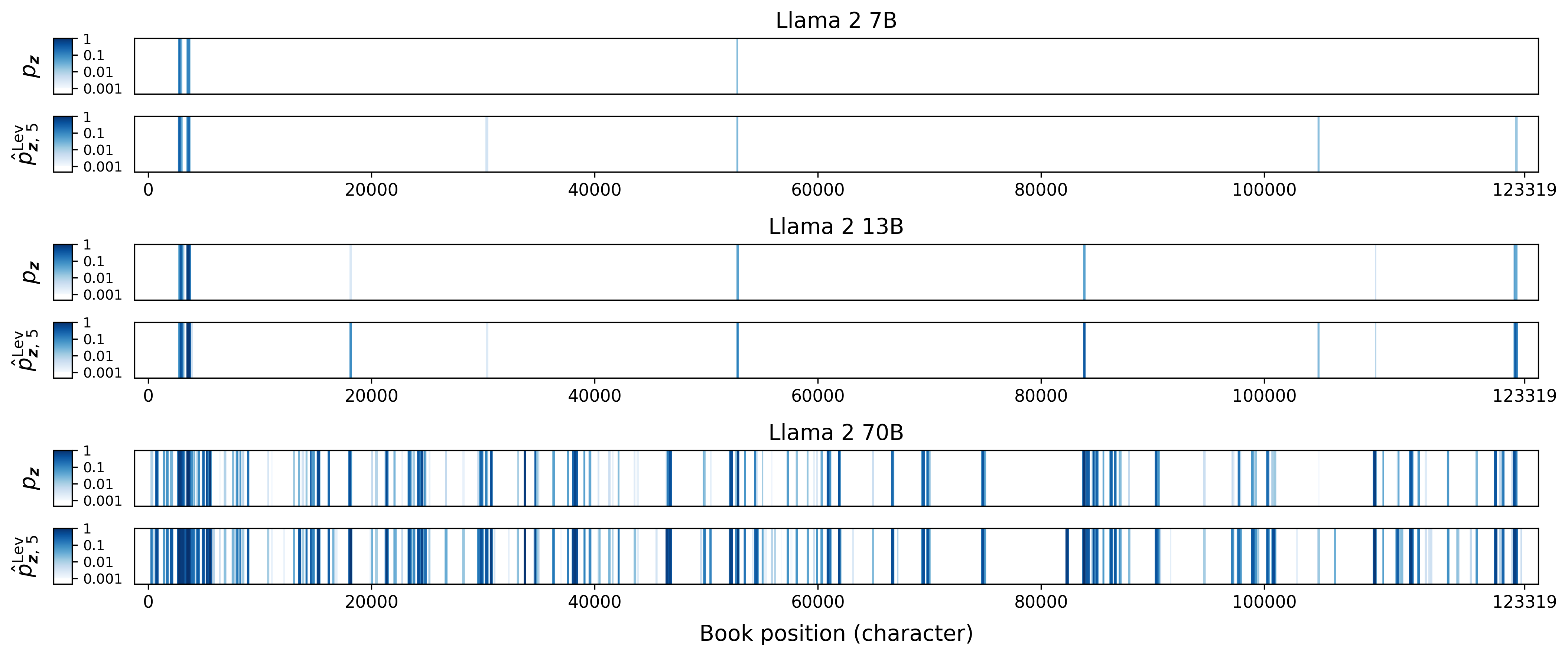}
    \subcaption{\emph{Winnie the Pooh}}
    \label{app:fig:winnie:heatmap}
\end{subfigure}
\caption{\textbf{Heatmaps comparing verbatim and near-verbatim extraction risk across books for \textsc{Llama 2} models.}
  For each book, we show three pairs of heatmaps---one pair for each \textsc{Llama 2} model size, with each pair showing the verbatim extraction probability ($\pseq$) and the near-verbatim extraction probability ($\hat{p}^\levshort_{\seq,5}$). 
  On a given heatmap, each vertical strip corresponds to a character position in the source text; 
  color intensity shows the \emph{maximum} extraction probability across all overlapping $50$-token suffixes that cover that position, on a $\log$ scale.
  White indicates no extractable suffix ($< \taumin$).}
\label{app:fig:heatmaps}
\end{figure*}

\clearpage
\subsection{Additional experiment with \textsc{Llama 3.1 8B}}\label{app:experiments:llama3}

We ran several additional experiments on \textsc{Llama 3.1 8B}, including for additional books. 
For brevity, we omit most of these results
In Figure~\ref{app:fig:pride:scatter-evo}, we offer one additional illustration for \emph{Pride and Prejudice}, as this book exhibits extensive degrees of memorization for this model.
We provide a set of scatter plots for increasing Levenshtein distance, visualizing how the risk evolves and how sequences become unlocked at different $\varepsilon$.
$\varepsilon=0$ is the verbatim case, so there are only \textcolor{seabornbluemid}{\textbf{blue}} points for that plot, and they all sit on the line $y=x$; 
we provide this as a reference, even though it just contains information about verbatim extraction.

\begin{figure*}[b!]
\centering
    \includegraphics[width=\linewidth]{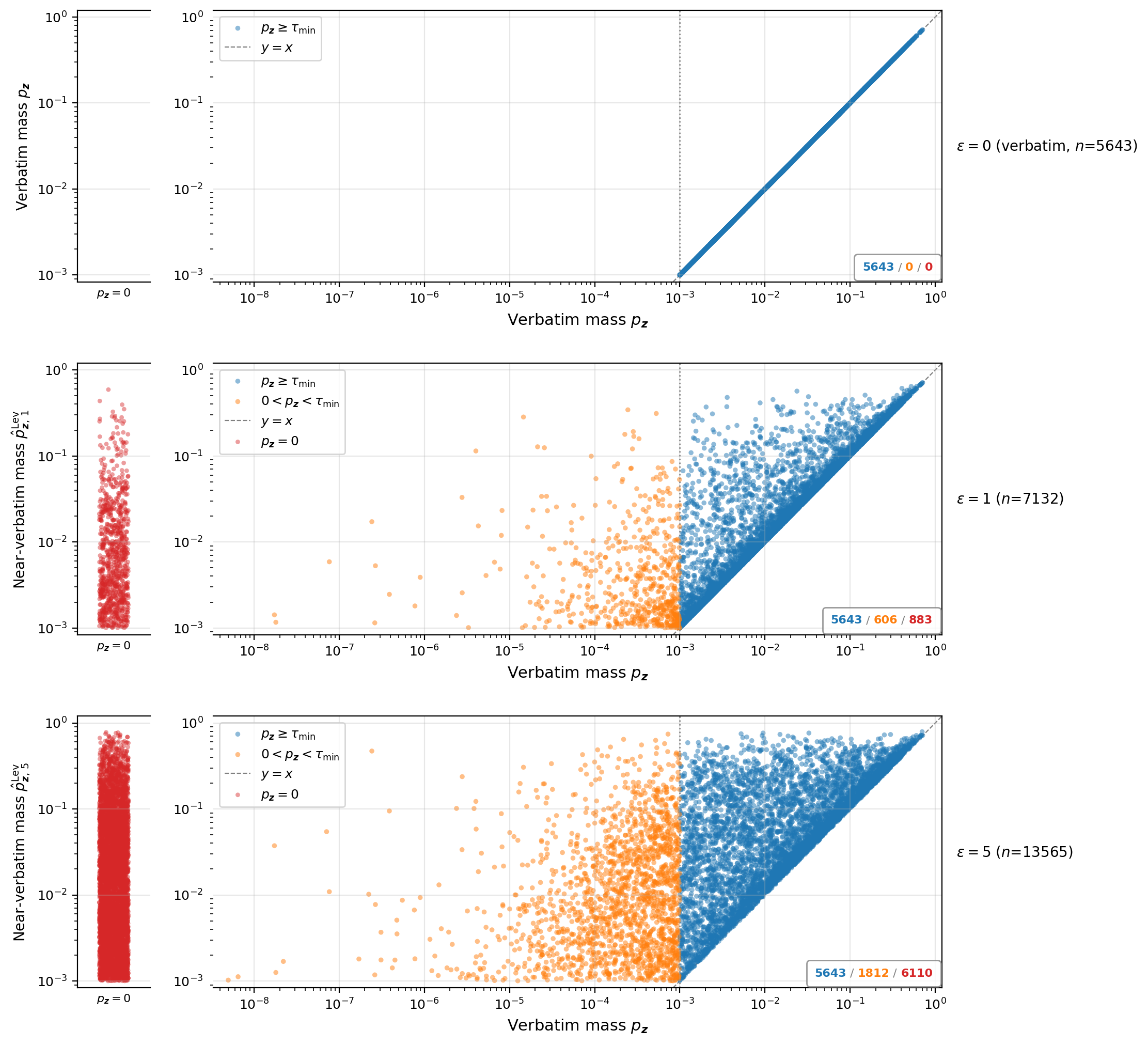}
\caption{\textbf{Evolution of near-verbatim mass vs.\ verbatim mass for \textsc{Llama 3.1 8B} and \emph{Pride and Prejudice}.} 
Each point is one sequence.
Axes show near-verbatim ($p_{\seq,5}^\levshort$, $\levshort\,\varepsilon\!=\!5$) vs.\ verbatim ($\pseq$) extraction mass on a $\log$--$\log$ scale.
\textcolor{seabornredmid}{\textbf{Red}}/\textcolor{seabornorangemid}{\textbf{orange}}
points are ``unlocked'' by near-verbatim extraction (to the left of the $\taumin$ dotted reference line, $\pseq\!<\!\taumin$, but $p_{\seq,5}^\levshort\!\geq\!\taumin$);
\textcolor{seabornbluemid}{\textbf{blue}} points are verbatim-extractable ($\pseq\!\geq\!\taumin$).
Points above the dashed $y\!=\!x$ line show increased extraction risk when near-verbatim mass is accounted for.
The top panel is a reference for what verbatim extraction conveys.
There are only \textcolor{seabornbluemid}{\textbf{blue}} points, and they are all on the line $y=x$; 
there are $5{,}643$ verbatim extractable sequences.
The middle panel shows $\varepsilon=1$; 
even at this distance, there are a large number of ``unlocked'' points that have zero verbatim mass (\textcolor{seabornredmid}{\textbf{red}}) or sub-threshold verbatim mass (\textcolor{seabornorangemid}{\textbf{orange}});
the total number of extractable points goes up from $5{,}643$ to $7{,}132$ sequences.
The bottom panel shows $\varepsilon=5$.
There are even more unlocked sequences ($13{,}565$), and the risk spreads upward for all three categories.\looseness=-1}
\label{app:fig:pride:scatter-evo}
\end{figure*}
\FloatBarrier

\subsection{Varying beam width}\label{app:experiments:width}

In our main experiments, we always set beam width $\bw=20$.
In this appendix, we justify this choice.
For \textsc{Llama 2} models and \emph{Winnie the Pooh}, we run a sweep of experiments for $\levshort\,\varepsilon=5$ $k$-CBS for beam widths $\bw=20,30,40$.\footnote{We initially ran these experiments using baseline $k$-CBS, but then re-ran them with these settings after concluding from the results in Appendix~\ref{app:sec:experiments:prune} that this is the best setting for running once for edit-distance-based extraction.}
These increases reflect a corresponding increase in forward passes (e.g, $\bw=40$ reflects a $2\times$ increase, see Appendix~\ref{app:sec:intuition:cost}, if we discount early termination).
Overall, we find that increasing the beam width comes at substantial compute cost, but minimal change in identified extraction or extraction mass. 
This is why we opt to use $\bw=20$; 
for substantially decreased cost, we effectively obtain the same information.

The three figures in this appendix support this conclusion.
In Figure~\ref{app:fig:winnie:bw-rate}, for each model size we show how extraction counts (rate $\%$) change according to $\bw$ (column) and post-processing for a chosen $\levshort\, \varepsilon$ (row). 
In Figure~\ref{app:fig:winnie:bw-risk}, we produce similar plots, but for how extraction risk mean $\pm$ standard deviation changes. 
In both, for each per-model plot within a given $\varepsilon$ row, red cell shading indicates a relative drop compared to the $\bw=40$ results on the right (so the $\bw=40$ cells are all white).
For instance, for \textsc{Llama 2 70B} verbatim ($\varepsilon=0$) post-processing of the results, $\bw=40$ identifies $665$ instances of extraction; $\bw=30$ closely tracks with $664$ (and is shaded very lightly to reflect this), and $\bw=20$ finds $658$ (and so is darker red). 
Note that shading is determined by normalizing for a given model's $\varepsilon$ results in a given row, so an absolute drop from $69$ to $66$ (for \textsc{Llama 2 13B} and $\varepsilon=5$, from $\bw=40$ to $\bw=20$) shades darker than an absolute drop from $988$ to $962$  (for \textsc{Llama 2 70B} and $\varepsilon=5$, from $\bw=40$ to $\bw=20$), as the latter reflects a smaller relative drop.
The same interpretation applies to the extraction risk plot shading.

Overall, from Figure~\ref{app:fig:winnie:bw-rate}, we find that widening the beam from $\bw=20$ to $\bw=40$ captures slightly more extractable sequences, but the gains are minimal. 
The largest absolute
difference is at 70B / $\varepsilon=5$, where $\vw=40$ finds $26$ more sequences
than $\bw=20$ (as noted above, $988$ vs $962$, a $0.42$ percentage point increase in rate from $15.64\%$ to $16.06\%$).
At 7B the difference is at most $1$ sequence; at 13B at most $3$.
For mean extraction risk, Figure~\ref{app:fig:winnie:bw-risk} shows that the per-sequence mass mean is similarly stable across beam widths.
The largest relative shift in mean mass is ${\sim}4.5\%$
(7B / $\varepsilon=5$: $0.050$ vs.\ $0.052$), and at 13B and 70B the mean masses are within $\sim1$--$2\%$.
Note that sometimes the mean is slightly higher at smaller beam widths $\bw$;
this is due to compositional effects of the extractable set (for larger beam widths, we sometimes retain additional lower-mass sequences that surpass $\taumin$, but can bring down the mean).

\begin{figure*}[t!]
\centering
    \includegraphics[width=\linewidth]{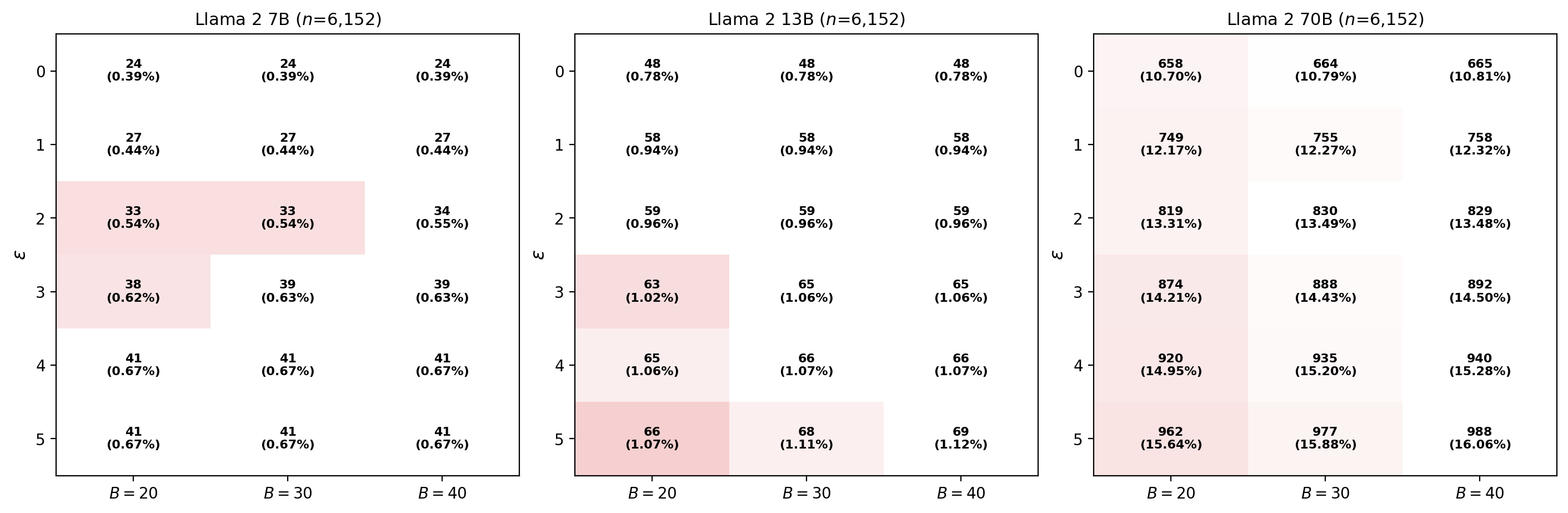}
    \caption{\textbf{Extraction counts and rates by beam width.} 
    For \emph{Winnie the Pooh} across \textsc{Llama 2} model sizes, each cell shows the number of near-verbatim extractable sequences ($\hat{p}^{\levshort}_{\seq,\,\varepsilon} \geq \taumin=0.001$), with the corresponding extraction rate in parentheses. 
    For a given model's plot, red shading in a cell indicates relative decreases compared
    to $\bw=40$ within each row (so the $\bw=40$ cells are all white). 
    In terms of absolute numbers, widening the beam from $\bw=20$ to $\bw=40$ yields at
    most $26$ additional extractable sequences (\textsc{Llama 2 70B}, $\levshort\,\varepsilon=5$).
    in terms of relative differences, \textsc{Llama 2 13B} shows the biggest relative decrease from $\bw=40$ to $\bw=20$ for $\levshort\,\varepsilon=5$ (which is why the red shading changes are more pronounced).\looseness=-1}
\label{app:fig:winnie:bw-rate}
\end{figure*}

\begin{figure*}[t!]
\centering
    \includegraphics[width=\linewidth]{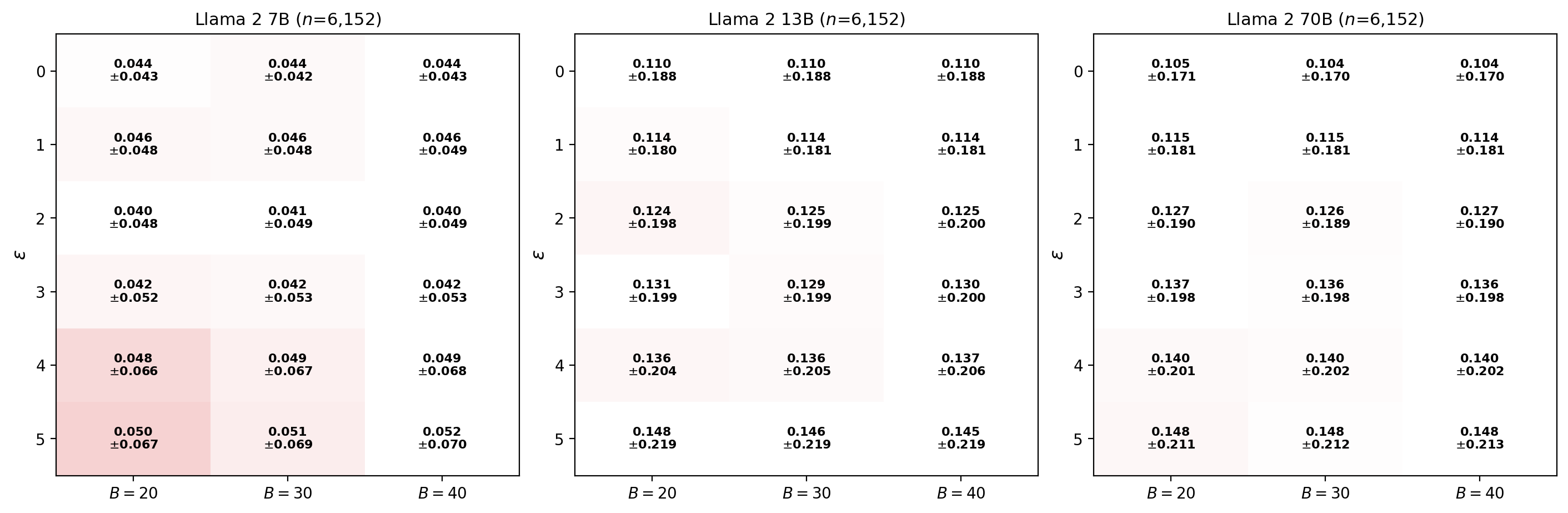}
\caption{\textbf{Extraction mass by beam width.} 
For \emph{Winnie the Pooh} across \textsc{Llama 2} model sizes, each cell shows mean extraction risk ($\pm1$ standard deviation) by beam width for extractable sequences. 
For a given model's plot, red shading in a cell indicates the relative drop in mean $\pm$ standard deviation, compared to $\bw=40$ within each row (so the $\bw=40$ cells are all white). 
The mass distribution is stable across beam widths, with relative differences under $5\%$.
Note that this measurement is about attempting to gauge stability---large swings in mass---as opposed to seeing how much mass decreases at smaller beam widths.
Mean mass can also shift downward at larger beam widths, if an expanded beam width reveals low-mass extractable sequences that smaller beam widths prune.
(See, e.g., \textsc{Llama 2 13B} for $\levshort\, \varepsilon=5$.)\looseness=-1}
\label{app:fig:winnie:bw-risk}
\end{figure*}

\begin{figure*}[t!]
\centering
    \includegraphics[width=.5\linewidth]{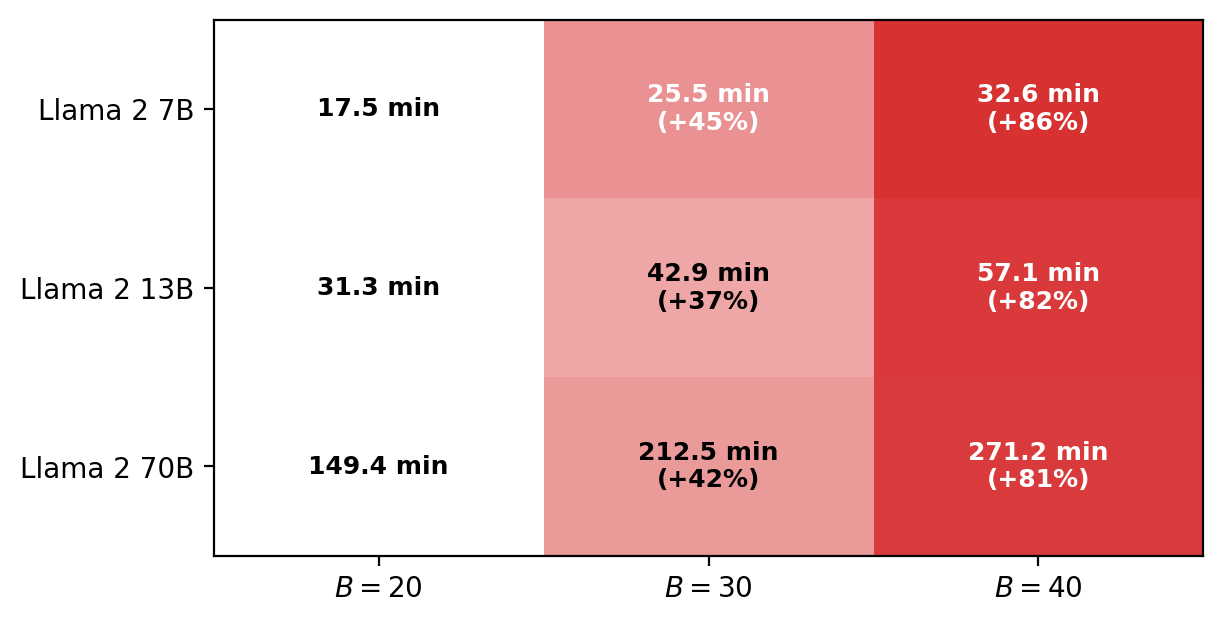}
\caption{\textbf{Wall-clock runtime (minutes) by beam width and model size.} 
The effective batch size is held roughly constant across beam widths
($\bw \times$ batch size $\approx$ $400$ for 70B, $\approx$ $200$ for 7B/13B).
Red cell shading indicates overhead relative to $\bw=20$. 
Doubling the beam width approximately doubles runtime, while the extraction gains shown in Figure~\ref{app:fig:winnie:bw-rate} are minimal.\looseness=-1}
\label{app:fig:winnie:bw-runtime}
\end{figure*}

For these minimal changes, using a larger beam width comes at significantly increased cost, with respect to wall-clock runtime, shown in Figure~\ref{app:fig:winnie:bw-runtime}. 
To make fair runtime comparisons, we choose batch sizes for $k$-CBS such that the effective processed number of sequences (\newterm{effective batch size}) is similar per iteration per configured run (Table~\ref{tab:exp-summary}). 
For a beam width $\bw$, a single iteration processes $\bw \cdot \text{batch size}$ sequences. 
For 70B, we use an effective batch size of $400$, so $\bw=20,30,40$ use batch sizes of $20$ (effective batch size $400$), $13$ ($390$), and $10$ ($400$), respectively.
This controls for GPU utilization: 
a fixed batch size with wider beams would increase parallelism, potentially masking the true per-step cost behind better GPU saturation. 
Holding the total sequences in being processed per iteration constant is a best effort at capturing how the runtime increase reflects the cost of wider beams.
A wider beam retains more mass, keeping the best-beam upper bound higher for longer, so sequences may terminate at different steps across $\bw$ settings.
Overall, going from $\bw=20$ to $\bw=40$ nearly doubles wall-clock time (slightly less, due to early termination). 
(Note that the shading is in the reverse direction of Figures~\ref{app:fig:winnie:bw-rate} and~\ref{app:fig:winnie:bw-risk}; 
we normalize runtime off of the cheapest run, which is for $\bw=20$.)

Together, these results justify using $\bw=20$ in our main experiments: 
doubling the beam width nearly doubles runtime while yielding minimal gains---a handful of marginal sequences at the extraction threshold.

\clearpage
\subsection{Comparing baseline and $\varepsilon$-viability pruned $k$-CBS}\label{app:sec:experiments:prune}

We run a full sweep of experiments for $\levshort\,\varepsilon=0,\ldots,5$ $\hamshort\,\varepsilon=0,\ldots,5$ to see how different pruning constraints impact extraction rate.
We also compare these full sweeps to baseline $k$-CBS post-processed results.
We run these experiments for \emph{Winnie the Pooh} across the three \textsc{Llama 2} sizes, summarized in two groups (per $\dist=\hamshort,\levshort$) of three models (7B, 13B, 7B).
The results are summarized in Figure~\ref{app:fig:winnie:prune-vs-rates}.
We also ran this entire suite of experiments on the negative control text, \emph{The People's Dictator} (first three chapters);
all counts across all run configurations are $0$. 

Overall, we find that the gains from tighter pruning are modest (tens of sequences out of thousands), so a single $\varepsilon=5$ run can be post-processed at any threshold with only minor loss in signal. 
Overall, the additional sequences captured by tighter pruning tend to send near the extraction threshold $\taumin$.
his is practically useful:
one run covers all distance thresholds of interest, rather than requiring separate runs per $\varepsilon$.
As expected, pruning with the Hamming distance finds fewer near-verbatim sequence than Levenshtein, as it is a stricter pruning criterion (Appendix~\ref{app:sec:nv:compute:ham-dominate-lev}). 

\paragraph{Interpreting the plots in this section.} 
Each per-model plot should be read vertically: 
each column fixes a distance threshold and compares how extraction counts (and corresponding rates) vary across configured run settings. 
The top row is baseline $k$-CBS, which does not perform pruning during the search (Section~\ref{sec:kcbs:baseline} \& Appendix~\ref{app:sec:kcbs}).
In the plots in Figure~\ref{app:fig:prune-lev}, subsequent rows show $\levshort$-pruned runs at
$\varepsilon = 0, \ldots, 5$;
the same is true for $\hamshort$-pruned runs in Figure~\ref{app:fig:prune-ham}. 
For a given plot, the diagonal---where pruning $\varepsilon$ matches the threshold used to filter outputs---represents each pruning setting's ``native'' operating point. 
Gray cells indicate distance thresholds above the run's pruning $\varepsilon$ (unavailable). Shading encodes the difference from baseline.
Bluer shading indicates increasingly tighter estimates of the corresponding extraction count (finding more extractable sequence), compared to the baseline (white).
Red shading would indicate identifying fewer such instances, which we do not observe in practice (i.e., pruned runs are always tighter than the baseline). 

\paragraph{Main takeaways.}
Every $\levshort$-pruned run and every $\hamshort$-pruned run finds at least as many extractable sequences as the un-pruned baseline at every applicable distance.
(Though, these runs are not guaranteed to find the same sequences; see Appendix~\ref{app:sec:prune:invariants}, counter example). 

The $\varepsilon=0$ pruned runs very closes match the verbatim probabilistic pipeline numbers. 
For instance, the verbatim probabilistic pipeline finds $24$ (7B), $48$ (13B), and $668$ (70B) extractable sequences, and for $\levshort\,\varepsilon=0$, counts exactly match. 
However, they find slightly different instances of extraction, even though the counts match.
This is due to the two using different pipelines--teacher forcing for verbatim probabilistic vs.\ autoregressive generation with KV caching for $k$-CBS.
Minor floating-point rounding differences can push borderline sequences across the extraction threshold in either direction.
At 70B with $\levshort\,\varepsilon=5$, the verbatim count ($\levshort\leq 0$ column) is $658$, compared to $668$ for verbatim probabilistic;
the net difference is $10$ sequences, but per-sequence comparisons reveal that $k$-CBS missed $11$ sequences identified by the verbatim probabilistic pipeline, and $k$-CBS identified $1$ sequence missed by verbatim probabilistic.
The un-pruned $k$-CBS baseline finds fewer verbatim matches than both of these other runs;
beam slot capacity goes to candidates with the highest mass at a given iteration, regardless of verbatim or near-verbatim match status. 
This crowds out verbatim matching candidates that are lower rank, in contrast to pruned $\varepsilon=0$ runs that concentrate all beam capacity on exact matches. 

Tighter pruning (i.e., setting a smaller $\varepsilon$) helps slightly at its ``native'' threshold, which is reflected on the diagonal of each plot.
Results on the diagonal consistently find a few more sequences. 
For instance, for 70B, the $0$ diagonal cell finds $668$ sequences, compared to $658$ in the $\varepsilon=5$ row;
similarly, the $1$ diagonal cell shows $763$ sequences, while the $\varepsilon=1$ row finds $749$. 
In general, tighter pruning frees beam slots for more viable candidates at that specific distance, revealing (relatively) ${\sim}1$--$2\%$ more extraction.

For the baseline runs, we also manually examine outputs that surpass $\taumin$ but are not within $\levshort\,\varepsilon=5$ of the target suffix. 
The results fall into two categories.
First, the more common case:
there are suffixes that are effectively near-verbatim matches, but exceed $\levshort=5$.
Based on visual inspection, these sequences are clearly memorized, but punctuation and formatting differences lead to inflated distances. 
Second, there are a handful of instances where the generated suffix is a near-verbatim match to \emph{different} text in the ground-truth book.
This type of outcome is what has encouraged others to develop suffix arrays over the training data to assess memorization~\citep{lee2022dedup, carlini2023quantifying}, rather than just comparing the ground-truth suffix for the given sequence to the generation.

\paragraph{Runtime.} 
Baseline $k$-CBS is the slowest, and then the pruned run times for different $\varepsilon$ increase with $\varepsilon$, as beam candidates can remain viable longer.
For 70B, $\levshort\, \varepsilon=5$ pruning is about $10\%$ faster in terms of wall-clock runtime compared to baseline $k$-CBS.
We cannot make a general claim about speedups;
it entirely depends on the amount of extractable sequences in the dataset being tested (with datasets with less memorization finishing faster, due to early termination from minimum probability testing and from pruning). 

\begin{figure*}[t!]
\centering
\begin{subfigure}{\linewidth}
    \includegraphics[width=\linewidth]{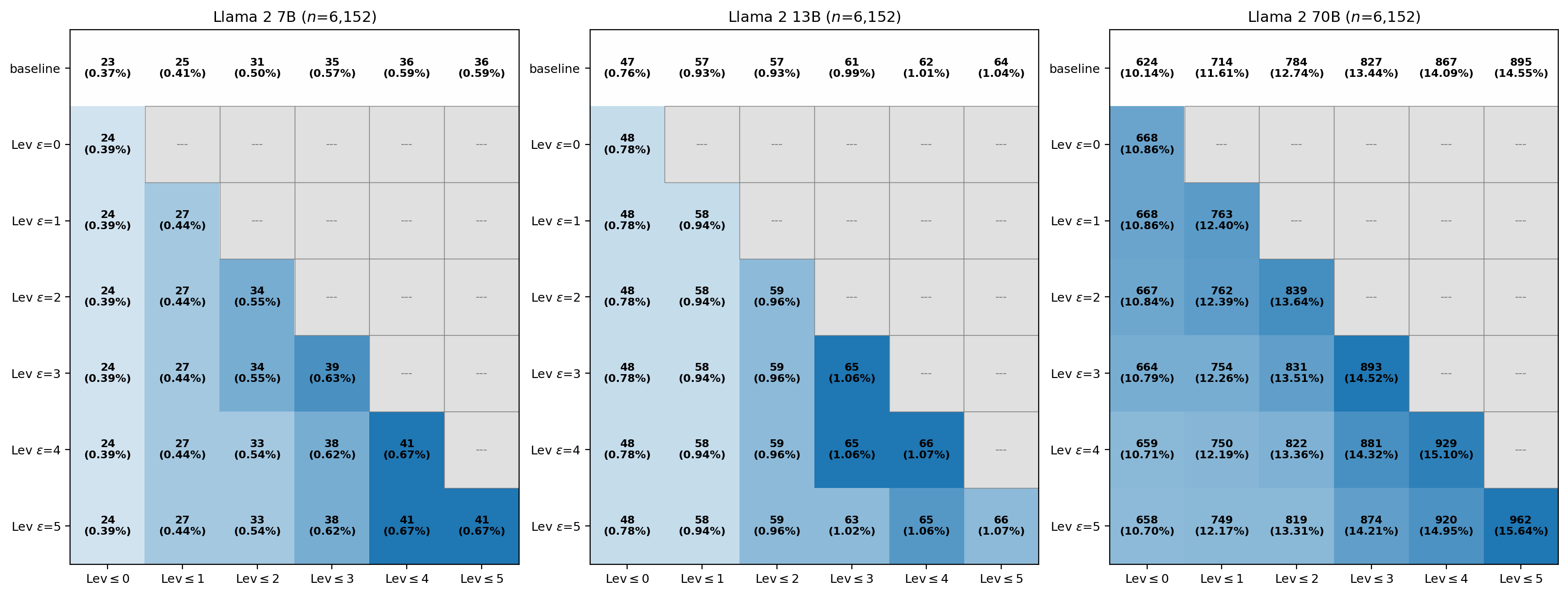}
    \subcaption{Levenshtein-pruned runs}
    \label{app:fig:prune-lev}
\end{subfigure}
\begin{subfigure}{\linewidth}
    \includegraphics[width=\linewidth]{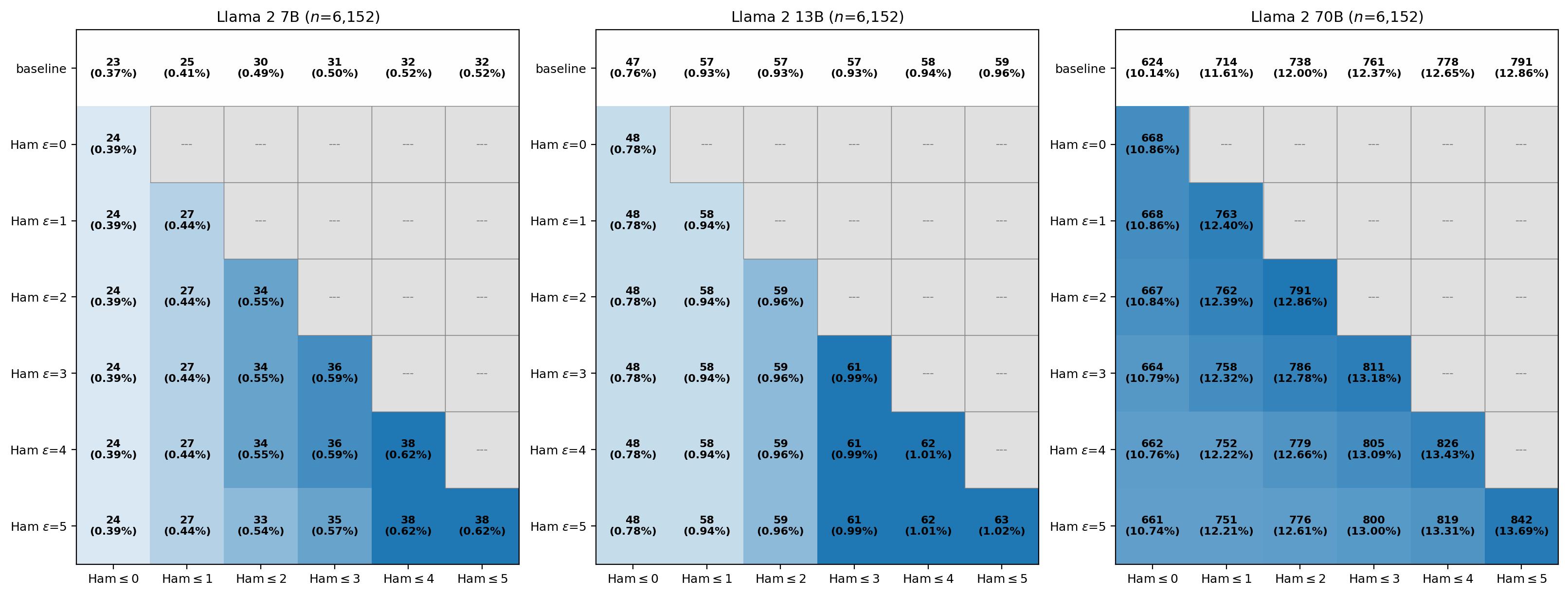}
    \subcaption{Hamming-pruned runs}
    \label{app:fig:prune-ham}
\end{subfigure}
\caption{\textbf{Extraction counts by configured run, based on filtering to a specific  distance threshold.}
For \emph{Winnie the Pooh} and each \textsc{Llama 2} model size, we run baseline $k$-CBS (Section~\ref{sec:kcbs:baseline} \& Appendix~\ref{app:sec:kcbs}), (\textbf{a}) $\levshort$-pruned $k$-CBS (Section~\ref{sec:pruning}  \& Appendix~\ref{app:sec:prune:lev}), and (\textbf{b}) $\hamshort$-pruned $k$-CBS (Section~\ref{sec:pruning}  \& Appendix~\ref{app:sec:prune:lev}).
For each plot, read vertically, each column fixes a distance threshold and compares how extraction counts vary across configured run settings. 
The top row in all plots is baseline $k$-CBS (no pruning); 
subsequent rows show pruned runs at $\varepsilon = 0, \ldots, 5$, for (\textbf{a}) the Levenhstein distance and (\textbf{b}) the Hamming distance. 
The diagonal (where pruning $\varepsilon$ row matches the evaluation threshold column) represents each pruning setting's ``native'' operating point. 
Gray cells indicate distance thresholds above the run's pruning $\varepsilon$ (unavailable). Blue cell shading encodes the difference from baseline, reflecting that pruning found more extractable sequences. 
White shows no change.
Red would indicate fewer extractable sequences (though we do not observe this).
}
\label{app:fig:winnie:prune-vs-rates}
\end{figure*}

\clearpage
\subsection{Examining baseline $k$-CBS outputs with other near-verbatim metrics}\label{app:sec:bleu}

The baseline $k$-CBS pipeline produces candidate continuations without any distance-viability-based pruning. 
We post-process these with Levenshtein and Hamming distances (e.g., see Appendix~\ref{app:sec:experiments:prune}), but the same outputs can be scored with any
similarity metric (or really any metric of interest). 

To illustrate this, we compute sentence-level BLEU score~\citep{bleu} between each candidate in the $\bw \cdot k$ $k$-CBS outputs and the ground-truth suffix.
We follow \citet{ippolito-etal-2023-preventing}, who define approximate memorization as BLEU $\geq 0.75$ over $50$-token suffixes. 
They pick this cutoff from manual inspection of decoded text.
(BLEU is computed on text, not tokenized text.) 
For each sequence, we sum the probability mass of all candidates whose sentence BLEU with the ground-truth suffix is $\geq 0.75$. 
We then say that a sequence is probabilistically BLEU-extractable if this mass $\geq \taumin$---the same probabilistic extraction criterion we use for edit-distance metrics.

\paragraph{BLEU score.}
\newterm{Bilingual Evaluation Understudy (BLEU)}~\citep{bleu} is a precision metric originally developed for machine translation applications.
In particular, \newterm{sentence-level BLEU} computes the geometric mean of $n$-gram precisions, typically for $n = 1, \dots, 4$, multiplied by a brevity penalty:
\begin{align}
\mathsf{BLEU} \triangleq \mathsf{BP} \cdot \exp \Bigl(\sum_{n=1}^{4}
\frac{1}{4} \log p_n\Bigr),
\end{align}
where $p_n$ is the fraction of $n$-grams in the hypothesis that appear in
the reference (using modified precision with clipped counts), and brevity penalty 
$\mathsf{BP} = \min(1,\, e^{1 - r/c})$ penalizes hypotheses shorter than
the reference ($c$ is the hypothesis length, $r$ is the reference length).

The standard configuration from \citet{bleu} uses $n = 1, \ldots, 4$ with uniform weights
$w_n = 1/4$ (which are omitted in the equation);
this is the default in all major implementations (including \texttt{nltk}, which we use, similar to \citet{ippolito-etal-2023-preventing}). 
$\mathsf{BLEU}$ ranges from $0$ to $1$.
A score of $0.75$ indicates that roughly $75\%$ of unigrams
through $4$-grams in the candidate appear in the reference;
this is high $n$-gram overlap, but not necessarily identical text. 
Unlike edit distance, $\mathsf{BLEU}$ does not require sequential alignment: 
it rewards shared $n$-grams regardless of position, so r-eorderings, insertions, and multi-token substitutions that preserve local $n$-gram structure can still yield a high score.
Note that since $\mathsf{BLEU}$ operates on text (not LLM tokens), $\mathsf{BP}$ is not always $1$, given that $50$-token suffixes may decode to different-length texts. 

\paragraph{Application to $k$-CBS.} 
For a given sequence, $k$-CBS returns $\bw \cdot k$ continuations.
We sum the probability mass of all candidates with sentence-level $\mathsf{BLEU} \geq 0.75$, computed with respect to the ground-truth suffix from the training data;
this is effectively our $\varepsilon$ criterion in the lower bound computation in Equation~\ref{eq:pseqe}.
For probabilistic extraction, we then apply the same $\taumin$ to produce $\hat{p}^{\mathsf{BLEU}}_{\seq, 0.75}$. 
We do this for the baseline $k$-CBS results for \emph{Winnie the Pooh}, which we report of Hamming and Levenshtein distance in Appendix~\ref{app:sec:experiments:prune}.
(We also run the associated negative controls on \emph{The People's Dictator}, first three chapters, which similarly yields no hits.)\looseness=-1 

\begin{table}[h]
\centering
\caption{\textbf{Extractable sequence counts under three similarity metrics on \textit{Winnie the Pooh}.} 
Each row counts sequences whose aggregated candidate mass under the given metric meets the extraction threshold $\taumin$, applied to baseline $k$-CBS outputs ($\bw\!=\!20$). 
The bottom section shows the overlap between Lev $d \leq 5$ and BLEU $\geq 0.75$: most extractable sequences are captured by both, but each metric finds a small number of sequences missed by the other.}
\label{tab:bleu_vs_lev}
\small
\begin{tabular}{l rrr}
\toprule
& \textsc{7B} & \textsc{13B} & \textsc{70B} \\
\midrule
\multicolumn{4}{l}{\textit{Extractable sequences by metric}} \\[2pt]
Verbatim & 23 (0.37\%) & 47 (0.76\%) & 624 (10.14\%) \\
$\levshort\, \varepsilon=5$ & 36 (0.59\%) & 64 (1.04\%) & 895 (14.55\%) \\
$\mathsf{BLEU} \geq 0.75$ & 41 (0.67\%) & 71 (1.15\%) & 1012 (16.45\%) \\
\midrule
\multicolumn{4}{l}{\textit{Set overlap: $\levshort\, \varepsilon \leq 5$ vs.\ $\mathsf{BLEU} \geq 0.75$}} \\[2pt]
Both & 36 & 64 & 889 \\
$\levshort\, \varepsilon=5$ only & 0 & 0 & 6 \\
$\mathsf{BLEU} \geq 0.75$ only & 5 & 7 & 123 \\
\bottomrule
\end{tabular}
\end{table}

Overall, $\mathsf{BLEU}$ is more permissive than $\levshort\, \varepsilon=5$.
As shown in the top of Table~\ref{tab:bleu_vs_lev}, in absolute counts (and therefore rates), it finds more extractable sequences at every model size. 
In the bottom of the table, we also examine the decomposition of extractable sequences, according to which are extractable under both metrics, under $\levshort\, \varepsilon=5$ only, and under $\mathsf{BLEU} \geq 0.75$ only. 
Overall, the two metrics largely agree with identifying near-verbatim extractable sequences, but their respective results are not subsets of each other. 

For instance, both metrics capture $889$ of $1{,}018$ unique extractable sequences for 70B.
Upon manual inspection of the outputs, disagreements in both directions indicate false negatives of the respective metric, not false positives of the other.
The sequences are all very close to identical to the verbatim target suffix, and fall through either check due to being below the respective filter thresholds.
For the $123$ sequences that are extractable under $\mathsf{BLEU} \geq$ only (i.e., that $\levshort\, \varepsilon=5$ misses), the differences reflect formatting changes---newline variations, punctuation style differences, ASCII replacements, slight syntactic rearrangements, etc.
These can shift many tokens (i.e., above our threshold), but preserve word-level content.
These candidates often carry large probability (e.g., $\textsf{BLEU}$-captured mass that is $>0.5$, where $\levshort$ distances are $6$--$15$, but clearly are memorized upon visual inspection). 

Coming from the other direction, the $\mathsf{BLEU}$ filter misses $6$ sequences at 70B that are extractable with respect to $\levshort\, \varepsilon=5$.
These consist of small token edits in fairly unusual text---made-up (sound-like words, onomatopoeia, etc.) from \emph{Winnie the Pooh} with capitalization changes.
Others include em-dash to double-hyphen substitutions, verb synonyms, etc.
These are within $5$ token edits for Levenshtein, but sufficient to break enough word-level $4$-grams to drop the $\mathsf{BLEU}$ score below the $0.75$ threshold.

\stopcontents[appendix]

\end{document}